\documentclass[10pt]{article}

\usepackage[utf8]{inputenc} % allow utf-8 input
\usepackage[T1]{fontenc}    % use 8-bit T1 fonts

\usepackage{amsfonts}
\usepackage{amsmath}
\usepackage{amscd}
\usepackage{amsthm}
\usepackage{amssymb}
\usepackage{amsfonts}       % blackboard math symbols
\usepackage{algorithm}
\usepackage{algpseudocode}
\usepackage{booktabs}       % professional-quality tables

\usepackage{url}            % simple URL typesetting

\usepackage{nicefrac}       % compact symbols for 1/2, etc.
\usepackage{microtype}      % microtypography
\usepackage{xcolor}         % colors
\usepackage{natbib}
\usepackage[colorlinks,linkcolor=blue,citecolor=blue]{hyperref}

\usepackage{helvet}  %Required
\usepackage{courier}  %Required
\usepackage{verbatim}
\usepackage{graphicx}  %Required
\usepackage{xcolor}

\usepackage{times}
\usepackage{latexsym}
\usepackage{inconsolata}
\usepackage{enumitem}
\usepackage{colortbl} % for \rowcolor in tables
\usepackage{subcaption} % Required for subfigures
\usepackage{float}
\usepackage{rotating} % Required for \rotatebox
\usepackage{multirow}
\usepackage{makecell}
\usepackage{caption}
\usepackage{array}
\usepackage{tabularx}
\usepackage{wrapfig}
\usepackage{fontawesome5}
\usepackage{listings}
\usepackage{calc}

\newlength{\hfwidth}
\newcommand{\hflogo}{%
  \setlength{\hfwidth}{1.2em}% Adjust size as needed
  \raisebox{-0.15em}{%
    \includegraphics[height=\hfwidth]{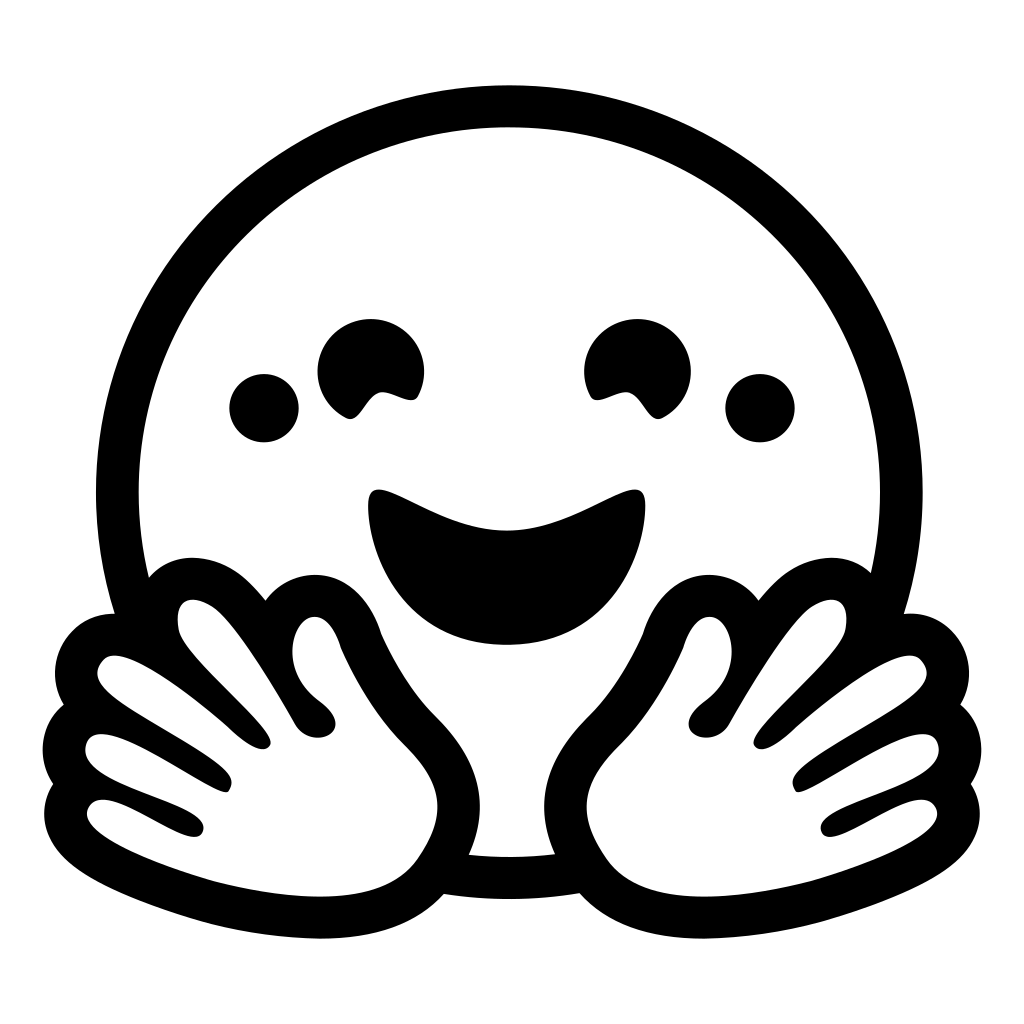}%
  }%
}

\usepackage[verbose=true,letterpaper]{geometry}
\newgeometry{
    textheight=9in,
    textwidth=6.5in,
    top=1in,
    headheight=12pt,
    headsep=25pt,
    footskip=30pt
}

\setlist[itemize]{itemsep=0.1cm, parsep=0pt, topsep=0.2cm, partopsep=0pt}
\setlist[enumerate]{itemsep=0.1cm, parsep=0pt, topsep=0.2cm, partopsep=0pt}

\theoremstyle{plain}

\newtheorem*{remark}{Remark}
\newtheorem{proposition}{Proposition}

\newtheorem{definition}{Definition}

\begin{document}
\title{Do Reasoning Models Enhance Embedding Models?}
\author{Wun Yu Chan$^1$, 
Shaojin Chen$^1$,
Huihao Jing$^1$,
Kwun Hang Lau$^1$,
Elton Chun-Chai Li$^1$, \\
Zihao Wang$^1$,
Haoran Li$^1$,
Yangqiu Song$^1$ \\
\\
$^1$\textit{CSE, HKUST} \\
\small Correspondance: \texttt{wychanbu@connect.ust.hk} \\
\small \faGithub \, \href{https://github.com/HKUST-KnowComp/Reasoning-Embedding}{Reasoning-Embedding} \,\,\hflogo\, \href{https://huggingface.co/collections/lucaswychan/reasoning-embedding}{Reasoning-Embedding}}

\maketitle
\begin{abstract}
State-of-the-art embedding models are increasingly derived from decoder-only Large Language Model (LLM) backbones adapted via contrastive learning. Given the emergence of reasoning models trained via Reinforcement Learning with Verifiable Rewards (RLVR), a natural question arises: do enhanced reasoning translate to superior semantic representations when these models serve as embedding initializations? Contrary to expectation, our evaluation on MTEB and BRIGHT reveals a \textbf{null effect}: embedding models initialized from RLVR-tuned backbones yield no consistent performance advantage over their base counterparts when subjected to identical training recipes. To unpack this paradox, we introduce \textbf{H}ierarchical \textbf{R}epresentation \textbf{S}imilarity \textbf{A}nalysis (HRSA), a framework that decomposes similarity across representation, geometry, and function levels. HRSA reveals that while RLVR induces irreversible latent manifold's local geometry reorganization and reversible coordinate basis drift, it preserves the global manifold geometry and linear readout. Consequently, subsequent contrastive learning drives strong alignment between base- and reasoning-initialized models, a phenomenon we term \textbf{Manifold Realignment}. Empirically, our findings suggest that unlike Supervised Fine-Tuning (SFT), RLVR optimizes trajectories within an existing semantic landscape rather than fundamentally restructuring the landscape itself.
\end{abstract}

\section{Introduction}
\label{sec: intro}

\begin{figure}[t]
    \centering
  \includegraphics[width=0.7\columnwidth]{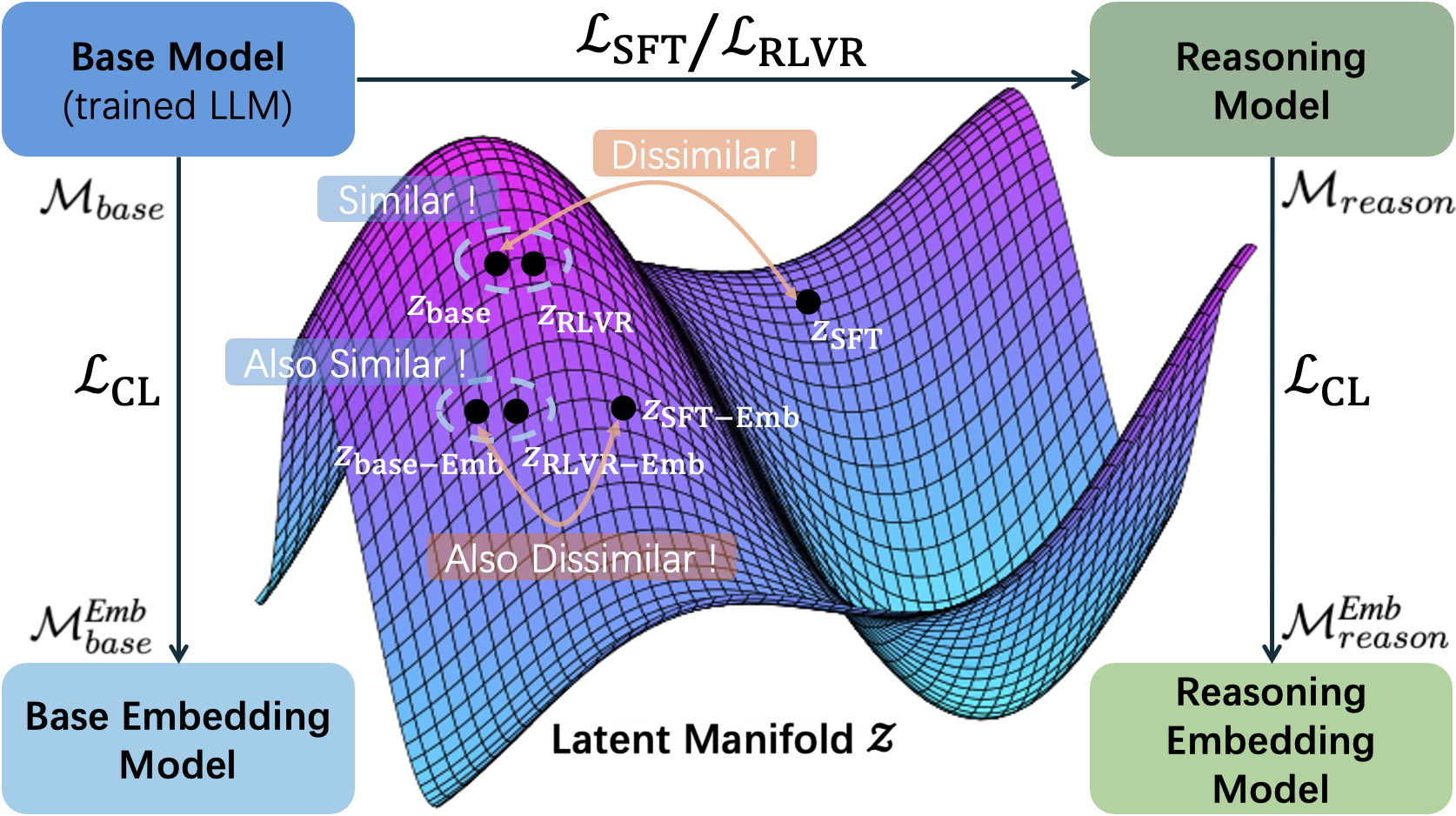}
  \caption{\textbf{Latent manifold and model relationships.} \textbf{CL}, \textbf{SFT}, and \textbf{RLVR} denote Contrastive Learning, Supervised Fine-Tuning, and Reinforcement Learning with Verifiable Rewards, respectively. $\textbf{z}$ indicates the representations of the corresponding models. Suffix ``-Emb'' is added to the model name to indicate the embedding model. We demonstrate the ideas of similar and dissimilar representations of RLVR-tuned pairs and SFT-tuned pairs, respectively.}
  \label{fig: thumbnail}
\end{figure}

Vector representations of text, known as text embeddings, are a core abstraction in modern natural language processing (NLP) \citep{word2vec}. 
As Large Language Models (LLMs) continue to evolve, embedding models have now been built by adapting decoder-only LLMs~\citep{NV-Embed,Qwen3-Embedding,Gemini-Embedding} as backbones to leverage the rich semantics and world knowledge stored in their parameters.

Most recently, reasoning models optimized via Reinforcement Learning with Verifiable Rewards (RLVR) on base models have demonstrated a qualitative leap in complex problem-solving and reasoning~\citep{DeepSeek-R1,Tulu3_with_RLVR_proposed,tinymodel-biglogic,GSPO}. This development raises a natural hypothesis for representation learning: \textit{Does the enhanced reasoning translate to a superior text embedding space?} Intuitively, a model that "thinks" more deeply should structure semantic relationships more effectively.

Counter-intuitively, our results reveal a \textbf{null effect}. Across comprehensive benchmarks including MTEB(Multilingual, v2) \citep{MMTEB}, MTEB(Code, v1) \citep{MTEB}, and BRIGHT \citep{BRIGHT}, embedding models initialized from RLVR-tuned reasoning models perform statistically identically to base-initialized models after contrastive learning \citep{InfoNCE,SimCSE}.
This observation presents a scientific puzzle: \textit{Why do reasoning and non-reasoning backbones yield indistinguishable results following contrastive learning?}

In this paper, we argue that existing performance metrics are insufficient for diagnosing the internal dynamics of representations. We introduce \textbf{H}ierarchical \textbf{R}epresentation \textbf{S}imilarity \textbf{A}nalysis (HRSA), a hierarchical analysis framework inspired by Representational Similarity Analysis (RSA)~\citep{revisited-similarity}. HRSA allows us to dissect model similarity at increasing levels of abstraction:

\begin{itemize}
    \item \textbf{Representation Level}: Focus on the coordinate basis and features.
    \item \textbf{Geometry Level}: Focus on the shape (geometry) of the latent manifold.
    \item \textbf{Function Level}: Focus on the input-output mappings.
\end{itemize}

Applying HRSA uncovers a phenomenon we term \emph{Manifold Realignment}. We find that RLVR largely preserves the \emph{global geometry} of the latent manifold, including the \emph{linear readout} associated with downstream tasks, while irreversibly reshaping the manifold’s \emph{local geometry}. The resulting drift in the coordinate basis is modest under typical training regimes but becomes pronounced under prolonged RLVR. Strikingly, when these backbones are later adapted into embedding models via contrastive learning, both base- and reasoning-initialized models exhibit strong realignment even in the presence of coordinate basis changes. We interpret the realignment as evidence that representational drift is largely reversible at the global level, yet accompanied by irreversible local distortions. Overall, our results suggest that, unlike SFT, RLVR primarily optimizes trajectories through an existing semantic landscape rather than fundamentally redrawing the landscape itself.

Our contributions are as follows:

\begin{enumerate}
    \item \textbf{Systematic Benchmarking}: We conduct the first controlled comparison of RLVR-optimized vs. its base model as backbones for text embeddings by fine-tuning a diverse suite of state-of-the-art reasoning models into embedding models and evaluate them against their base counterparts, establishing that current RLVR methods do not inherently improve embedding quality. 
    \item \textbf{The HRSA Framework}: We propose a hierarchical RSA framework (Representation, Geometry, Function) to diagnose why models behave similarly, offering a toolkit for future interpretability studies, and unifying the disorganized RSA framework.
    \item \textbf{Discover Manifold Realignment}: We demonstrate that RLVR do not fundamentally alter the latent manifold, but it can reorganize the local neighborhood structure, and only the coordinate basis will be drifted when training is prolonged. However, the contrastive learning can overwrite the reversible drift and exhibit strong alignment between base- and reasoning-initialized embedding models.
\end{enumerate}

\begin{table*}[t]
    \centering
    \small % Slightly reduce font size to fit width comfortably
    \setlength{\tabcolsep}{1.0pt} % Adjust column spacing

    \caption{
       \textbf{Mean embedding benchmark performance (3 seeds).} We compare the base backbone $\mathcal{M}_{base}$ versus its RLVR-tuned reasoning model backbone $\mathcal{M}_{reason}$; we also include an SFT-tuned backbone for reference. The $\Delta$ (\textbf{Std}) column (\textcolor{gray}{gray}) shows the mean performance gap $\pm$ standard deviation. The near-zero deltas for RLVR indicate that RLVR largely preserves the base model’s semantic effectiveness, contrasting with larger shifts under SFT.
    }

    \begin{tabular}{l ccc c ccc c ccc}
        \toprule
        & \multicolumn{3}{c}{\small \textbf{MTEB(Multilingual, v2)}} && \multicolumn{3}{c}{\small \textbf{MTEB(Code, v1)}} && \multicolumn{3}{c}{\small \textbf{BRIGHT}} \\
        \cmidrule{2-4} \cmidrule{6-8} \cmidrule{10-12}
        \textbf{Model Pair Backbone} & 
        $\mathcal{M}_{base}^\textit{Emb}$ & $\mathcal{M}_{reason}^\textit{Emb}$ & \textbf{$\Delta$ (Std)} && 
        $\mathcal{M}_{base}^\textit{Emb}$ & $\mathcal{M}_{reason}^\textit{Emb}$ & \textbf{$\Delta$ (Std)} && 
        $\mathcal{M}_{base}^\textit{Emb}$ & $\mathcal{M}_{reason}^\textit{Emb}$ & \textbf{$\Delta$ (Std)} \\
        \midrule
        
        % =======================================================
        % SECTION: SFT
        % =======================================================
        \multicolumn{12}{l}{\textbf{SFT}} \\
        % SFT Pair (Baseline) - Red Background maintained
        \scalebox{0.8}{\href{https://huggingface.co/Qwen/Qwen3-0.6B-Base}{Qwen3-0.6B-Base} vs \href{https://huggingface.co/Qwen/Qwen3-0.6B}{Qwen3-0.6B}} &
        53.50 & 41.47 & \textcolor{gray}{-12.03}~{\tiny$\pm$\,0.14} && 
        55.29 & 56.28 & \textcolor{gray}{+0.99}~{\tiny$\pm$\,0.05} &&
        13.06 & 13.71 & \textcolor{gray}{+0.65}~{\tiny$\pm$\,0.06} \\ 

        \midrule

        % =======================================================
        % SECTION: RLVR
        % =======================================================
        \multicolumn{12}{l}{\textbf{RLVR}} \\
        
        % Pair 1
        \scalebox{0.8}{\href{https://huggingface.co/Qwen/Qwen2.5-1.5B}{Qwen2.5-1.5B} vs \href{https://huggingface.co/hkust-nlp/Qwen-2.5-1.5B-SimpleRL-Zoo}{Qwen2.5-1.5B-SRL-Zoo}} & 
        54.73 & 54.54 & \textcolor{gray}{-0.19}~{\tiny$\pm$\,0.07} && 
        58.98 & 58.72 & \textcolor{gray}{-0.26}~{\tiny$\pm$\,0.08} &&
        17.71 & 17.89 & \textcolor{gray}{+0.18}~{\tiny$\pm$\,0.04} \\ 
        
        % Pair 2
        \scalebox{0.8}{\href{https://huggingface.co/Qwen/Qwen2.5-0.5B}{Qwen2.5-0.5B} vs \href{https://huggingface.co/hkust-nlp/Qwen-2.5-0.5B-SimpleRL-Zoo}{Qwen2.5-0.5B-SRL-Zoo}} & 
        51.25 & 51.27 & \textcolor{gray}{+0.02}~{\tiny$\pm$\,0.09} && 
        57.41 & 57.53 & \textcolor{gray}{+0.12}~{\tiny$\pm$\,0.07} &&
        14.05 & 13.99 & \textcolor{gray}{-0.06}~{\tiny$\pm$\,0.05} \\
        
        % Pair 3 
        \scalebox{0.8}{\href{https://huggingface.co/deepseek-ai/DeepSeek-R1-Distill-Qwen-1.5B}{DS-Distill-1.5B} vs \href{https://huggingface.co/ nvidia/Nemotron-Research-Reasoning-Qwen-1.5B}{NV-ProRL}} & 
        46.19 & 46.25 & \textcolor{gray}{+0.06}~{\tiny$\pm$\,0.06} && 
        45.47 & 45.87 & \textcolor{gray}{+0.40}~{\tiny$\pm$\,0.07} &&
        9.02  & 9.47  & \textcolor{gray}{+0.45}~{\tiny$\pm$\,0.02} \\
        
        % Pair 4
        \scalebox{0.8}{\href{https://huggingface.co/Qwen/Qwen3-4B}{Qwen3-4B} vs \href{https://huggingface.co/TianHongZXY/Qwen3-4B-PSR}{Qwen3-4B-PSR}} & 
        59.85 & 59.79 & \textcolor{gray}{-0.06}~{\tiny$\pm$\,0.05} && 
        63.90 & 64.57 & \textcolor{gray}{+0.67}~{\tiny$\pm$\,0.03} &&
        18.10 & 18.17 & \textcolor{gray}{+0.07}~{\tiny$\pm$\,0.03} \\
        
        \bottomrule
    \end{tabular}

    \label{tab: embedding performance}
\end{table*}
\section{Embedding Model Performances}
\label{sec: embedding benchmarks performance}

We first unify our terminology by explicitly separating a \emph{starting checkpoint} from the \emph{fine-tuning stage}.
\paragraph{Backbone LLMs.}
Given a base model $\mathcal{M}_{base}$ as a trained LLM, we consider a reasoning model $\mathcal{M}_{reason}$ as an LLM that undergoes either Supervised Fine-Tuning (SFT) or RLVR directly on top of the base model.
We focus on \textit{zero-RL}~\citep{DeepSeek-R1} where RLVR starts directly from the base model without performing a warm-start SFT stage. Concretely, we evaluate and compare a matched pair of $\mathcal{M}_{base}$ and $\mathcal{M}_{reason}$, where $\mathcal{M}_{reason}$ must be fine-tuned on $\mathcal{M}_{base}$. The SFT-tuned pairs are used as an explicit control to highlight the very close similarity observed in the RLVR-tuned comparisons.

\paragraph{Embedding models.}
We term the base embedding model $\mathcal{M}_{base}^\textit{Emb}$ and reasoning embedding model $\mathcal{M}_{reason}^\textit{Emb}$ with the backbone $\mathcal{M}_{base}$ and $\mathcal{M}_{reason}$ respectively. The embedding models are formed by removing the language modeling head and applying a pooling operator to the final-layer hidden states to produce a fixed-dimensional vector. We train embedding models with an InfoNCE objective~\citep{InfoNCE} to align semantically similar texts. Within a pair the two embedding models share identical architectures and training recipes, differing only in their backbone initialization (i.e., $\mathcal{M}_{base}$ vs.\ $\mathcal{M}_{reason}$). Appendix~\ref{appendix: embedding-model-training} provides training details.
% \section{Embedding Benchmark Performances}
% \label{sec: embedding benchmarks performance}

To rigorously assess the impact of RLVR optimization on embedding benchmark, we trained and evaluated multiple matched pairs consisting $\mathcal{M}_{base}^\textit{Emb}$ and $\mathcal{M}_{reason}^\textit{Emb}$. We evaluated these models across a diverse suite of benchmarks, including MTEB(Multilingual, v2) \citep{MMTEB}, MTEB(Code, v1) \citep{MTEB}, and BRIGHT \citep{BRIGHT}, to ensure coverage of retrieval, clustering, and semantic similarity tasks, as well as the data in the same domain as trained in RLVR.

The results, presented in Table~\ref{tab: embedding performance}, reveal that $\mathcal{M}_{reason}^\textit{Emb}$ with RLVR-tuned backbone $\mathcal{M}_{reason}$ consistently achieve performance parity with $\mathcal{M}_{base}^\textit{Emb}$ across all benchmarks. Rather than interpreting this as a limitation, we view it as a significant indicator of representational robustness. The RLVR process refines the model's trajectory-generation policy without destructively overwriting the rich world knowledge and semantic relationships established during pre-training.
\section{The HRSA Framework: Dissecting Model Similarity}
\label{sec: hrsa}

In Section~\ref{sec: embedding benchmarks performance}, we established that $\mathcal{M}_{reason}^\textit{Emb}$ with RLVR-tuned backbone maintain performance parity with $\mathcal{M}_{base}^\textit{Emb}$ across all benchmarks, exhibiting no degradation in general semantic tasks. This macroscopic observation suggests a hypothesis that the RLVR optimization trajectory preserves the intrinsic geometry of the pre-trained \textbf{Latent Manifold $\mathcal{Z}$}, altering only the policy for traversing it rather than the landscape itself.

To rigorously test this hypothesis, we must look beyond aggregate benchmark scores, which can mask internal representational shifts. We introduce \textbf{HRSA} to dissect the relationship between $\mathcal{M}_{base}$ and $\mathcal{M}_{reason}$ at three nested levels of abstraction. 

Crucially, HRSA is not defined by the specific metrics used in this study (e.g., CKA), but by the invariance properties required at each level of abstraction. Researchers can substitute metrics or theoretical constraints, provided they respect the hierarchy's invariance rules. See Table~\ref{tab: appendix_hrsa_extensibility}.

\subsection{Common Setup and Notation}
\label{sec:common-notation}

\begin{figure*}[t]
    \includegraphics[width=\linewidth]{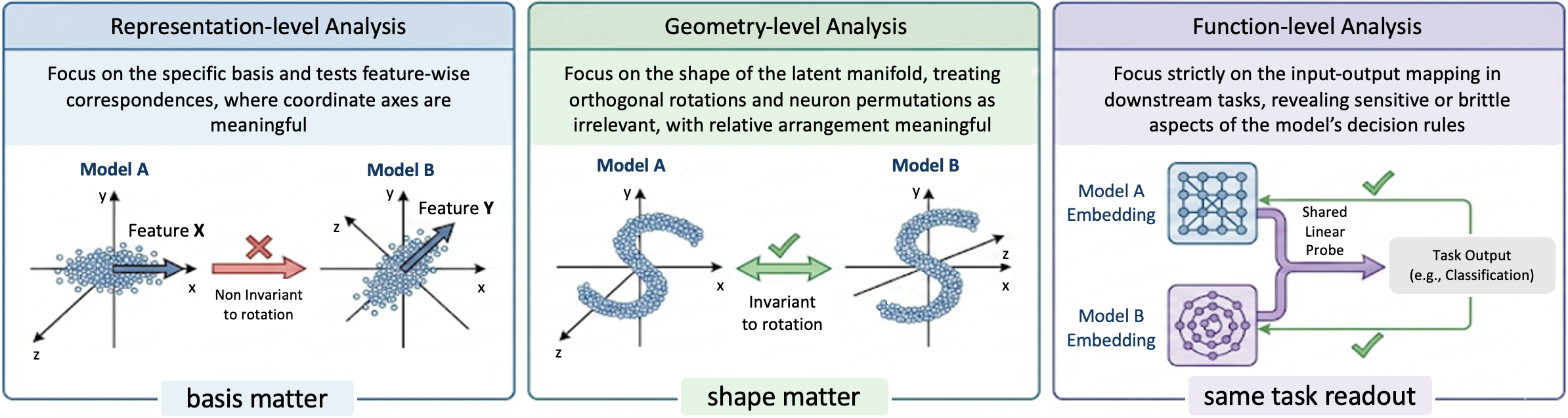}
    % \vspace{-0.2in}
    \caption{The overview of HRSA.}
    \vspace{-0.2in}
\end{figure*}

To analyze the structural differences between models, we compare their representations on a shared sequence of $N$ token positions. Let $\mathbf{X}, \mathbf{Y} \in \mathbb{R}^{N \times D}$ denote the $D$-dimensional token-level representation matrices produced by $\mathcal{M}_{base}$ (or $\mathcal{M}_{base}^{Emb}$) and $\mathcal{M}_{reason}$ (or $\mathcal{M}_{reason}^{Emb}$) respectively. The $i$-th row of each matrix, denoted as $\mathbf{x}_i$ or $\mathbf{y}_i$, represents a single token embedding such that $\mathbf{x}_i, \mathbf{y}_i \in \mathcal{Z}$, where $\mathcal{Z} \subset \mathbb{R}^D$ is the latent manifold induced by the distribution of the mapped inputs within the high-dimensional ambient space.

\subsection{Representation-Level Analysis}
\label{sec: representation-level}

Representation-level analysis focuses on the \emph{specific coordinate basis} of the latent manifold and tests feature-wise correspondences between models. At this level, we treat the coordinate basis themselves as meaningful objects, where rotating or permuting features can change the outcome of the analysis. In other words, representation-level metrics are \emph{not} invariant to orthogonal transformations or neuron permutations (see Appendix~\ref{appendix: representation-level-proof} for a formal discussion).

Intuitively, this level investigates \emph{whether two models implement similar features along similar coordinate basis or realize similar solutions in very different coordinate systems}. We operationalize this question with two complementary tools:

\begin{enumerate}
    \item \textbf{Dimension-Wise Correlation}: probe direct, axis-aligned feature correspondence~\citep{dimension-wise-correlation}.
    \item \textbf{Orthogonal Procrustes Analysis}: probe global linear equivalence~\citep{schonemann1966generalized}.
\end{enumerate}

\subsubsection{Dimension-Wise Correlation}

Dimension-wise correlation tests whether each coordinate in one model can be matched to the \emph{same} coordinate in another model. Let $\mathbf{X}_{:j}$ and $\mathbf{Y}_{:j}$ denote the $j$-th column vectors (features) of the matrices $\mathbf{X}$ and $\mathbf{Y}$, respectively, corresponding to feature $j$ across all token positions. After centering each column over tokens, we define the correlation for dimension $j$ as
\vspace{-3pt} % Adjust as needed
\begin{equation}
\label{eq:dimwise-corr}
\rho_j(\mathbf{X}, \mathbf{Y})
=
\frac{(\mathbf{X}_{:j})^\top \mathbf{Y}_{:j}}{\|\mathbf{X}_{:j}\| \|\mathbf{Y}_{:j}\|},
\quad
j = 1, \dots, D.
\end{equation}
We summarize $\{\rho_j\}_{j=1}^D$ via mean.
High per-dimension correlations indicate that many features are already aligned one-to-one without any transformation, while low per-dimension correlations suggest that, even if the models encode similar information overall, information carried by a single feature in one model may be distributed across multiple features in the other.

\subsubsection{Orthogonal Procrustes Analysis}

Dimension-wise correlation is strict: it does not allow any mixing between feature dimensions. Orthogonal Procrustes analysis relaxes this by asking whether one representation space can be mapped to the other via a single orthogonal transformation. Formally, we solve
% \vspace{-5pt} % Adjust as needed
\begin{equation}
\label{eq:procrustes}
O^* = \arg\min_{O^\top O = I}
\;
\|\mathbf{X} O - \mathbf{Y}\|_F^2,
\end{equation}
where $O^* \in \mathbb{R}^{D \times D}$ is orthogonal and $\|\cdot\|_F$ denotes the Frobenius norm. This objective allows a global orthogonal mixing of features: each feature of $\mathbf{Y}$ can be an orthogonal combination of features in $\mathbf{X}$. If $O^\star$ is a near-diagonal or near-permutation, then features can be matched almost one-to-one after a simple rotation or permutation. This corresponds to relatively localized, interpretable feature correspondences.
In contrast, if $O^\star$ is dense, then each feature in one model is a distributed combination of many features in the other. The same information may be present, but in an entangled, non-localized form.
We consider the inverse row entropy of $O^*$ as the quantified metric. See Appendix~\ref{appendix: opa} for the details.

% \paragraph{Summary.}
% Dimension-wise correlations provide a local, axis-wise view of similarity, while Orthogonal Procrustes provides a global, joint linear alignment of the two representational spaces. Together, they reveal whether base models and reasoning models differ only by simple reparameterizations (e.g., rotations) or by deeper changes in which features exist at all.

\subsection{Geometry-Level Analysis}
\label{sec: geometry-level}

\begin{table*}[t]
\caption{\textbf{HRSA result summary.} It shows how different training algorithms impact the model's manifold. SFT causes fundamental restructuring, whereas RLVR acts as a trajectory optimization. Contrastive Learning (CL) successfully realigns the latent manifold.}
\label{tab:manifold_realignment}
\centering
\footnotesize % 1. Reduces font size
\renewcommand{\arraystretch}{1.1} % 2. Reduces row height
\setlength{\tabcolsep}{3.5pt} % 3. Reduces space between columns

\begin{tabular*}{\textwidth}{@{\extracolsep{\fill}} l l | c c c @{}}
\toprule
\textbf{Level} & \textbf{Metric Focus} & \textbf{SFT (Backbone)} & \textbf{RLVR (Backbone)} & \textbf{Post-CL (Embedding)} \\ 
& & $\mathcal{M}_{base}$ vs $\mathcal{M}_{reason}$ & $\mathcal{M}_{base}$ vs $\mathcal{M}_{reason}$ & $\mathcal{M}_{base}^{Emb}$ vs $\mathcal{M}_{reason}^{Emb}$ \\
\midrule
% Section 1
\textbf{1. Representation} & \textit{Coordinate Basis} & Destructive Mixing & \textbf{Preserved} & Re-Aligned \\
\addlinespace[0.5ex]

% Section 2
\multirow{2}{*}{\textbf{2. Geometry}} 
 & \textit{Global Geometry} & Anisotropic Distortion & \textbf{Isometric} & Stable \\
 & \textit{Local Geometry} & Reorganized & Reorganized & Irreversible \\
\addlinespace[0.5ex]

% Section 3
\textbf{3. Function} & \textit{Linear Readout} & Degraded & \textbf{Transferred} & Aligned \\

\midrule
% Summary Row - UPDATED: Added '|' to the multicolumn spec
\multicolumn{2}{l|}{\textit{\textbf{Manifold Status}}} & \textit{Fundamental Restructuring} & \textit{Trajectory Optimization} & \textit{\textbf{Manifold Realignment}} \\
\bottomrule
\end{tabular*}
\end{table*}

Geometry-level analysis moves one step up in abstraction. Instead of caring about the specific coordinate system, we focus on the \emph{shape} of the latent manifold. At this level, orthogonal rotations and neuron permutations are treated as irrelevant; the relative arrangement of points is paramount. Geometry-level metrics are therefore invariant to changes of coordinate basis, but sensitive to deformations that alter distances or local neighborhoods. See Appendix~\ref{appendix: geometry-level-proof}.

Conceptually, this level investigates \emph{whether two models organize embeddings into similar manifold shapes even using different axes}. We study geometry-level similarity using two complementary metrics:

\begin{enumerate}
    \item \textbf{Linear Centered Kernel Alignment (Linear CKA)}: measures \emph{global geometry} of manifold, via their Gram matrices, up to orthogonal transforms and isotropic scaling~\citep{revisited-similarity}.
    \item \textbf{$k$-Nearest Neighbors ($k$-NN) Overlap}: measures \emph{local geometry} of manifold by quantifying the preservation of nearest-neighbor relationships~\citep{knn}.
\end{enumerate}

\subsubsection{Linear CKA}

CKA compares representations via their kernel matrices rather than their raw coordinates. In Linear CKA, we consider the linear kernel $K_X = \mathbf{X}\mathbf{X}^\top$ and $K_Y = \mathbf{Y}\mathbf{Y}^\top$, where $K_X, K_Y \in \mathbb{R}^{N \times N}$. The Hilbert--Schmidt Independence Criterion (HSIC)~\citep{hsic} between the kernels $K_X$ and $K_Y$ is
\vspace{-3pt} % Adjust as needed
\begin{equation}
\label{eq:hsic}
\mathrm{HSIC}(K_X, K_Y)
=
\frac{1}{(N-1)^2} \, \mathrm{tr}(K_X H K_Y H).
\end{equation}
where $H = I - \frac{1}{N}\mathbf{1}\mathbf{1}^\top \in \mathbb{R}^{N \times N}$ is the centering matrix, $I$ is the identity matrix and $\mathbf{1}$ is a vector of $1$. Linear CKA is then

\begin{equation}
\label{eq:cka}
\mathrm{CKA}(\mathbf{X}, \mathbf{Y})
=
\frac{\mathrm{HSIC}(K_X, K_Y)}{\sqrt{\mathrm{HSIC}(K_X, K_X) \, \mathrm{HSIC}(K_Y, K_Y)}}.
\end{equation}

Linear CKA quantifies how similarly two models organize the \emph{global geometry} of manifolds, capturing features like cluster structure and anisotropy, by assessing whether their manifolds can be aligned through orthogonal transformations and uniform scaling, without the need for nonlinear deformations.

\citet{decodable-info} has shown that Linear CKA quantifies the average alignment between optimal linear readouts across a distribution of decoding tasks. However, it can be manipulated without large changes in functional behavior under high dimensions~\citep{cka-reliability,rep-sim-capture-func}, so it should not be interpreted as a direct proxy for linear separability or task equivalence. Here, we use Linear CKA specifically as a global geometry descriptor.

\subsubsection{$k$-NN Overlap}
\label{sec:knn-overlap}

While CKA captures global manifold geometry, $k$-NN overlap focuses on local manifold geometry. Intuitively, it investigates whether each embedding’s local neighborhood is preserved between two models.

Let the cosine similarity be $s_{\mathbf{Z}}(i,j) = \frac{z_i^\top z_j}{\|z_i\|\,\|z_j\|}$, where $z_i \in \mathbb{R}^D$ is the $i$-th embedding (row) of the representation matrix $\mathbf{Z}$. We define the $k$-nearest neighbor sets under models with representations $\mathbf{X}$ and $\mathbf{Y}$ as
$N_k^X(i) = \operatorname{TopK}_j \, s_{\mathbf{X}}(i, j)$ and $N_k^Y(i) = \operatorname{TopK}_j \, s_{\mathbf{Y}}(i, j)$, respectively.
The $k$-NN overlap score $\tilde{s}_k$ is
% \vspace{-10pt} % Adjust as needed
\begin{equation}
\label{eq:knn-overlap}
\tilde{s}_k
=
\frac{1}{N}
\sum_{i=1}^N
\frac{\left|N_k^X(i) \cap N_k^Y(i)\right|}
{\left|N_k^X(i) \cup N_k^Y(i)\right|}
\end{equation}
where we use the Jaccard index to quantify agreement of neighbor sets. Because neighborhood relations are preserved under orthogonal transforms and permutations, but disrupted by non-isometric distortions (e.g., anisotropic scaling), $k$-NN overlap directly reflects how similarly two models instantiate the \emph{local geometry} of the manifold.

\subsection{Function-Level Analysis}
\label{sec: function-level}

\begin{figure*}[t]
    \centering
    \small 
    % 1. INCREASED TABCOLSEP (General horizontal padding)
    \setlength{\tabcolsep}{3pt} 
    \renewcommand{\arraystretch}{1.05} 

    % DEFINITIONS
    \newlength{\imgwidth}
    \setlength{\imgwidth}{0.17\textwidth} 

    % Horizontal Header
    \newcommand{\headerbox}[2]{%
        \parbox{\linewidth}{%
            \centering
            \colorbox{#1}{\makebox[\dimexpr\linewidth-2\fboxsep\relax]{\textbf{#2}}}%
        }%
    }

    % Vertical Header 
    \newcommand{\verticalbox}[2]{%
        \rotatebox[origin=c]{90}{%
            \colorbox{#1}{\makebox[\dimexpr\imgwidth-2\fboxsep\relax]{\scriptsize \textbf{#2}}}%
        }%
    }

    % 2. MODIFIED COLUMN SPECIFIER
    % Added @{\hspace{2em}} between the 5th and 6th column to separate the two heatmap groups
    \begin{tabular}{c c m{\imgwidth} m{\imgwidth} c @{\hspace{2em}} m{\imgwidth} m{\imgwidth} c}
        
        % ROW 1: TITLES
        & & 
        \multicolumn{2}{c}{\textbf{Dimension-Wise Correlation}} & 
        & 
        \multicolumn{2}{c}{\textbf{Linear CKA}} & 
        \\
        
        % ROW 2: HEADERS (SFT / RLVR)
        & & 
        \headerbox{red!20}{SFT} & \headerbox{blue!20}{RLVR} & 
        & 
        \headerbox{red!20}{SFT} & \headerbox{blue!20}{RLVR} & 
        \\

        % ROW 3: LLM CONTENT
        % Y-Axis Label
        \multirow{2}{*}{\rotatebox{90}{\textbf{Base Model Layer Index}}} & 
        
        % Row Label
        \verticalbox{orange!20}{LLM} & 
        
        % Dim-Wise Imgs
        \begin{subfigure}{\linewidth}
            \includegraphics[width=\linewidth]{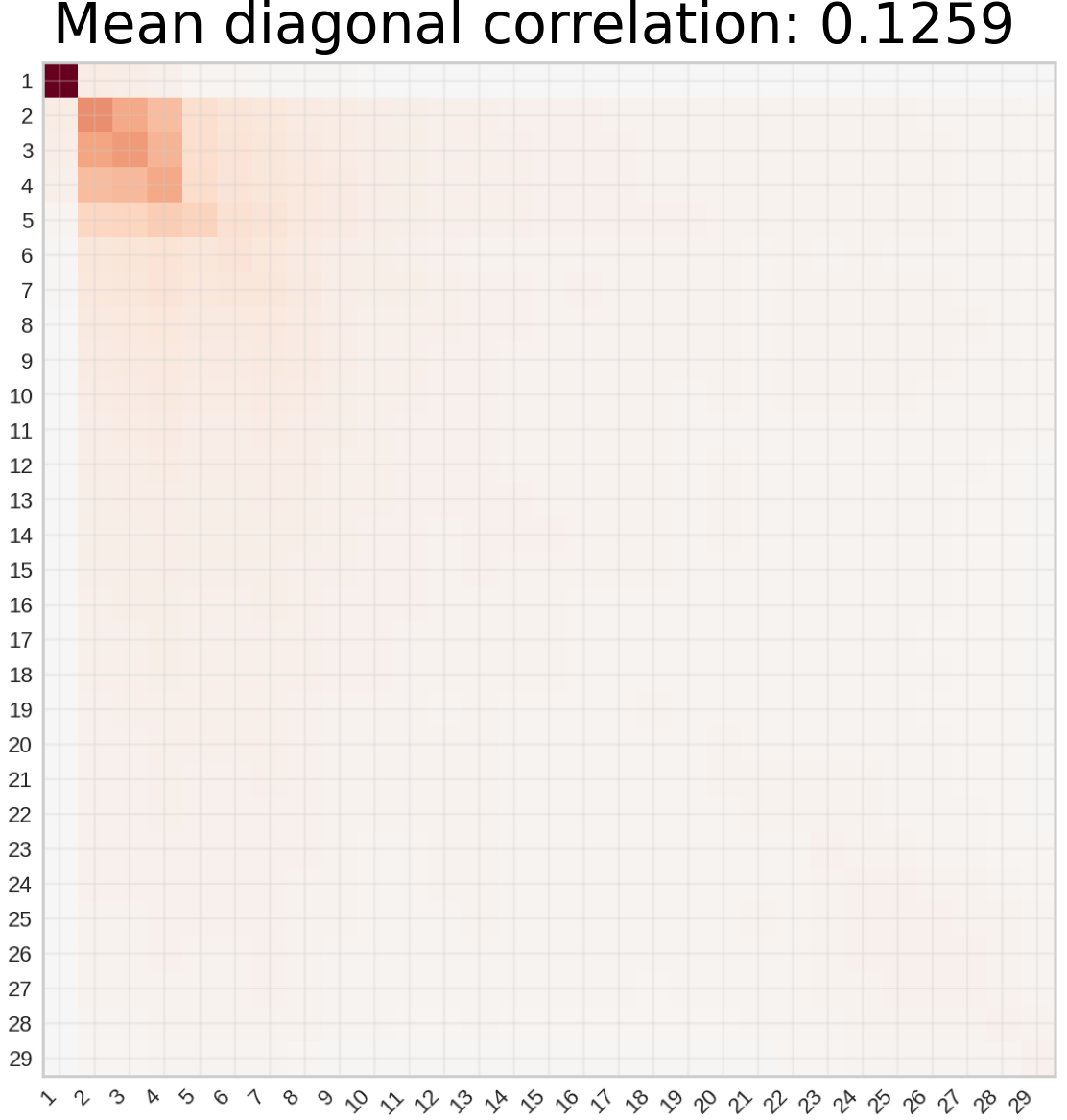}
            \caption*{\tiny Qwen2.5 vs DS} 
        \end{subfigure} & 
        \begin{subfigure}{\linewidth}
            \includegraphics[width=\linewidth]{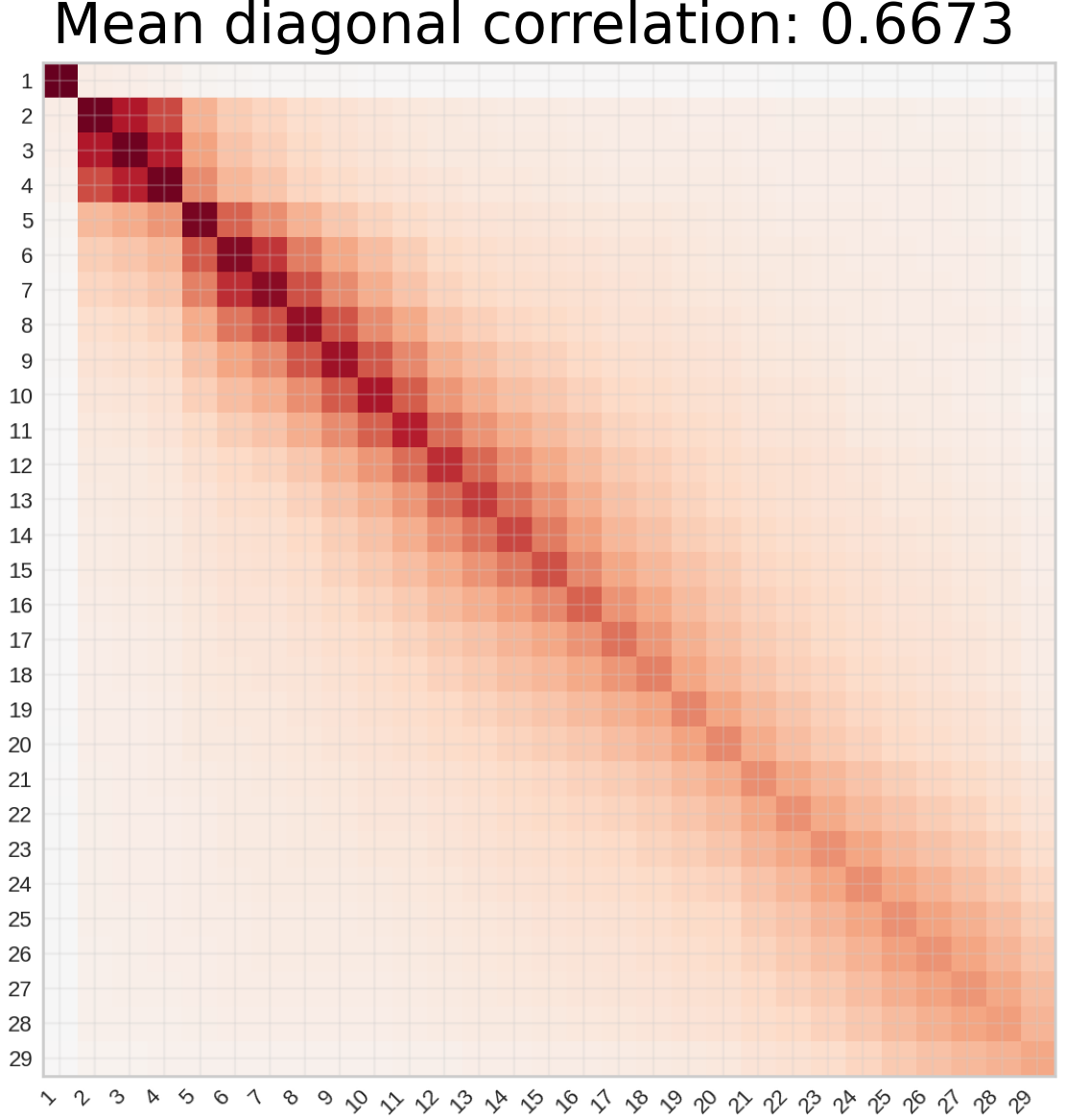}
            \caption*{\tiny DS vs ProRL}
        \end{subfigure} & 
        
        % Cbar 1 
        \raisebox{0.6cm}{\multirow{2}{*}{\includegraphics[height=6cm, width=0.04\textwidth, keepaspectratio]{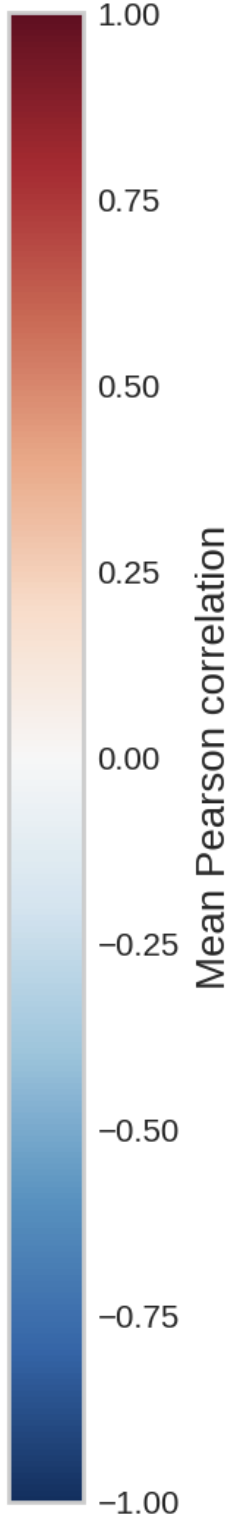}}} & 

        % CKA Imgs
        \begin{subfigure}{\linewidth}
            \includegraphics[width=\linewidth]{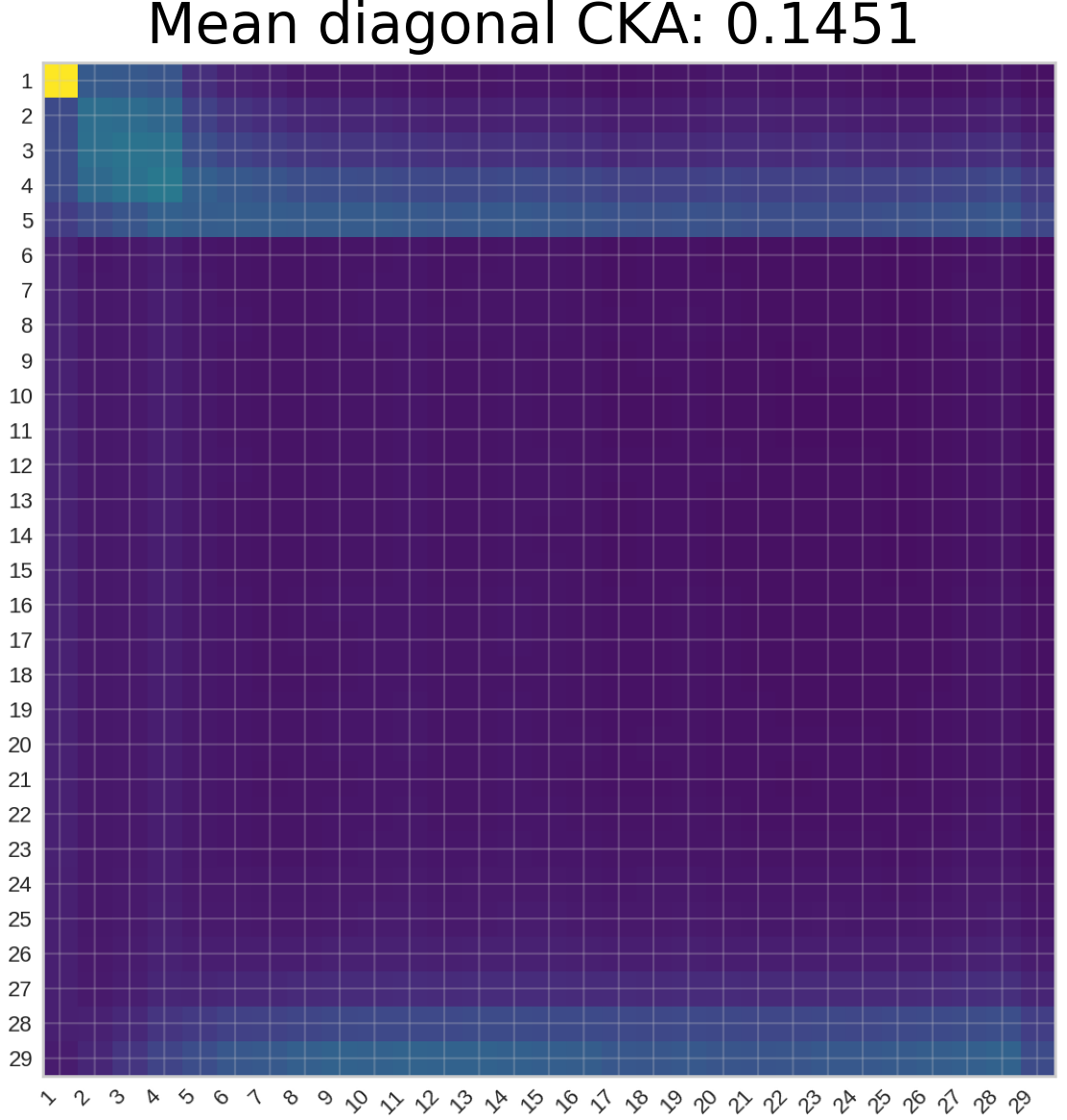}
            \caption*{\tiny Qwen2.5 vs DS}
        \end{subfigure} & 
        \begin{subfigure}{\linewidth}
            \includegraphics[width=\linewidth]{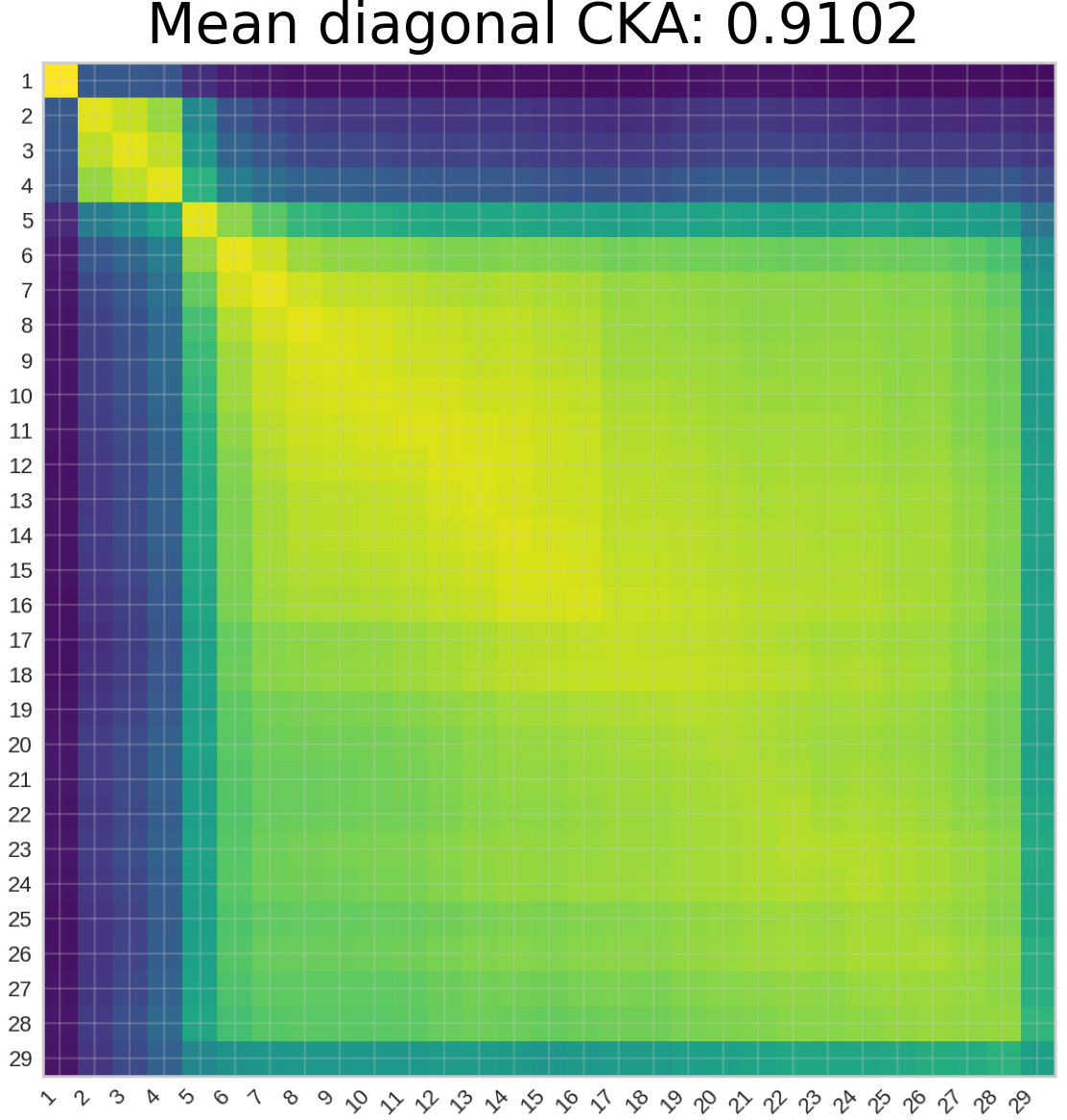}
            \caption*{\tiny DS vs ProRL}
        \end{subfigure} &
        
        % Cbar 2
        \raisebox{0.8cm}{\multirow{2}{*}{\includegraphics[height=6cm, width=0.04\textwidth, keepaspectratio]{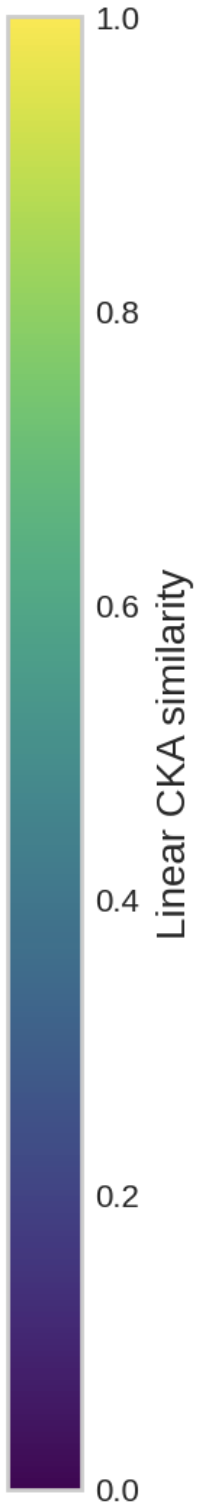}} }
        \\

        % ROW 4: EMBEDDING MODEL CONTENT
        % Y-Axis (Empty)
        & 
        
        % Row Label
        \verticalbox{orange!20}{Embed Model} & 
        
        % Dim-Wise Imgs
        \begin{subfigure}{\linewidth}
            \includegraphics[width=\linewidth]{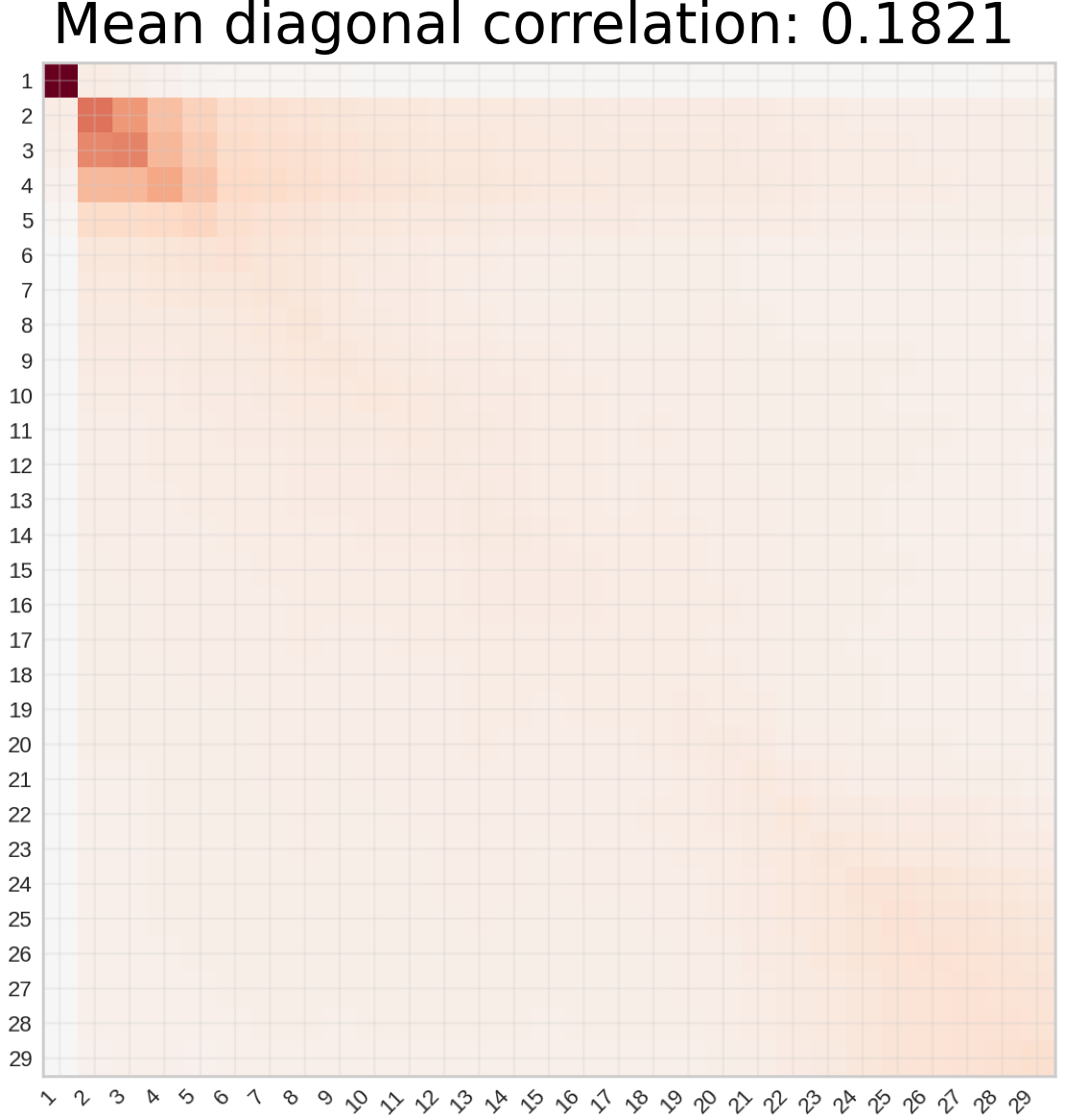}
            \caption*{\tiny Qwen2.5-\textit{Emb} vs DS-\textit{Emb}}
        \end{subfigure} & 
        \begin{subfigure}{\linewidth}
            \includegraphics[width=\linewidth]{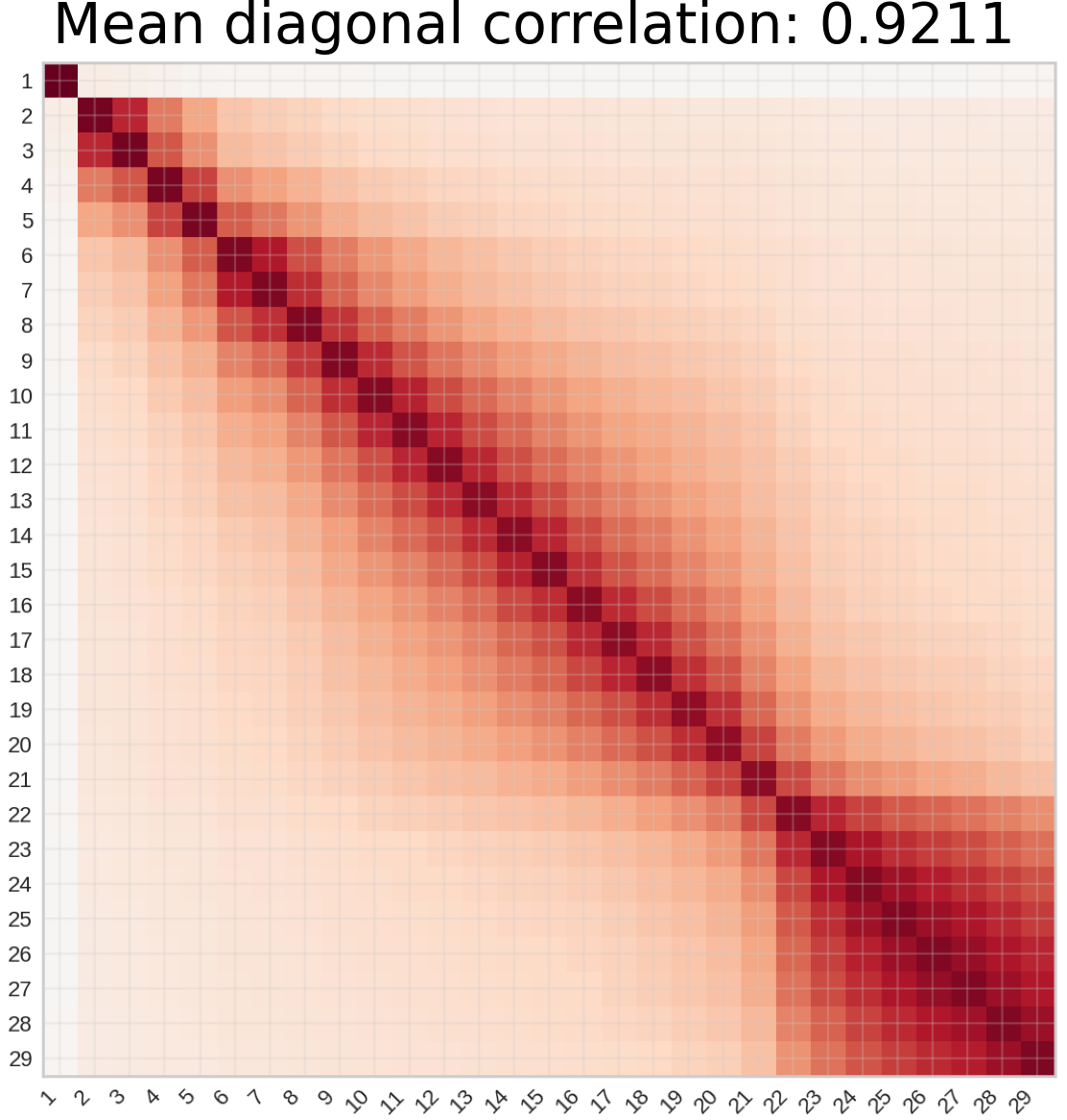}
            \caption*{\tiny DS-\textit{Emb} vs ProRL-\textit{Emb}}
        \end{subfigure} & 
        
        % Cbar 1 (Empty)
        & 
        
        % CKA Imgs
        \begin{subfigure}{\linewidth}
            \includegraphics[width=\linewidth]{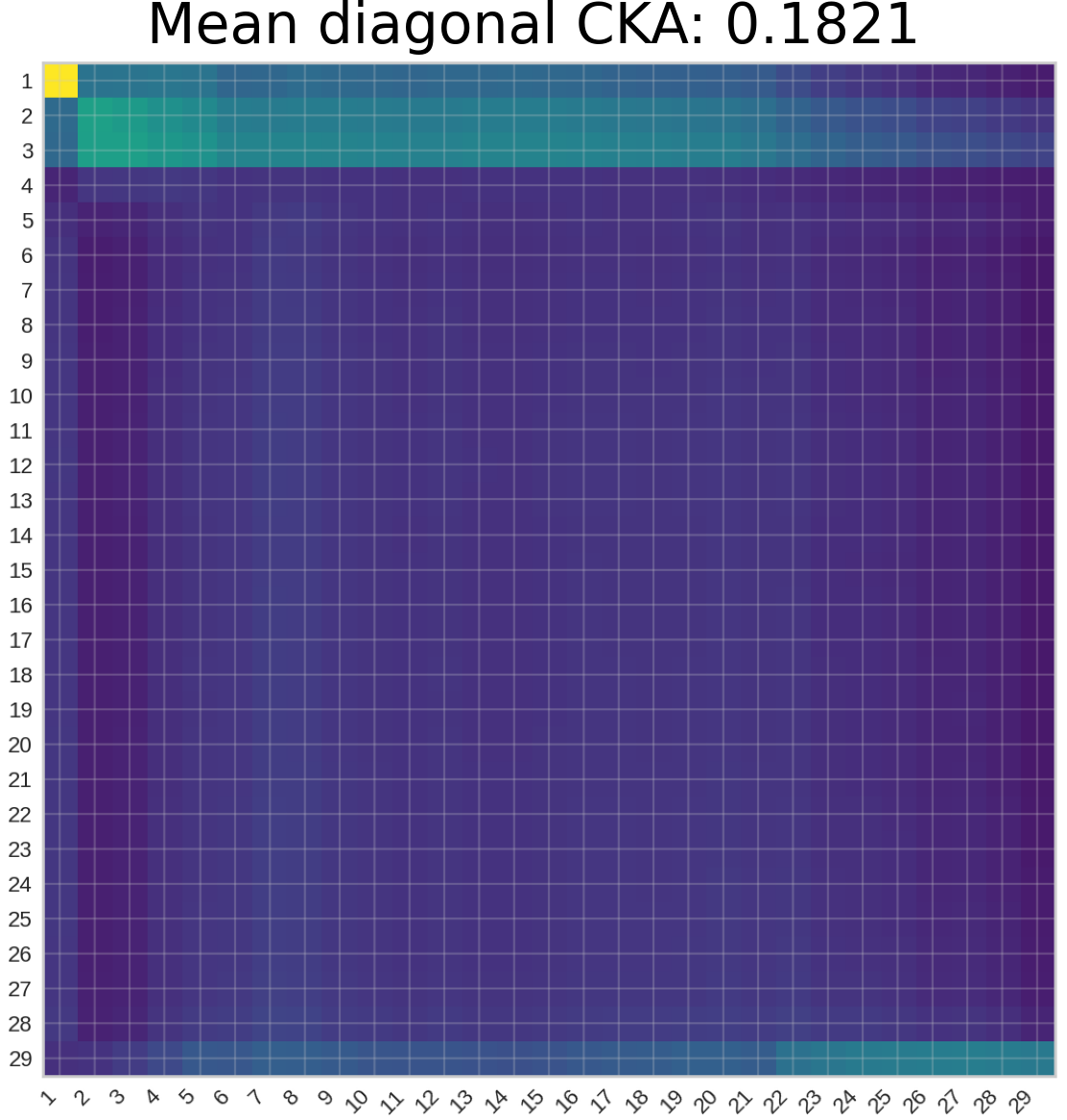}
            \caption*{\tiny Qwen2.5-\textit{Emb} vs DS-\textit{Emb}}
        \end{subfigure} & 
        \begin{subfigure}{\linewidth}
            \includegraphics[width=\linewidth]{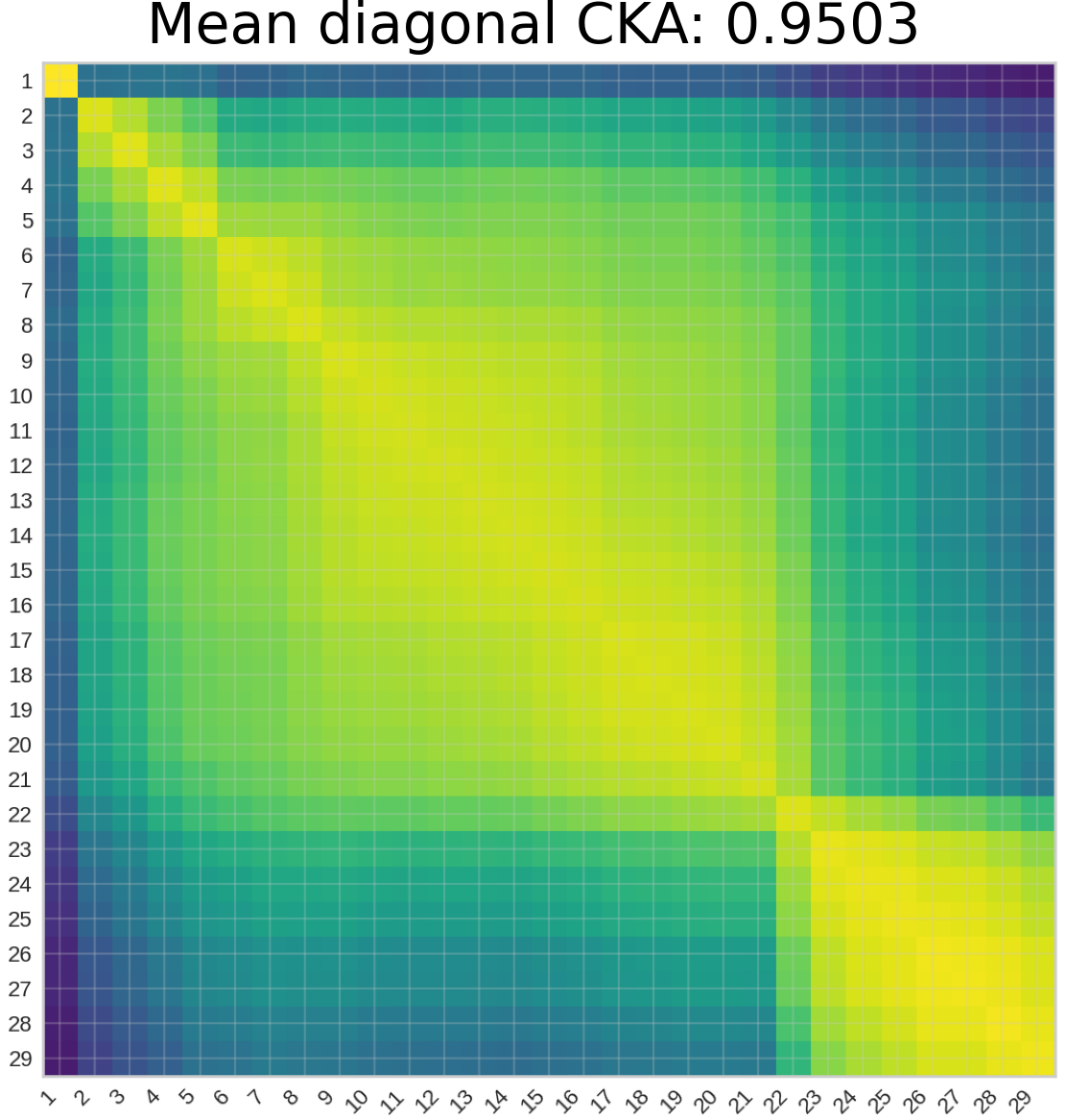}
            \caption*{\tiny DS-\textit{Emb} vs ProRL-\textit{Emb}}
        \end{subfigure} &

        % Cbar 2 (Empty)
        \\
        
    \end{tabular}

    % X-Axis Label
    \vspace{0.2em}
    \centerline{\textbf{Reasoning Model Layer Index}}

    \caption{\textbf{Heatmap of Dimension-Wise Correlation (left) and Linear CKA (right).} \textbf{Columns}: SFT vs. RLVR. \textbf{Rows}: $\mathcal{M}_{base}$ vs. $\mathcal{M}_{reason}$ and $\mathcal{M}_{base}^\textit{Emb}$ vs. $\mathcal{M}_{reason}^\textit{Emb}$.}
    \label{fig: corr_and_cka_figures}
    \vspace{-0.15in}
\end{figure*}

Function-level analysis abstracts away from the internal representation manifold to focus strictly on \emph{the input–output transformations exhibited during downstream tasks}. Two models may have very different representation- or geometry-level metrics, yet still be functionally similar if the same tasks are solvable with comparable readouts. Conversely, even modest changes in embeddings can yield different behaviors under specific decoders. See Appendix~\ref{appendix: function-level-proof}.

We instantiate function-level similarity with the \textbf{Cross-Model Linear Probes}~\citep{cross-model-linear-probe} to test whether the \emph{same} linear readout generalizes across models.

\subsubsection{Cross-Model Linear Probes}

Cross-model linear probes provide a task-conditioned measure of function-level similarity between embedding spaces. Let $y \in \mathbb{R}^N$ be the labels. We first fit a linear probe on $\mathbf{X}$:
% \vspace{-5pt} % Adjust as needed
\begin{equation}
\label{eq:probe-on-X}
\hat{y} = \mathbf{X} W_X + b_X
\end{equation}
where $(W_X, b_X)$ are learned via logistic or ridge regression, depending on the task. We then \emph{freeze} $(W_X, b_X)$ and apply the same linear map to representations from the other model. High cross-model performance (probe trained on $\mathbf{X}$, evaluated on $\mathbf{Y}$) relative to self-performance (trained and evaluated on $\mathbf{X}$) indicates that the same linear decision boundary is useful in both spaces. This implies strong function-level similarity: two models support essentially the same set of linearly decodable functions for that task.

% \paragraph{Summary.}
% Because function-level metrics operate purely on observable outputs or task-specific readouts, they are agnostic to the internal parameterization and representation geometry. They capture whether two systems \emph{implement} similar input--output functions in practice, even when their representation- and geometry-level metrics differ.

\subsection{From Representations to Functions: How the Levels Fit Together}

% =============================================================================  
% orthogonal matrix O 
% =============================================================================

\begin{wraptable}{r}{0.5\columnwidth}  % 'r' for right, 'l' for left
    \vspace{0.2in}
    \small
    \centering
    \setlength{\tabcolsep}{2.5pt}
    \renewcommand{\arraystretch}{1.15}
    \caption{The inverse row entropy of the orthogonal matrix $O^*$. For each training method (SFT or RLVR), we report $\mathcal{M}_{base}$ vs. $\mathcal{M}_{reason}$ and $\mathcal{M}_{base}^\textit{Emb}$ vs. $\mathcal{M}_{reason}^\textit{Emb}$ (suffix \textit{-Emb}) comparisons. A higher inverse row entropy indicates $O^*$ corresponds to one-to-one feature mapping.}
    
    \begin{tabular}{@{}lc@{}}
        \toprule
        \textbf{Model Pair} & \textbf{Inverse Row Entropy}$\, \uparrow$ \\
        \midrule
        \textbf{SFT:} {Qwen2.5} vs {DS} & 0.108 \\
        \quad $\rightarrow$ {Qwen2.5}-\textit{Emb} vs {DS}-\textit{Emb} & 0.142 \\
        
        \midrule
        
        \textbf{RLVR:} {DS} vs {ProRL} & \textbf{0.161} \\
        \quad $\rightarrow$ {DS}-\textit{Emb} vs {ProRL}-\textit{Emb} & \textbf{0.863} \\
        \bottomrule
    \end{tabular}
    \label{tab: orthogonal O numbers}
\end{wraptable}

HRSA forms a hierarchy of abstraction over the same underlying representations. \textbf{Representation level} asks whether the latent manifolds of two models share the same coordinate basis. \textbf{Geometry level} discards the choice of coordinate basis and asks whether the latent manifold has a similar shape globally and locally. \textbf{Function level} discards most geometric detail and asks which input--output mappings are supported and realized.

By separating these levels, we can distinguish:
\begin{itemize}
    \item Cases where $\mathcal{M}_{base}$ and $\mathcal{M}_{reason}$ differ mainly by a reparameterization (e.g., rotation) but preserve geometry and function.
    \item Cases where global or local geometry changes but downstream behavior remains similar, suggesting redundant internal solutions.
    \item Cases where modest representational changes induce large functional differences in readout directions, revealing sensitive or brittle aspects of the model’s decision rules.
\end{itemize}

This decomposition allows us to turn the initial puzzle—why $\mathcal{M}_{base}^\textit{Emb}$ and $\mathcal{M}_{reason}^\textit{Emb}$ look so similar on benchmarks—into a structured investigation of \emph{where} any differences live. Is the reason of the negligible deviation lived in their underlying backbone models?
\section{Evaluation Setups}
\label{sec: setups}

To empirically validate the hypothesis, we apply the HRSA framework across two dimensions: LLMs ($\mathcal{M}_{base}$ vs. $\mathcal{M}_{reason}$) and downstream adaptation ($\mathcal{M}_{base}^\textit{Emb}$ vs. $\mathcal{M}_{reason}^\textit{Emb}$). Furthermore, we extend this analysis by also comparing the SFT-tuned reasoning models, showing the clear differences in SFT-tuned and RLVR-tuned reasoning models. See Appendix~\ref{appendix: additional-results} for more experiment results.

\paragraph{Datasets.}
To verify if the latent manifold is preserved even within reasoning trajectories, we construct a Chain-of-Thought (CoT) dataset and use the hard-level subset for evaluation. Further generation details are provided in Appendix~\ref{appendix: cot-datasets}. In function-level analysis, we use the AG's News Topic Classification Dataset~\citep{zhang2015characterlevel} to evaluate the linear readout directions.

\paragraph{Models.}
For SFT comparison, we use \textit{Qwen2.5-Math-1.5B} (base)~\citep{qwen2.5-series} vs. \textit{DeepSeek-R1-Distill-Qwen-1.5B} (reasoning)~\citep{DeepSeek-R1}. For RLVR comparison, we use \textit{DeepSeek-R1-Distill-Qwen-1.5B} (base) vs. \textit{Nemotron-Research-Reasoning-Qwen-1.5B}~\citep{NV-ProRL-model} (reasoning, trained with prolonged RLVR). We abbreviate these as \textbf{Qwen2.5}, \textbf{DS}, and \textbf{ProRL} respectively. Downstream embedding models add ``\textit{-Emb}'' suffix. Additional results for RLVR models with different training algorithms (GRPO~\citep{DeepSeekMath}, DAPO~\citep{DAPO}) and training datasets are in Appendix~\ref{appendix: additional-results}, showing consistent latent manifold across variations. We also demonstrate the training dynamic of manifold realignment.

% Instead of only considering the final model checkpoint, we also conduct experiments on the intermediate model checkpoints to demonstrate the training dynamic of manifold realignment.

Our HRSA analysis examines model activations at every layer, before any pooling. Specifically, for each model in a matched pair with $L$ layers, we collect the entire set of hidden states: $\{\mathbf{X}_l\}_{l=1}^L$ and $\{\mathbf{Y}_l\}_{l=1}^L$, where each $\mathbf{X}_l, \mathbf{Y}_l \in \mathbb{R}^{N \times D}$. This per-layer, per-token perspective preserves the full representational structure for a more comprehensive analysis, avoiding any information loss due to pooling.
\section{Results}
\label{sec: experiments}

% =============================================================================  
% cross-model linear probe 
% =============================================================================
\newcolumntype{Y}{>{\raggedright\arraybackslash}X}

\begin{wraptable}{r}{0.5\columnwidth}  % 'r' for right, 'l' for left
\centering
\caption{$k$-NN mean overlap across layers between $\mathcal{M}_{base}$ and $\mathcal{M}_{reason}$ (and their Emb variants). Higher mean overlap indicates more preservation in \emph{local geometry} of latent manifold.}
\label{tab: knn-neighborhood-overlap}

\footnotesize
\setlength{\tabcolsep}{3.5pt} % tighter columns [web:27]
\renewcommand{\arraystretch}{1.12}

\begin{tabularx}{0.48\columnwidth}{@{}Yccc@{}}
\toprule
\multirow{3}{*}{\textbf{Model Pairs}}
  & \multicolumn{3}{c}{\textbf{Mean Overlap} \(\uparrow\)} \\
\cmidrule(lr){2-4}
  & \(\,k{=}5\,\) & \(\,k{=}10\,\) & \(\,k{=}50\,\) \\
\midrule
\textbf{SFT:} Qwen2.5 vs DS
  & 0.052 & 0.068 & 0.132\\
  % & 0.052 $\pm$ 0.019 & 0.068 $\pm$ 0.027 & 0.132 $\pm$ 0.062 \\
\quad $\rightarrow$ Qwen2.5-\textit{Emb} vs DS-\textit{Emb}
  & 0.069 & 0.068 & 0.091 \\
  % & 0.069 $\pm$ 0.026 & 0.068 $\pm$ 0.027 & 0.091 $\pm$ 0.055 \\

\midrule
\textbf{RLVR:} {\small DS} vs {\small ProRL} 
& \textbf{0.455} & \textbf{0.484} & \textbf{0.577} \\
% & 0.455 $\pm$ 0.073 & 0.484 $\pm$ 0.065 & 0.577 $\pm$ 0.042 \\
\quad $\rightarrow$ {\small DS}-\textit{Emb} vs {\small ProRL}-\textit{Emb}
& \textbf{0.451} & \textbf{0.474} & \textbf{0.531} \\
% & 0.451 $\pm$ 0.070 & 0.474 $\pm$ 0.069 & 0.531 $\pm$ 0.063 \\
\bottomrule
\end{tabularx}
\end{wraptable}
\begin{figure*}[t]
    \centering
    \setlength{\tabcolsep}{1.5pt} % Minimal horizontal padding
    \renewcommand{\arraystretch}{1.2} % Better vertical spacing

    % ---------------------------------------------------------
    % Helper command for consistent color headers
    % Usage: \headerbox{width}{color}{text}
    % ---------------------------------------------------------
    \newcommand{\headerbox}[3]{%
        \parbox{#1}{%
            \centering
            % makebox forces the text/color to fill the calculated width minus the box padding
            \colorbox{#2}{\makebox[\dimexpr#1-2\fboxsep\relax]{\textbf{#3}}}%
        }%
    }

    \begin{tabular}{c c c c c}
        % =========================================================
        % ROW 1: HEADERS
        % =========================================================
        % Empty cell for Y-axis label column
        & 
        % SFT Group (Red)
        % Width = 2 * Image Width (0.235*2 = 0.47) + 2 * tabcolsep (the gap between columns)
        \multicolumn{2}{c}{\headerbox{\dimexpr0.47\textwidth+2\tabcolsep\relax}{red!20}{SFT}} & 
        
        % RLVR Group (Blue)
        \multicolumn{2}{c}{\headerbox{\dimexpr0.47\textwidth+2\tabcolsep\relax}{blue!20}{RLVR}} \\
        
        \noalign{\vspace{2pt}} % Small vertical adjustment

        % =========================================================
        % ROW 2: CONTENT
        % =========================================================
        % 1. Y-Axis Label
        \raisebox{0.55\height}{\rotatebox{90}{\textbf{Accuracy}}} &
        
        % 2. SFT Image 1
        \begin{subfigure}[b]{0.235\textwidth}
            \includegraphics[width=\textwidth]{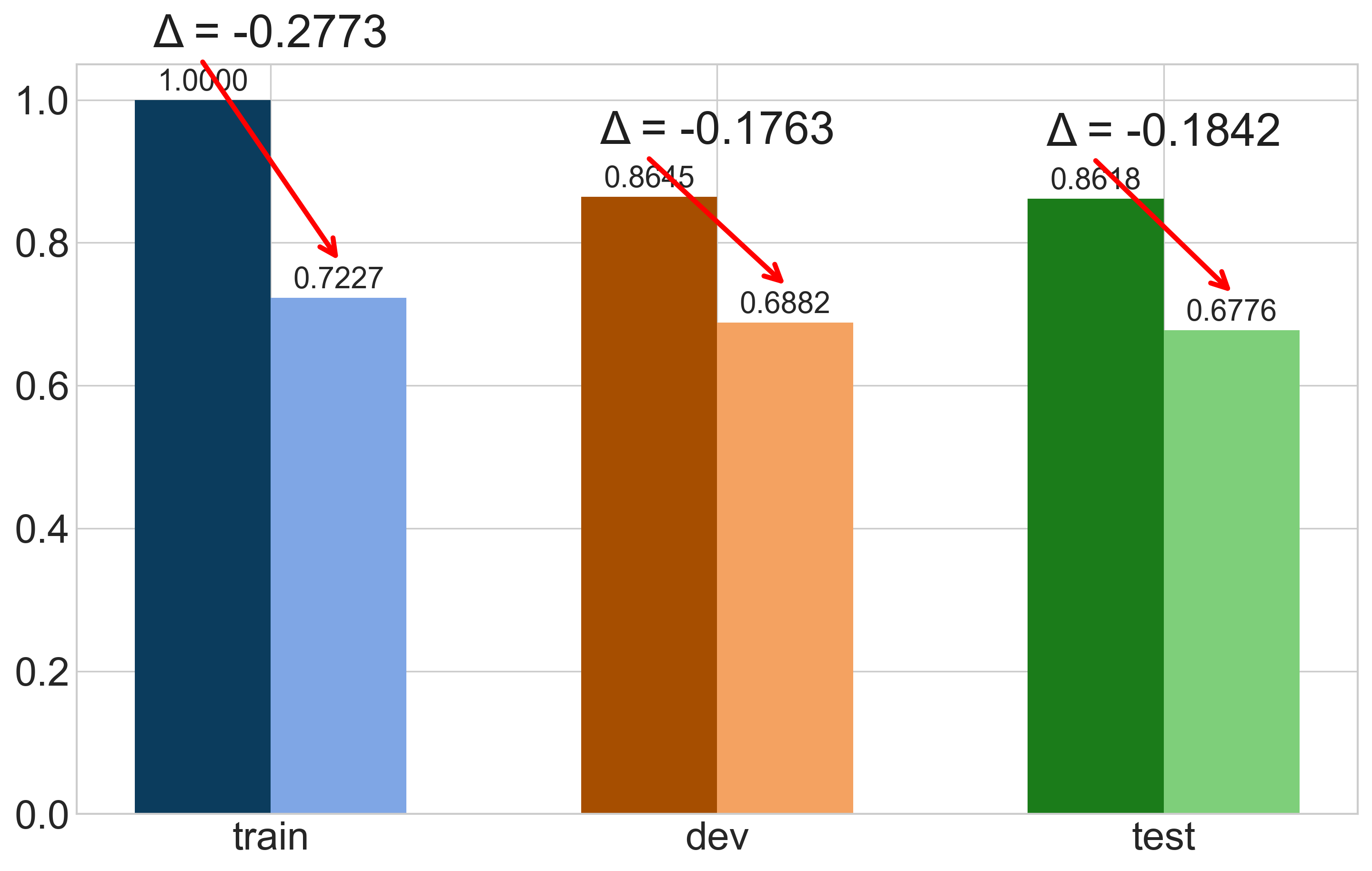}
            \caption{\scriptsize Qwen2.5 vs DS}
        \end{subfigure} &
        
        % 3. SFT Image 2
        \begin{subfigure}[b]{0.235\textwidth}
            \includegraphics[width=\textwidth]{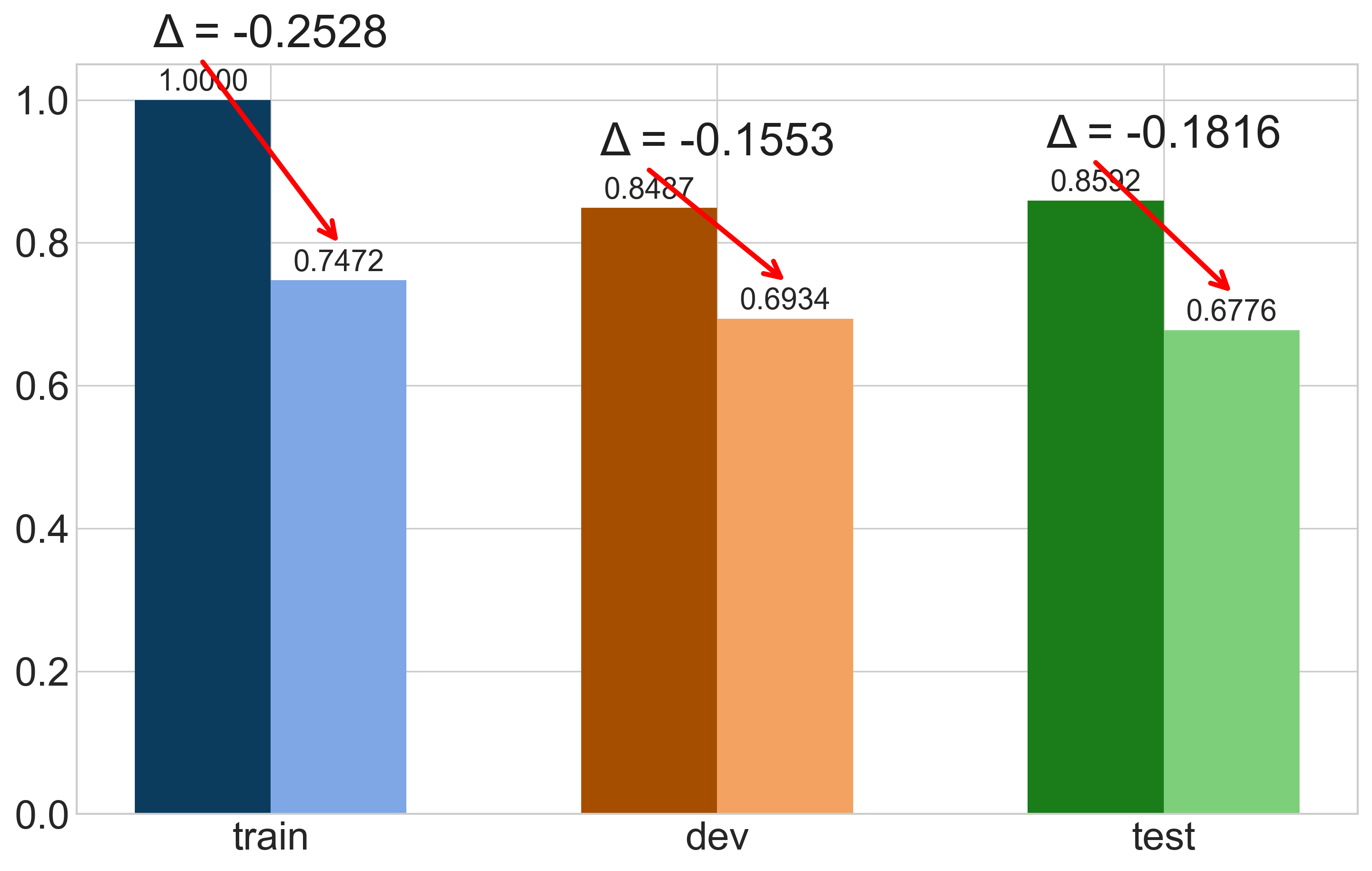}
            \caption{\scriptsize Qwen2.5-\textit{Emb} vs DS-\textit{Emb}}
        \end{subfigure} &
        
        % 4. RLVR Image 1
        \begin{subfigure}[b]{0.235\textwidth}
            \includegraphics[width=\textwidth]{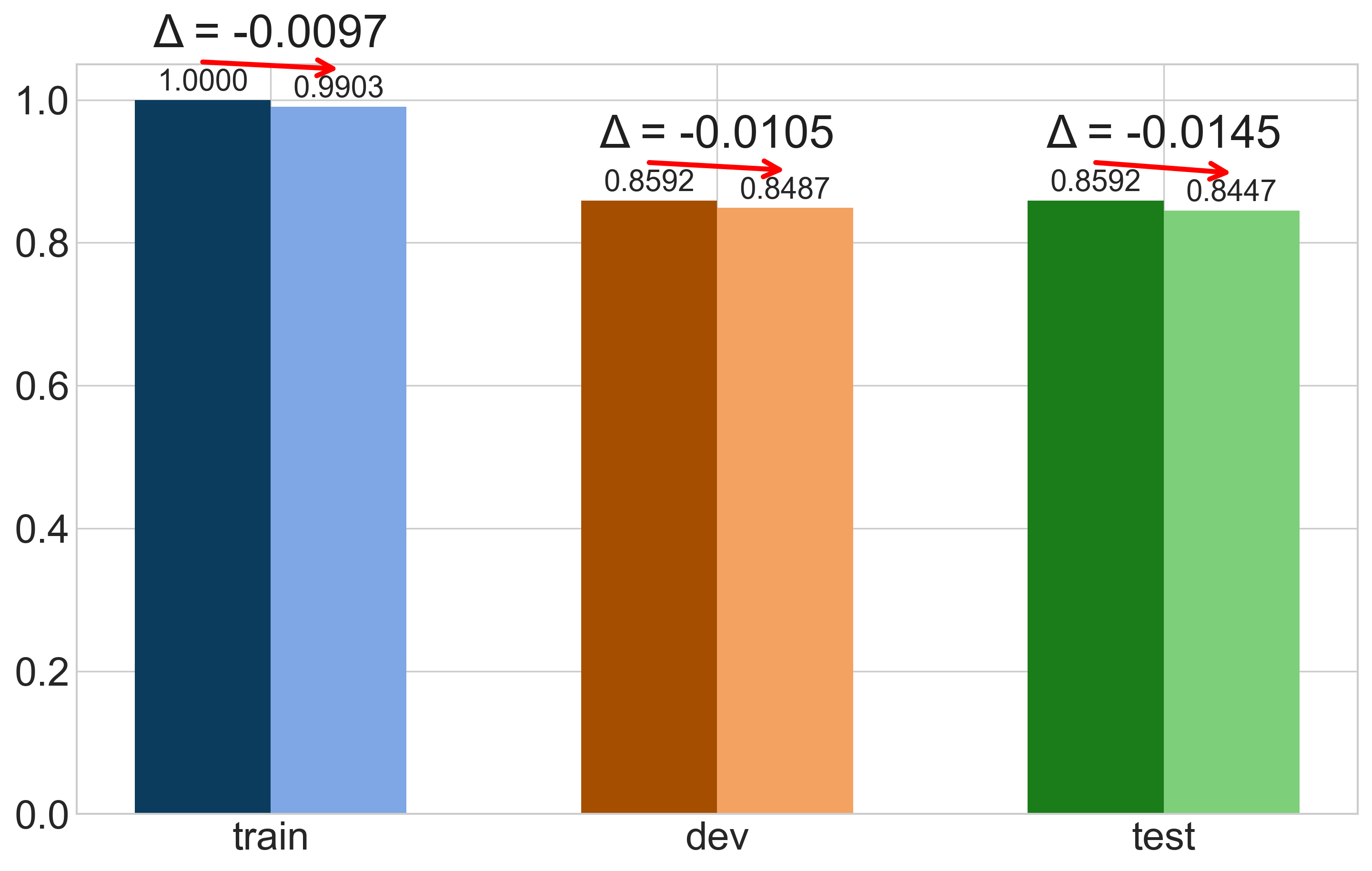}
            \caption{\scriptsize DS vs ProRL}
        \end{subfigure} &
        
        % 5. RLVR Image 2
        \begin{subfigure}[b]{0.235\textwidth}
            \includegraphics[width=\textwidth]{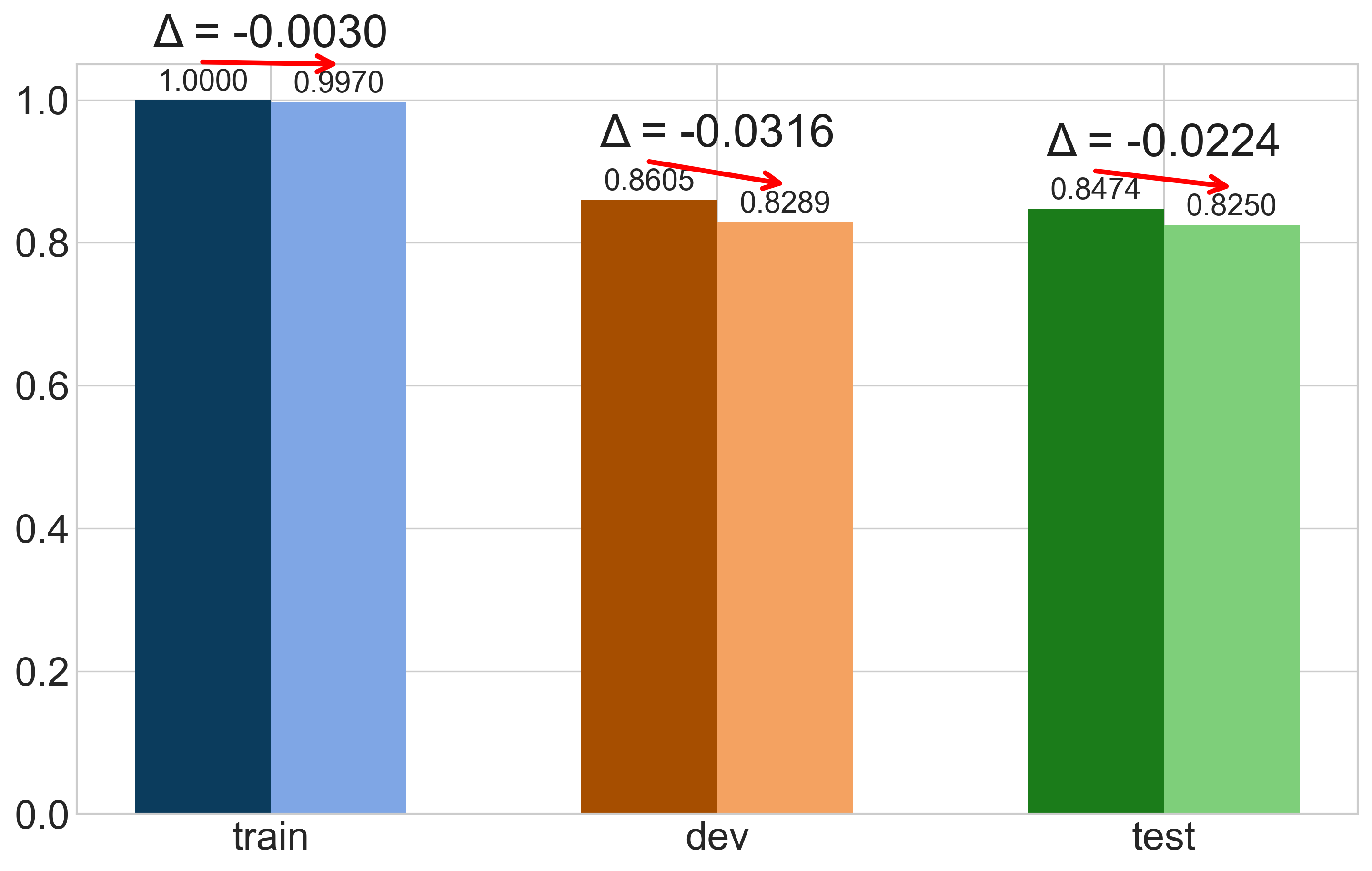}
            \caption{\scriptsize DS-\textit{Emb} vs ProRL-\textit{Emb}}
        \end{subfigure} \\
        
    \end{tabular}

    % =========================================================
    % CAPTION
    % =========================================================
    \vspace{-0.1in}
    \caption{\textbf{Cross-Model Linear Probe Results.} For each dataset split (train, dev, test), the left bar corresponds to $\mathcal{M}_{base}$ (or $\mathcal{M}_{base}^\textit{Emb}$) and the right bar corresponds to $\mathcal{M}_{reason}$ (or $\mathcal{M}_{reason}^\textit{Emb}$). The linear probe is trained on $\mathcal{M}_{base}$ (or $\mathcal{M}_{base}^\textit{Emb}$ in embedding model analysis) representations and evaluated on both models. The smaller the $\Delta$, the stronger the cross-model linear probe transfer.}
    \vspace{-0.1in}
    \label{fig: cross-model-linear-probe}
\end{figure*}

\subsection{Representation-Level Results}

\paragraph{Dimension-Wise Correlation}
Figure~\ref{fig: corr_and_cka_figures} (left) shows that SFT yields weak axis-aligned feature correspondence, while RLVR retains substantially higher per-dimension correlations. Notably, the clearest deviation from diagonal structure appears only under \emph{prolonged} RLVR (our main RLVR example), whereas contrastive learning largely restores axis alignment between the resulting embedding models, consistent with \emph{Manifold Realignment}.

\paragraph{Orthogonal Procrustes Analysis}
Table~\ref{tab: orthogonal O numbers} supports this global alignment perspective. While SFT results in a dense orthogonal map $O^\star$ (implying high feature mixing), prolonged RLVR yields an $O^\star$ that is nearly a permutation matrix, becoming strongly permutative after contrastive learning.
Table~\ref{tab: appendix-orthogonal-base-vs-reasoning} shows that $O^\star$ is already near-permutation for most $\mathcal{M}_{base}$ vs. RLVR-tuned $\mathcal{M}_{reason}$ comparisons. This suggests RLVR does not induce feature mixing; instead, coordinate basis drift is limited to prolonged training scenarios (as in ProRL). RLVR encourages the model to construct correct paths using existing capabilities, learning only the sequence of feature activations required for rewards. Thus, the coordinate basis remains largely unchanged.

% \begin{align}
%     \min_{P} \min_{O} \|XO - PY\|_F
% \end{align}
\subsection{Geometry-Level Results}

\paragraph{Linear CKA}
Figure~\ref{fig: corr_and_cka_figures} (right) shows a sharp contrast in \emph{global manifold geometry}. Linear CKA drops under SFT but remains high under RLVR, consistent with an approximately isometric relationship. After contrastive learning, the $\mathcal{M}_{base}^\textit{Emb}$ and $\mathcal{M}_{reason}^\textit{Emb}$ move even closer in CKA, highlighting \emph{Manifold Realignment} at the geometry level. RLVR functions as a near-isometric transformation, rigidly preserving the shape of the latent manifold. Consequently, the semantic distances established during pre-training remain invariant, which explains why downstream embedding performance does not improve.

\paragraph{$k$-NN Overlap}
Table~\ref{tab: knn-neighborhood-overlap} shows that RLVR preserves substantially more \emph{local} structure (higher mean overlap) than SFT, yet overlap remains substantially below $1$, indicating local geometry reorganization. This gap persists even when the embedding model manifolds are pulled closer by contrastive learning, showing the idea that RLVR introduces irreversible local geometry reorganization, which is different to the rigid global geometry. We hypothesize that this irreversible local reorganization reflects RLVR optimization in grouping related reasoning steps effectively, clusters the decision trajectory without altering the global semantic map.

\subsection{Function-Level Results}

\begin{wrapfigure}{r}{0.5\columnwidth}  % 'r' for right, 'l' for left
    \vspace{-0.61in}
    \centering
    \includegraphics[width=0.48\columnwidth]{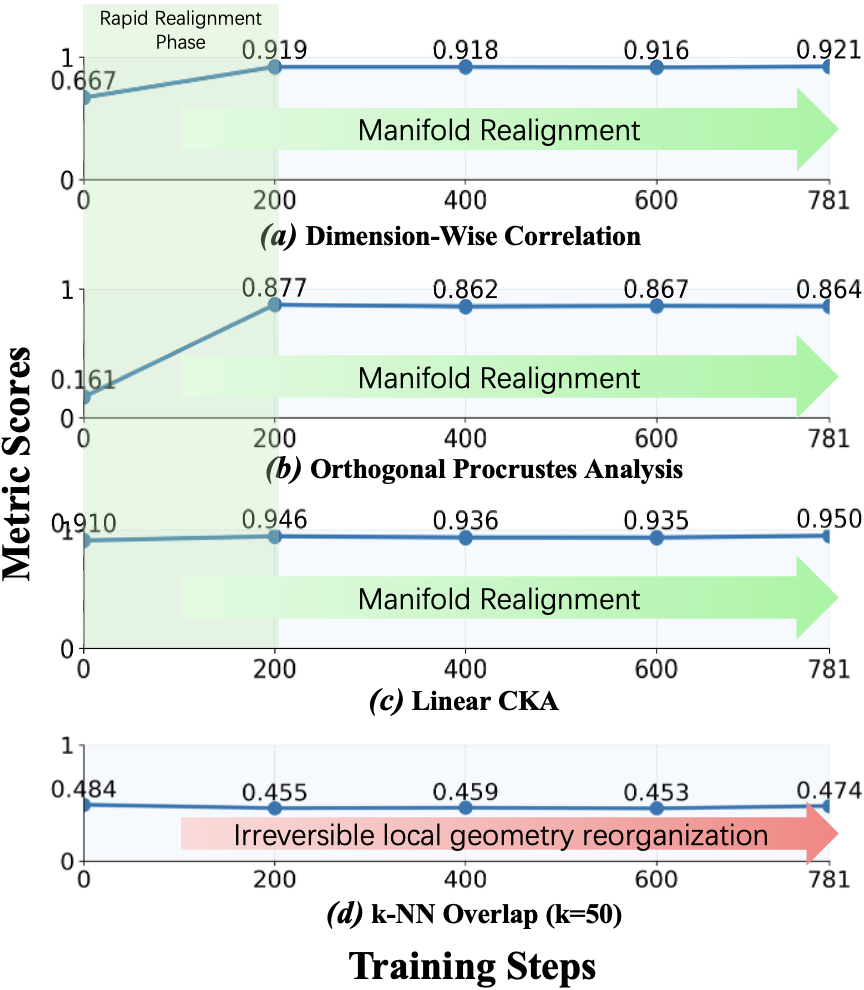}
    \caption{\textbf{The training dynamics of the embedding model pairs DS-$Emb$ vs ProRL-$Emb$.} Step 0 indicates LLM backbones, and step 781 indicates the final checkpoint of the embedding models.}
    \label{fig: training-progress}
    \vspace{-0.3in}
    
\end{wrapfigure}

\paragraph{Cross-Model Linear Probes}
Figure~\ref{fig: cross-model-linear-probe} shows stronger cross-model probe transfer under RLVR than SFT, implying that task-relevant linear readout directions of the latent manifold are more stable. For the embedding model pairs, transfer remains consistently high, reflecting \emph{Manifold Realignment}: contrastive learning maintains strong functional alignment even when local geometry do not fully coincide.

\subsection{Manifold Realignment in Training Dynamics}

Figure~\ref{fig: training-progress} illustrates the dynamics of adapting LLMs to embedding models over training steps. By applying HRSA to intermediate checkpoints, we observe that manifold realignment occurs rapidly in the early training stages (Steps 0–200), after which representational similarity stabilizes. This trajectory demonstrates the manifold realignment, which contrastive learning effectively drives strong alignment between base- and reasoning-initialized embedding models. In contrast, the $k$-NN mean overlap across layers decreases during this process, confirming that the RLVR-induced reorganization of local geometry is irreversible.
\section{Related Works}
\label{sec:relate}

\subsection{Reinforcement Learning with Verifiable Rewards (RLVR)}

RLVR optimizes models using deterministic, verifiable rewards rather than heuristic preference signals~\citep{DeepSeek-R1}. Recent analyses suggest RLVR stays close to the pretrained solution (e.g., KL-anchored/on-policy behavior)~\citep{RlRazor} and improves via weight updates that avoid large principal-subspace changes~\citep{ThePathNotTaken}, without introducing fundamentally novel reasoning beyond the base model~\citep{yue2025does}. However, these works do not directly characterize the \emph{representational} changes induced by RLVR; we show (via HRSA) that RLVR largely preserves global manifold structure while reorganizing local geometry and, with prolonged training, exhibiting some coordinate basis drift.

\subsection{Embedding Models}

Many state-of-the-art text embedding models now leverage decoder-only LLM backbones with bidirectional attention and contrastive training to produce strong encoders~\citep{Qwen3-Embedding,Gemini-Embedding}. While reward-driven or RL-based embedding learning has been explored~\citep{EmbeddingRL,SearchR3}, it remains unclear whether RLVR-tuned reasoning models improve embedding geometry or retrieval. Our study directly tests this connection and finds that RLVR-tuned reasoning models do not reliably enhance embedding quality.

\subsection{Representational Similarity Analysis}

Representational similarity analysis (RSA) and related metrics (e.g., CKA) are widely used to compare layer representations across models and tasks~\citep{RSA2008,revisited-similarity,RSA-LLMs,RSA-LLMs2,CKAonLLMs}. Prior work typically reports single-level alignment and does not organize how changes manifest across abstraction levels, nor does it connect RLVR update properties to representation geometry~\citep{RlRazor,Sparsity}. We address this with HRSA, which disentangles coordinate basis, manifold geometry, and readout-direction changes, showing substantial global preservation in RLVR-tuned reasoning models.

\section{Discussion and Conclusion}

In this paper, we introduced HRSA, a hierarchical representation similarity analysis framework for diagnosing how training reshapes the latent manifold, and conducted the first systematic benchmarking of RLVR-optimized vs. its base model as backbones for text embedding models. Applying HRSA to base backbones and their RLVR-tuned backbones, we identified which components of the latent manifold change and characterized a consistent pattern we term manifold realignment. Across settings, RLVR largely preserves global geometry and linear readout, while producing irreversible reorganization of local geometry. Coordinate basis drift emerges primarily under prolonged RLVR, but appears reversible: subsequent contrastive learning corrects this drift and reinstates strong realignment.

These results support the view that RLVR primarily optimizes trajectories through an existing semantic landscape rather than rewriting that landscape itself. As latent-space-centric paradigms such as World Models~\citep{world-model} and JEPA~\citep{LLM-JEPA} gain prominence, our findings point to a practical trade-off: RLVR tends to preserve the base model’s representational backbone (which may help retain broad generalization), yet on its own is unlikely to fundamentally improve the underlying global organization of the latent manifold. Put differently, RLVR seems to “move behavior” mainly by reshaping local geometry (how nearby states relate) while leaving the large-scale coordinate system and linear readout mostly intact.

Our analysis also suggests an actionable hypothesis for training design: if RLVR’s distinctive footprint is local geometry reorganization under global geometry stability, then similar behavior might be achievable via SFT augmented with geometry- and basis-aware regularization. For example, one could explicitly constrain global manifold distances or penalize excessive coordinate basis drift while encouraging controlled local geometry reorganization. Testing whether such constrained SFT can match RLVR’s representational effects offers a concrete direction for follow-up work.

Several open questions remain about the mechanism. In particular, we do not yet fully explain \emph{why} RLVR produces persistent local geometry reorganization while leaving global geometry and linear readout directions relatively stable, nor what training signals govern the onset and reversibility of coordinate basis drift. Progress here may require controlled interventions (e.g., reward shaping, curriculum, or KL/entropy constraints) paired with HRSA to isolate which components of the RLVR objective drive each geometric effect.

Finally, while our experiments focus on text embedding models, the hierarchy of effects uncovered by HRSA reflects a training-agnostic geometric signature rather than a modality-specific artifact. We therefore expect manifold realignment to be a general phenomenon that extends to representation learning in vision and audio, and we position HRSA as a practical diagnostic to verify this claim across modalities and objectives.

\nocite{GritLM}
\nocite{PSR-model}
\nocite{Polaris2025}
\nocite{Chat-LLM}
\nocite{PPO}

\bibliography{ref}
\bibliographystyle{icml2026}

\newpage
\appendix
\section{Embedding Model Training}
\label{appendix: embedding-model-training}

In this section, we reveal all the training details of the embedding models.

\subsection{Training Details}
We optimize the InfoNCE loss~\citep{InfoNCE} defined in Equation~\ref{eq:infonce}. This objective aims to maximize the similarity between the query $q$ and the positive passage $p$, while simultaneously minimizing the similarity between $q$ and the negative passages. Let $B$ denote the set of in-batch passages (which includes $p$ and negatives from other instances), $\mathcal{N}$ be the set of hard negatives, and sim be the cosine similarity. The loss is calculated as:
\begin{equation}
\label{eq:infonce}
\mathcal{L}(q, p, B, \mathcal{N}) = - \log \frac{\exp(\text{sim}(q, p) / \tau)}{\sum_{d \in B \cup \mathcal{N}} \exp(\text{sim}(q, d) / \tau)}
\end{equation}

We select decoder-only LLMs as the embedding model backbone, take the last layer's activation as the final output, and perform mean pooling to obtain a fixed-dimension embedding vector. We also enable bi-directional attention in the backbone by discarding the causal attention mask to capture more semantic details and relationships between tokens.
We use mixed precision with \texttt{bfloat16} and gradient checkpointing to reduce the memory pressure on the hardware. 
We use Flash Attention 2 as the attention backend algorithm. For more details on the settings, the reader can refer to Table~\ref{tab: training-parameters}. We employ the instruction-tuning technique. In particular, we use the instruction template \texttt{Instruction: \{instruction\}\textbackslash nQuery: {query}}, where \texttt{\{instruction\}} and \texttt{\{query\}} are the placeholders for the instruction and query, respectively. All of our training is conducted on 4x Nvidia L20 GPUs, with VRAM 44GB per GPU.

\begin{table}[h]
    \centering
    \small % Consistent font size
    \setlength{\tabcolsep}{12pt} % Comfortable column spacing
    \caption{Training Hyperparameters}
    \label{tab: training-parameters}
    
    \renewcommand{\arraystretch}{1.2} % Comfortable row spacing
    
    \begin{tabular}{l l}
        \toprule
        \textbf{Variables} & \textbf{Values} \\
        \midrule
        Batch Size                       & 2048 \\
        Learning Rate (LR)               & $2\times 10^{-5}$ \\
        LR Warm-up Ratio                 & 0.03 \\
        LR Scheduler                     & Cosine \\
        Weight Decay                     & 0.05 \\
        Optimizer                        & AdamW \\
        Padding Side                     & Right \\
        Number of data                   & 1,603,172 \\
        Number of training steps         & 782 \\
        Number of hard negatives         & 3 \\
        Temperature                      & 0.02 \\
        Pooling                          & Mean \\
        \bottomrule
    \end{tabular}
\end{table}

Although many prior works~\citep{Qwen3-Embedding,Gemini-Embedding,NV-Embed} use LoRA to train the embedding models, in our work, we discard it, since we find that training without LoRA yields better performance, and full parameters can better record the training dynamics. See Table~\ref{tab: lora-comparison}.

\begin{table}[h]
    \centering
    \small % Matches the font size of the previous table
    \setlength{\tabcolsep}{10pt} % Adds breathing room between columns
    \caption{LoRA comparison on performance in MTEB (Multilingual, v2).}

    \begin{tabular}{l c}
        \toprule
        \textbf{Model} & \textbf{Performance} \\
        \midrule
        
        % =======================================================
        % SECTION: WITHOUT LORA
        % =======================================================
        \multicolumn{2}{l}{\textbf{With LoRA}} \\
        DS-Distill-Qwen-1.5B\textit{-Emb} & 42.450 \\
        NV-ProRL\textit{-Emb}             & 42.064 \\

        \midrule

        % =======================================================
        % SECTION: WITH LORA
        % =======================================================
        \multicolumn{2}{l}{\textbf{Without LoRA}} \\
        DS-Distill-Qwen-1.5B\textit{-Emb} & \textbf{46.185} \\
        
        % Optional: Highlight the best result similar to the SFT highlight in the previous table
        NV-ProRL\textit{-Emb}             & \textbf{46.247} \\
        \bottomrule
    \end{tabular}
    \label{tab: lora-comparison}
\end{table}

\subsection{Training Data Statistics}

We consider a wide range of datasets, forming the training dataset by composing 11 separate datasets. We used \textit{Qwen3-Embedding-0.6B}~\citep{Qwen3-Embedding} to mine 3 hard negatives per query, and employ the positive-aware hard negative mining technique introduced in \citet{NV-Retriever}, with 95\% margin to the positive score. 

\begin{table}[h]
  \centering
  \small
  \caption{Training Datasets Details}

  \begin{tabular}{l  c}
    \toprule
    \textbf{Datasets} & \textbf{Number of Samples} \\
    \midrule
     FEVER & 105,893 \\
     NaturalQuestions & 97,912 \\
     NLI  & 277,217 \\
     MSMARCO & 499,184 \\
     Quora & 94,443 \\
     Mr.Tydi & 102,796 \\
     DUReader & 17,493 \\
     TriviaQA & 65,465 \\
     HotpotQA & 167,808 \\
     SQuAD & 84,494 \\
     T2Ranking & 90,467 \\
     \midrule
     Total & 1,603,172 \\
     \bottomrule
  \end{tabular}
  \label{tab: training-datasets}
\end{table}

% ======================== HRSA Proof ========================

\section{HRSA Proof}
\label{appendix: proof}

\begin{table}[t]
\caption{\textbf{HRSA Framework Extensibility.} The HRSA framework is defined by invariance properties, not specific metrics. Researchers can select alternative metrics (right column) for different modalities or theoretical needs, provided they respect the invariance constraints of the target analysis level.}
\label{tab: appendix_hrsa_extensibility}
\centering
\small  % This command now controls the font size
\begin{tabular}{@{}l p{3.5cm} l p{4.5cm}@{}}
\toprule
\textbf{Level} & \textbf{Invariance Constraints} & \textbf{Default Metric} & \textbf{Alternative Valid Metrics} \\ 
\midrule
\multirow{5}{*}{\textbf{Representation}} 
 & \textbf{Non-Invariant to:} & \multirow{2}{*}{\makecell[l]{Dimension-Wise\\Correlation}} & \textbf{Optimal Transport (Wasserstein)} \\
 & Orthogonal Transformation & & {\footnotesize Measures cost to move mass from basis $X$ to $Y$ without rotation.} \\
 \cmidrule{3-4}
 & \textbf{Goal:} Assess alignment & \multirow{2}{*}{\makecell[l]{Orthogonal\\Procrustes ($O^*$)}} & \textbf{Manifold Alignment Loss} \\
 & of specific axes. & & {\footnotesize Direct penalization of feature mismatch.} \\
\midrule
\multirow{5}{*}{\textbf{Geometry}} 
 & \textbf{Invariant to:} & \multirow{2}{*}{Linear CKA} & \textbf{RBF Kernel CKA} \\
 & Orthogonal Transformation & & {\footnotesize Captures non-linear similarity.} \\
 \cmidrule{3-4}
 & \textbf{Non-Invariant to:} & \multirow{3}{*}{\makecell[l]{k-NN Overlap\\(Jaccard)}} & \textbf{Riemannian Metrics} \\
 & Invertible Linear Transforms & & {\footnotesize Geodesic distance comparison.} \\
 & (Scaling/Shear) & & \\
\midrule
\multirow{4}{*}{\textbf{Function}} 
 & \textbf{Invariant to:} & \multirow{2}{*}{\makecell[l]{Linear Probing\\Transfer}} & \textbf{Mutual Information $I(X;Y)$} \\
 & Any transform preserving & & {\footnotesize Information theoretic upper bound.} \\
 \cmidrule{3-4}
 & the decision boundary. & \multirow{2}{*}{\makecell[l]{Zero-Shot\\Accuracy}} & \textbf{Behavioral Consistency} \\
 & & & {\footnotesize Exact match on downstream tasks.} \\
\bottomrule
\end{tabular}
\end{table}

In this section, we provide more details on HRSA, including all the invariance properties of each level analysis and the proof of their invariance properties.

We emphasize again that HRSA is not dependent on the specific metrics selected for this study, such as Dimension-Wise Correlation or Linear CKA. Rather, it is grounded in the hierarchy of invariance properties established earlier. Consequently, any metric that satisfies the invariance requirements of a specific level can be employed to analyze that level's focus. Refer to Table~\ref{tab: appendix_hrsa_extensibility} for a summary of these properties and a catalog of alternative valid metrics.

\subsection{Representation-Level Proof}
\label{appendix: representation-level-proof}

\begin{definition}[Representation-Level Analysis]
  \label{def:rep_level}
  The representation-level analysis examines the explicit coordinate basis of the latent manifold. A metric at this level must demonstrate sensitivity to coordinate basis rotations. Specifically:
  \begin{itemize}
    \item \textbf{Non-invariant to:} Orthogonal transformations (Rotation/Permutation) and General Linear transformations.
  \end{itemize}
\end{definition}

\subsubsection{Dimension-Wise Correlation}
Recall the definition of Dimension-Wise Correlation from equation~\ref{eq:dimwise-corr}.

\begin{proposition}
  \label{prop:dimwise_ortho}
  Dimension-Wise Correlation is non-invariant to orthogonal transformations.
\end{proposition}

\begin{proof}
  Let $Q \in \mathbb{R}^{D \times D}$ be an orthogonal matrix ($Q^\top Q = I$) such that $X' = XQ$. The $j$-th column becomes $x'_{:j} = \sum_{k=1}^D X_{:k} Q_{kj}$.
  The correlation of the $j$-th column becomes:
  \begin{equation}
    \rho_j(XQ, Y) = \frac{(\sum_k X_{:k} Q_{kj})^\top y_{:j}}{\|\sum_k X_{:k} Q_{kj}\|_2 \|y_{:j}\|_2}.
  \end{equation}
  Since $Q$ mixes information from multiple columns $X_{:k}$ into the new column $x'_{:j}$, the correlation with the fixed target $y_{:j}$ changes arbitrarily depending on $Q$. Thus, $\rho_j(XQ, Y) \neq \rho_j(X, Y)$, satisfying the requirement for coordinate basis sensitivity.
\end{proof}

\subsubsection{Orthogonal Procrustes Analysis}
\label{appendix: opa}

Recall the Orthogonal Procrustes solution $O^*$ defined in Equation~\ref{eq:procrustes}. To quantify the extent of coordinate alignment, we introduce the \emph{inverse row entropy}, denoted as $H_{\text{inv}}$. We interpret the squared elements of each row in $O^*$ as a probability distribution. This is mathematically valid because $O^*$ is orthogonal, meaning its rows have unit Euclidean norm (i.e., $\sum_j (O^*_{ij})^2 = 1$).

We compute $H_{\text{inv}}$ by calculating the mean row entropy, normalizing it by the maximum possible entropy ($\log D$), and taking the complement:

\begin{align*}
H &= - \frac{1}{D \log D} \sum_{i=1}^D \sum_{j=1}^D (O^*_{ij})^2 \log (O^*_{ij})^2 \\
H_{\text{inv}} &= 1 - H
\end{align*}
where $O^*_{ij}$ denotes the element of $O^*$ at row $i$ and column $j$, and $D$ represents the dimensionality. The intermediate term $H$ is normalized to the range $[0, 1]$. Consequently, a higher $H_{\text{inv}}$ indicates that the coordinate basis is preserved (i.e., $O^*$ is sparse and approximates a permutation matrix), whereas a lower $H_{\text{inv}}$ indicates that features are "smeared" or rotated across multiple dimensions.

\begin{proposition}
  \label{prop:procrustes_ortho}
  The structure of the optimal mapping in Orthogonal Procrustes Analysis is non-invariant to orthogonal transformations.
\end{proposition}

\begin{proof}
  For any orthogonal matrices $Q,R \in \mathbb{R}^{D\times D}$, if we transform $X,Y$ to $X' = XQ$, $Y' = YR$, then an optimal map for the new problem is
  \begin{equation}
    O^*(X',Y') = Q^\top O^*(X,Y) R.
  \end{equation}
  This is a conjugation of $O^*$ by orthogonal matrices, which in general destroys diagonality or one-hot structure.
\end{proof}

\begin{remark}
  One may claim that Orthogonal Procrustes Analysis should be classified as a geometry-level measurement because the residual is invariant to orthogonal transformation. In our work, we only focus on the structure of $O^*$, specifically by considering its inverse row entropy. As shown in Proposition~\ref{prop:procrustes_ortho}, $O^*$ remains dependent on the chosen coordinate system.
\end{remark}

\subsection{Geometry-Level Proof}
\label{appendix: geometry-level-proof}

\begin{definition}[Geometry-Level Analysis]
  \label{def:geo_level}
  The geometry-level analysis examines the intrinsic shape and topology of the latent manifold $\mathcal{Z}$. Metrics at this level must quantify the arrangement of points relative to one another, independent of the specific coordinate system used to describe them.
  \begin{itemize}
    \item \textbf{Invariant to:} Similarity transformations, defined as the composition of orthogonal rotation/reflection ($Q \in \mathbb{R}^{D \times D}, Q^\top Q = I$) and isotropic scaling ($c \in \mathbb{R}, c > 0$).
    \item \textbf{Non-invariant to:} Anisotropic linear transformations (e.g., non-uniform scaling, shearing) where the transformation matrix $A$ satisfies $A^\top A \neq cI$.
  \end{itemize}
\end{definition}

In the following, we provide proofs for the invariance properties of Linear CKA and Cosine $k$-NN Overlap.

\subsubsection{Linear CKA}
Recall that Linear CKA is defined via the Hilbert--Schmidt Independence Criterion (HSIC) of centered Gram matrices. Let $K_X = XX^\top$ and $H = I - \frac{1}{N}\mathbf{1}\mathbf{1}^\top$.

\begin{proposition}
\label{prop:cka_sim_hsic}
Linear CKA is invariant to similarity transformations $X \mapsto cXQ$ where $c > 0$ and $Q$ is orthogonal.
\end{proposition}

\begin{proof}
Let $X' = cXQ$. We first derive the Gram matrix for the transformed representation:
\begin{equation}
K_{X'} = (cXQ)(cXQ)^\top = c^2 X Q Q^\top X^\top.
\end{equation}
Since $Q$ is orthogonal ($Q Q^\top = I$), this simplifies to:
\begin{equation}
K_{X'} = c^2 X X^\top = c^2 K_X.
\end{equation}

Now we examine the HSIC term in the numerator. Using the property $\mathrm{tr}(cA) = c\,\mathrm{tr}(A)$:
\begin{equation}
\begin{aligned}
\mathrm{HSIC}(K_{X'}, K_Y) &= \frac{1}{(N-1)^2} \mathrm{tr}(K_{X'} H K_Y H) \\
                            &= \frac{1}{(N-1)^2} \mathrm{tr}(c^2 K_X H K_Y H) \\
                            &= c^2 \cdot \mathrm{HSIC}(K_X, K_Y).
\end{aligned}
\end{equation}

Similarly, for the normalization term in the denominator:
\begin{equation}
\begin{aligned}
\mathrm{HSIC}(K_{X'}, K_{X'}) &= \frac{1}{(N-1)^2} \mathrm{tr}(c^2 K_X H c^2 K_X H) \\
&= c^4 \cdot \mathrm{HSIC}(K_X, K_X).
\end{aligned}
\end{equation}

Substituting these into the full Linear CKA equation:
\begin{equation}
\begin{aligned}
\mathrm{CKA}(X', Y) &= \frac{c^2 \mathrm{HSIC}(K_X, K_Y)}{\sqrt{c^4 \mathrm{HSIC}(K_X, K_X) \cdot \mathrm{HSIC}(K_Y, K_Y)}} \\
&= \frac{c^2 \mathrm{HSIC}(K_X, K_Y)}{c^2 \sqrt{\mathrm{HSIC}(K_X, K_X) \cdot \mathrm{HSIC}(K_Y, K_Y)}} \\
&= \mathrm{CKA}(X, Y).
\end{aligned}
\end{equation}
The scalar factors cancel perfectly, proving invariance.
\end{proof}

\begin{proposition}
\label{prop:cka_aniso}
Linear CKA is generally non-invariant to anisotropic linear transformations.
\end{proposition}

\begin{proof}
Let $X' = XA$, where $A \in \mathbb{R}^{D \times D}$ is invertible and anisotropic ($AA^\top \neq cI$). The Gram matrix becomes:
\begin{equation}
K_{X'} = X A A^\top X^\top.
\end{equation}
Let $M = A A^\top$. The numerator HSIC term becomes proportional to $\mathrm{tr}(X M X^\top H K_Y H)$. Unlike the isotropic case, the matrix $M$ is "trapped" between $X$ and $X^\top$ inside the trace. Unless $M$ is a scalar multiple of the identity, it reweights the singular values of $X$, effectively altering the principal components of the representation space. Since Linear CKA measures the alignment of these principal components, $\mathrm{CKA}(XA, Y) \neq \mathrm{CKA}(X, Y)$.
\end{proof}

\subsubsection{$k$-NN Overlap}
As defined in the Section~\ref{sec:knn-overlap}, $k$-NN overlap relies on the ranking of cosine similarities $s(u, v) = \frac{u^\top v}{\|u\|\|v\|}$.

\begin{proposition}
\label{prop:knn_sim}
Cosine-based $k$-NN Overlap is invariant to similarity transformations.
\end{proposition}

\begin{proof}
Let $x$ and $y$ be any two embedding vectors (rows of $X$). We apply the transformation $x' = cQx$ and $y' = cQy$, with $c > 0$ and $Q^\top Q = I$. The cosine similarity between the transformed vectors is:
\begin{equation}
\begin{aligned}
s(x', y') &= \frac{(cQx)^\top (cQy)}{\|cQx\| \|cQy\|} \\
&= \frac{c^2 x^\top Q^\top Q y}{\sqrt{(cQx)^\top (cQx)} \sqrt{(cQy)^\top (cQy)}}.
\end{aligned}
\end{equation}
Using the orthogonality property $Q^\top Q = I$:
\begin{equation}
\begin{aligned}
s(x', y') &= \frac{c^2 x^\top y}{\sqrt{c^2 x^\top x} \sqrt{c^2 y^\top y}} \\
&= \frac{c^2 (x^\top y)}{c \|x\| \cdot c \|y\|} \\
&= \frac{x^\top y}{\|x\| \|y\|} = s(x, y).
\end{aligned}
\end{equation}
Since the pairwise similarity scores remain exactly the same, the ranking of neighbors is preserved. Thus, the set of top-$k$ nearest neighbors is identical: $N_k^{X'}(i) = N_k^X(i)$, and the overlap score is invariant.
\end{proof}

\begin{proposition}
\label{prop:knn_aniso}
Cosine-based $k$-NN Overlap is generally non-invariant to anisotropic linear transformations.
\end{proposition}

\begin{proof}
Let $x' = Ax$ and $y' = Ay$ with anisotropic $A$. The transformed similarity is:
\begin{equation}
s(x', y') = \frac{x^\top A^\top A y}{\sqrt{x^\top A^\top A x} \sqrt{y^\top A^\top A y}}.
\end{equation}
Let $M = A^\top A$. This expression represents the cosine of the angle between $x$ and $y$ in a space equipped with the inner product $\langle u, v \rangle_M = u^\top M v$. 

Because $A$ is anisotropic, $M$ has distinct eigenvalues. This transformation distorts angles: vectors aligned with the large eigenvectors of $M$ are "pulled" closer together in angular space, while vectors aligned with small eigenvectors are pushed apart.

Consequently, if we have $s(x, y) > s(x, z)$ (meaning $y$ is a closer neighbor to $x$ than $z$), an anisotropic $A$ can reverse this relationship such that $s(x', z') > s(x', y')$. This alters the composition of the $k$-nearest neighbor sets, changing the overlap score.
\end{proof}

\subsection{Function-Level Proof}
\label{appendix: function-level-proof}

\begin{definition}[Function-Level Analysis]
  \label{def:func_level}
  The function-level analysis examines the usable information accessible via linear readouts (probes) or the final behavioral output. This level specifically tests whether two models share the same ``readout directions'' for solving a task.
  \begin{itemize}
    \item \textbf{Invariant to:} Isomorphic transformations \textit{if and only if} the readout mechanism is transformed correspondingly.
    \item \textbf{Non-invariant to:} Linear Reparameterization under a \textit{fixed} readout hypothesis.
  \end{itemize}
\end{definition}

\subsubsection{Cross-Model Linear Probes}
Let $w_X^*$ be the optimal probe weights for task $Z$ on representations $X$, i.e., $w_X^* = \text{argmin}_w \|Xw - Z\|^2$. We evaluate these weights on $Y$: $\text{Error} = \|Y w_X^* - Z\|^2$.

\begin{proposition}
  \label{prop:cross_model_probe}
  Cross-Model Linear Probes are non-invariant to linear reparameterization under a fixed readout.
\end{proposition}

\begin{proof}
  Assume $Y$ contains the exact same information as $X$ but is linearly transformed: $Y = XA$ (where $A$ is invertible).
  The prediction using transferred weights is:
  \begin{equation}
    \hat{Z}_{Y} = Y w_X^* = (X A) w_X^*.
  \end{equation}
  The original prediction was $\hat{Z}_X = X w_X^*$.
  For the predictions to be identical ($\hat{Z}_Y = \hat{Z}_X$) for all $X$, we require $X A w_X^* = X w_X^*$, implying $A w_X^* = w_X^*$.
  This equality only holds if $w_X^*$ is an eigenvector of $A$ with eigenvalue 1. For a general transformation $A$, $A w_X^* \neq w_X^*$.
  Therefore, even if $Y$ is geometrically isomorphic to $X$, the cross-model probe will fail if the \textit{direction} of the solution has shifted. This proves the metric satisfies the requirement set in Definition~\ref{def:func_level}.
\end{proof}

% ================================================================================================ 
% LLM Pairs
% ================================================================================================ 

\begin{table*}[t]
    \centering
    \caption{LLM pairs used in additional HRSA analyses, separated by the training algorithms (SFT, RLVR).}

    \small % Slightly reduce font size to fit width comfortably
    \setlength{\tabcolsep}{7pt} % Adjust column spacing
    \begin{tabular}{l l l l}
        \toprule
        \textbf{Base Model $\mathcal{M}_{base}$} & \textbf{Reasoning Model $\mathcal{M}_{reason}$} & \textbf{Algorithm} & \textbf{Data} \\
        \midrule

        \multicolumn{4}{l}{\textbf{SFT}} \\
        \href{https://huggingface.co/Qwen/Qwen2.5-Math-1.5B}{Qwen2.5-Math-1.5B} &
        \href{https://huggingface.co/deepseek-ai/DeepSeek-R1-Distill-Qwen-1.5B}{DeepSeek-R1-Distill-Qwen-1.5B} &
        SFT & Mixed \\

        \midrule
        \multicolumn{4}{l}{\textbf{RLVR}} \\

        \href{https://huggingface.co/Qwen/Qwen3-4B}{Qwen3-4B} &
        \href{https://huggingface.co/POLARIS-Project/Polaris-4B-Preview}{Polaris-4B-Preview} &
        DAPO & Math \\

        \href{https://huggingface.co/deepseek-ai/DeepSeek-R1-Distill-Qwen-7B}{DeepSeek-R1-Distill-Qwen-7B} &
        \href{https://huggingface.co/POLARIS-Project/Polaris-7B-Preview}{Polaris-7B-Preview} &
        DAPO & Math \\

        \href{https://huggingface.co/Qwen/Qwen2.5-7B}{Qwen2.5-7B} &
        \href{https://huggingface.co/princeton-nlp/zero__ppo__think__Qwen2.5-7B}{zero\_\_ppo\_\_think\_\_Qwen2.5-7B} &
        PPO & Chat \\

        \href{https://huggingface.co/Qwen/Qwen2.5-1.5B}{Qwen2.5-1.5B} &
        \href{https://huggingface.co/hkust-nlp/Qwen-2.5-1.5B-SimpleRL-Zoo}{Qwen-2.5-1.5B-SimpleRL-Zoo} &
        GRPO & Math \\

        \href{https://huggingface.co/Qwen/Qwen2.5-0.5B}{Qwen2.5-0.5B} &
        \href{https://huggingface.co/hkust-nlp/Qwen-2.5-0.5B-SimpleRL-Zoo}{Qwen-2.5-0.5B-SimpleRL-Zoo} &
        GRPO & Math \\

        \href{https://huggingface.co/deepseek-ai/DeepSeek-R1-Distill-Qwen-1.5B}{DeepSeek-R1-Distill-Qwen-1.5B} &
        \href{https://huggingface.co/nvidia/Nemotron-Research-Reasoning-Qwen-1.5B}{Nemotron-Research-Reasoning-Qwen-1.5B} &
        GRPO & Math \\

        \href{https://huggingface.co/Qwen/Qwen3-4B}{Qwen3-4B} &
        \href{https://huggingface.co/TianHongZXY/Qwen3-4B-PSR}{Qwen3-4B-PSR} &
        PSR & Math \\
        \bottomrule
    \end{tabular}
    \label{tab: appendix-model-details}
\end{table*}

% ================================================================================================ 
% Embedding Model Pairs
% ================================================================================================ 

\begin{table*}[t]
    \centering
    \caption{Embedding model pairs used in additional HRSA analyses. All of the embedding models are trained on the same dataset with InfoNCE loss. They are separated by the training algorithms used to train their reasoning model backbone.}

    \small % Slightly reduce font size to fit width comfortably
    \setlength{\tabcolsep}{7pt} % Adjust column spacing
    \begin{tabular}{l l}
        \toprule
        \textbf{Base Embedding Model} $\mathcal{M}_{base}^\textit{Emb}$ & \textbf{Reasoning Embedding Model} $\mathcal{M}_{reason}^\textit{Emb}$ \\
        \midrule

        \multicolumn{2}{l}{\textbf{SFT}} \\
        \href{https://huggingface.co/lucaswychan/Qwen2.5-Math-1.5B-Reasoning-Embedding}{Qwen2.5-Math-1.5B-\textit{Emb}} & \href{https://huggingface.co/lucaswychan/DeepSeek-R1-Distill-Qwen-1.5B-Reasoning-Embedding}{DeepSeek-R1-Distill-Qwen-1.5B-\textit{Emb}} \\
        \href{https://huggingface.co/lucaswychan/Qwen3-0.6B-Base-Reasoning-Embedding}{Qwen3-0.6B-Base-\textit{Emb}} &
        \href{https://huggingface.co/lucaswychan/Qwen3-0.6B-Reasoning-Embedding}{Qwen3-0.6B-\textit{Emb}} \\

        \midrule
        \multicolumn{2}{l}{\textbf{RLVR}} \\
        \href{https://huggingface.co/lucaswychan/Qwen2.5-1.5B-Reasoning-Embedding}{Qwen2.5-1.5B-\textit{Emb}} &
        \href{https://huggingface.co/lucaswychan/Qwen-2.5-1.5B-SimpleRL-Zoo-Reasoning-Embedding}{Qwen-2.5-1.5B-SimpleRL-Zoo-\textit{Emb}} \\

        \href{https://huggingface.co/lucaswychan/Qwen2.5-0.5B-Reasoning-Embedding}{Qwen2.5-0.5B-\textit{Emb}} &
        \href{https://huggingface.co/lucaswychan/Qwen-2.5-0.5B-SimpleRL-Zoo-Reasoning-Embedding}{Qwen-2.5-0.5B-SimpleRL-Zoo-\textit{Emb}} \\

        \href{https://huggingface.co/lucaswychan/DeepSeek-R1-Distill-Qwen-1.5B-Reasoning-Embedding}{DeepSeek-R1-Distill-Qwen-1.5B-\textit{Emb}} &
        \href{https://huggingface.co/lucaswychan/Nemotron-Research-Reasoning-Qwen-1.5B-Reasoning-Embedding}{Nemotron-Research-Reasoning-Qwen-1.5B-\textit{Emb}} \\

        \href{https://huggingface.co/lucaswychan/Qwen3-4B-Reasoning-Embedding}{Qwen3-4B-\textit{Emb}} &
        \href{https://huggingface.co/lucaswychan/Qwen3-4B-PSR-Reasoning-Embedding}{Qwen3-4B-PSR-\textit{Emb}} \\
        \bottomrule
    \end{tabular}
    \label{tab: appendix-embedding-model-details}
\end{table*}

% ======================== CoT Datasets ========================
\section{CoT Datasets}
\label{appendix: cot-datasets}

To rigorously verify if the latent manifold is preserved within reasoning trajectories (as discussed in Section~\ref{sec: setups}), we constructed a specialized \textbf{CoT-Activations dataset}. Unlike standard semantic datasets, this corpus focuses on long-range, multi-step reasoning traces generated by state-of-the-art reasoning models.

\subsection{Dataset Composition and Hierarchy}

We curated a diverse suite of mathematical reasoning benchmarks to ensure our analysis covers varying degrees of reasoning complexity, ranging from elementary arithmetic to competition-level problem solving. The dataset is stratified into three difficulty levels:

\begin{itemize}
    \item \textbf{Easy:} Sourced from \textbf{GSM8K}~\citep{GSM8K}, focusing on grade-school math word problems that require multi-step arithmetic but limited abstract reasoning.
    \item \textbf{Moderate:} Sourced from \textbf{MATH-500}~\citep{MATH} (a curated subset of the MATH benchmark including AMC/AIME problems) and \textbf{NuminaMath} (CN K-12 curriculum). These datasets introduce higher-dimensional algebraic and geometric reasoning.
    \item \textbf{Hard:} Sourced from \textbf{LiveMathBench}~\citep{LiveMathBench} (2025 Hard Subset). These are recent competition-level problems requiring extremely long context windows and complex logical deductions.
\end{itemize}

Table~\ref{tab:cot_stats} summarizes the statistics of the generated CoT dataset.

\subsection{Generation Protocol}

To extract high-quality reasoning traces, we utilized \textit{Qwen3-32B}~\citep{qwen3-series} as the generator backbone. The generation process was designed to maximize the explicitness of the internal reasoning process (the ``chain of thought'').

\paragraph{Inference Configuration.} We enabled the internal ``thinking'' mode (\texttt{enable\_thinking: true}) to expose the raw reasoning tokens before the final answer. The generation parameters were set to temperature $T=0.6$ and nucleus sampling probability $p=0.95$ to balance creativity with logical coherence.

\paragraph{Token Limits.} To accommodate deep reasoning, we set a high context limit. For standard datasets, we allowed up to 8,000 reasoning tokens. For the \textbf{LiveMathBench} subset, we removed the CoT token limit entirely to allow for exhaustive search trajectories in hard problems.

\paragraph{Prompting.} We employed a standardized two-message chat format to enforce rigorous step-by-step reasoning. The \textbf{System Prompt} was defined as:
\begin{lstlisting}[caption={System Prompt of the CoT dataset generation}, label={appendix: system-message-generation}]
You are a helpful and rigorous math reasoning assistant.
\end{lstlisting}
The \textbf{User Prompt} wrapped the specific dataset problem with instructions to act as a competition solver:
\begin{lstlisting}[caption={User Prompt of the CoT dataset generation}, label={appendix: user-message-generation}]
You are an expert competition math solver. Read the problem carefully and solve it step by 
step.

Problem:
{Problem}
\end{lstlisting}

\subsection{Quality Control and Evaluation}

To ensure that our latent manifold analysis is based on valid reasoning trajectories rather than hallucinations, we implemented a strict verification pipeline using an LLM-as-a-Judge approach.

We employed \textbf{DeepSeek-V3.2-exp}~\citep{deepseekai2024deepseekv32} as the external evaluator. To ensure deterministic and strictly formatted outputs, we used greedy decoding ($T=0$) and explicitly disabled the model's internal chain-of-thought feature (\texttt{enable\_thinking: False}).

The interaction was structured as follows:

\begin{lstlisting}[caption={System Prompt of the external evaluator}, label={appendix: system-message-external-evaluator}]
You are a precise math answer evaluator. Respond only with 0 or 1.
\end{lstlisting}

\begin{lstlisting}[caption={User Prompt of the external evaluator}, label={appendix: user-message-external-evaluator}]
You are an expert math problem evaluator. Your task is to determine if the provided answer 
correctly solves the given problem.
        
Problem: {Problem}
Answer: {Answer}

Evaluate whether the answer is correct. Respond with ONLY "1" if the answer is correct, or 
"0" if it is incorrect. Do not provide any explanation.
\end{lstlisting}

The evaluator's output was parsed using a simple inclusion check: if the token ``1'' appeared in the response, the reasoning trace was marked as valid (\texttt{correctness\_label} $= 1$); otherwise, it was discarded.

% ==========================================
% TABLE MOVED HERE TO FLOAT TO TOP OF PAGE
% ==========================================
\begin{table}[H]
    \centering
    \caption{
        \textbf{Statistics of the Generated CoT Reasoning Dataset.} 
        We report the yield of our generation pipeline across difficulty tiers. 
        \textbf{Total}: Number of initial prompts; \textbf{Valid}: Traces that passed the correctness verification (Correctness $=1$). 
        \textbf{Acc.}: The effective yield rate (Valid / Total).
    }

    \small 
    \setlength{\tabcolsep}{8pt} % Adjusted spacing for a balanced look
    \begin{tabular}{l ccc}
        \toprule
        & \multicolumn{3}{c}{\textbf{Generation Statistics}} \\
        \cmidrule{2-4}
        \textbf{Dataset (Difficulty)} & \textbf{Total} & \textbf{Valid} & \textbf{Acc. (\%)} \\
        \midrule
        
        GSM8K \textit{(Easy)} & 
        500 & 471 & 92.80 \\ 
        NuminaMath \textit{(Moderate)} & 
        161 & 154 & 91.30 \\
        
        MATH-500 \textit{(Moderate)} & 
        479 & 364 & 75.78 \\
        
        LiveMathBench \textit{(Hard)} & 
        57  & 32  & 56.14 \\
        
        \midrule
        
        \textbf{Total / Average} & 
        \textbf{1,197} & \textbf{1,021} & \textbf{85.30} \\
        
        \bottomrule
    \end{tabular}
    \label{tab:cot_stats}
\end{table}

\section{Additional Results}
\label{appendix: additional-results}

In this section, we demonstrate additional results of HRSA applying on more model pairs, including $\mathcal{M}_{base}$ vs. $\mathcal{M}_{reason}$ and $\mathcal{M}_{base}^\textit{Emb}$ vs. $\mathcal{M}_{reason}^\textit{Emb}$ (suffix $-Emb$). See Table~\ref{tab: appendix-model-details} and Table~\ref{tab: appendix-embedding-model-details} for the detailed model pairs.

Instead of only considering the CoT dataset, we also apply HRSA with the MMLU-Pro~\citep{MMLU-Pro} dataset to study the difference (if any) between the models in a general field rather than only the maths domain.

% Helper for rotated left label
% NOTE: Requires \usepackage{graphicx} and \usepackage{array}
\newcommand{\rotlabel}[1]{%
    \raisebox{-0.5\height}{\rotatebox{90}{\parbox{3.5cm}{\centering\scriptsize \textbf{#1}}}}%
}

% Helper for COLORED Data Section Title
% NOTE: Requires \usepackage[table]{xcolor}
\newcommand{\sectionbox}[1]{%
    \noindent\colorbox{blue!20}{%
        \parbox{\dimexpr\linewidth-2\fboxsep\relax}{%
            \centering \vspace{0.1em}
            {\large \textbf{\textit{#1}}}
            \vspace{0.1em}
        }%
    }\par
}

\clearpage
% ================================================================================================ 
% Dimension-Wise Correlation 
% Base Model vs. Reasoning Model
%================================================================================================ 

\begin{figure*}[p]
    \centering
    {\large \textbf{Dimension-Wise Correlation}} \par\smallskip
    {\large \textbf{Base Models $\mathcal{M}_{base}$ vs. Reasoning Models $\mathcal{M}_{reason}$}} \par\medskip

    % --- Configuration ---
    \setlength{\tabcolsep}{1pt}

    % % Helper for rotated left label
    % % NOTE: Requires \usepackage{graphicx} and \usepackage{array}
    % \newcommand{\rotlabel}[1]{%
    %     \raisebox{-0.5\height}{\rotatebox{90}{\parbox{3.5cm}{\centering\scriptsize \textbf{#1}}}}%
    % }

    % % Helper for COLORED Data Section Title
    % % NOTE: Requires \usepackage[table]{xcolor}
    % \newcommand{\sectionbox}[1]{%
    %     \noindent\colorbox{blue!20}{%
    %         \parbox{\dimexpr\linewidth-2\fboxsep\relax}{%
    %             \centering \vspace{0.1em}
    %             {\large \textbf{\textit{#1}}}
    %             \vspace{0.1em}
    %         }%
    %     }\par
    % }

    \sectionbox{Dataset: CoT Datset}
    \vspace{1ex} % Adds a small vertical space for better separation

    % Center the main content grid and ensure it does not exceed the text width
    \makebox[\textwidth][c]{%
    \begin{tabular}{
        c @{\hspace{1pt}} c  @{\hspace{0.5em}}
        c @{\hspace{1pt}} c  @{\hspace{0.5em}}
        c @{\hspace{1pt}} c  @{\hspace{0.5em}}
        c @{\hspace{1pt}} c
    }
        % --- Row 1 ---
        % === MODIFIED CELL START ===
        % We wrap the minipage in a colorbox. 
        % \fboxsep controls the padding between the color edge and the image.
        \rotlabel{Qwen2.5-Math-1.5B} &
        \setlength{\fboxsep}{3pt}% 
        \colorbox{red!20}{%  <-- CHANGE COLOR HERE (e.g., yellow!20, blue!10)
            \begin{minipage}{0.20\textwidth}
                \centering
                \includegraphics[width=\linewidth]{figures/re_experiments/embedding_correlation_evaluation/Qwen2.5-Math-1.5B_vs_DeepSeek-R1-Distill-Qwen-1.5B/qwen3_32b_LiveMathbench_evaluated/plots/corr_heatmap_Qwen2.5-Math-1.5B_vs_DeepSeek-R1-Distill-Qwen-1.5B_processed.png}
                \par\vspace{2pt} % Small space between image and caption
                \scriptsize \textbf{DeepSeek-R1-Distill-Qwen-1.5B}
            \end{minipage}%
        } &

        \rotlabel{Qwen3-4B} &
        \begin{minipage}{0.20\textwidth} % Adjusted width
            \includegraphics[width=\linewidth]{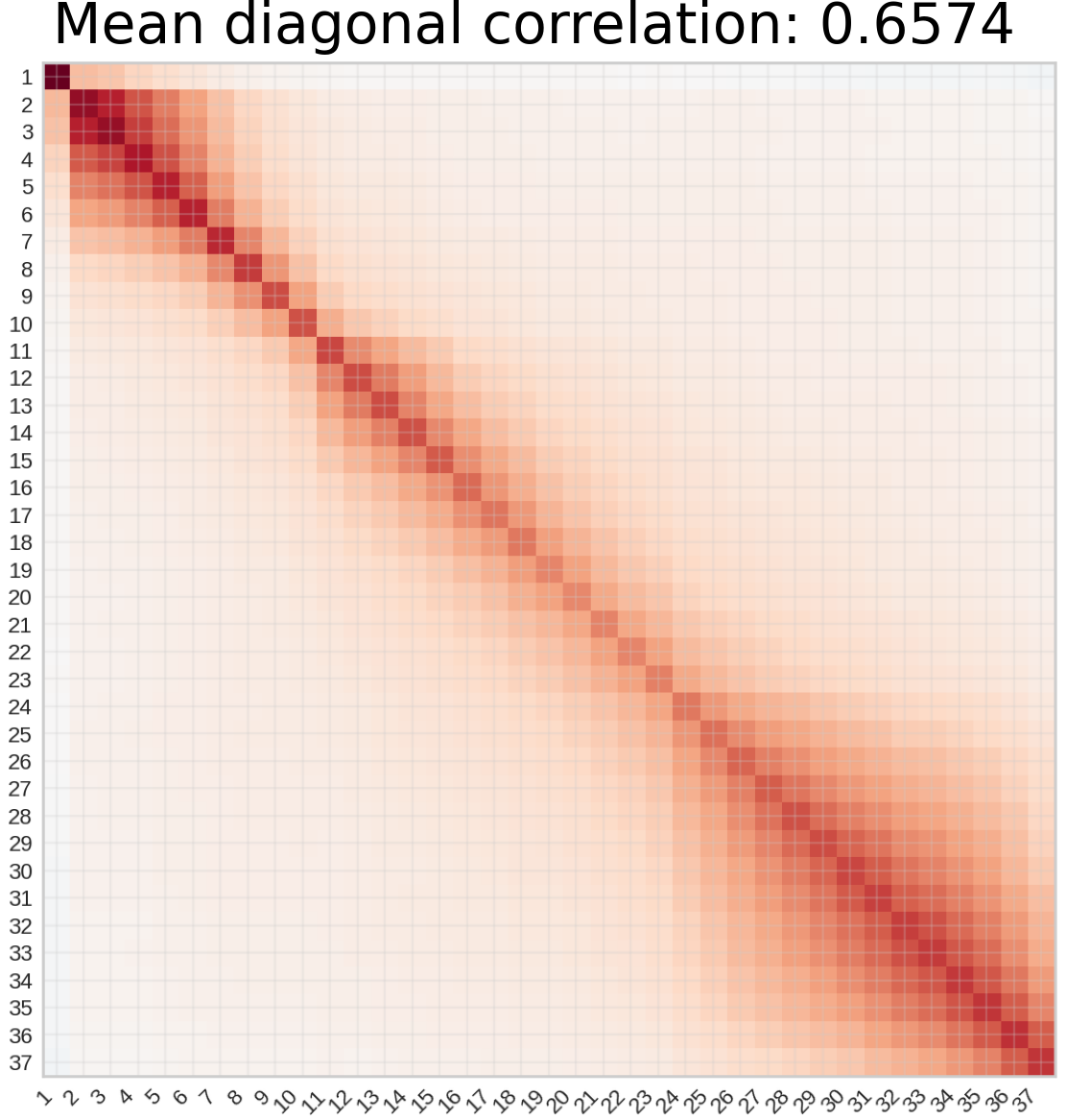} \\
            \centering \scriptsize \textbf{Polaris-4B-Preview}
        \end{minipage} &

        \rotlabel{DeepSeek-R1-Distill-Qwen-7B} &
        \begin{minipage}{0.20\textwidth} % Adjusted width
            \includegraphics[width=\linewidth]{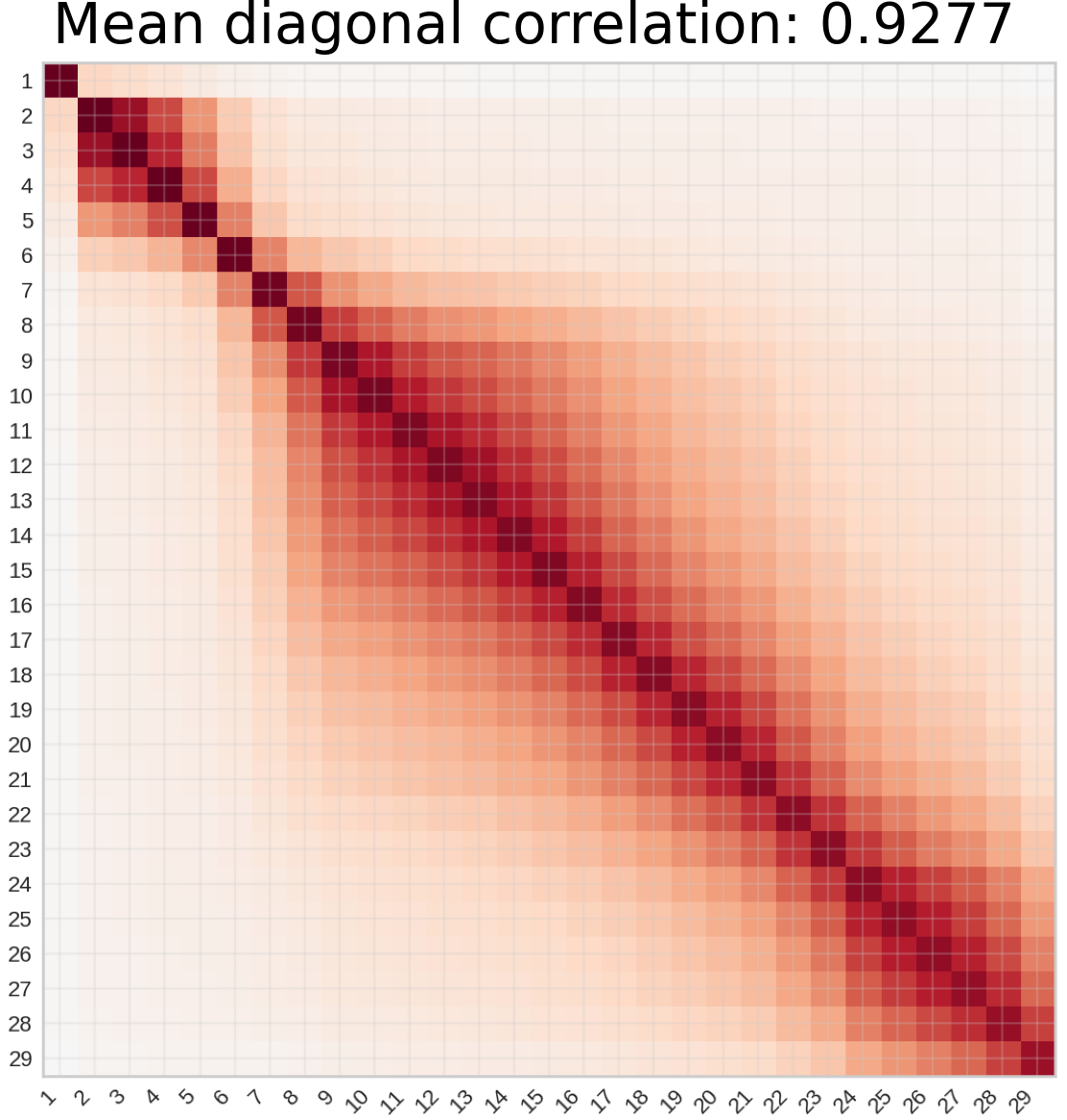} \\
            \centering \scriptsize \textbf{Polaris-7B-Preview}
        \end{minipage} &

        \rotlabel{Qwen2.5-7B} &
        \begin{minipage}{0.20\textwidth} % Adjusted width
            \includegraphics[width=\linewidth]{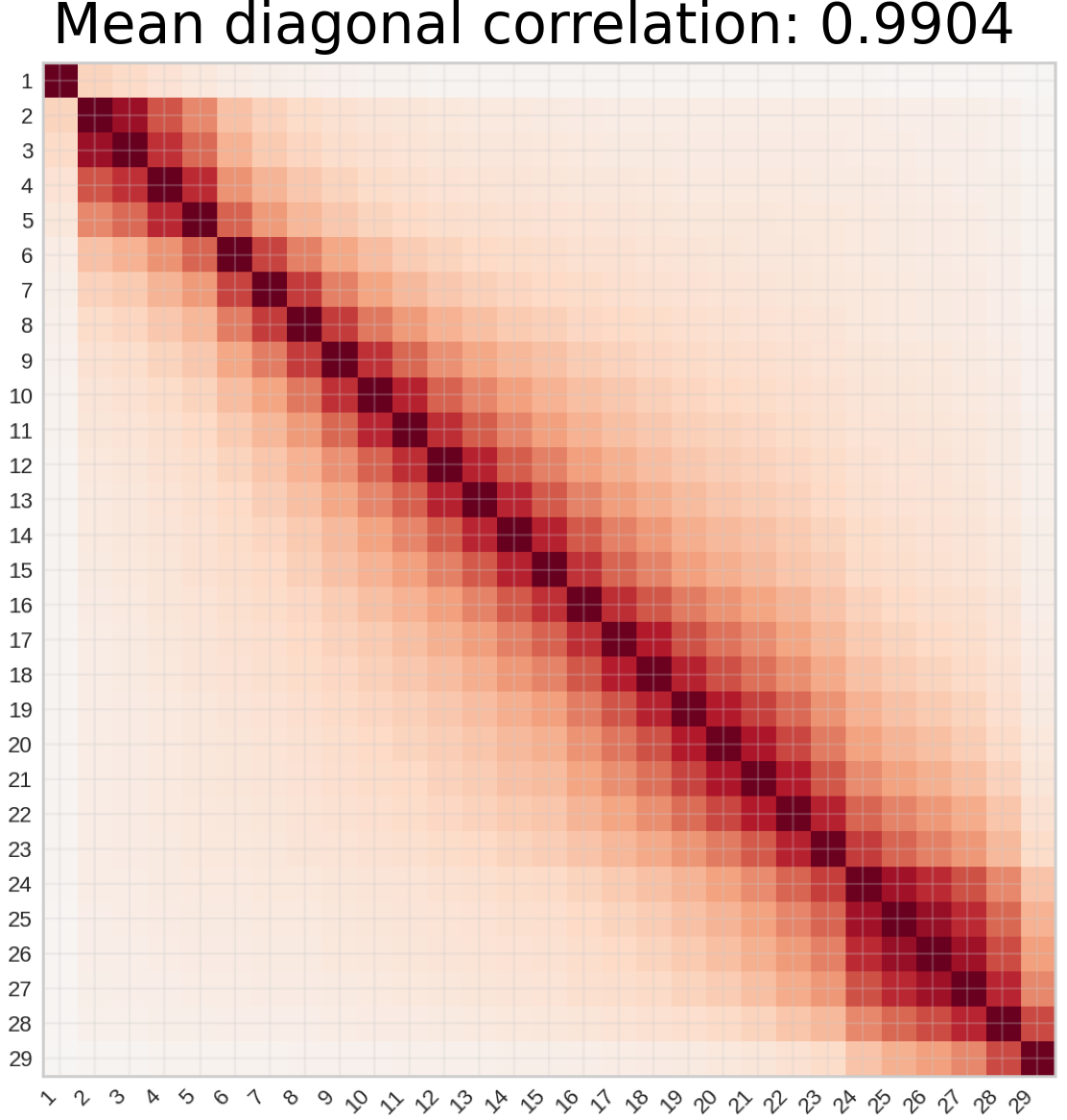} \\
            \centering \scriptsize \textbf{zero\_\_ppo\_\_think\_\_Qwen2.5-7B}
        \end{minipage} \\

        % \multicolumn{8}{c}{\vspace{0.1em}} \\ % Spacing between rows

        % --- Row 2 ---
        \rotlabel{Qwen2.5-1.5B} &
        \begin{minipage}{0.20\textwidth} % Adjusted width
            \includegraphics[width=\linewidth]{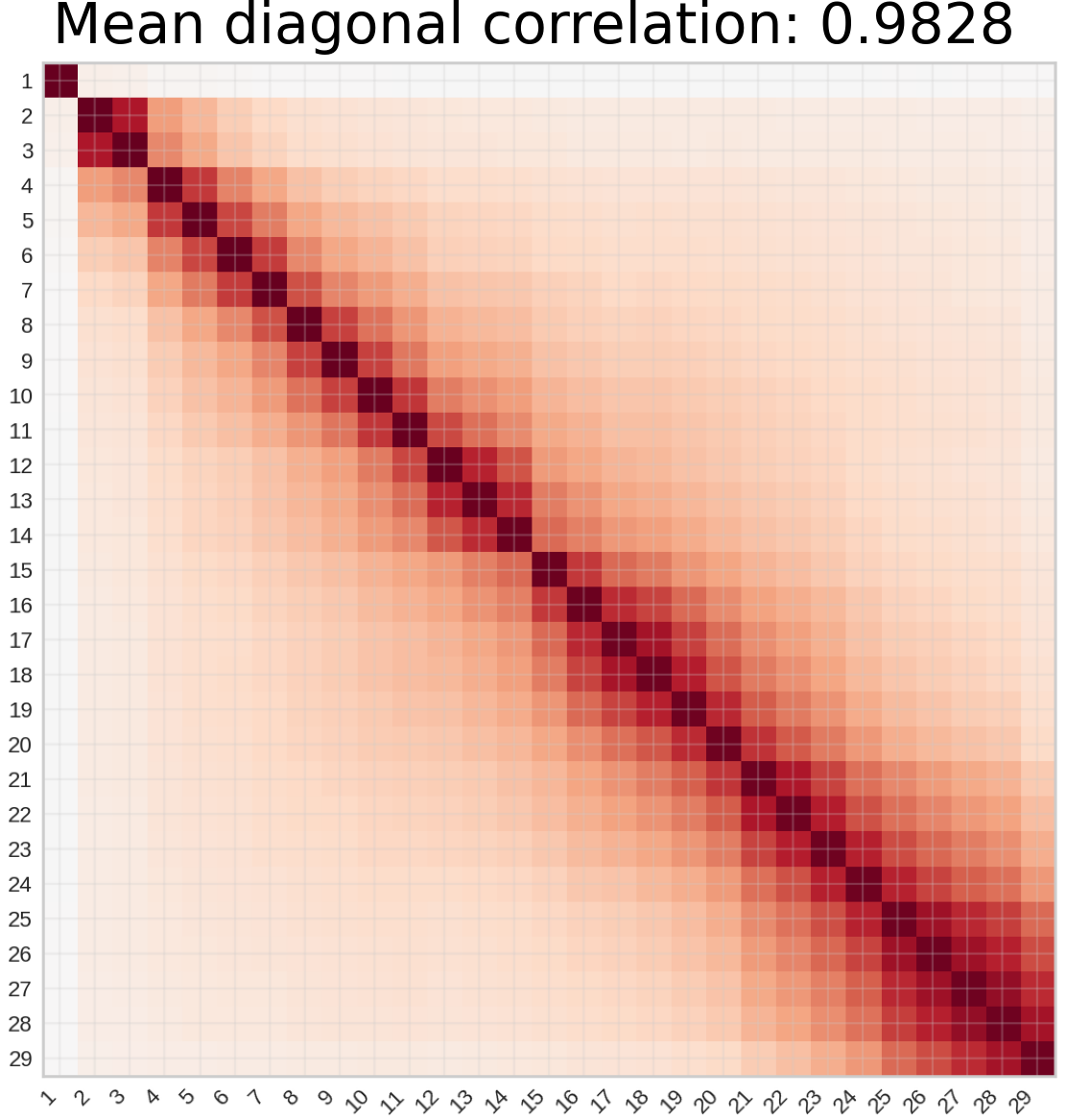} \\
            \centering \scriptsize \textbf{Qwen-2.5-1.5B-SimpleRL-Zoo}
        \end{minipage} &

        \rotlabel{Qwen2.5-0.5B} &
        \begin{minipage}{0.20\textwidth} % Adjusted width
            \includegraphics[width=\linewidth]{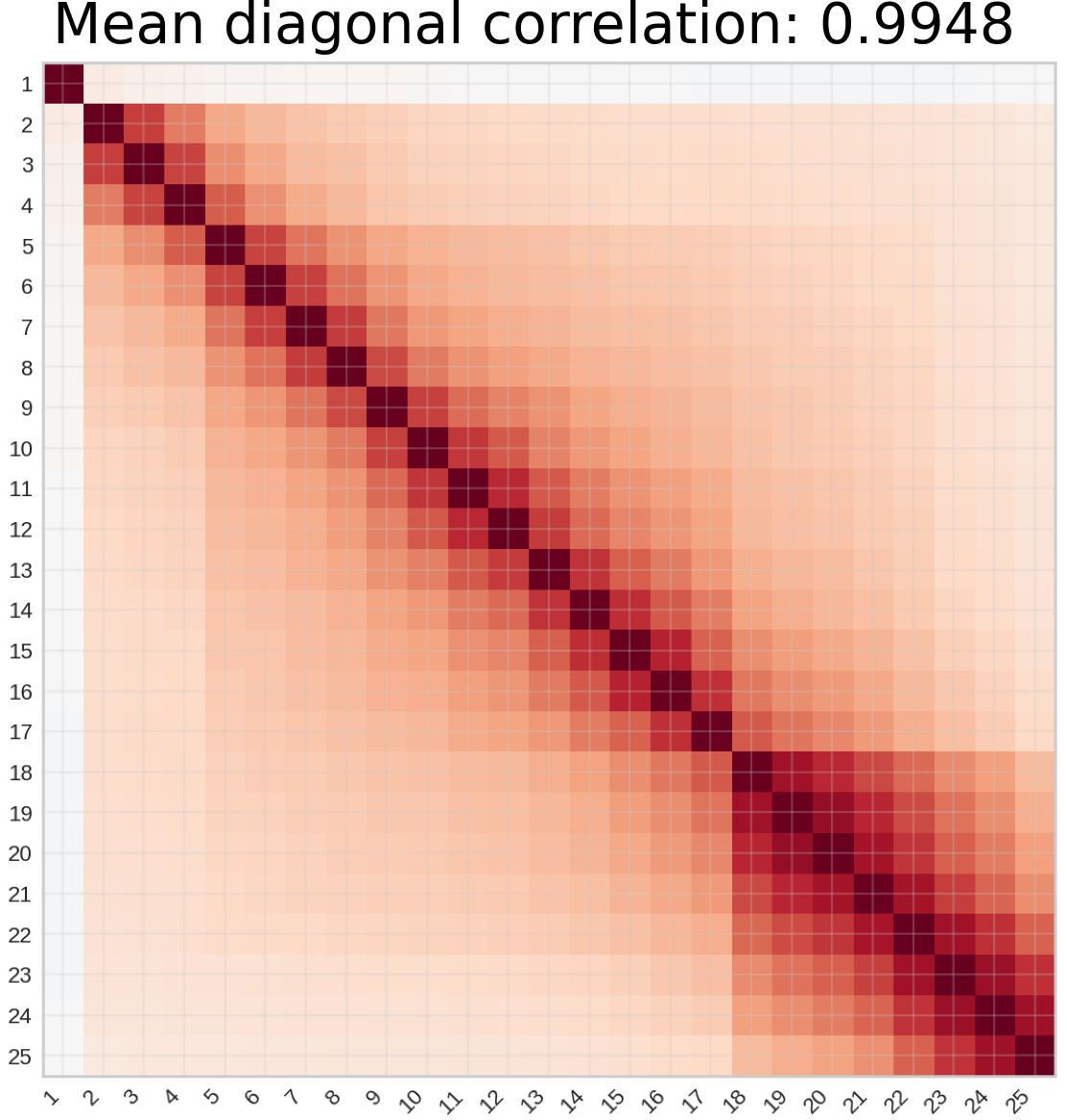} \\
            \centering \scriptsize \textbf{Qwen-2.5-0.5B-SimpleRL-Zoo}
        \end{minipage} &

        \rotlabel{DeepSeek-R1-Distill-Qwen-1.5B} &
        \begin{minipage}{0.20\textwidth} % Adjusted width
            \includegraphics[width=\linewidth]{figures/re_experiments/embedding_correlation_evaluation/DeepSeek-R1-Distill-Qwen-1.5B_vs_Nemotron-Research-Reasoning-Qwen-1.5B/qwen3_32b_LiveMathbench_evaluated/plots/corr_heatmap_DeepSeek-R1-Distill-Qwen-1.5B_vs_Nemotron-Research-Reasoning-Qwen-1.5B_processed.png} \\
            \centering \scriptsize \textbf{Nemotron-Research-Reasoning-Qwen-1.5B}
        \end{minipage} &

        \rotlabel{Qwen3-4B} &
        \begin{minipage}{0.20\textwidth} % Adjusted width
            \includegraphics[width=\linewidth]{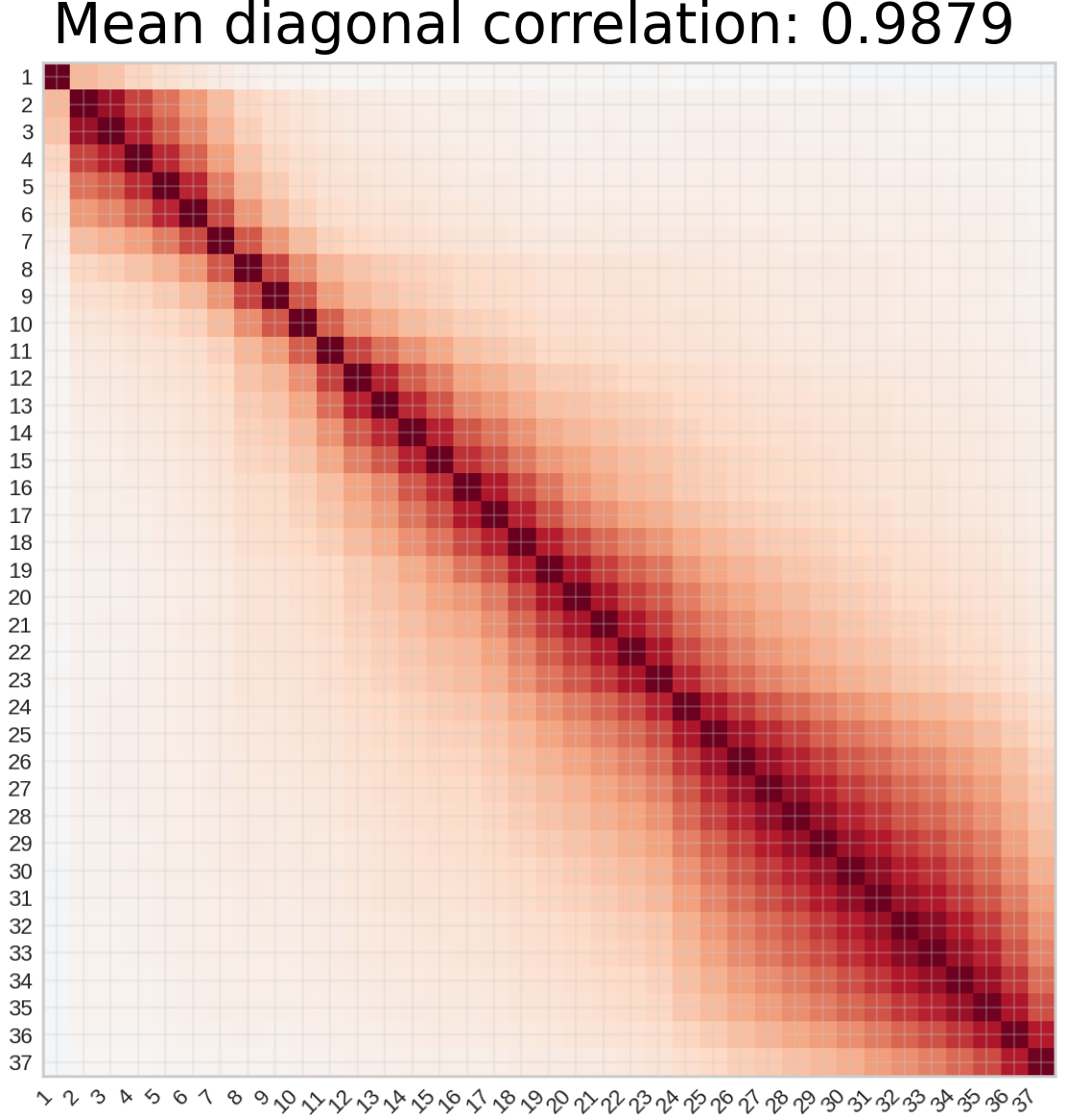} \\
            \centering \scriptsize \textbf{Qwen3-4B-PSR}
        \end{minipage} \\
    \end{tabular}%
    } % End of \makebox

    \vspace{0.5em}
    \rotatebox{-90}{\includegraphics[height=6cm, width=\linewidth, keepaspectratio]{figures/re_experiments/embedding_correlation_evaluation/corr_heatmap.png}} 
    
    \vspace{1.0em}

    \sectionbox{Dataset: MMLU-Pro}
    \vspace{1ex} % Adds a small vertical space for better separation

    % Center the main content grid and ensure it does not exceed the text width
    \makebox[\textwidth][c]{%
    \begin{tabular}{
        c @{\hspace{1pt}} c  @{\hspace{0.5em}}
        c @{\hspace{1pt}} c  @{\hspace{0.5em}}
        c @{\hspace{1pt}} c  @{\hspace{0.5em}}
        c @{\hspace{1pt}} c
    }
        % --- Row 1 ---

        \rotlabel{Qwen2.5-Math-1.5B} &
        \setlength{\fboxsep}{3pt}% 
        \colorbox{red!20}{%  <-- CHANGE COLOR HERE (e.g., yellow!20, blue!10)
            \begin{minipage}{0.20\textwidth}
                \centering
                \includegraphics[width=\linewidth]{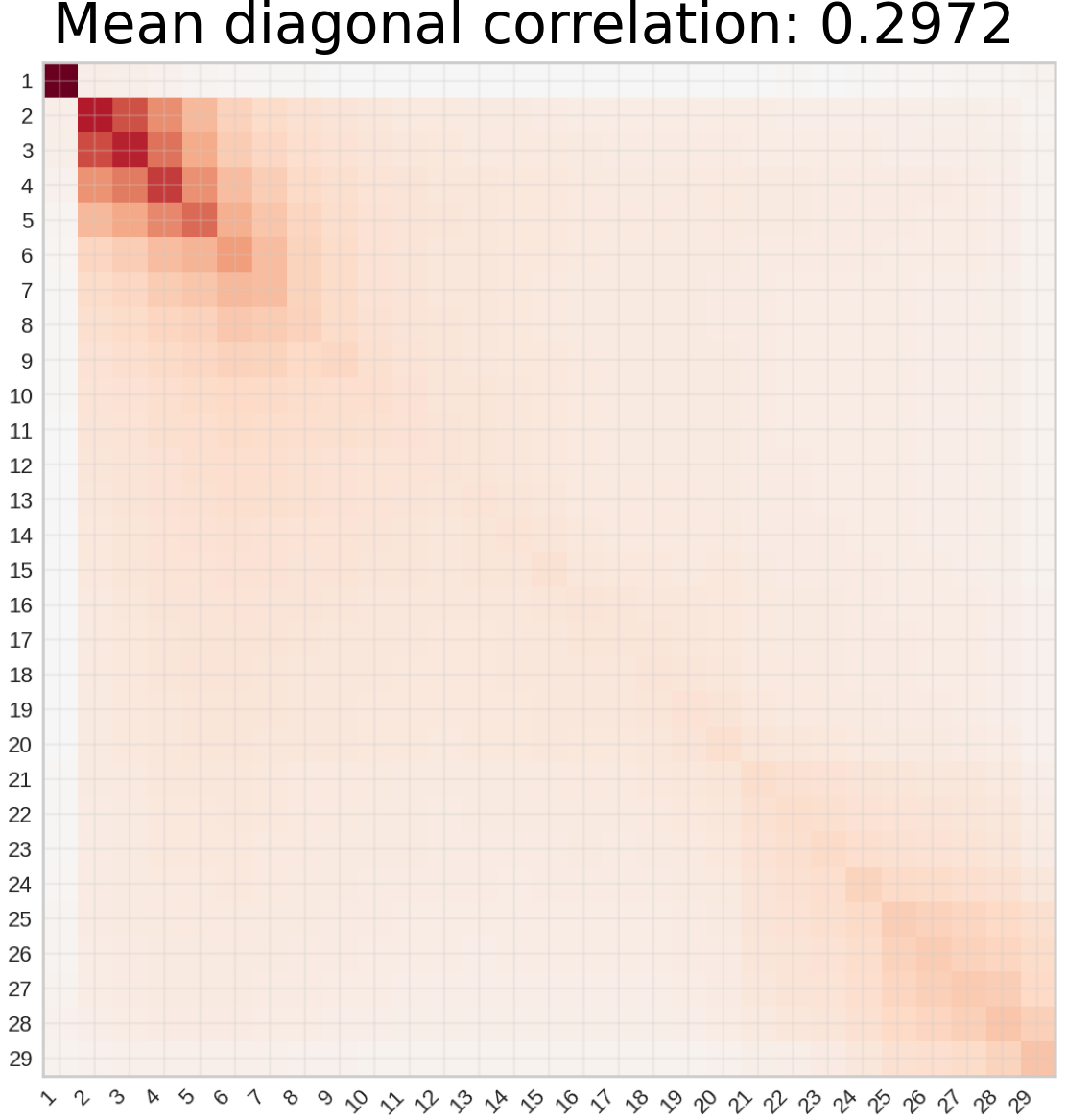}
                \par\vspace{2pt} % Small space between image and caption
                \scriptsize \textbf{DeepSeek-R1-Distill-Qwen-1.5B}
            \end{minipage}%
        } &

        \rotlabel{Qwen3-4B} &
        \begin{minipage}{0.20\textwidth} % Adjusted width
            \includegraphics[width=\linewidth]{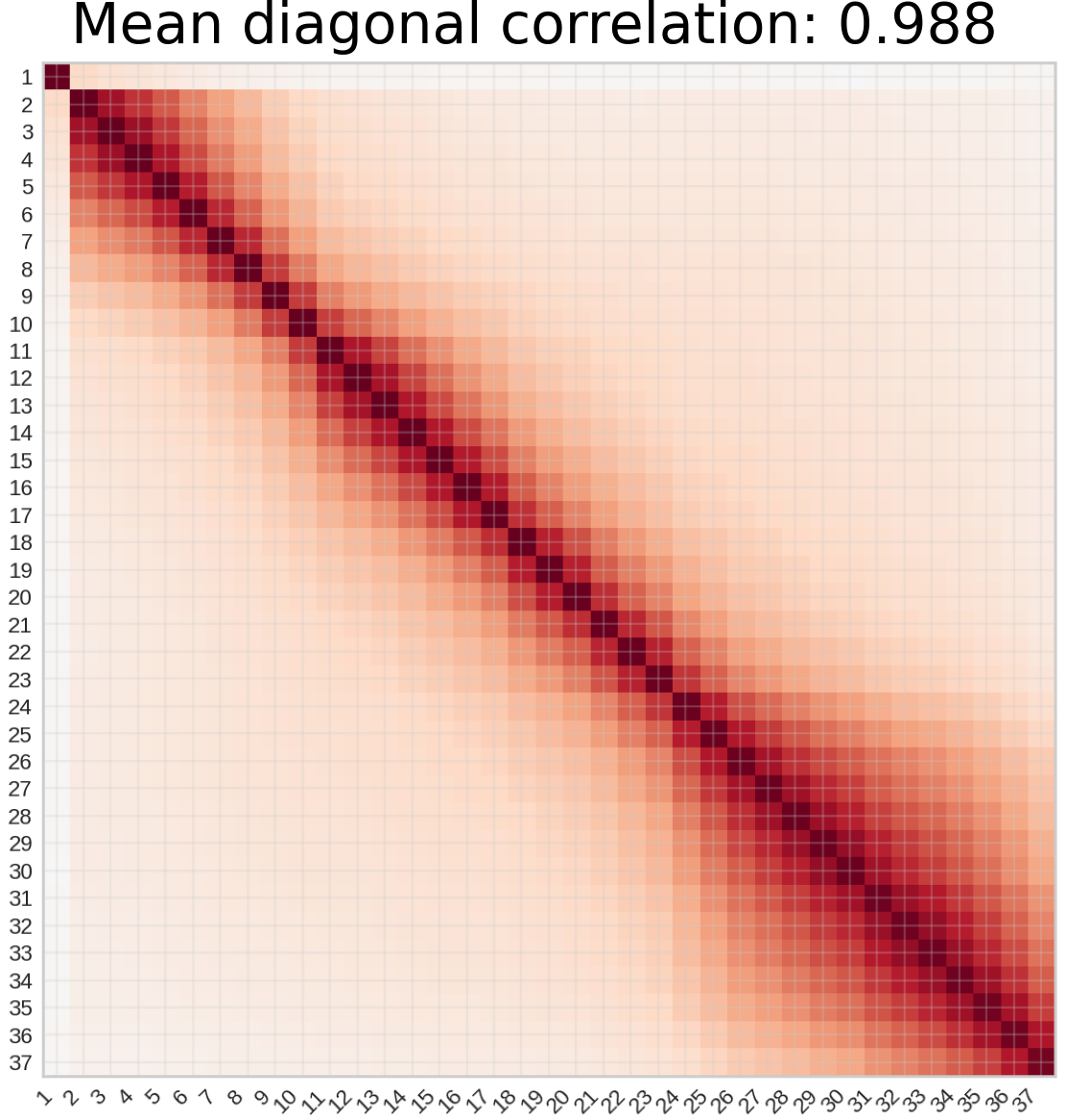} \\
            \centering \scriptsize \textbf{Polaris-4B-Preview}
        \end{minipage} &

        \rotlabel{DeepSeek-R1-Distill-Qwen-7B} &
        \begin{minipage}{0.20\textwidth} % Adjusted width
            \includegraphics[width=\linewidth]{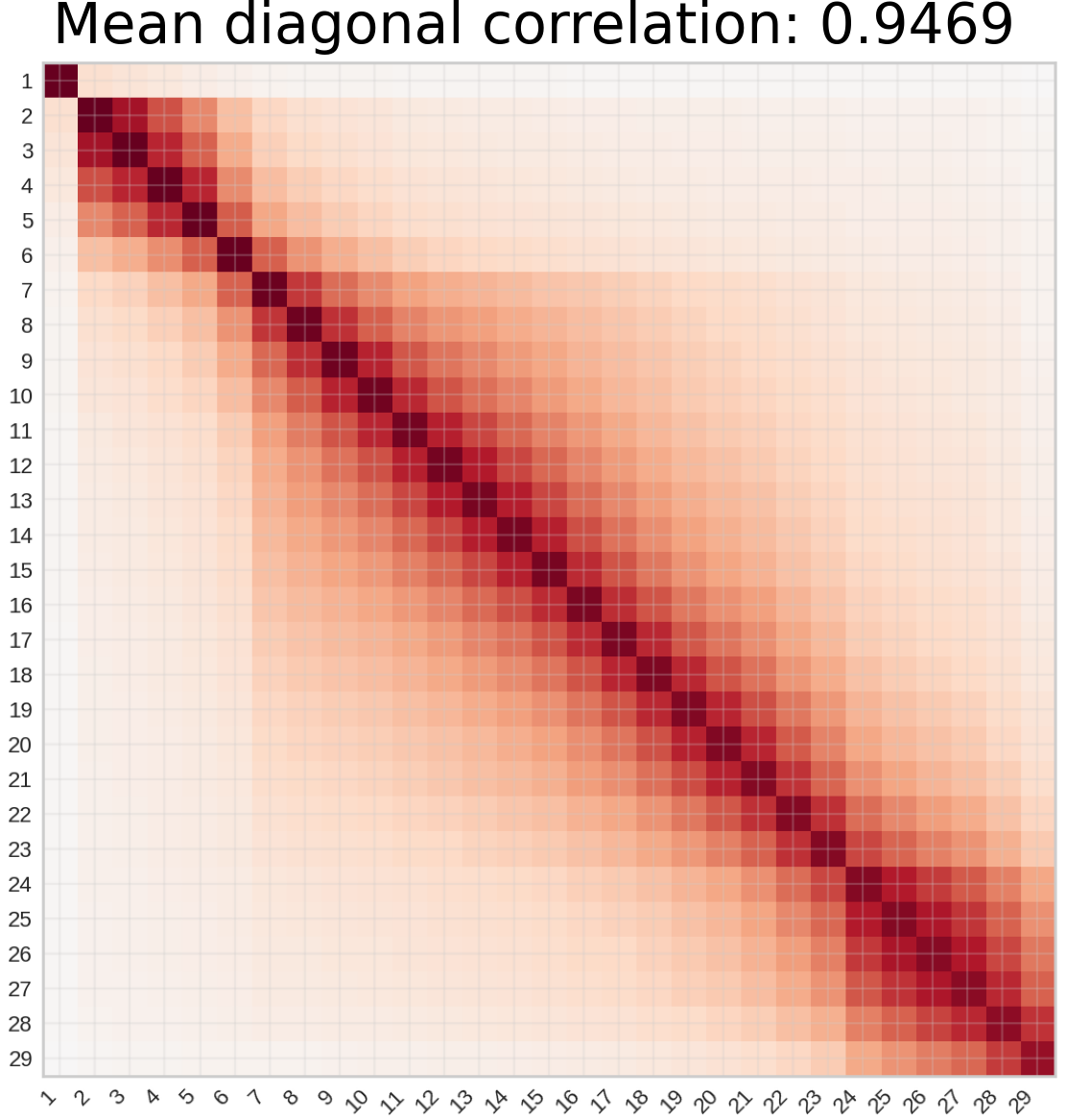} \\
            \centering \scriptsize \textbf{Polaris-7B-Preview}
        \end{minipage} &

        \rotlabel{Qwen2.5-7B} &
        \begin{minipage}{0.20\textwidth} % Adjusted width
            \includegraphics[width=\linewidth]{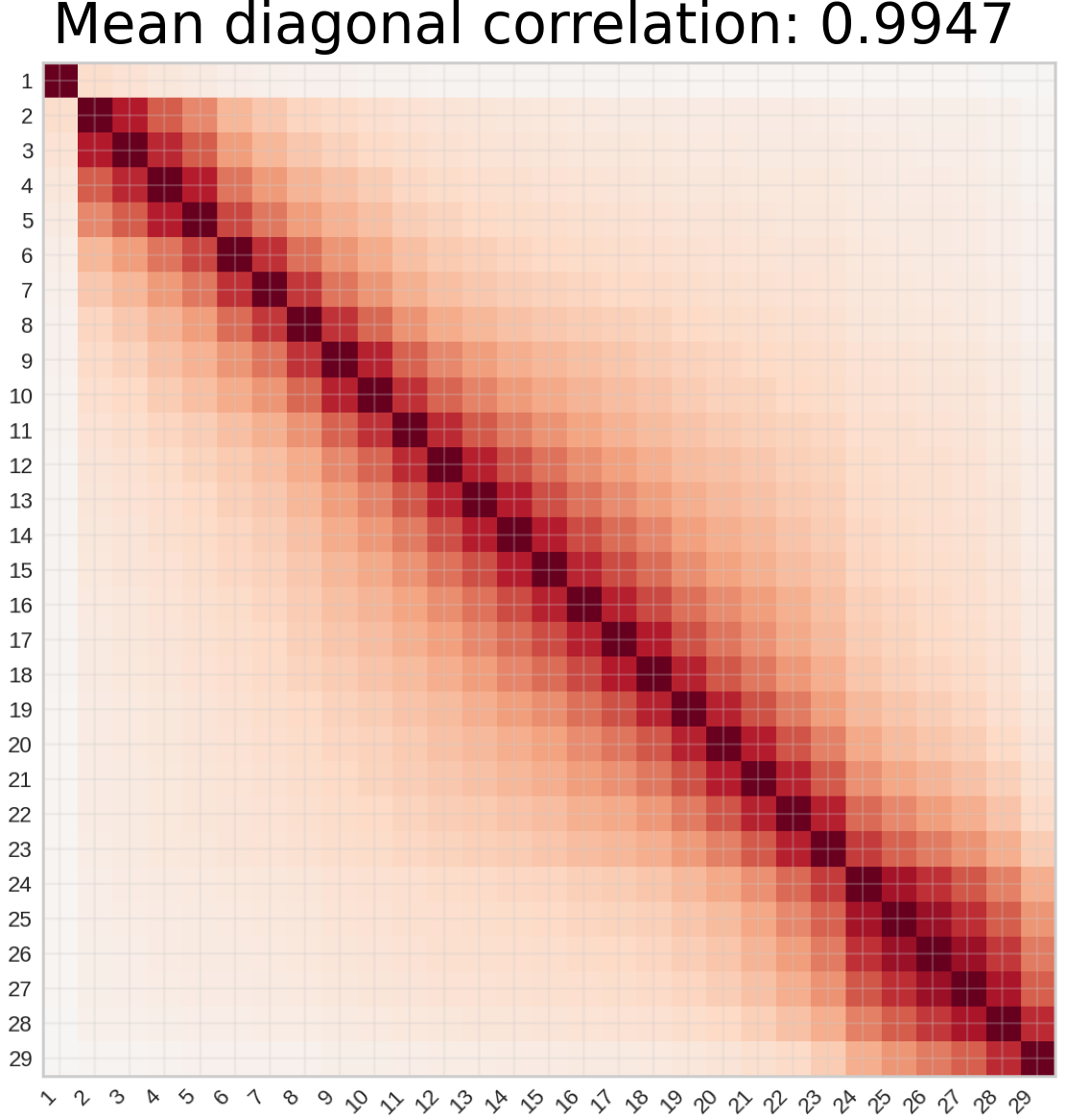} \\
            \centering \scriptsize \textbf{zero\_\_ppo\_\_think\_\_Qwen2.5-7B}
        \end{minipage} \\

        % \multicolumn{8}{c}{\vspace{0.1em}} \\ % Spacing between rows

        % --- Row 2 ---
        \rotlabel{Qwen2.5-1.5B} &
        \begin{minipage}{0.20\textwidth} % Adjusted width
            \includegraphics[width=\linewidth]{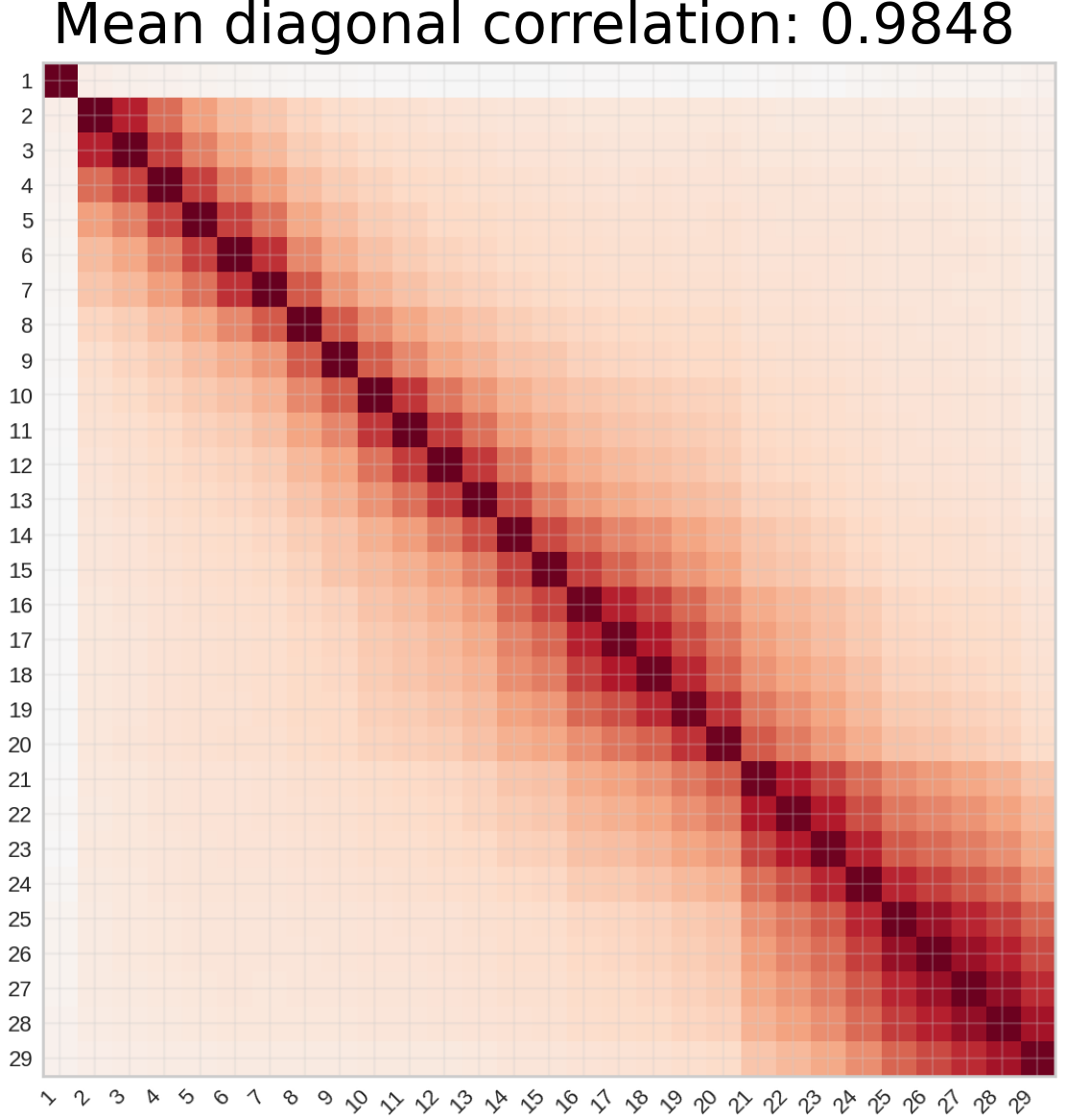} \\
            \centering \scriptsize \textbf{Qwen-2.5-1.5B-SimpleRL-Zoo}
        \end{minipage} &

        \rotlabel{Qwen2.5-0.5B} &
        \begin{minipage}{0.20\textwidth} % Adjusted width
            \includegraphics[width=\linewidth]{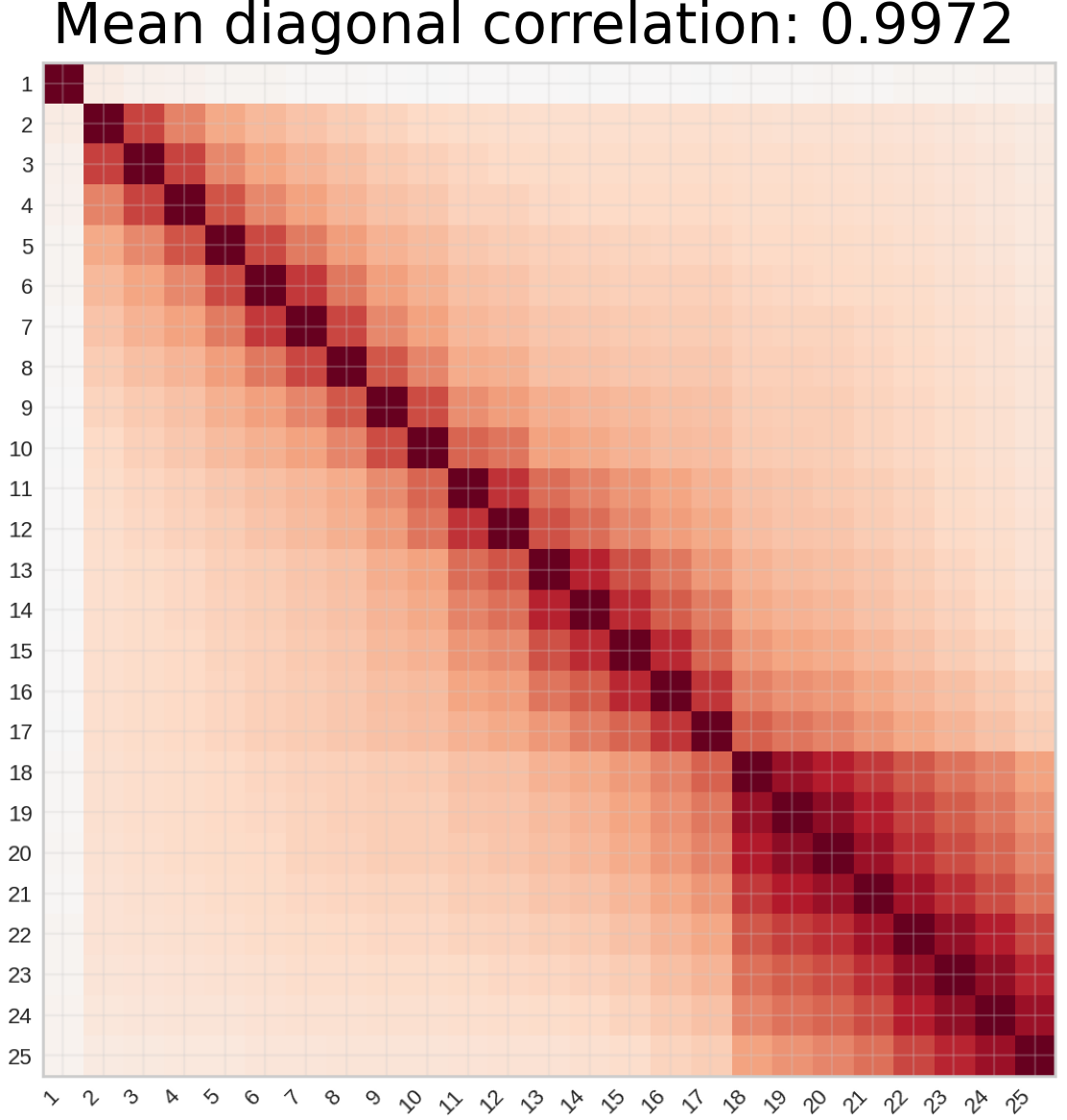} \\
            \centering \scriptsize \textbf{Qwen-2.5-0.5B-SimpleRL-Zoo}
        \end{minipage} &

        \rotlabel{DeepSeek-R1-Distill-Qwen-1.5B} &
        \begin{minipage}{0.20\textwidth} % Adjusted width
            \includegraphics[width=\linewidth]{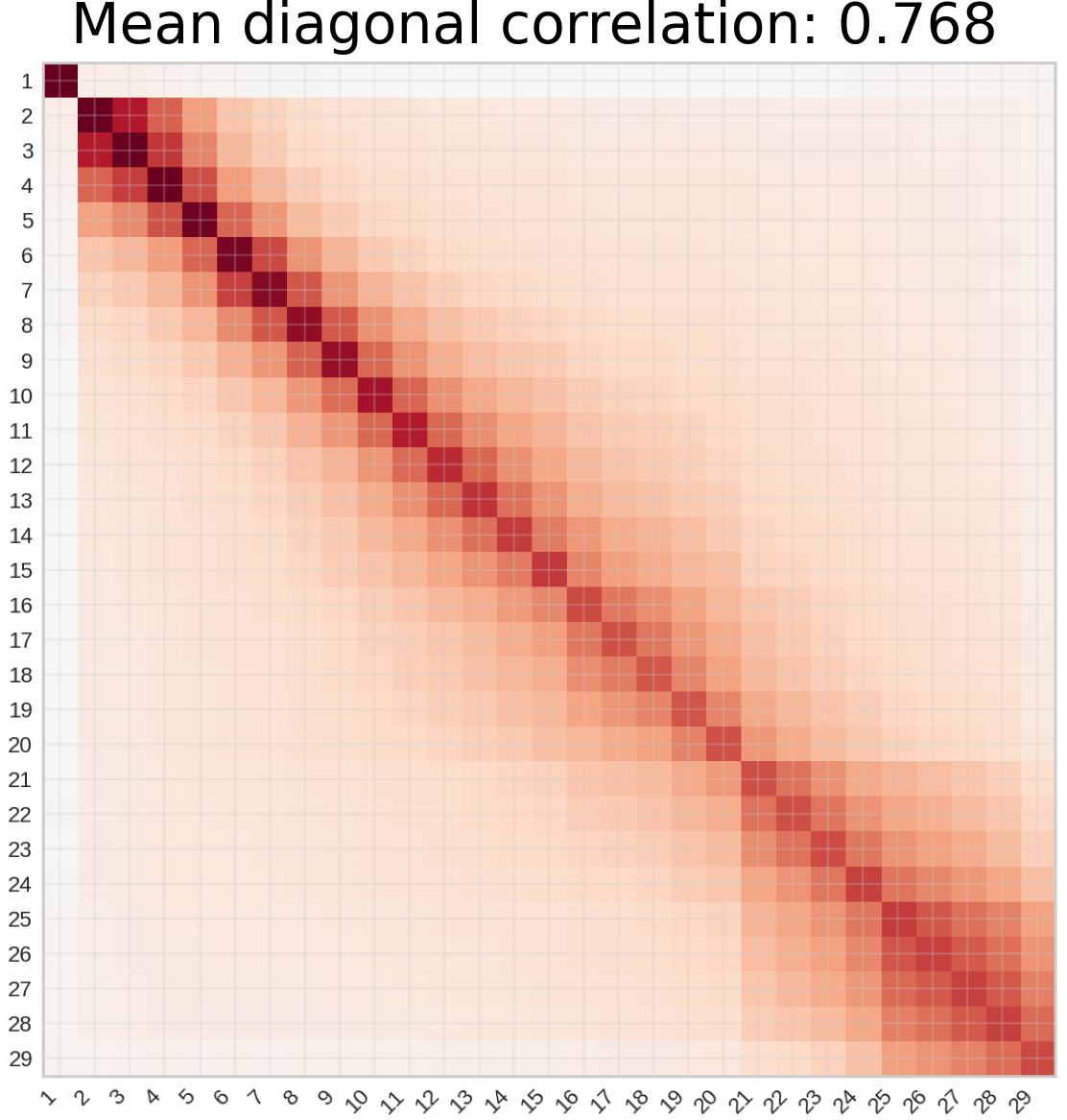} \\
            \centering \scriptsize \textbf{Nemotron-Research-Reasoning-Qwen-1.5B}
        \end{minipage} &

        \rotlabel{Qwen3-4B} &
        \begin{minipage}{0.20\textwidth} % Adjusted width
            \includegraphics[width=\linewidth]{figures/re_experiments/embedding_correlation_evaluation/Qwen3-4B_vs_Qwen3-4B-PSR/mmlu_pro_100_samples/plots/corr_heatmap_Qwen3-4B_vs_Qwen3-4B-PSR_processed.png} \\
            \centering \scriptsize \textbf{Qwen3-4B-PSR}
        \end{minipage} \\
    \end{tabular}%
    } % End of \makebox

    \vspace{0.5em}
    \rotatebox{-90}{\includegraphics[height=6cm, width=\linewidth, keepaspectratio]{figures/re_experiments/embedding_correlation_evaluation/corr_heatmap.png}} 

    % The second section and other elements from the original code were commented out or misplaced,
    % and have been removed for clarity and to focus on the main table.

    \caption{Additional Results on Dimension-Wise Correlation separated by dataset. The vertical axis and horizontal axis are Base Model Layer Index and Reasoning Model Layer Index, respectively. The \textbf{\textcolor{red}{red}} background indicates SFT-tuned pairs.}
    \label{fig: appendix-corr-base-vs-reasoning}
\end{figure*}
% ================================================================================================ 
% Dimension-Wise Correlation 
% Base Embedding Model vs. Reasoning Embedding Model
%================================================================================================ 

\begin{figure*}[p]
    \centering
    {\large \textbf{Dimension-Wise Correlation}} \par\smallskip
    {\large \textbf{Base Embedding Models $\mathcal{M}_{base}^{Emb}$ vs. Reasoning Embedding Models $\mathcal{M}_{reason}^{Emb}$}} \par\medskip

    % --- Configuration ---
    \setlength{\tabcolsep}{1pt}

    % % Helper for rotated left label
    % % NOTE: Requires \usepackage{graphicx} and \usepackage{array}
    % \newcommand{\rotlabel}[1]{%
    %     \raisebox{-0.5\height}{\rotatebox{90}{\parbox{3.5cm}{\centering\scriptsize \textbf{#1}}}}%
    % }

    % % Helper for COLORED Data Section Title
    % % NOTE: Requires \usepackage[table]{xcolor}
    % \newcommand{\sectionbox}[1]{%
    %     \noindent\colorbox{blue!20}{%
    %         \parbox{\dimexpr\linewidth-2\fboxsep\relax}{%
    %             \centering \vspace{0.1em}
    %             {\large \textbf{\textit{#1}}}
    %             \vspace{0.1em}
    %         }%
    %     }\par
    % }

    \sectionbox{Dataset: CoT Datset}
    \vspace{1ex} % Adds a small vertical space for better separation

    % Center the main content grid and ensure it does not exceed the text width
    \makebox[\textwidth][c]{%
    \begin{tabular}{
        c @{\hspace{1pt}} c  @{\hspace{0.5em}}
        c @{\hspace{1pt}} c  @{\hspace{0.5em}}
        c @{\hspace{1pt}} c
    }
        % --- Row 1 ---
        \rotlabel{Qwen2.5-Math-1.5B-\textit{Emb}} &
        \setlength{\fboxsep}{3pt}%
        \colorbox{red!20}{%
            \begin{minipage}{0.18\textwidth} % Adjusted width for better fit
                \includegraphics[width=\linewidth]{figures/re_experiments/embedding_correlation_evaluation/Qwen2.5-Math-1.5B-Reasoning-Embedding_vs_DeepSeek-R1-Distill-Qwen-1.5B-Reasoning-Embedding/qwen3_32b_LiveMathbench_evaluated/plots/corr_heatmap_Qwen2.5-Math-1.5B-Reasoning-Embedding_vs_DeepSeek-R1-Distill-Qwen-1.5B-Reasoning-Embedding_processed.png}
                \centering \scriptsize \textbf{DeepSeek-R1-Distill-Qwen-1.5B-\textit{Emb}}
            \end{minipage}%
        } &

        \rotlabel{Qwen3-0.6B-Base-\textit{Emb}} &
        \setlength{\fboxsep}{3pt}%
        \colorbox{red!20}{%
            \begin{minipage}{0.18\textwidth} % Adjusted width
                \includegraphics[width=\linewidth]{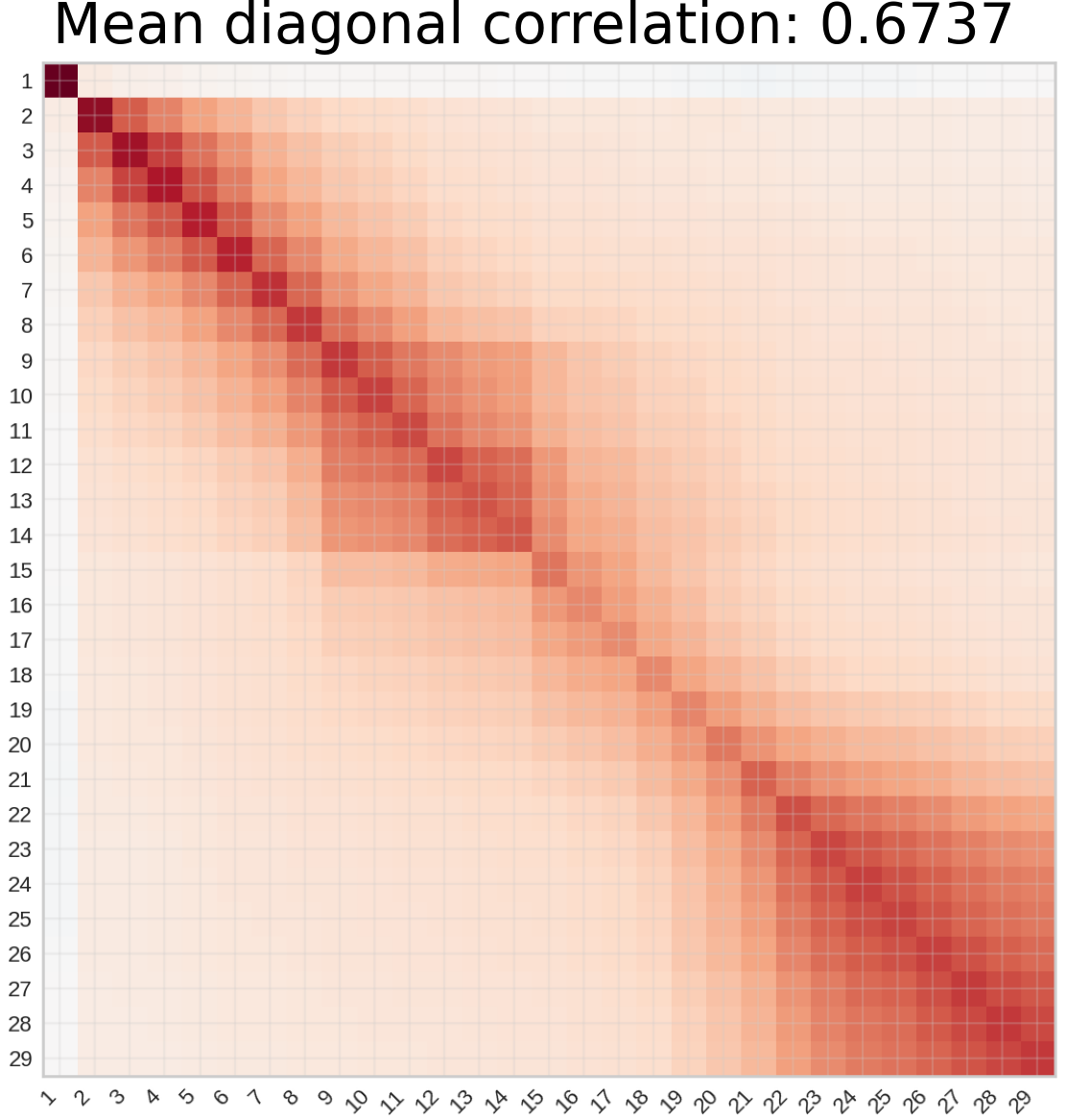} \\
                \centering \scriptsize \textbf{Qwen3-0.6B-\textit{Emb}}
            \end{minipage}%
        } &

        \rotlabel{Qwen2.5-1.5B-\textit{Emb}} &
        \begin{minipage}{0.18\textwidth} % Adjusted width
            \includegraphics[width=\linewidth]{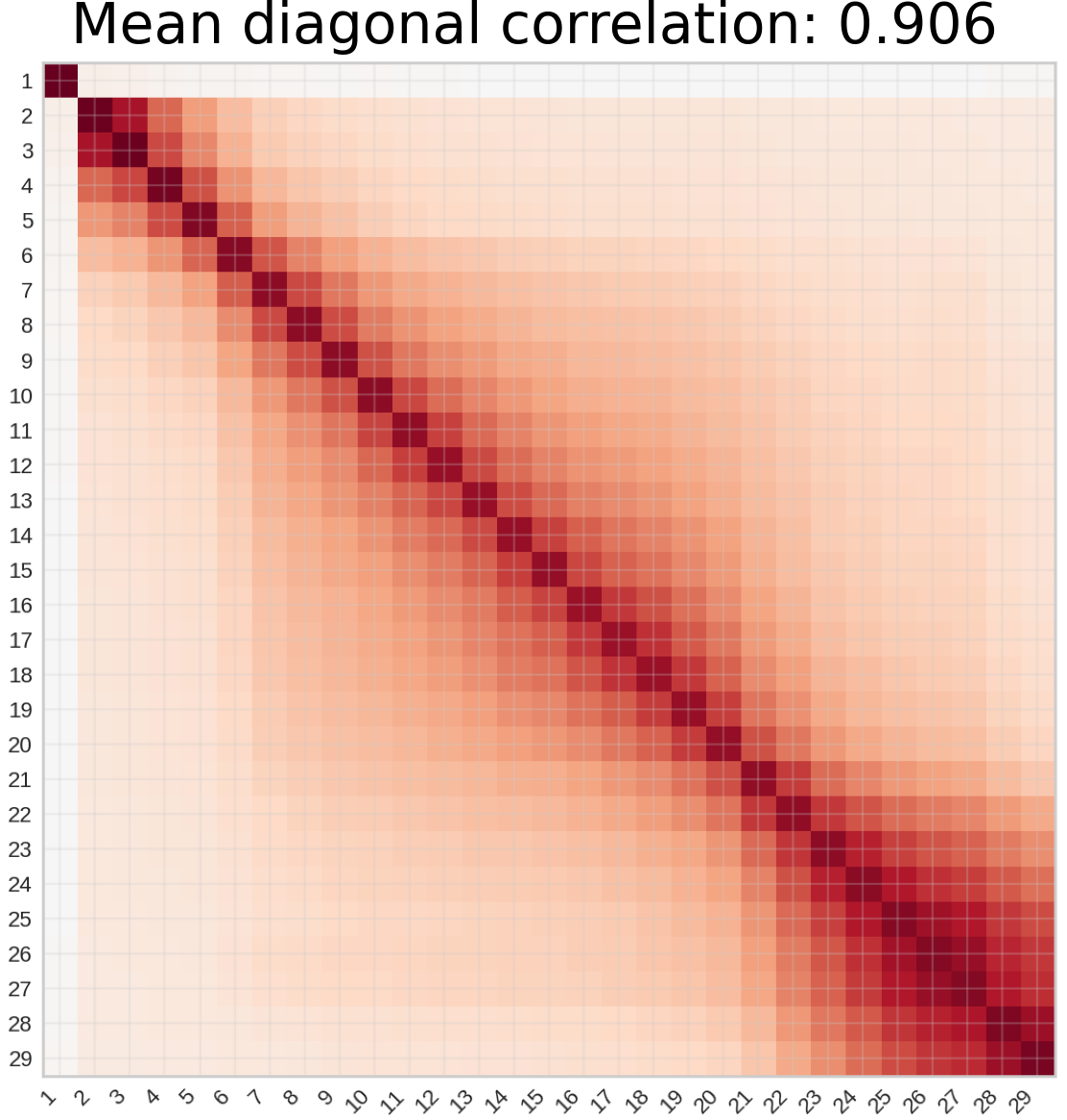} \\
            \centering \scriptsize \textbf{Qwen-2.5-1.5B-SimpleRL-Zoo-\textit{Emb}}
        \end{minipage} \\

        % \multicolumn{6}{c}{\vspace{0.1em}} \\ % Spacing between rows

        % --- Row 2 ---
        \rotlabel{Qwen2.5-0.5B-\textit{Emb}} &
        \begin{minipage}{0.18\textwidth} % Adjusted width
            \includegraphics[width=\linewidth]{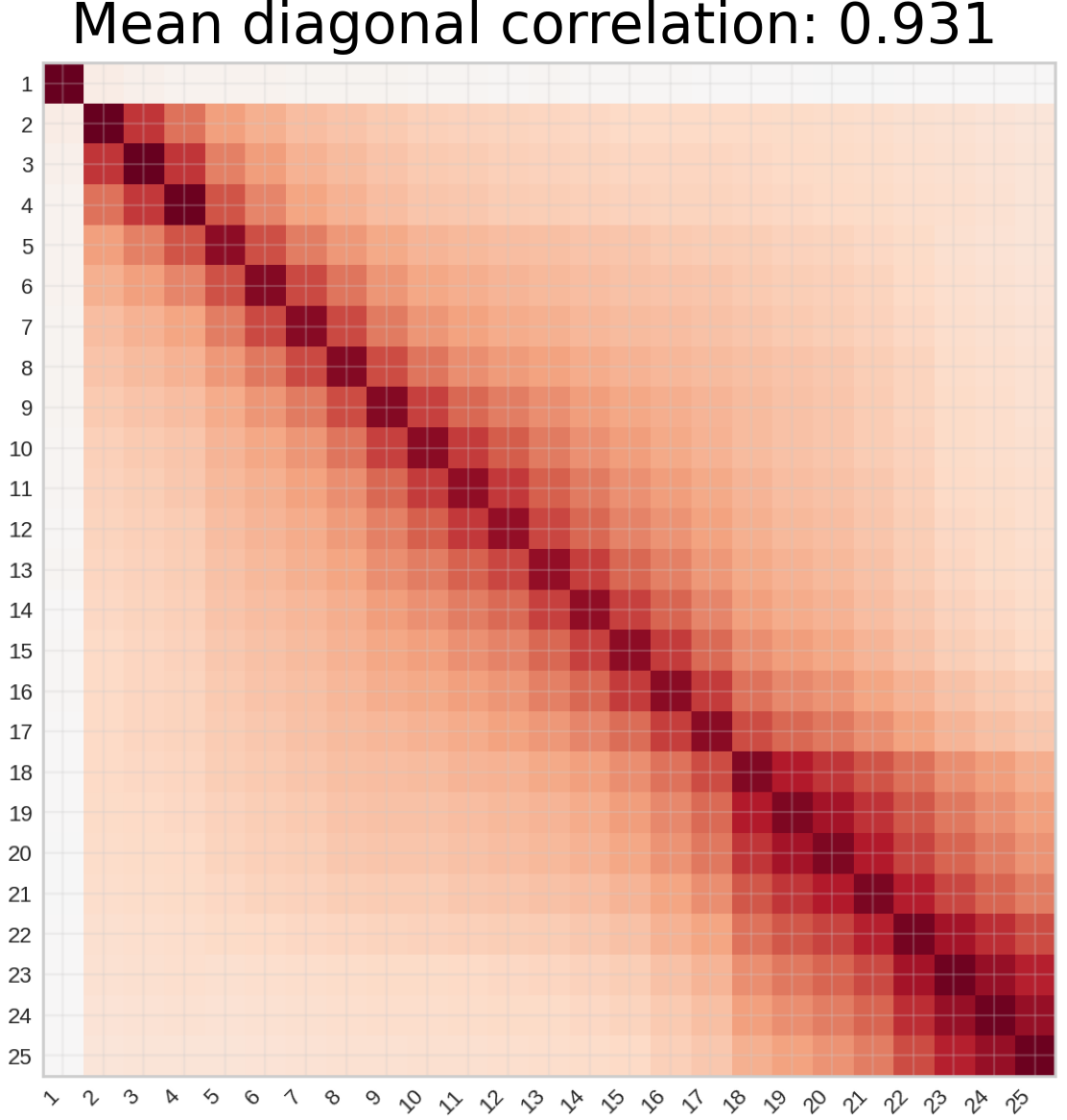} \\
            \centering \scriptsize \textbf{Qwen-2.5-0.5B-SimpleRL-Zoo-\textit{Emb}}
        \end{minipage} &

        \rotlabel{DeepSeek-R1-Distill-Qwen-1.5B-\textit{Emb}} &
        \begin{minipage}{0.18\textwidth} % Adjusted width
            \includegraphics[width=\linewidth]{figures/re_experiments/embedding_correlation_evaluation/DeepSeek-R1-Distill-Qwen-1.5B-Reasoning-Embedding_vs_Nemotron-Research-Reasoning-Qwen-1.5B-Reasoning-Embedding/qwen3_32b_LiveMathbench_evaluated/plots/corr_heatmap_DeepSeek-R1-Distill-Qwen-1.5B-Reasoning-Embedding_vs_Nemotron-Research-Reasoning-Qwen-1.5B-Reasoning-Embedding_processed.png} \\
            \centering \scriptsize \textbf{Nemotron-Research-Reasoning-Qwen-1.5B-\textit{Emb}}
        \end{minipage} &

        \rotlabel{Qwen3-4B-\textit{Emb}} &
        \begin{minipage}{0.18\textwidth} % Adjusted width
            \includegraphics[width=\linewidth]{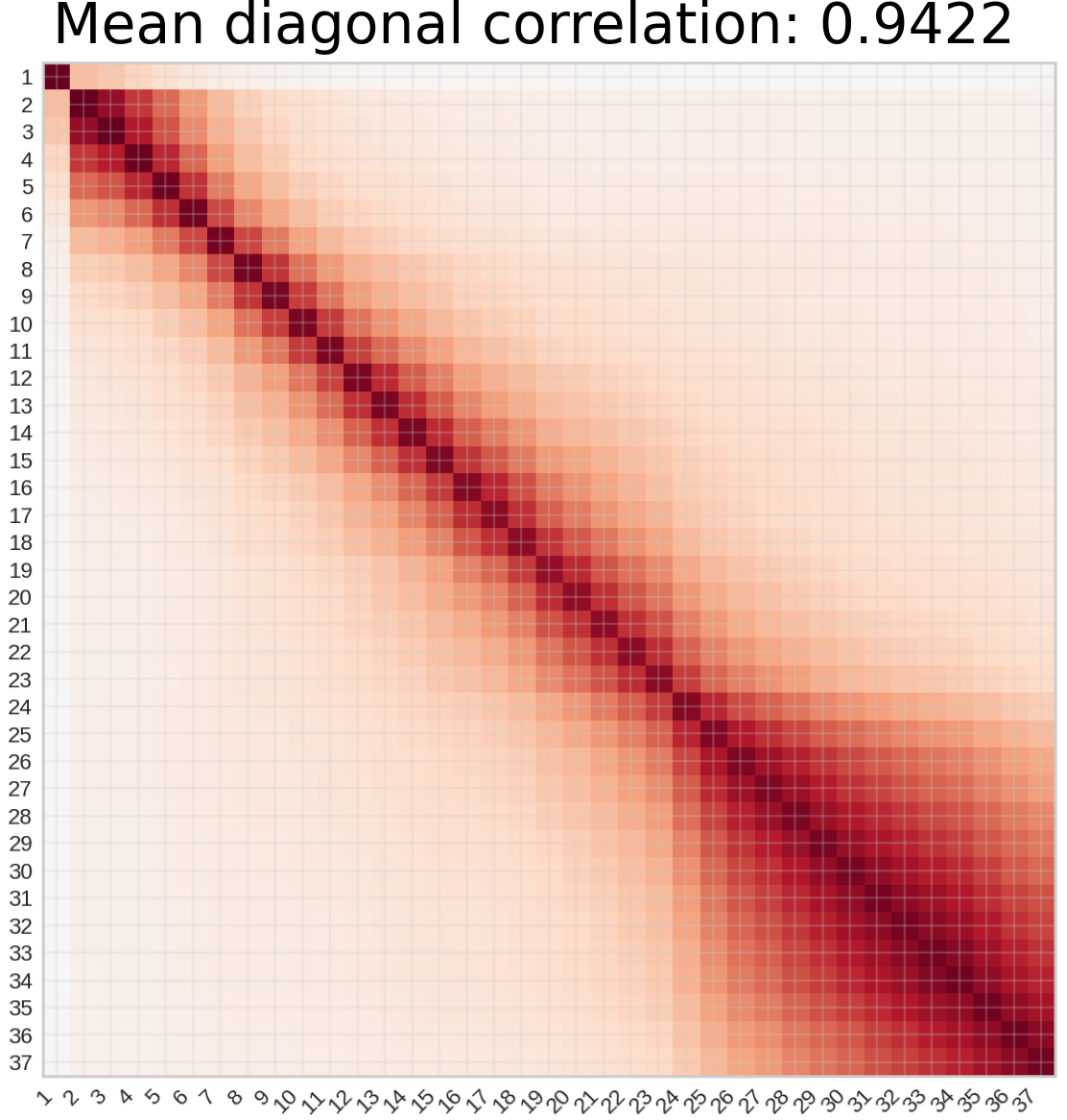} \\
            \centering \scriptsize \textbf{Qwen3-4B-PSR-\textit{Emb}}
        \end{minipage} \\
    \end{tabular}%
    } % End of \makebox

    \vspace{0.5em}
    \rotatebox{-90}{\includegraphics[height=6cm, width=\linewidth, keepaspectratio]{figures/re_experiments/embedding_correlation_evaluation/corr_heatmap.png}} 
    
    \vspace{1.0em}

    \sectionbox{Dataset: MMLU-Pro}
    \vspace{1ex} % Adds a small vertical space for better separation

    % Center the main content grid and ensure it does not exceed the text width
    \makebox[\textwidth][c]{%
    \begin{tabular}{
        c @{\hspace{1pt}} c  @{\hspace{0.5em}}
        c @{\hspace{1pt}} c  @{\hspace{0.5em}}
        c @{\hspace{1pt}} c
    }
        % --- Row 1 ---
        \rotlabel{Qwen2.5-Math-1.5B-\textit{Emb}} &
        \setlength{\fboxsep}{3pt}%
        \colorbox{red!20}{%
            \begin{minipage}{0.18\textwidth} % Adjusted width for better fit
                \includegraphics[width=\linewidth]{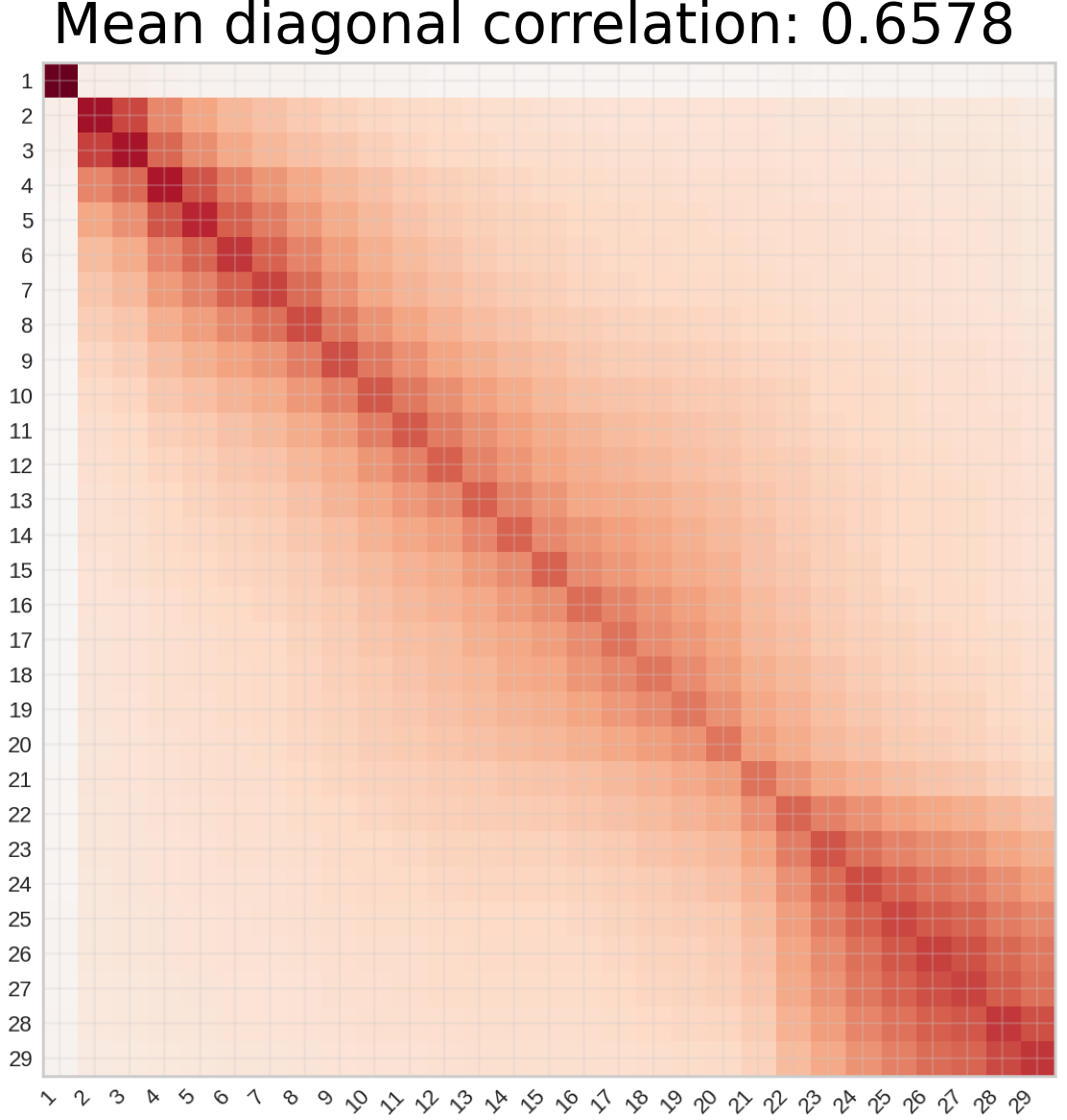}
                \centering \scriptsize \textbf{DeepSeek-R1-Distill-Qwen-1.5B-\textit{Emb}}
            \end{minipage}%
        } &

        \rotlabel{Qwen3-0.6B-Base-\textit{Emb}} &
        \setlength{\fboxsep}{3pt}%
        \colorbox{red!20}{%
            \begin{minipage}{0.18\textwidth} % Adjusted width
                \includegraphics[width=\linewidth]{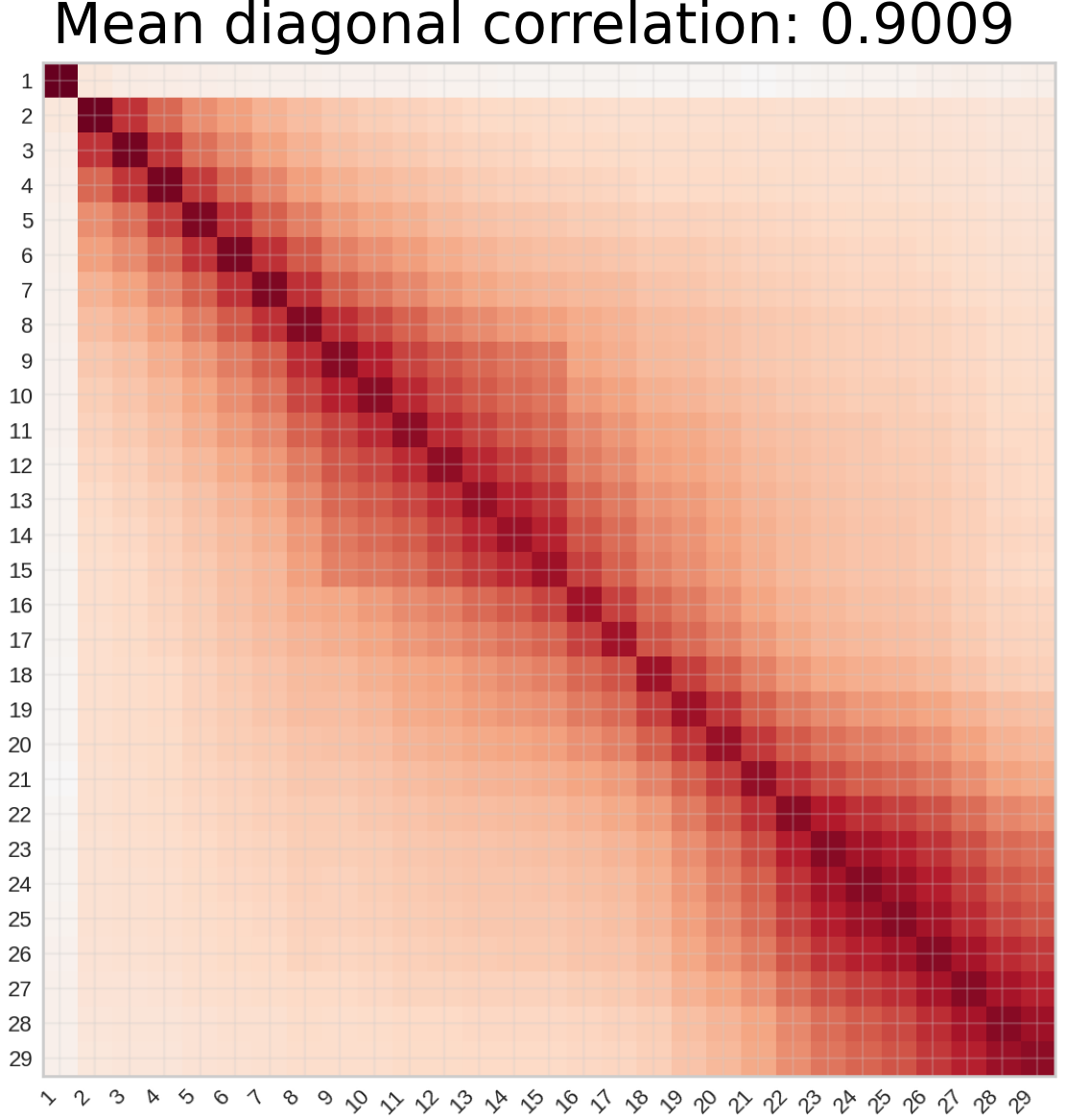} \\
                \centering \scriptsize \textbf{Qwen3-0.6B-\textit{Emb}}
            \end{minipage}%
        } &

        \rotlabel{Qwen2.5-1.5B-\textit{Emb}} &
        \begin{minipage}{0.18\textwidth} % Adjusted width
            \includegraphics[width=\linewidth]{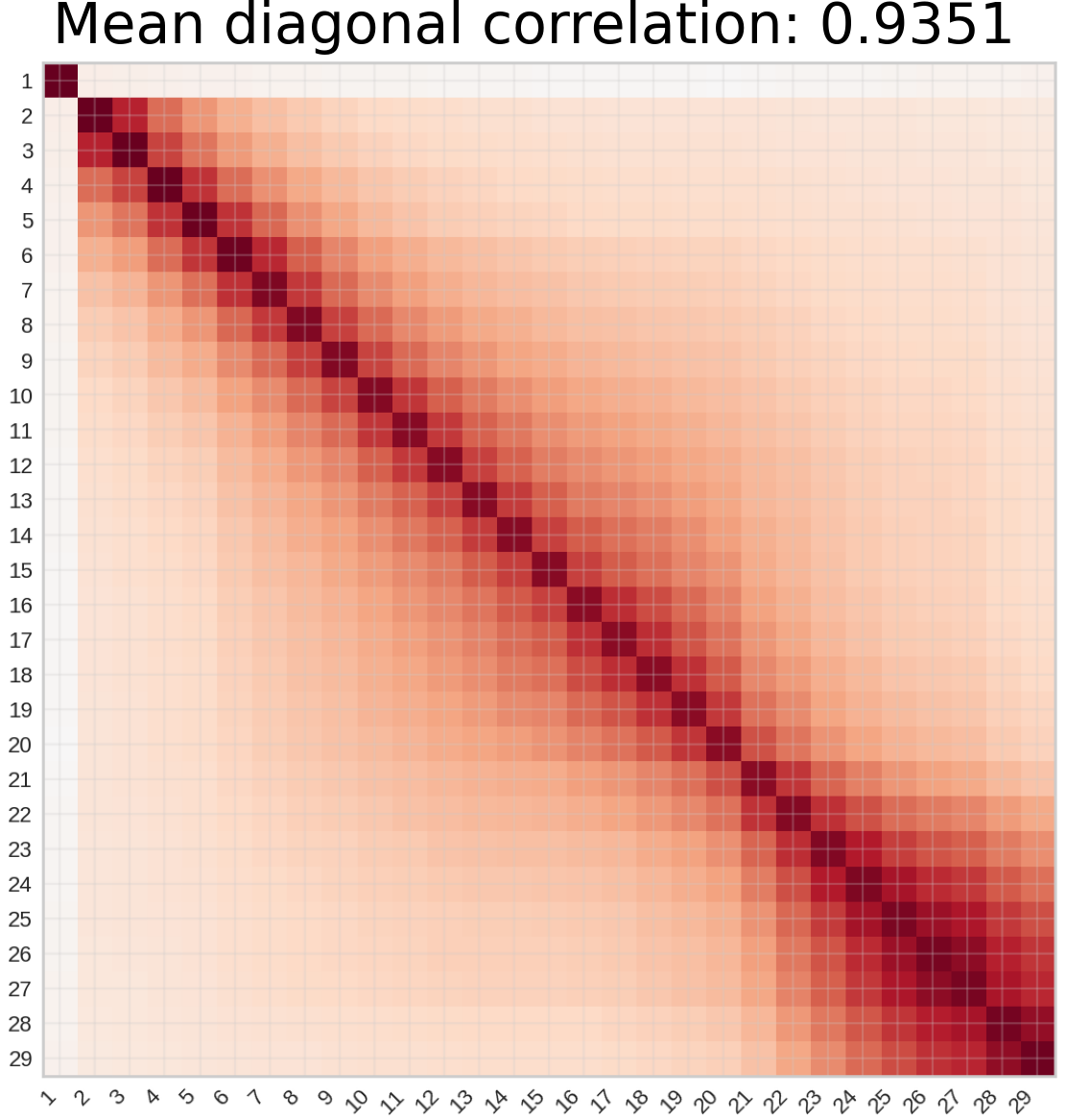} \\
            \centering \scriptsize \textbf{Qwen-2.5-1.5B-SimpleRL-Zoo-\textit{Emb}}
        \end{minipage} \\

        % \multicolumn{6}{c}{\vspace{0.1em}} \\ % Spacing between rows

        % --- Row 2 ---
        \rotlabel{Qwen2.5-0.5B-\textit{Emb}} &
        \begin{minipage}{0.18\textwidth} % Adjusted width
            \includegraphics[width=\linewidth]{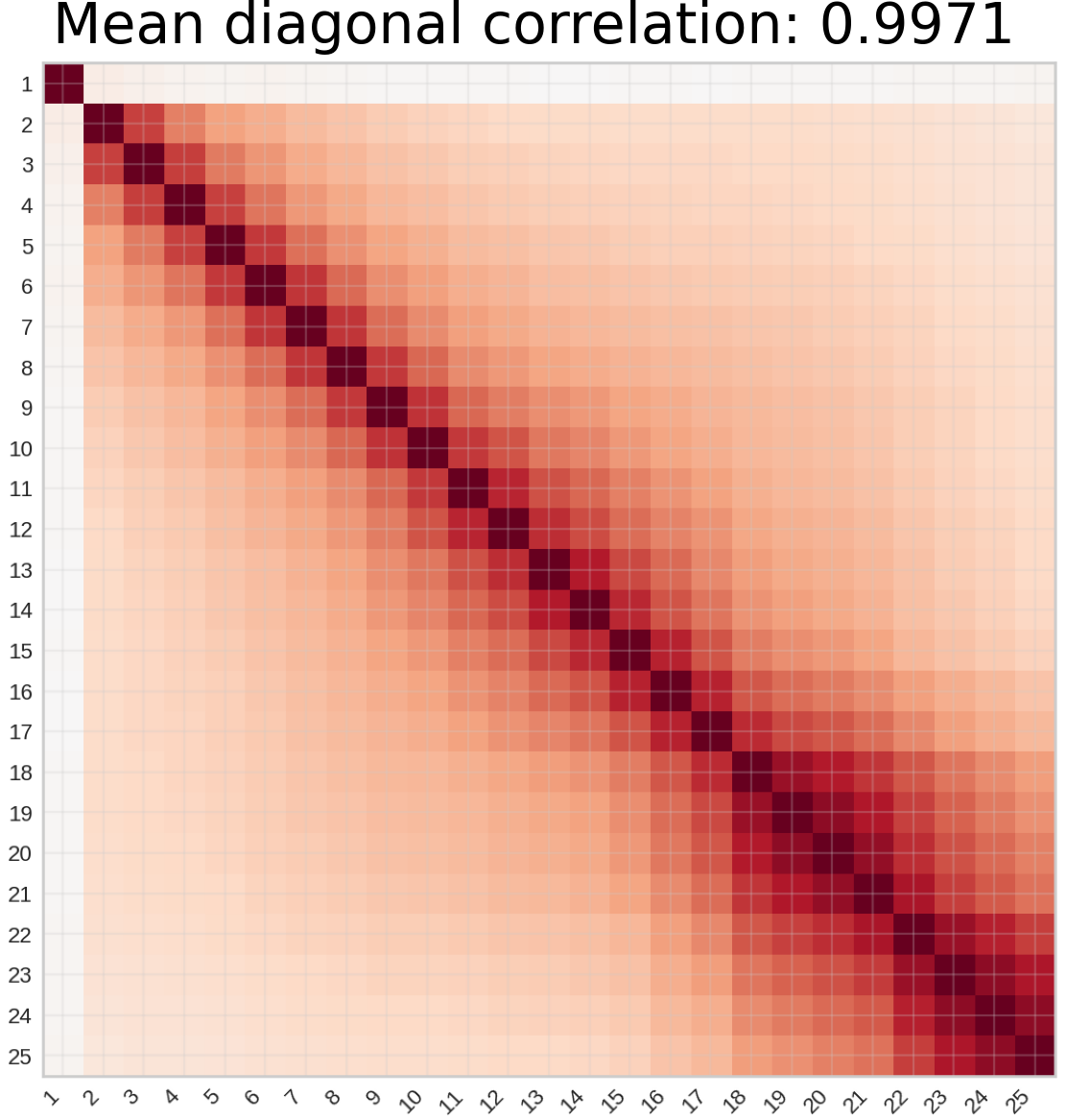} \\
            \centering \scriptsize \textbf{Qwen-2.5-0.5B-SimpleRL-Zoo-\textit{Emb}}
        \end{minipage} &

        \rotlabel{DeepSeek-R1-Distill-Qwen-1.5B-\textit{Emb}} &
        \begin{minipage}{0.18\textwidth} % Adjusted width
            \includegraphics[width=\linewidth]{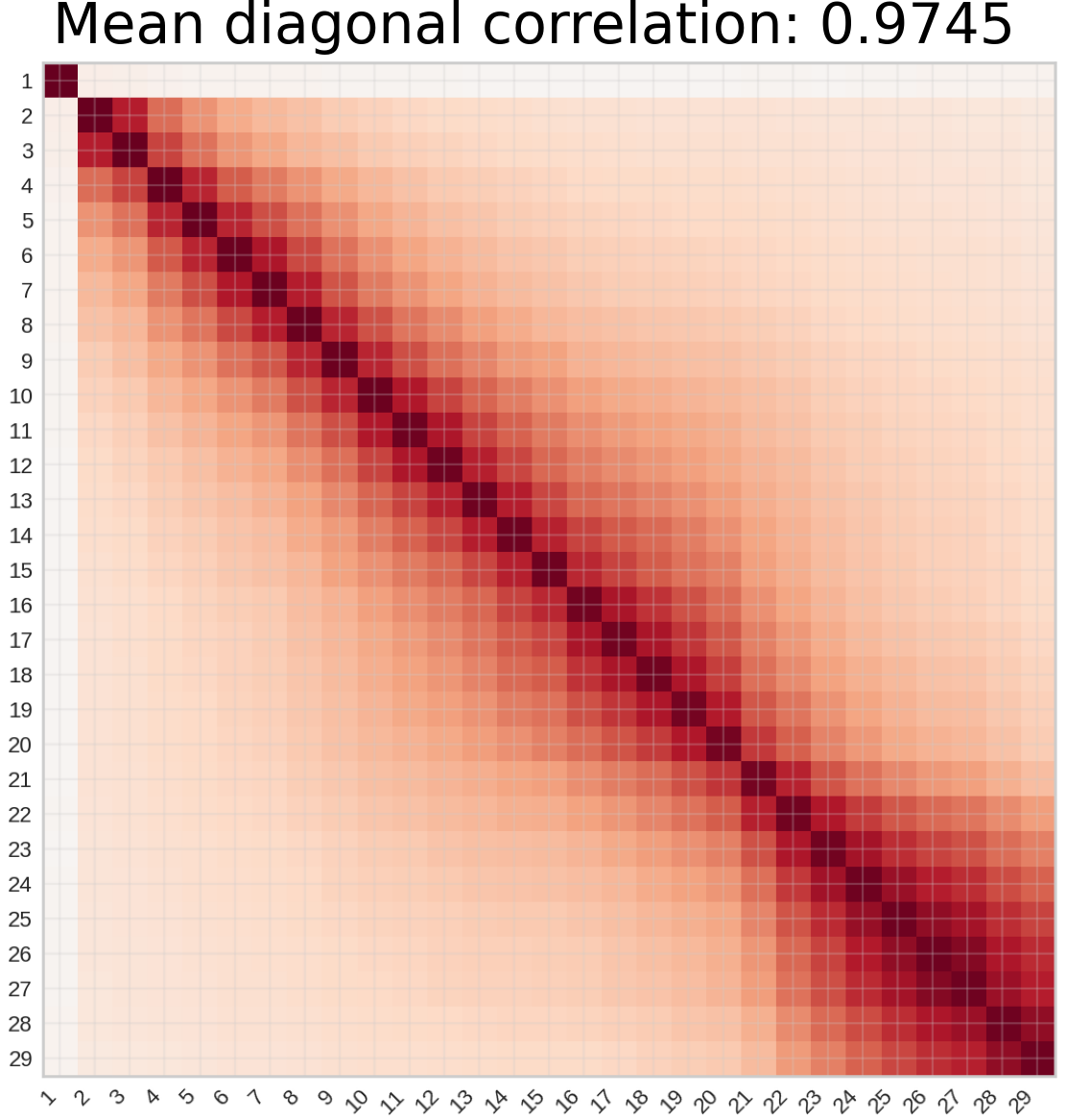} \\
            \centering \scriptsize \textbf{Nemotron-Research-Reasoning-Qwen-1.5B-\textit{Emb}}
        \end{minipage} &

        \rotlabel{Qwen3-4B-\textit{Emb}} &
        \begin{minipage}{0.18\textwidth} % Adjusted width
            \includegraphics[width=\linewidth]{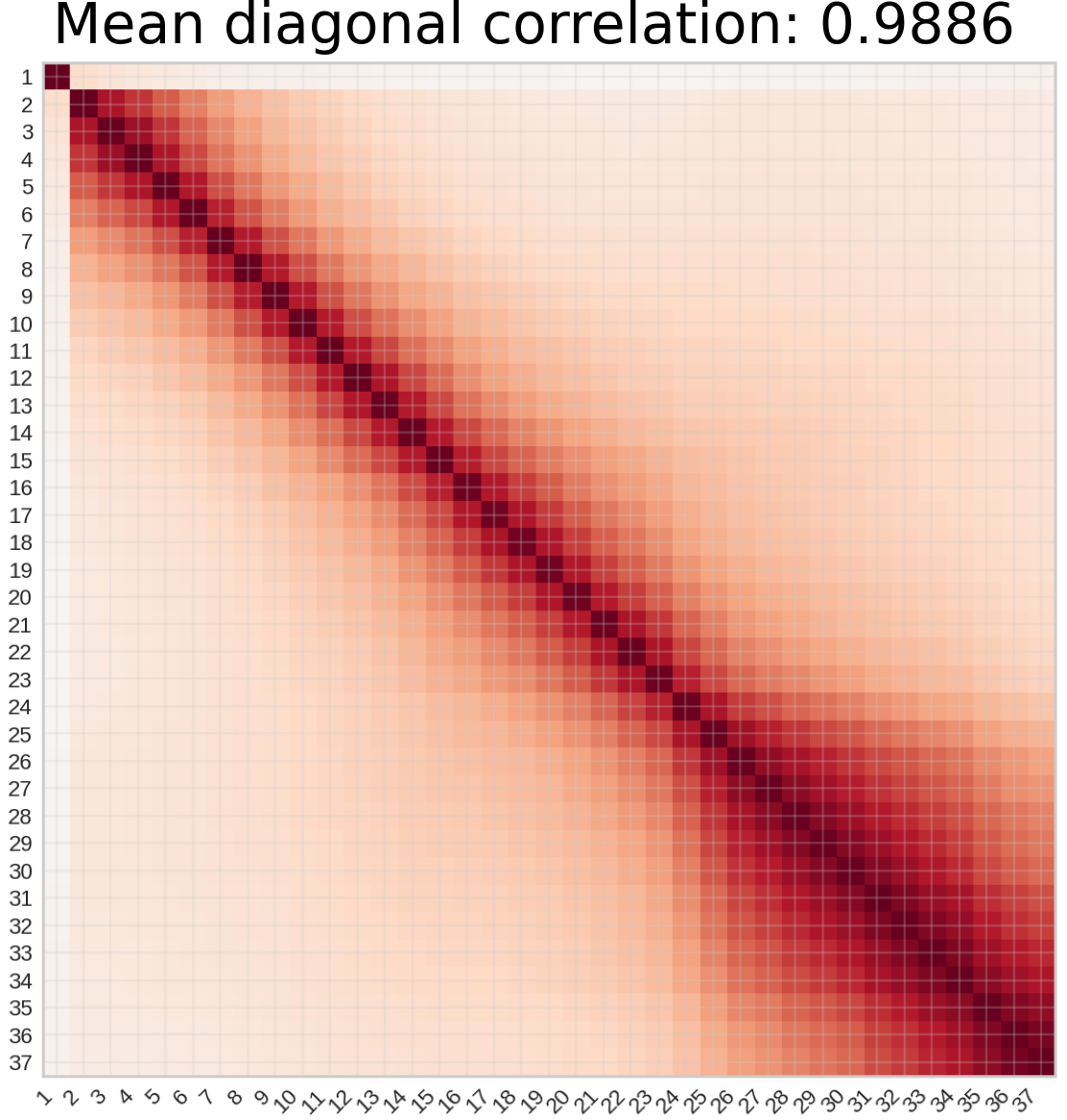} \\
            \centering \scriptsize \textbf{Qwen3-4B-PSR-\textit{Emb}}
        \end{minipage} \\
    \end{tabular}%
    } % End of \makebox

    \vspace{0.5em}
    \rotatebox{-90}{\includegraphics[height=6cm, width=\linewidth, keepaspectratio]{figures/re_experiments/embedding_correlation_evaluation/corr_heatmap.png}} 

    % The second section and other elements from the original code were commented out or misplaced,
    % and have been removed for clarity and to focus on the main table.

    \caption{Additional Results on Dimension-Wise Correlation separated by dataset. The vertical axis and horizontal axis are Base Model Layer Index and Reasoning Model Layer Index, respectively. The \textbf{\textcolor{red}{red}} background indicates their backbone LLMs are SFT-tuned pairs.}
    \label{fig: appendix-corr-base-embedding-vs-reasoning-embedding}
\end{figure*}

% ================================================================================================ 
% Orthogonal Matrix $O$ 
% Base Model vs. Reasoning Model
%================================================================================================ 

\begin{table*}[p]
    \centering
    \setlength{\tabcolsep}{2.5pt}
    \renewcommand{\arraystretch}{1.15}

    \caption{Inverse row entropy $H_{\text{inv}}$ of the orthogonal matrix $O^*$ for $\mathcal{M}_{base}$ vs. $\mathcal{M}_{reason}$ across different datasets. Higher inverse row entropy indicates more axis-aligned correspondence, while lower inverse row entropy indicates more globally mixed features. The model pairs are separated by the algorithm used to train their reasoning model $\mathcal{M}_{reason}$.}

    {\large \textbf{Orthogonal Procrustes Analysis}} \par\medskip
    {\large \textbf{Base Models $\mathcal{M}_{base}$ vs. Reasoning Models $\mathcal{M}_{reason}$}} \par\medskip

    \sectionbox{Dataset: CoT Dataset}
    \vspace{1ex} \par\medskip
    \begin{tabular}{@{}lc@{}}
        \toprule
        \textbf{Model Pair} & \textbf{$H_{\text{inv}}\, \uparrow$} \\
        \midrule
        \textbf{SFT} & \\
        Qwen2.5-Math-1.5B vs DeepSeek-R1-Distill-Qwen-1.5B & 0.1076 \\
        \midrule
        \textbf{RLVR} & \\
        Qwen3-4B vs Polaris-4B-Preview & 0.4365 \\
        DeepSeek-R1-Distill-Qwen-7B vs Polaris-7B-Preview & 0.6576 \\
        Qwen2.5-7B vs zero\_\_ppo\_\_think\_\_Qwen2.5-7B & 0.6338 \\
        Qwen2.5-1.5B vs Qwen-2.5-1.5B-SimpleRL-Zoo & 0.9481 \\
        Qwen2.5-0.5B vs Qwen-2.5-0.5B-SimpleRL-Zoo & 0.9923 \\
        DeepSeek-R1-Distill-Qwen-1.5B vs Nemotron-Research-Reasoning-Qwen-1.5B & 0.1613 \\
        Qwen3-4B vs Qwen3-4B-PSR & 0.8122 \\
        \bottomrule
    \end{tabular}

    \vspace{2em}

    \sectionbox{Dataset: MMLU-Pro}
    \vspace{1ex} \par\medskip
    \begin{tabular}{@{}lcc@{}}
        \toprule
        \textbf{Model Pair} & \textbf{$H_{\text{inv}}\, \uparrow$} \\
        \midrule
        \textbf{SFT} & \\
        Qwen2.5-Math-1.5B vs DeepSeek-R1-Distill-Qwen-1.5B & 0.2229 \\
        \midrule
        \textbf{RLVR} & \\
        Qwen3-4B vs Polaris-4B-Preview & 0.9711 \\
        DeepSeek-R1-Distill-Qwen-7B vs Polaris-7B-Preview & 0.9623 \\
        Qwen2.5-7B vs zero\_\_ppo\_\_think\_\_Qwen2.5-7B & 0.9922 \\
        Qwen2.5-1.5B vs Qwen-2.5-1.5B-SimpleRL-Zoo & 0.9963 \\
        Qwen2.5-0.5B vs Qwen-2.5-0.5B-SimpleRL-Zoo & 0.9981 \\
        DeepSeek-R1-Distill-Qwen-1.5B vs Nemotron-Research-Reasoning-Qwen-1.5B & 0.8336 \\
        Qwen3-4B vs Qwen3-4B-PSR & 0.9843 \\
        \bottomrule
    \end{tabular}
    \label{tab: appendix-orthogonal-base-vs-reasoning}
\end{table*}
% ================================================================================================ 
% Orthogonal Matrix $O$ 
% Base Embedding Model vs. Reasoning Embedding Model
%================================================================================================ 

\begin{table*}[p]
    \centering
    \setlength{\tabcolsep}{2.5pt}
    \renewcommand{\arraystretch}{1.15}
    
    \caption{Inverse row entropy $H_{\text{inv}}$ of the orthogonal matrix $O^*$ for $\mathcal{M}_{base}^{Emb}$ vs. $\mathcal{M}_{reason}^{Emb}$ across different datasets. Higher inverse row entropy indicates more axis-aligned correspondence, while lower inverse row entropy indicates more globally mixed features. The model pairs are separated by the algorithm used to train their reasoning model backbone $\mathcal{M}_{reason}$.}

    {\large \textbf{Orthogonal Procrustes Analysis}} \par\medskip
    {\large \textbf{Base Embedding Models $\mathcal{M}_{base}^{Emb}$ vs. Reasoning Embedding Models $\mathcal{M}_{reason}^{Emb}$}} \par\medskip

    \sectionbox{Dataset: CoT Dataset}
    \vspace{1ex} \par\medskip
    \begin{tabular}{@{}lcc@{}}
        \toprule
        \textbf{Model Pair} & \textbf{ $H_{\text{inv}}\, \uparrow$} \\
        \midrule
        \textbf{SFT} & \\
        Qwen2.5-Math-1.5B-\textit{Emb} vs DeepSeek-R1-Distill-Qwen-1.5B-\textit{Emb} & 0.1429 \\
        Qwen3-0.6B-Base-\textit{Emb} vs Qwen3-0.6B-\textit{Emb} & 0.4915 \\
        \midrule
        \textbf{RLVR} & \\
        Qwen2.5-1.5B-\textit{Emb} vs Qwen-2.5-1.5B-SimpleRL-Zoo-\textit{Emb} & 0.8826 \\
        Qwen2.5-0.5B-\textit{Emb} vs Qwen-2.5-0.5B-SimpleRL-Zoo-\textit{Emb} & 0.9835 \\
        DeepSeek-R1-Distill-Qwen-\textit{Emb} vs Nemotron-Research-Reasoning-Qwen-\textit{Emb} & 0.8637 \\
        Qwen3-4B-\textit{Emb} vs Qwen3-4B-PSR-\textit{Emb} & 0.5863 \\
        \bottomrule
    \end{tabular}

    \vspace{2em}

    \sectionbox{Dataset: MMLU-Pro}
    \vspace{1ex} \par\medskip
    \begin{tabular}{@{}lcc@{}}
        \toprule
        \textbf{Model Pair} & \textbf{$H_{\text{inv}}\, \uparrow$} \\
        \midrule
        \textbf{SFT} & \\
        Qwen2.5-Math-1.5B-\textit{Emb} vs DeepSeek-R1-Distill-Qwen-1.5B-\textit{Emb} & 0.7105 \\
        Qwen3-0.6B-Base-\textit{Emb} vs Qwen3-0.6B-\textit{Emb} & 0.9164 \\
        \midrule
        \textbf{RLVR} & \\
        Qwen2.5-1.5B-\textit{Emb} vs Qwen-2.5-1.5B-SimpleRL-Zoo-\textit{Emb} & 0.9794 \\
        Qwen2.5-0.5B-\textit{Emb} vs Qwen-2.5-0.5B-SimpleRL-Zoo-\textit{Emb} & 0.9978 \\
        DeepSeek-R1-Distill-Qwen-\textit{Emb} vs Nemotron-Research-Reasoning-Qwen-\textit{Emb} & 0.9814 \\
        Qwen3-4B-\textit{Emb} vs Qwen3-4B-PSR-\textit{Emb} & 0.9933 \\
        \bottomrule
    \end{tabular}
    \label{tab: appendix-orthogonal-base-embedding-vs-reasoning-embedding}
\end{table*}

% ================================================================================================ 
% Linear CKA 
% Base Model vs. Reasoning Model
%================================================================================================ 

\begin{figure*}[p]
    \centering
    {\large \textbf{Linear CKA}} \par\smallskip
    {\large \textbf{Base Models $\mathcal{M}_{base}$ vs. Reasoning Models $\mathcal{M}_{reason}$}} \par\medskip

    % --- Configuration ---
    \setlength{\tabcolsep}{1pt}

    \sectionbox{Dataset: CoT Datset}
    \vspace{1ex} % Adds a small vertical space for better separation

    % Center the main content grid and ensure it does not exceed the text width
    \makebox[\textwidth][c]{%
    \begin{tabular}{
        c @{\hspace{1pt}} c  @{\hspace{0.5em}}
        c @{\hspace{1pt}} c  @{\hspace{0.5em}}
        c @{\hspace{1pt}} c  @{\hspace{0.5em}}
        c @{\hspace{1pt}} c
    }
        % --- Row 1 ---
        % === MODIFIED CELL START ===
        % We wrap the minipage in a colorbox. 
        % \fboxsep controls the padding between the color edge and the image.
        \rotlabel{Qwen2.5-Math-1.5B} &
        \setlength{\fboxsep}{3pt}% 
        \colorbox{red!20}{%  <-- CHANGE COLOR HERE (e.g., yellow!20, blue!10)
            \begin{minipage}{0.20\textwidth}
                \centering
                \includegraphics[width=\linewidth]{figures/re_experiments/cka_evaluation/Qwen2.5-Math-1.5B_vs_DeepSeek-R1-Distill-Qwen-1.5B/qwen3_32b_LiveMathbench_evaluated/cka_heatmap_Qwen2.5-Math-1.5B_vs_DeepSeek-R1-Distill-Qwen-1.5B_processed.png}
                \par\vspace{2pt} % Small space between image and caption
                \scriptsize \textbf{DeepSeek-R1-Distill-Qwen-1.5B}
            \end{minipage}%
        } &

        \rotlabel{Qwen3-4B} &
        \begin{minipage}{0.20\textwidth} % Adjusted width
            \includegraphics[width=\linewidth]{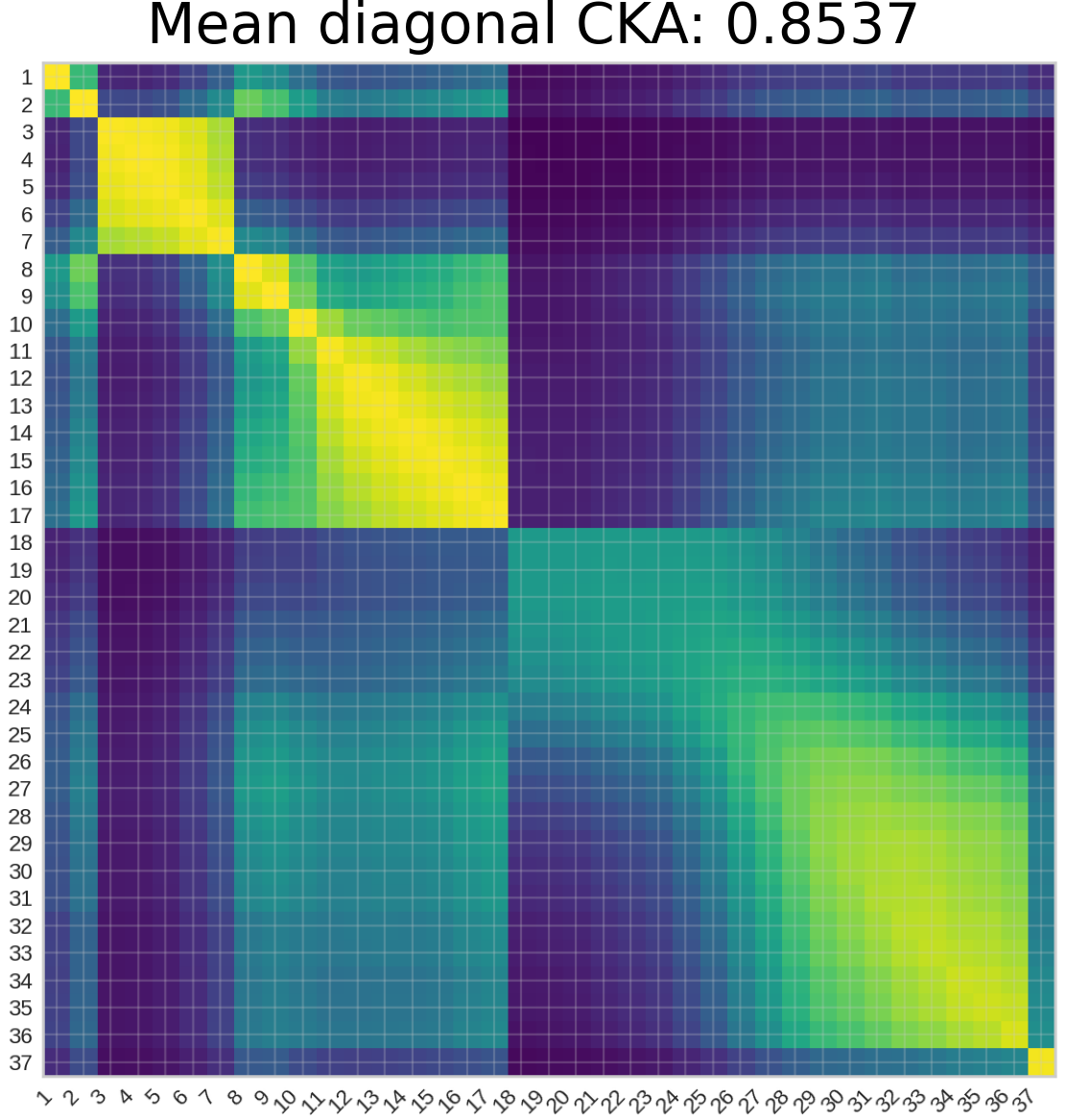} \\
            \centering \scriptsize \textbf{Polaris-4B-Preview}
        \end{minipage} &

        \rotlabel{DeepSeek-R1-Distill-Qwen-7B} &
        \begin{minipage}{0.20\textwidth} % Adjusted width
            \includegraphics[width=\linewidth]{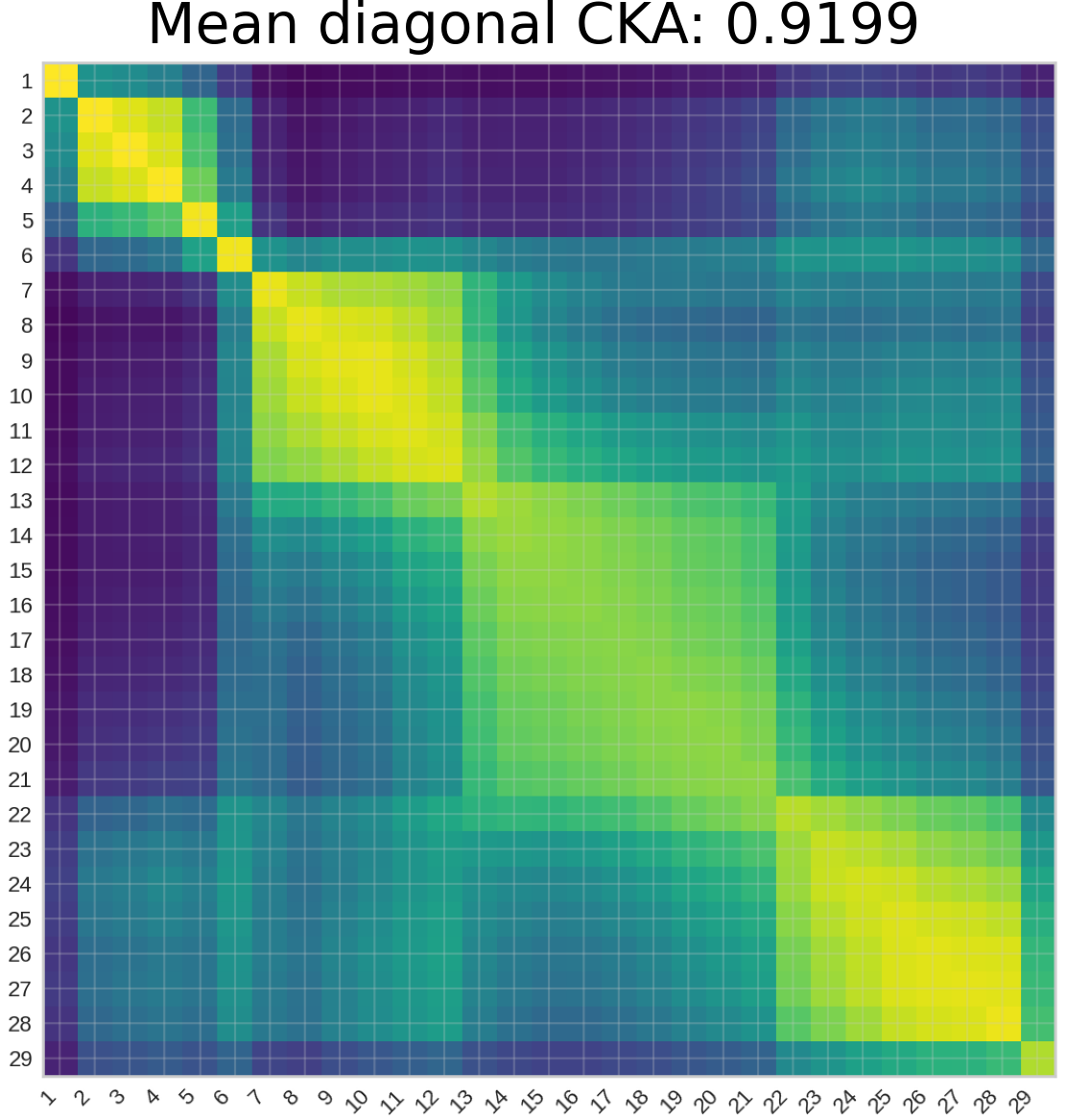} \\
            \centering \scriptsize \textbf{Polaris-7B-Preview}
        \end{minipage} &

        \rotlabel{Qwen2.5-7B} &
        \begin{minipage}{0.20\textwidth} % Adjusted width
            \includegraphics[width=\linewidth]{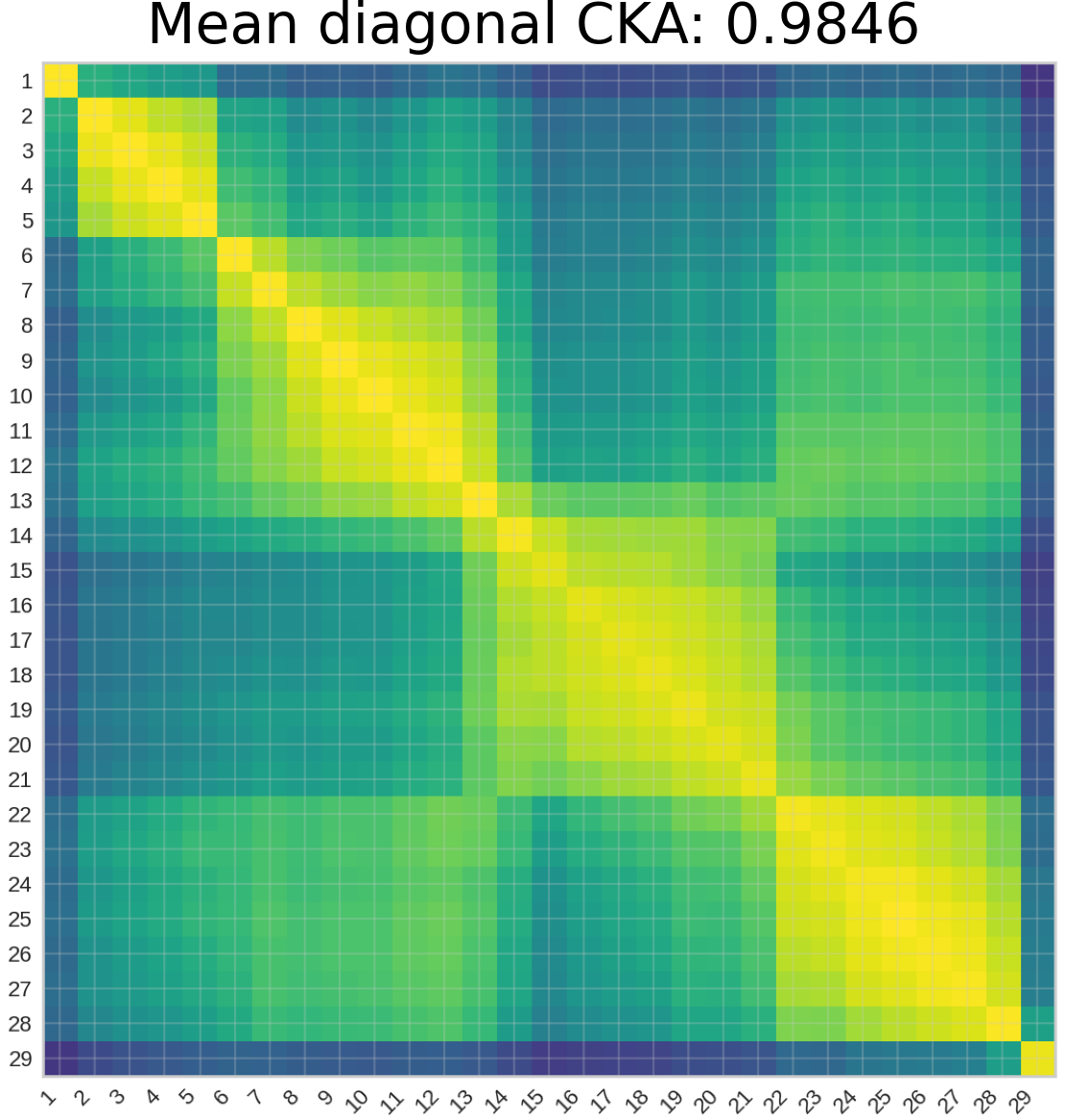} \\
            \centering \scriptsize \textbf{zero\_\_ppo\_\_think\_\_Qwen2.5-7B}
        \end{minipage} \\

        % \multicolumn{8}{c}{\vspace{0.1em}} \\ % Spacing between rows

        % --- Row 2 ---
        \rotlabel{Qwen2.5-1.5B} &
        \begin{minipage}{0.20\textwidth} % Adjusted width
            \includegraphics[width=\linewidth]{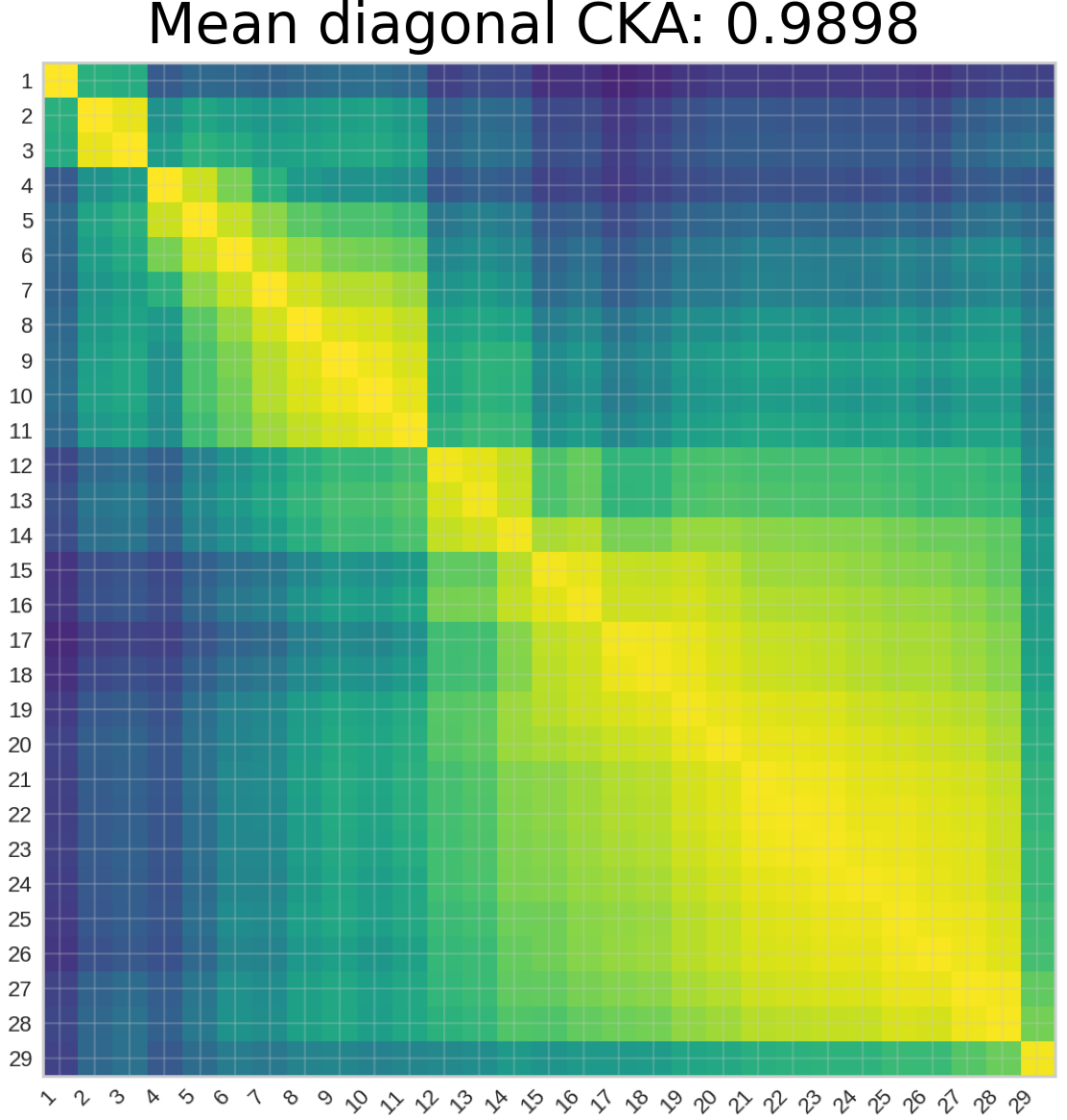} \\
            \centering \scriptsize \textbf{Qwen-2.5-1.5B-SimpleRL-Zoo}
        \end{minipage} &

        \rotlabel{Qwen2.5-0.5B} &
        \begin{minipage}{0.20\textwidth} % Adjusted width
            \includegraphics[width=\linewidth]{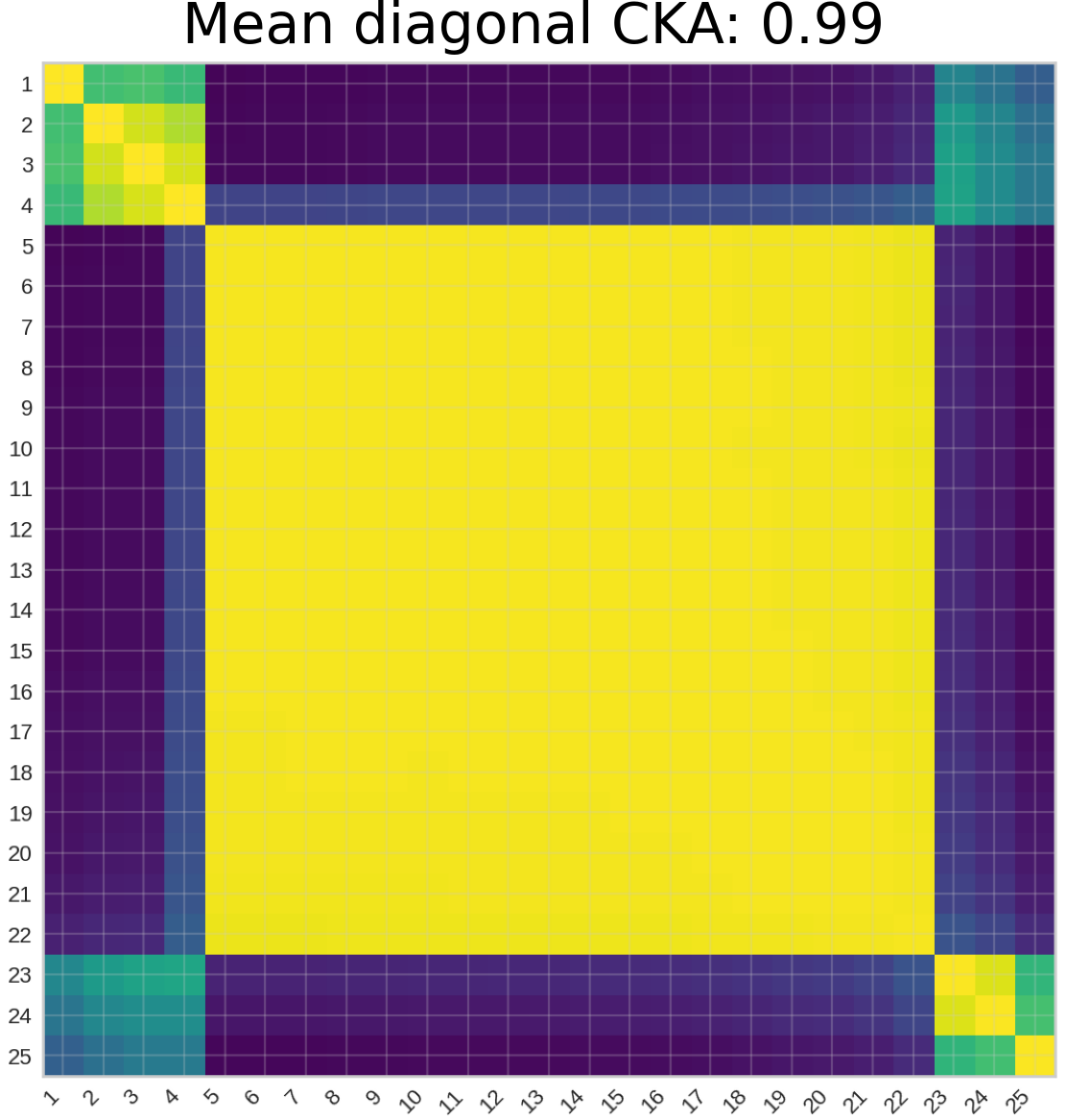} \\
            \centering \scriptsize \textbf{Qwen-2.5-0.5B-SimpleRL-Zoo}
        \end{minipage} &

        \rotlabel{DeepSeek-R1-Distill-Qwen-1.5B} &
        \begin{minipage}{0.20\textwidth} % Adjusted width
            \includegraphics[width=\linewidth]{figures/re_experiments/cka_evaluation/DeepSeek-R1-Distill-Qwen-1.5B_vs_Nemotron-Research-Reasoning-Qwen-1.5B/qwen3_32b_LiveMathbench_evaluated/cka_heatmap_DeepSeek-R1-Distill-Qwen-1.5B_vs_Nemotron-Research-Reasoning-Qwen-1.5B_processed.png} \\
            \centering \scriptsize \textbf{Nemotron-Research-Reasoning-Qwen-1.5B}
        \end{minipage} &

        \rotlabel{Qwen3-4B} &
        \begin{minipage}{0.20\textwidth} % Adjusted width
            \includegraphics[width=\linewidth]{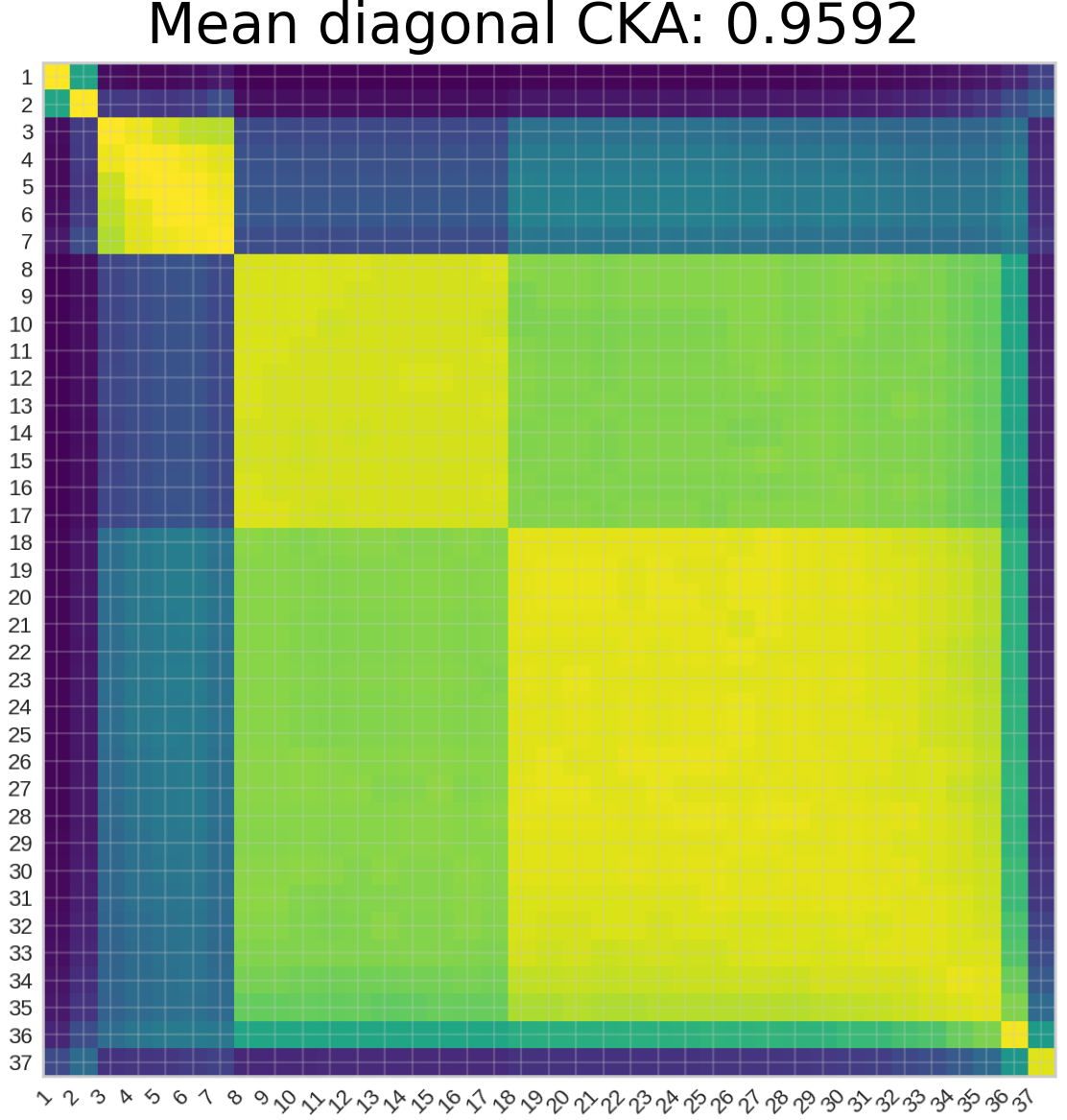} \\
            \centering \scriptsize \textbf{Qwen3-4B-PSR}
        \end{minipage} \\
    \end{tabular}%
    } % End of \makebox

    \vspace{0.5em}
    \rotatebox{-90}{\includegraphics[height=6cm, width=\linewidth, keepaspectratio]{figures/re_experiments/cka_evaluation/cka_heatmap.png}} 
    
    \vspace{1.0em}

    \sectionbox{Dataset: MMLU-Pro}
    \vspace{1ex} % Adds a small vertical space for better separation

    % Center the main content grid and ensure it does not exceed the text width
    \makebox[\textwidth][c]{%
    \begin{tabular}{
        c @{\hspace{1pt}} c  @{\hspace{0.5em}}
        c @{\hspace{1pt}} c  @{\hspace{0.5em}}
        c @{\hspace{1pt}} c  @{\hspace{0.5em}}
        c @{\hspace{1pt}} c
    }
        % --- Row 1 ---
        \rotlabel{Qwen2.5-Math-1.5B} &
        \setlength{\fboxsep}{3pt}% 
        \colorbox{red!20}{%  <-- CHANGE COLOR HERE (e.g., yellow!20, blue!10)
            \begin{minipage}{0.20\textwidth}
                \centering
                \includegraphics[width=\linewidth]{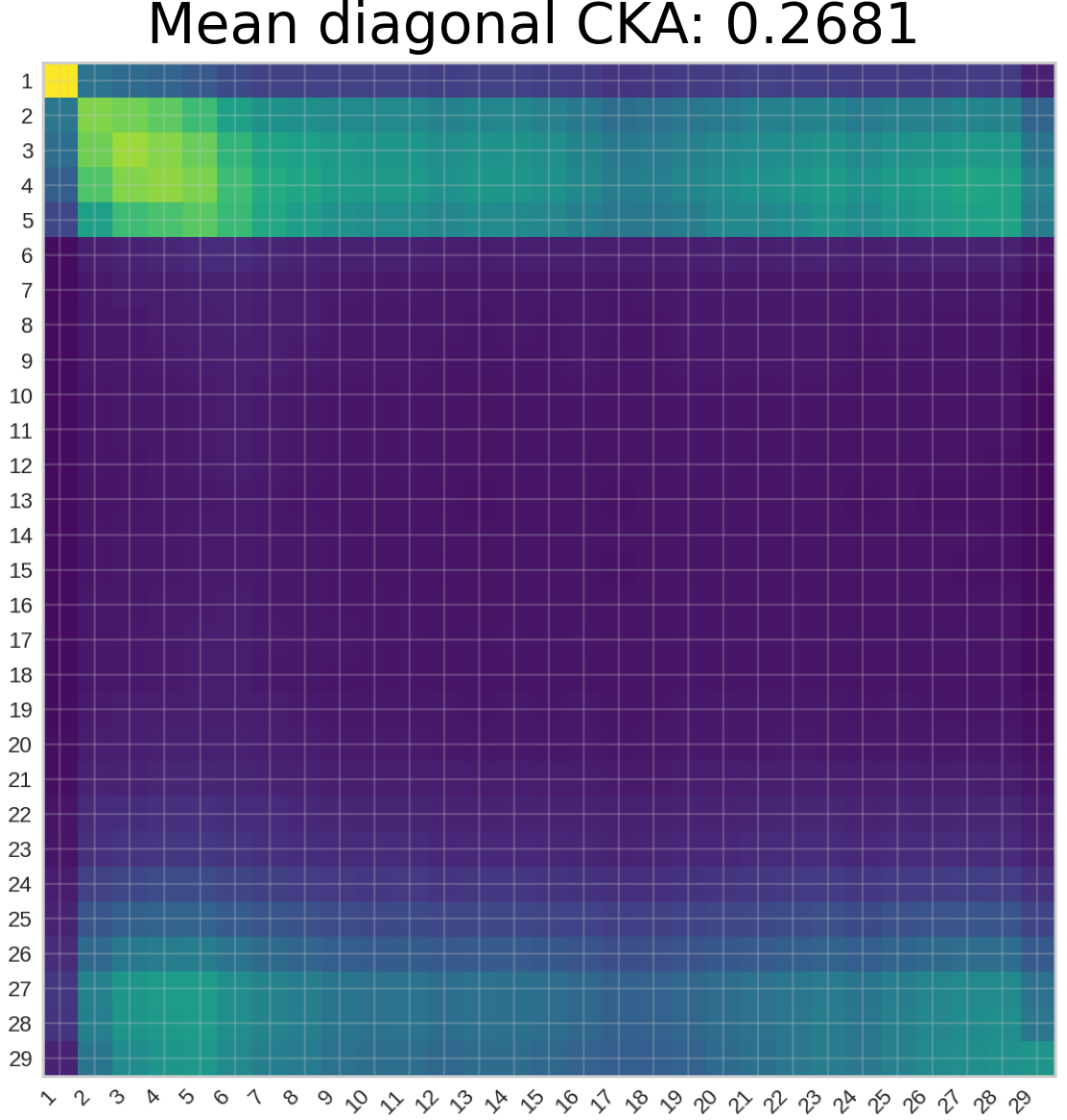}
                \par\vspace{2pt} % Small space between image and caption
                \scriptsize \textbf{DeepSeek-R1-Distill-Qwen-1.5B}
            \end{minipage}%
        } &

        \rotlabel{Qwen3-4B} &
        \begin{minipage}{0.20\textwidth} % Adjusted width
            \includegraphics[width=\linewidth]{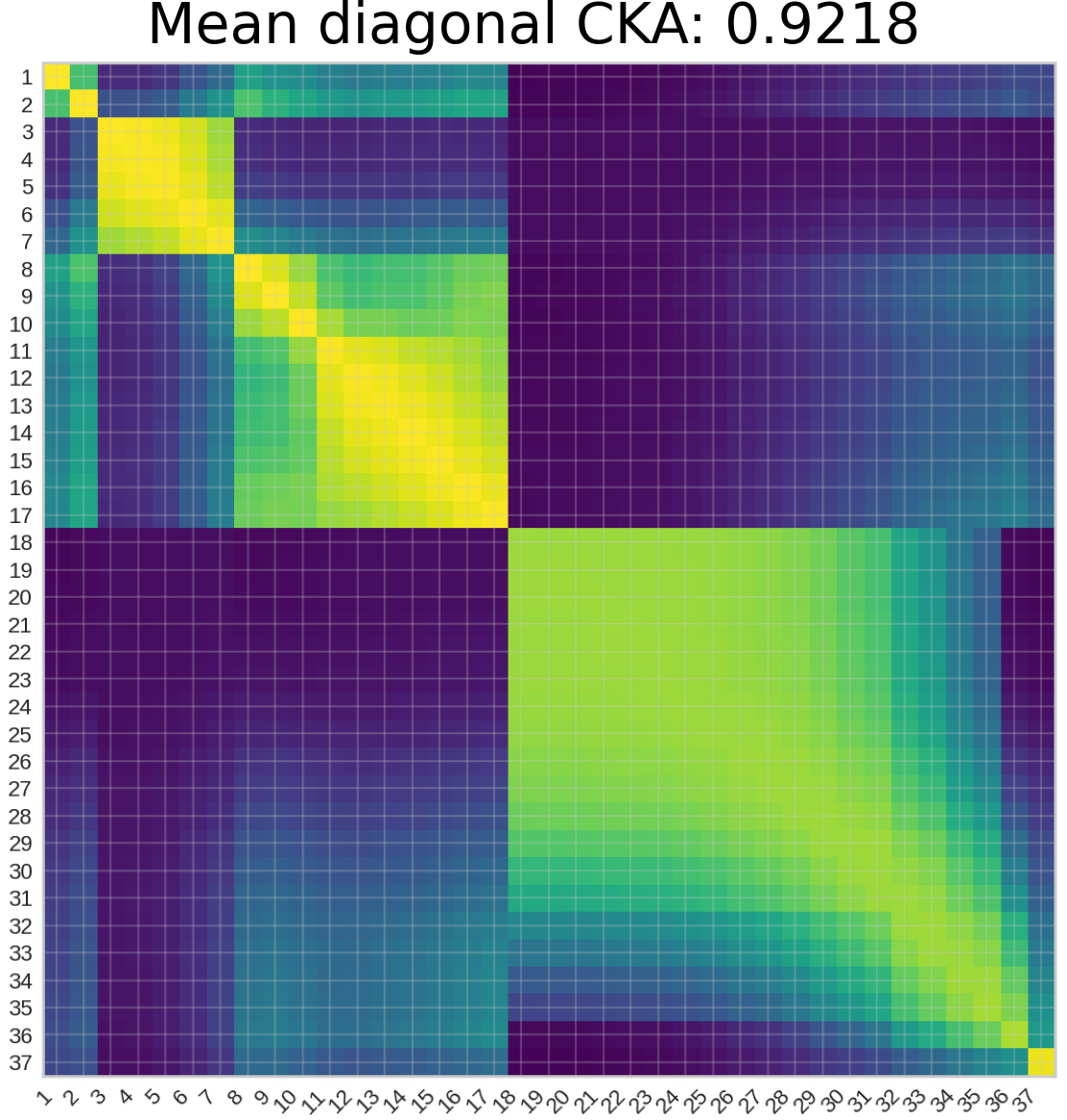} \\
            \centering \scriptsize \textbf{Polaris-4B-Preview}
        \end{minipage} &

        \rotlabel{DeepSeek-R1-Distill-Qwen-7B} &
        \begin{minipage}{0.20\textwidth} % Adjusted width
            \includegraphics[width=\linewidth]{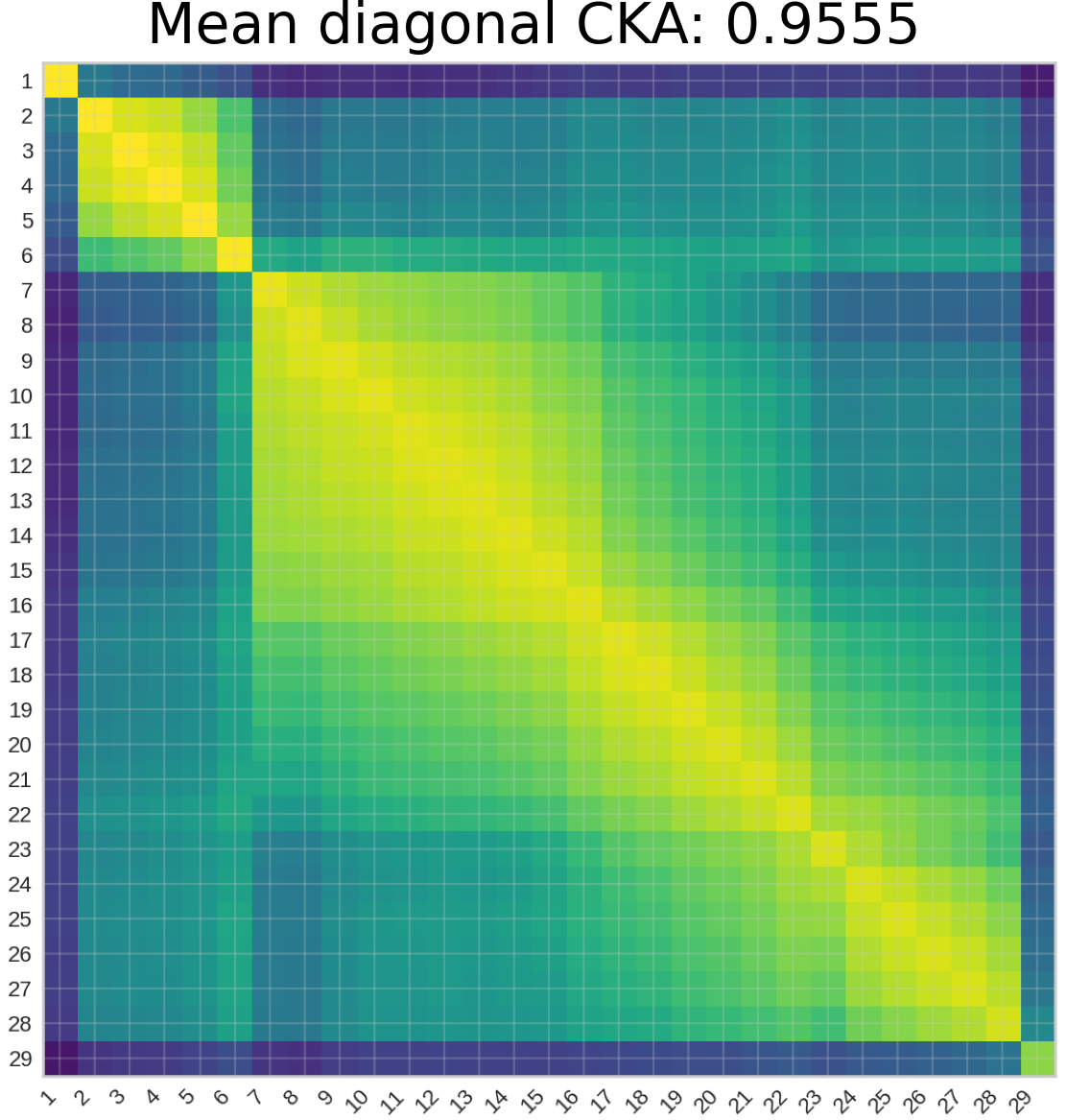} \\
            \centering \scriptsize \textbf{Polaris-7B-Preview}
        \end{minipage} &

        \rotlabel{Qwen2.5-7B} &
        \begin{minipage}{0.20\textwidth} % Adjusted width
            \includegraphics[width=\linewidth]{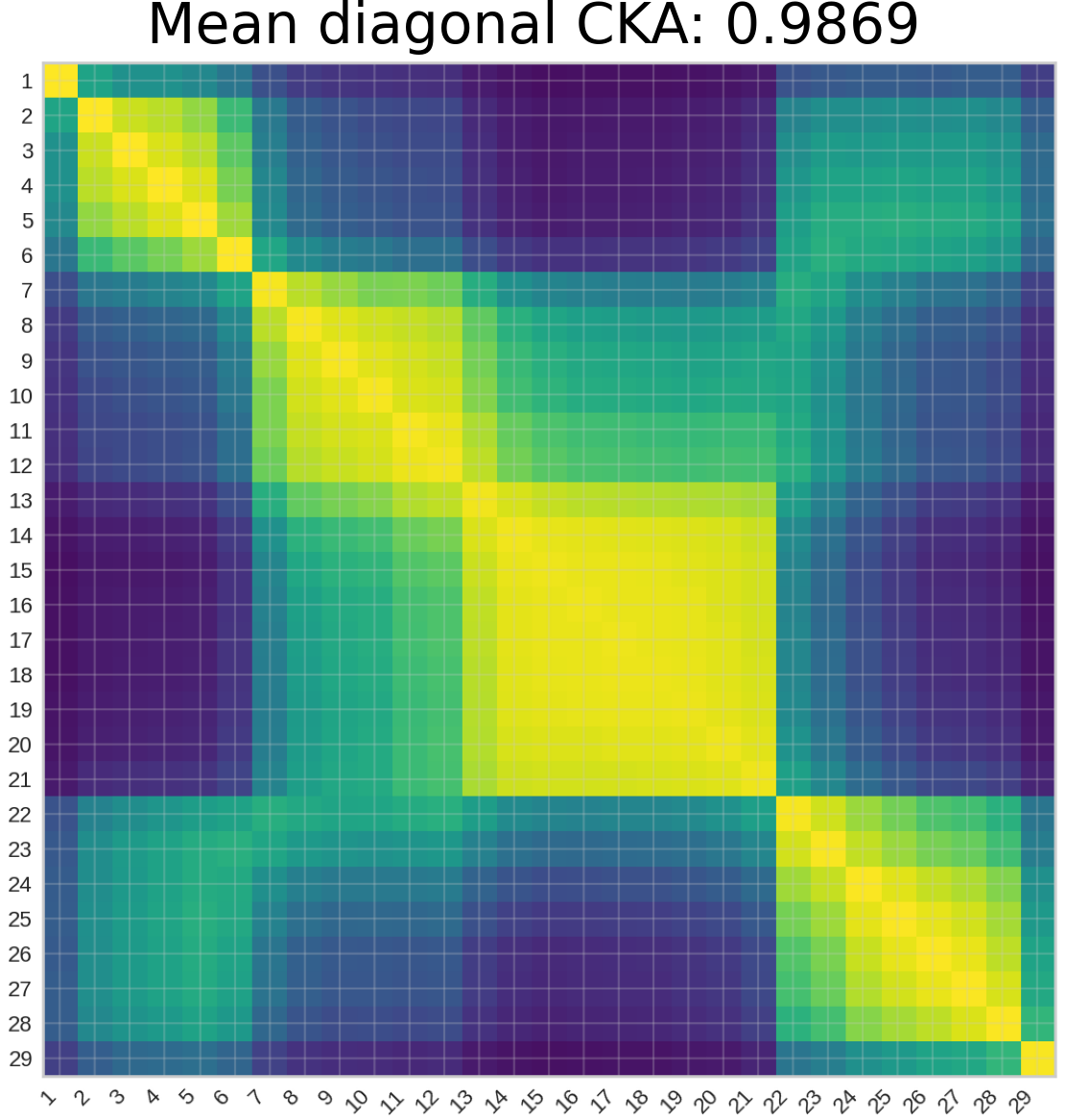} \\
            \centering \scriptsize \textbf{zero\_\_ppo\_\_think\_\_Qwen2.5-7B}
        \end{minipage} \\

        % \multicolumn{8}{c}{\vspace{0.1em}} \\ % Spacing between rows

        % --- Row 2 ---
        \rotlabel{Qwen2.5-1.5B} &
        \begin{minipage}{0.20\textwidth} % Adjusted width
            \includegraphics[width=\linewidth]{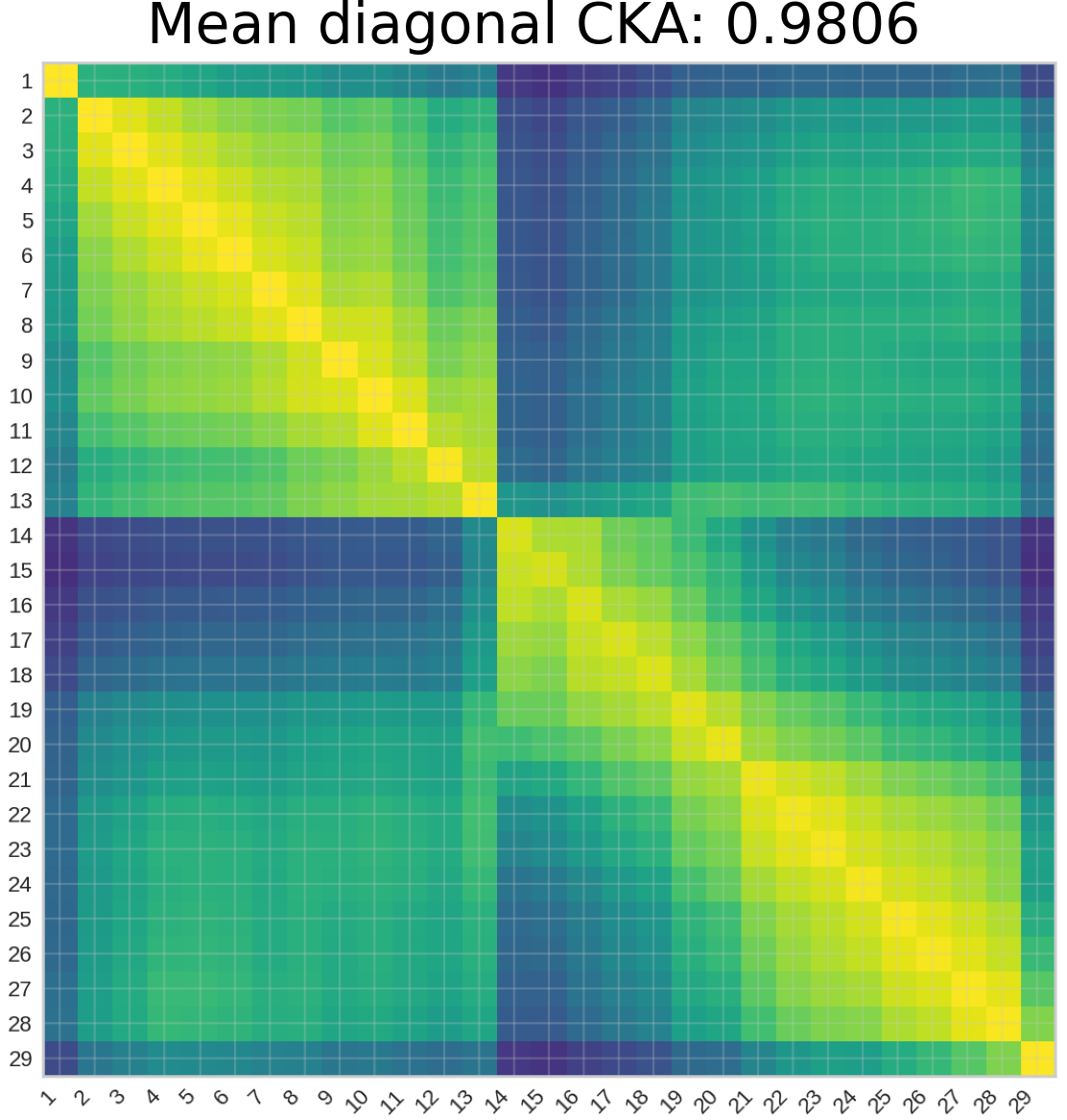} \\
            \centering \scriptsize \textbf{Qwen-2.5-1.5B-SimpleRL-Zoo}
        \end{minipage} &

        \rotlabel{Qwen2.5-0.5B} &
        \begin{minipage}{0.20\textwidth} % Adjusted width
            \includegraphics[width=\linewidth]{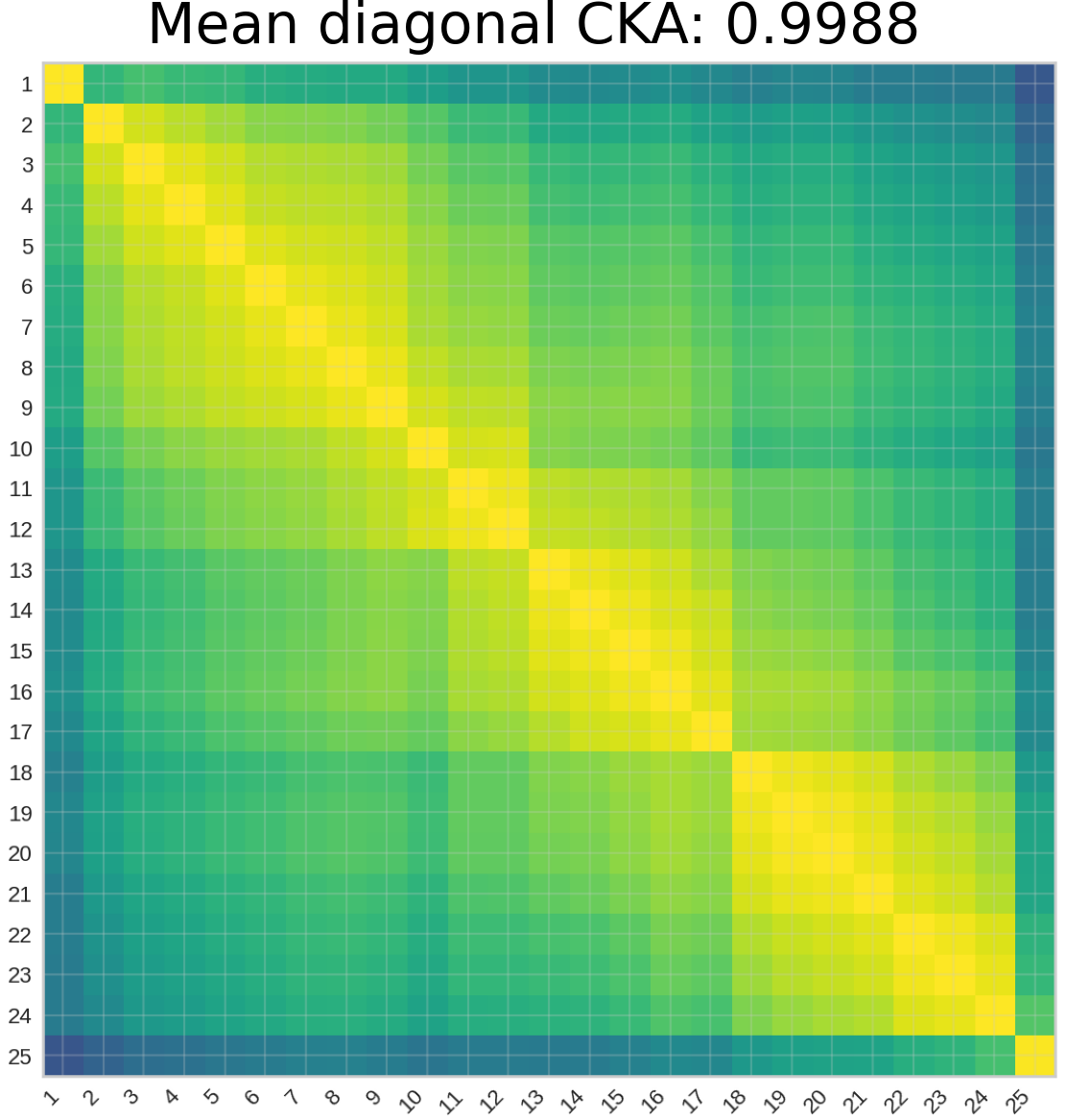} \\
            \centering \scriptsize \textbf{Qwen-2.5-0.5B-SimpleRL-Zoo}
        \end{minipage} &

        \rotlabel{DeepSeek-R1-Distill-Qwen-1.5B} &
        \begin{minipage}{0.20\textwidth} % Adjusted width
            \includegraphics[width=\linewidth]{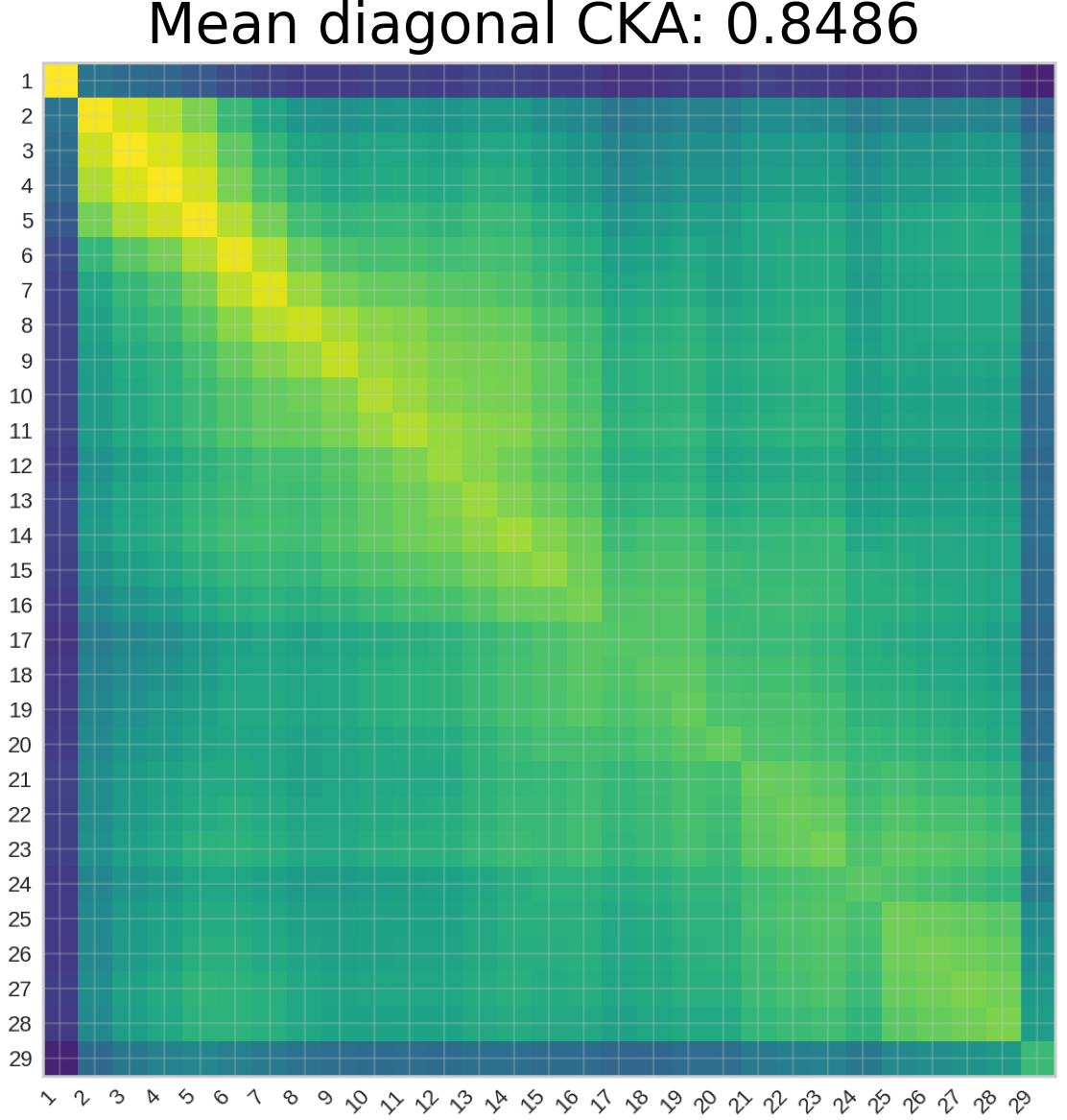} \\
            \centering \scriptsize \textbf{Nemotron-Research-Reasoning-Qwen-1.5B}
        \end{minipage} &

        \rotlabel{Qwen3-4B} &
        \begin{minipage}{0.20\textwidth} % Adjusted width
            \includegraphics[width=\linewidth]{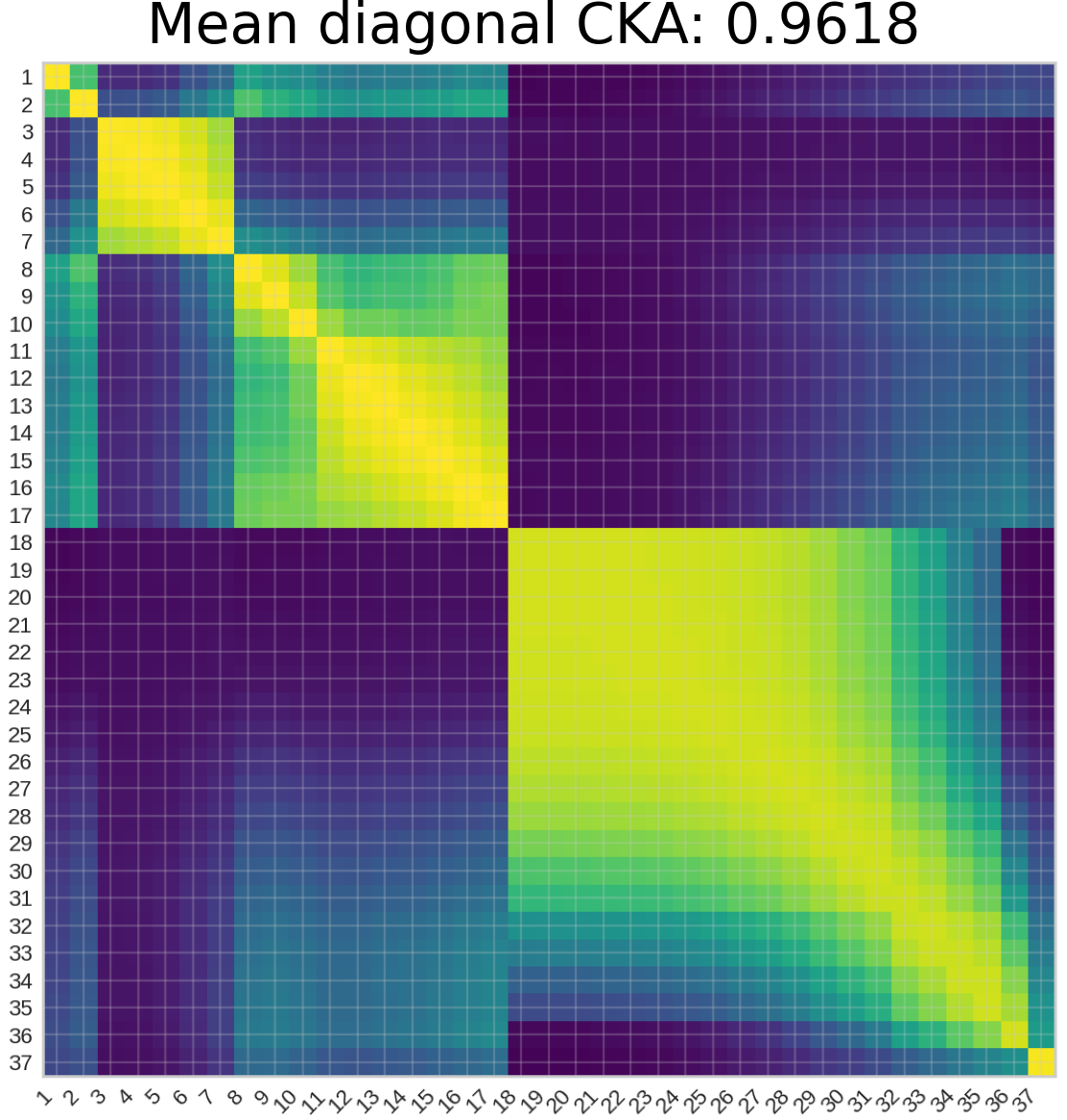} \\
            \centering \scriptsize \textbf{Qwen3-4B-PSR}
        \end{minipage} \\
    \end{tabular}%
    } % End of \makebox

    \vspace{0.5em}
    \rotatebox{-90}{\includegraphics[height=6cm, width=\linewidth, keepaspectratio]{figures/re_experiments/cka_evaluation/cka_heatmap.png}} 

    % The second section and other elements from the original code were commented out or misplaced,
    % and have been removed for clarity and to focus on the main table.

    \caption{Additional Results on Linear CKA separated by dataset. The vertical axis and horizontal axis are Base Model Layer Index and Reasoning Model Layer Index, respectively. The \textbf{\textcolor{red}{red}} background indicates SFT-tuned pairs.}
    \label{fig: appendix-corr-base-vs-reasoning}
\end{figure*}
% ================================================================================================ 
% Linear CKA 
% Base Embedding Model vs. Reasoning Embedding Model
%================================================================================================ 

\begin{figure*}[p]
    \centering
    {\large \textbf{Linear CKA}} \par\smallskip
    {\large \textbf{Base Embedding Models $\mathcal{M}_{base}^{Emb}$ vs. Reasoning Embedding Models $\mathcal{M}_{reason}^{Emb}$}} \par\medskip

    % --- Configuration ---
    \setlength{\tabcolsep}{1pt}

    \sectionbox{Dataset: CoT Datset}
    \vspace{1ex} % Adds a small vertical space for better separation

    % Center the main content grid and ensure it does not exceed the text width
    \makebox[\textwidth][c]{%
    \begin{tabular}{
        c @{\hspace{1pt}} c  @{\hspace{0.5em}}
        c @{\hspace{1pt}} c  @{\hspace{0.5em}}
        c @{\hspace{1pt}} c
    }
        % --- Row 1 ---
        \rotlabel{Qwen2.5-Math-1.5B-\textit{Emb}} &
        \setlength{\fboxsep}{3pt}% 
        \colorbox{red!20}{%  <-- CHANGE COLOR HERE (e.g., yellow!20, blue!10)
            \begin{minipage}{0.18\textwidth}
                \centering
                \includegraphics[width=\linewidth]{figures/re_experiments/cka_evaluation/Qwen2.5-Math-1.5B-Reasoning-Embedding_vs_DeepSeek-R1-Distill-Qwen-1.5B-Reasoning-Embedding/qwen3_32b_LiveMathbench_evaluated/cka_heatmap_Qwen2.5-Math-1.5B-Reasoning-Embedding_vs_DeepSeek-R1-Distill-Qwen-1.5B-Reasoning-Embedding_processed.png}
                \par\vspace{2pt} % Small space between image and caption
                \scriptsize \textbf{DeepSeek-R1-Distill-Qwen-1.5B-\textit{Emb}}
            \end{minipage}%
        } &

        \rotlabel{Qwen3-0.6B-Base-\textit{Emb}} &
        \setlength{\fboxsep}{3pt}% 
        \colorbox{red!20}{%  <-- CHANGE COLOR HERE (e.g., yellow!20, blue!10)
            \begin{minipage}{0.18\textwidth}
                \centering
                \includegraphics[width=\linewidth]{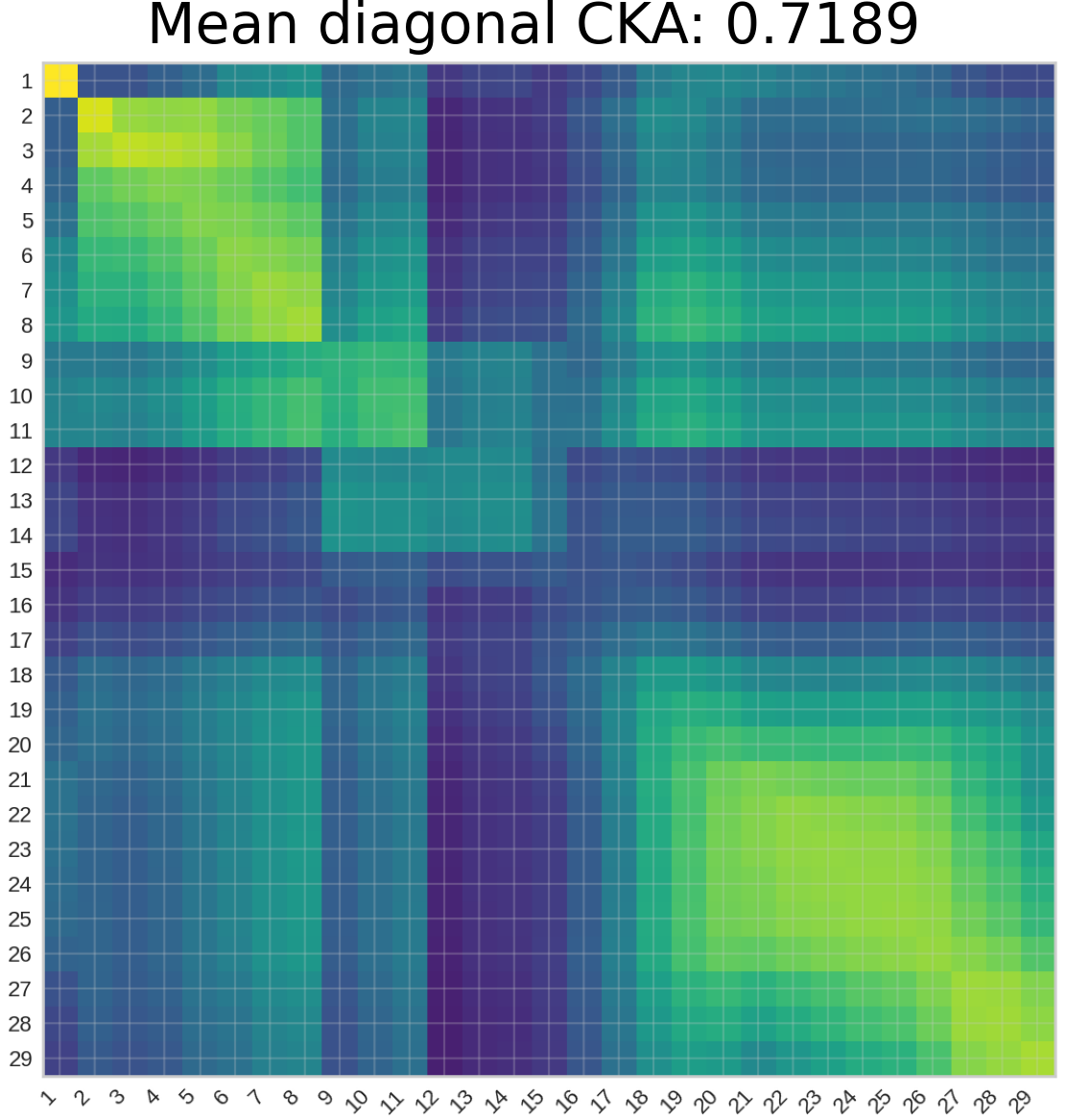}
                \par\vspace{2pt} % Small space between image and caption
                \scriptsize \textbf{Qwen3-0.6B-\textit{Emb}}
            \end{minipage}%
        } &
    
        \rotlabel{Qwen2.5-1.5B-\textit{Emb}} &
        \begin{minipage}{0.18\textwidth} % Adjusted width
            \includegraphics[width=\linewidth]{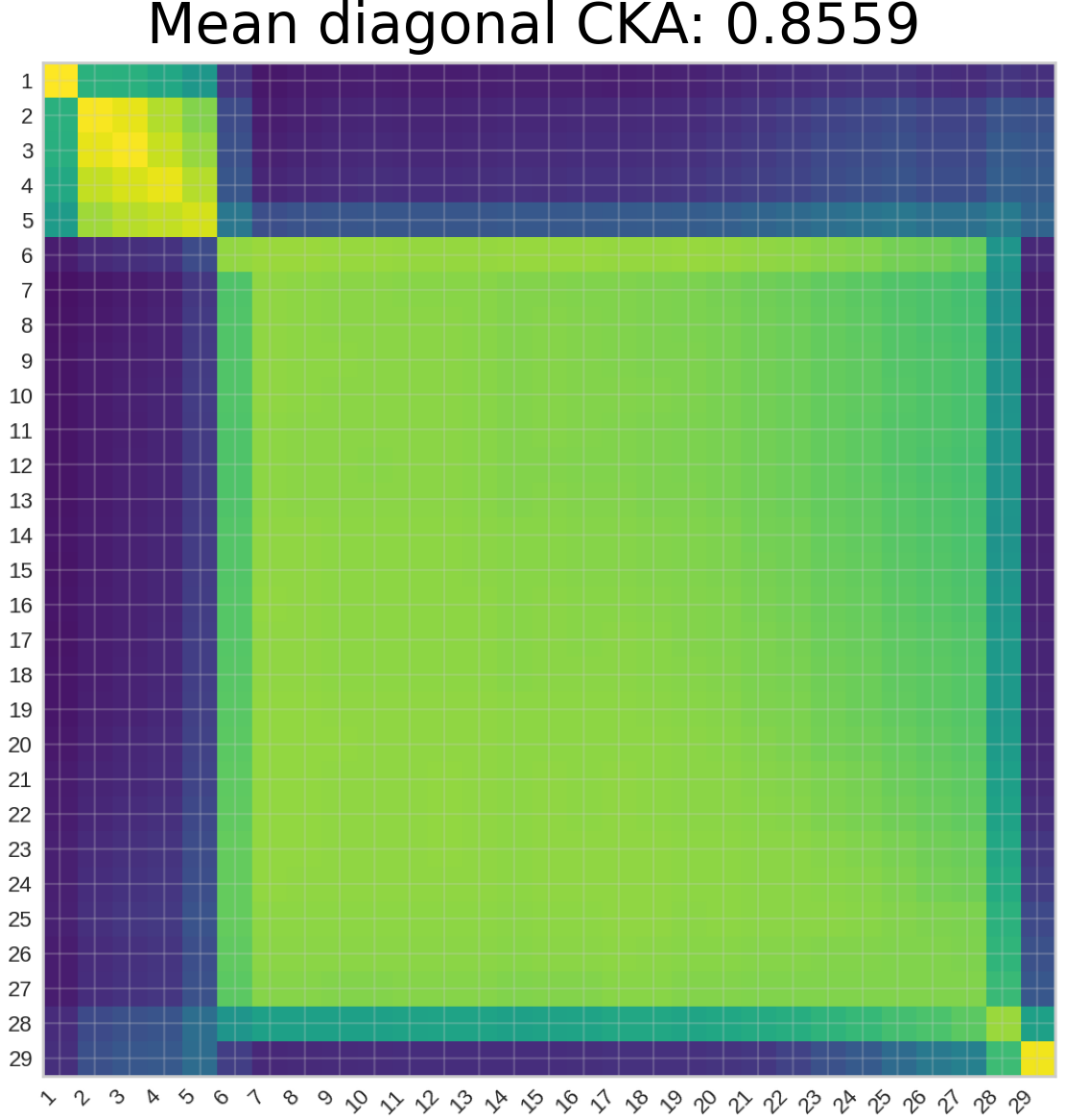} \\
            \centering \scriptsize \textbf{Qwen-2.5-1.5B-SimpleRL-Zoo-\textit{Emb}}
        \end{minipage} \\

        % \multicolumn{6}{c}{\vspace{0.1em}} \\ % Spacing between rows

        % --- Row 2 ---
        \rotlabel{Qwen2.5-0.5B-\textit{Emb}} &
        \begin{minipage}{0.18\textwidth} % Adjusted width
            \includegraphics[width=\linewidth]{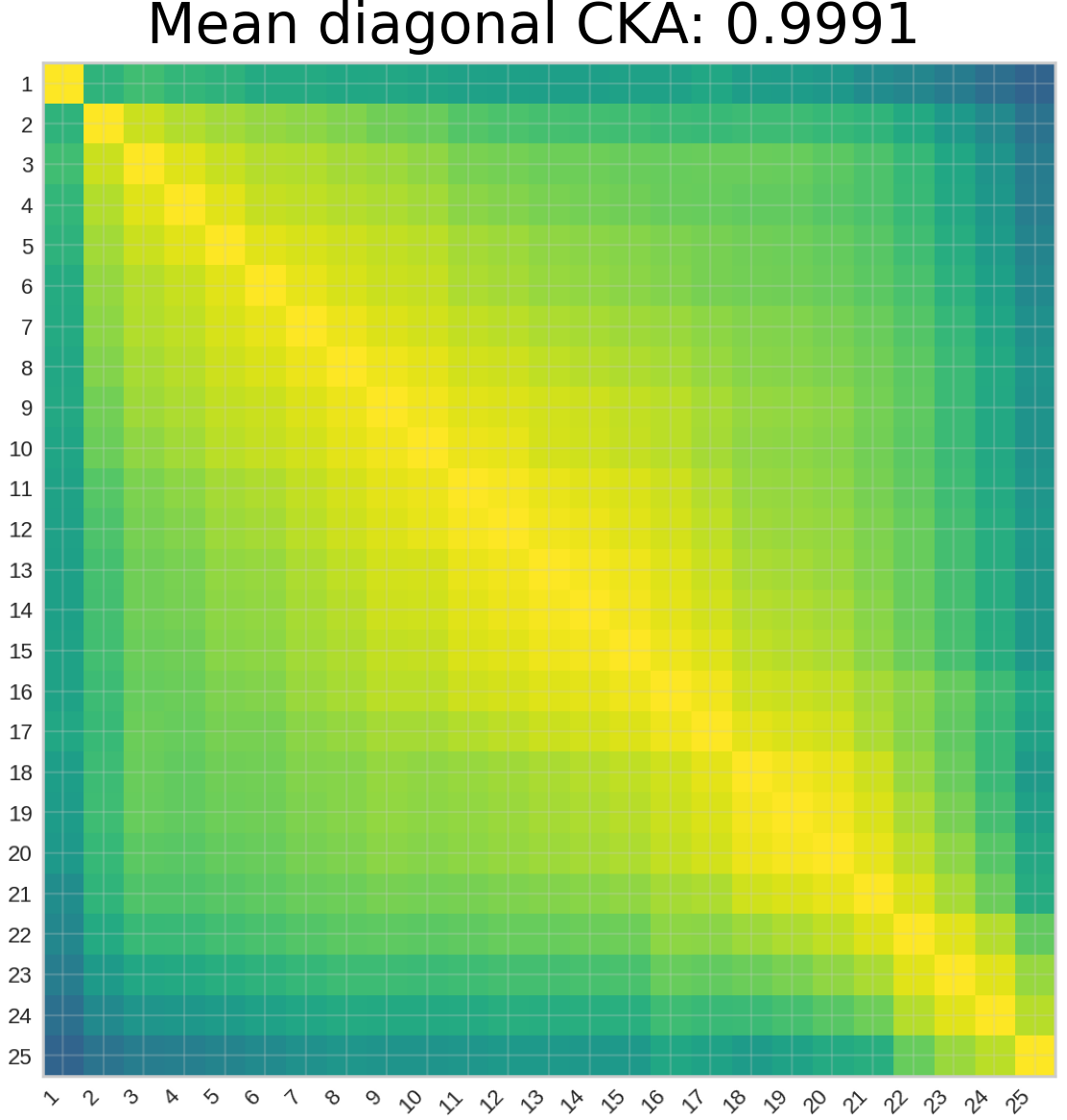} \\
            \centering \scriptsize \textbf{Qwen-2.5-0.5B-SimpleRL-Zoo-\textit{Emb}}
        \end{minipage} &

        \rotlabel{DeepSeek-R1-Distill-Qwen-1.5B-\textit{Emb}} &
        \begin{minipage}{0.18\textwidth} % Adjusted width
            \includegraphics[width=\linewidth]{figures/re_experiments/cka_evaluation/DeepSeek-R1-Distill-Qwen-1.5B-Reasoning-Embedding_vs_Nemotron-Research-Reasoning-Qwen-1.5B-Reasoning-Embedding/qwen3_32b_LiveMathbench_evaluated/cka_heatmap_DeepSeek-R1-Distill-Qwen-1.5B-Reasoning-Embedding_vs_Nemotron-Research-Reasoning-Qwen-1.5B-Reasoning-Embedding_processed.png} \\
            \centering \scriptsize \textbf{Nemotron-Research-Reasoning-Qwen-1.5B-\textit{Emb}}
        \end{minipage} &

        \rotlabel{Qwen3-4B-\textit{Emb}} &
        \begin{minipage}{0.18\textwidth} % Adjusted width
            \includegraphics[width=\linewidth]{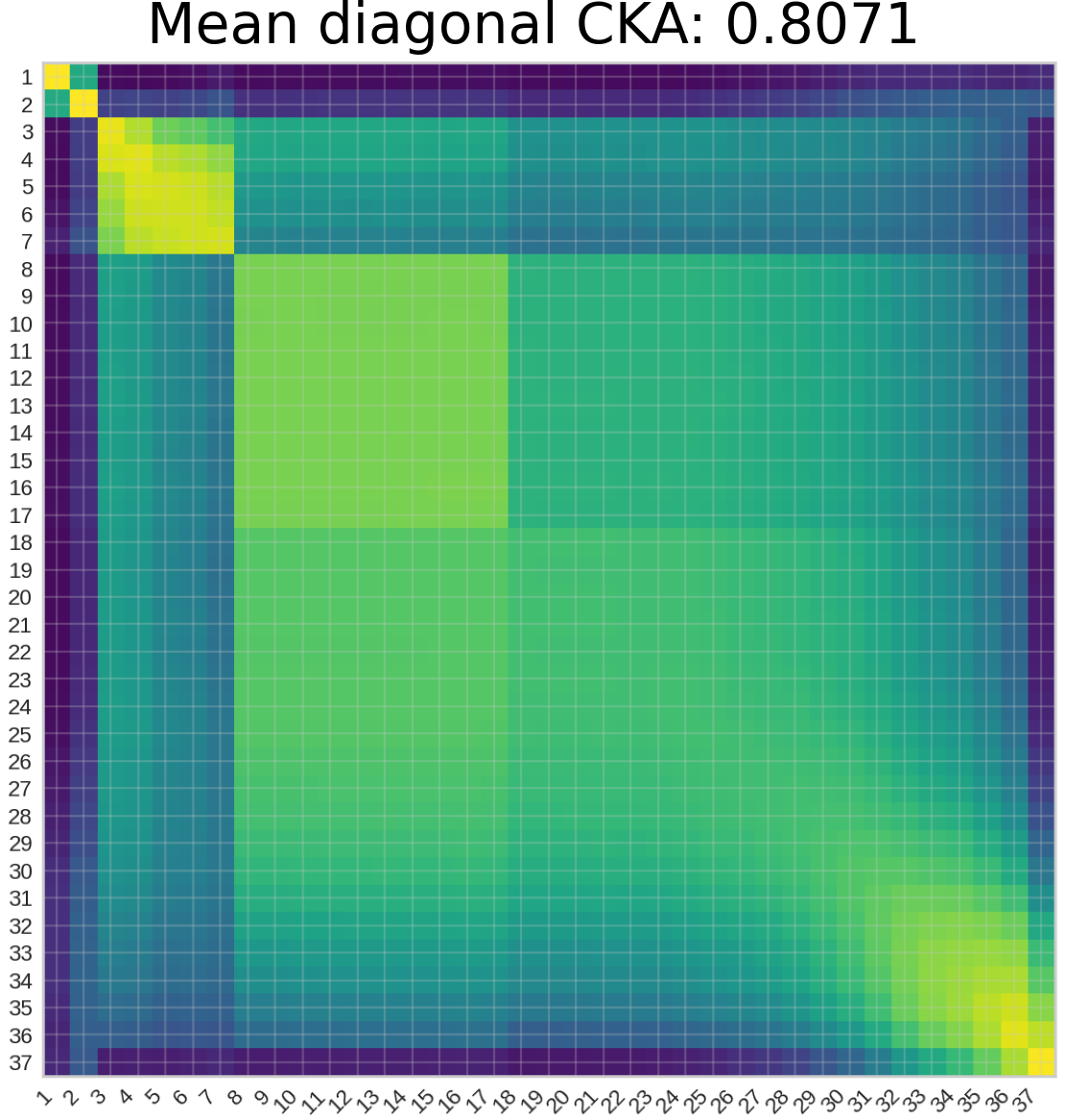} \\
            \centering \scriptsize \textbf{Qwen3-4B-PSR-\textit{Emb}}
        \end{minipage} \\
    \end{tabular}%
    } % End of \makebox

    \vspace{0.5em}
    \rotatebox{-90}{\includegraphics[height=6cm, width=\linewidth, keepaspectratio]{figures/re_experiments/cka_evaluation/cka_heatmap.png}} 
    
    \vspace{1.0em}

    \sectionbox{Dataset: MMLU-Pro}
    \vspace{1ex} % Adds a small vertical space for better separation

    % Center the main content grid and ensure it does not exceed the text width
    \makebox[\textwidth][c]{%
    \begin{tabular}{
        c @{\hspace{1pt}} c  @{\hspace{0.5em}}
        c @{\hspace{1pt}} c  @{\hspace{0.5em}}
        c @{\hspace{1pt}} c
    }
        % --- Row 1 ---
        \rotlabel{Qwen2.5-Math-1.5B-\textit{Emb}} &
        \setlength{\fboxsep}{3pt}% 
        \colorbox{red!20}{%  <-- CHANGE COLOR HERE (e.g., yellow!20, blue!10)
            \begin{minipage}{0.18\textwidth}
                \centering
                \includegraphics[width=\linewidth]{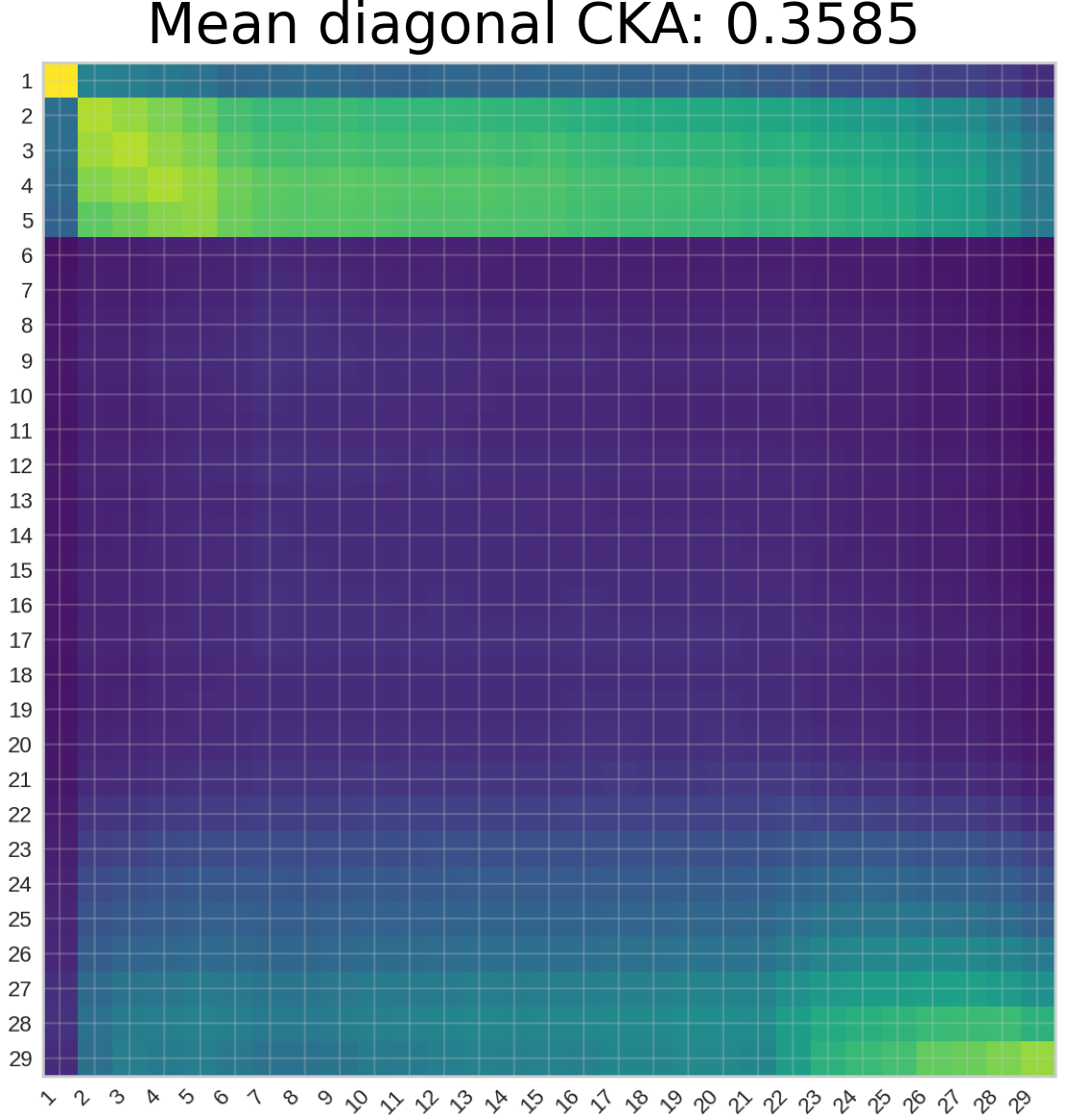}
                \par\vspace{2pt} % Small space between image and caption
                \scriptsize \textbf{DeepSeek-R1-Distill-Qwen-1.5B-\textit{Emb}}
            \end{minipage}%
        } &

        \rotlabel{Qwen3-0.6B-Base-\textit{Emb}} &
        \setlength{\fboxsep}{3pt}% 
        \colorbox{red!20}{%  <-- CHANGE COLOR HERE (e.g., yellow!20, blue!10)
            \begin{minipage}{0.18\textwidth}
                \centering
                \includegraphics[width=\linewidth]{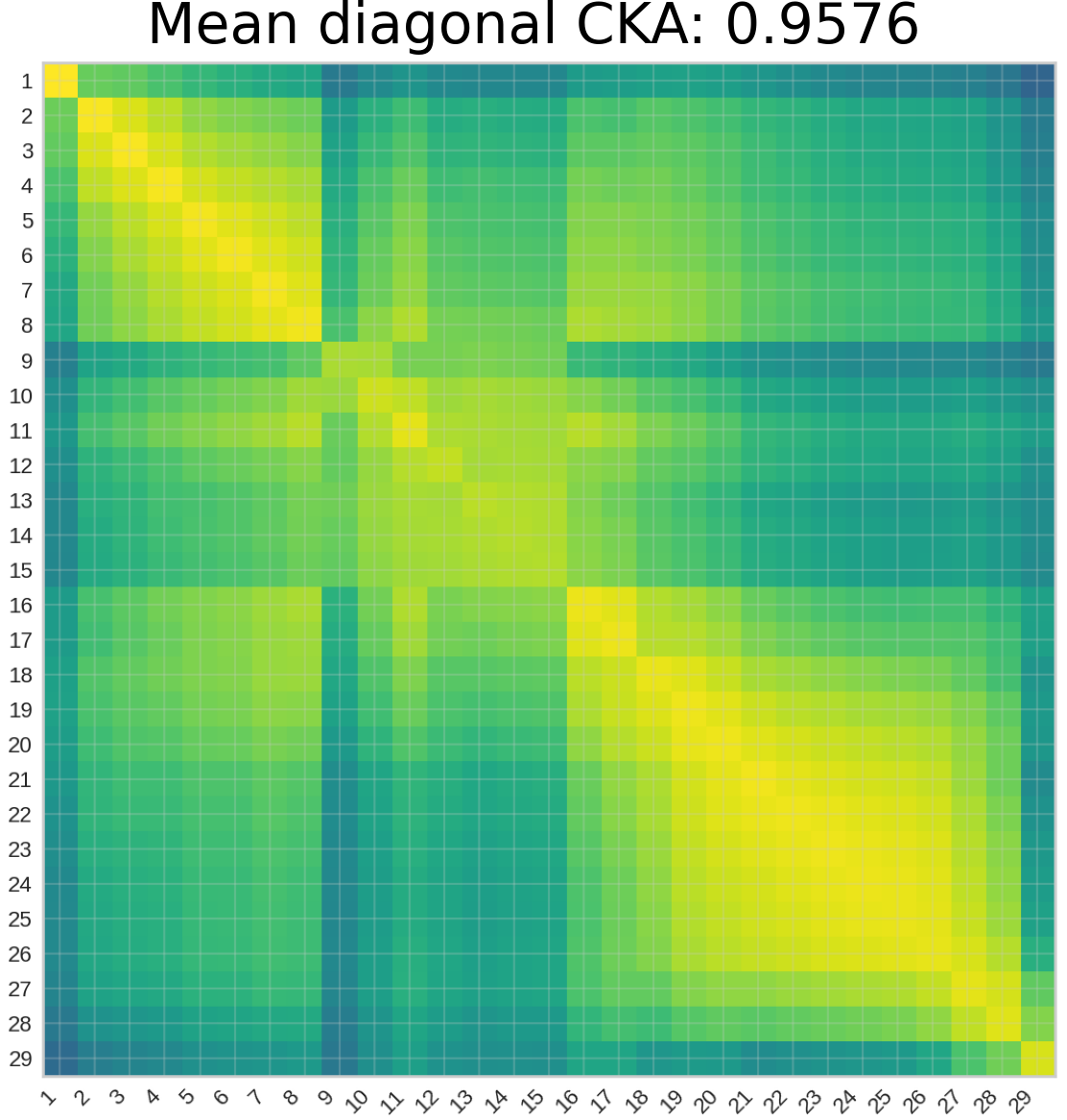}
                \par\vspace{2pt} % Small space between image and caption
                \scriptsize \textbf{Qwen3-0.6B-\textit{Emb}}
            \end{minipage}%
        } &

        \rotlabel{Qwen2.5-1.5B-\textit{Emb}} &
        \begin{minipage}{0.18\textwidth} % Adjusted width
            \includegraphics[width=\linewidth]{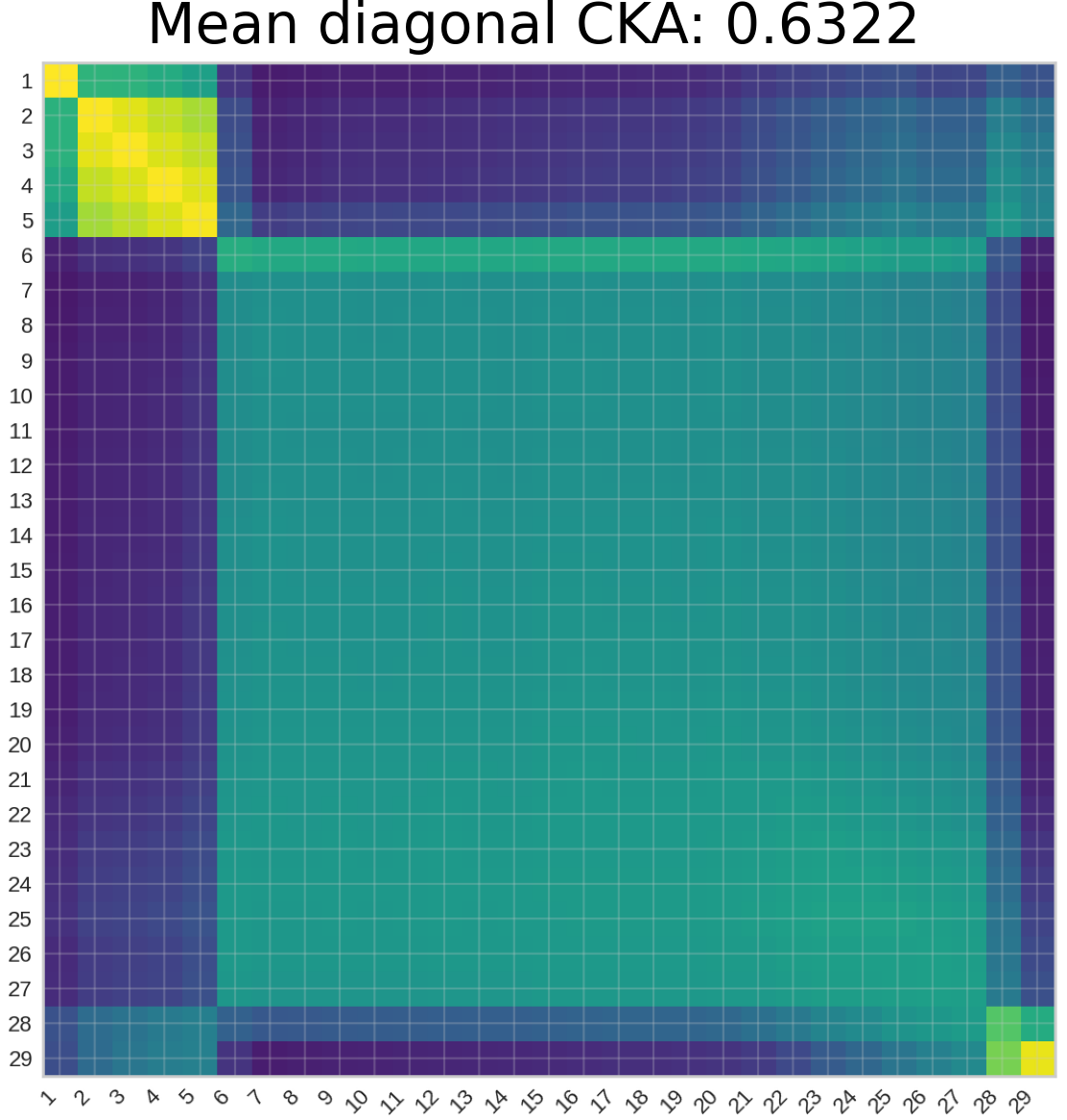} \\
            \centering \scriptsize \textbf{Qwen-2.5-1.5B-SimpleRL-Zoo-\textit{Emb}}
        \end{minipage} \\

        % \multicolumn{6}{c}{\vspace{0.1em}} \\ % Spacing between rows

        % --- Row 2 ---
        \rotlabel{Qwen2.5-0.5B-\textit{Emb}} &
        \begin{minipage}{0.18\textwidth} % Adjusted width
            \includegraphics[width=\linewidth]{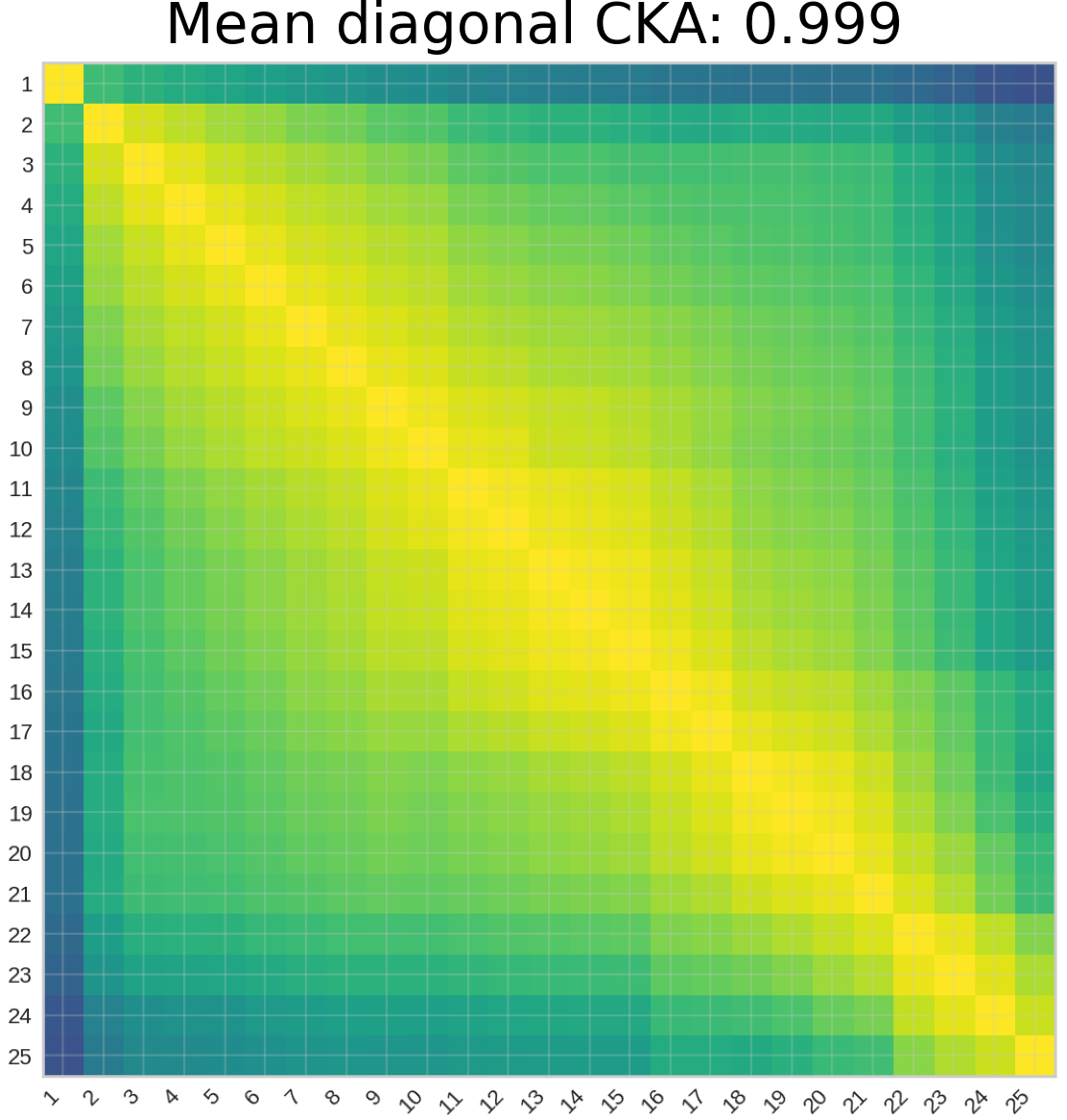} \\
            \centering \scriptsize \textbf{Qwen-2.5-0.5B-SimpleRL-Zoo-\textit{Emb}}
        \end{minipage} &

        \rotlabel{DeepSeek-R1-Distill-Qwen-1.5B-\textit{Emb}} &
        \begin{minipage}{0.18\textwidth} % Adjusted width
            \includegraphics[width=\linewidth]{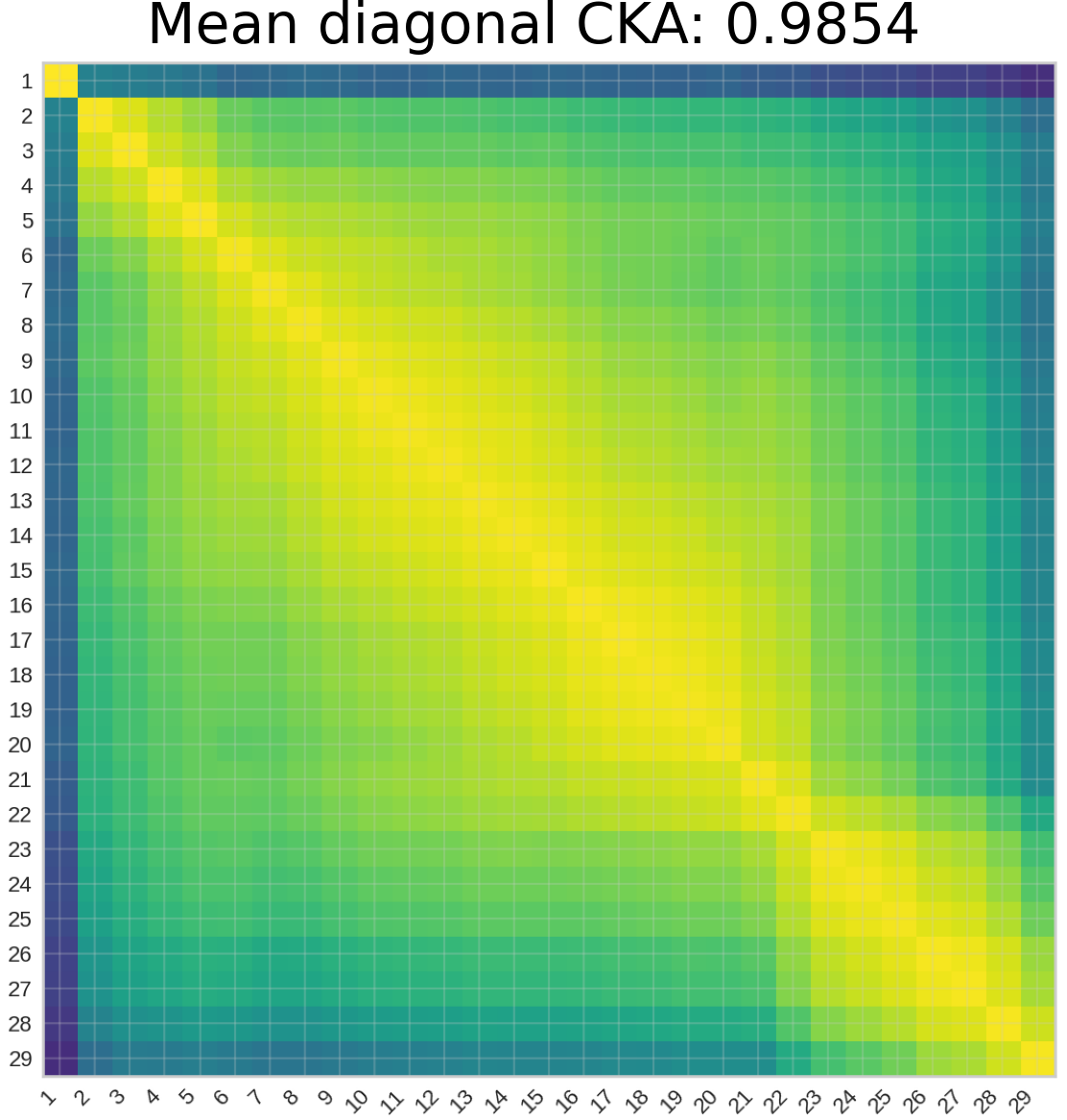} \\
            \centering \scriptsize \textbf{Nemotron-Research-Reasoning-Qwen-1.5B-\textit{Emb}}
        \end{minipage} &

        \rotlabel{Qwen3-4B-\textit{Emb}} &
        \begin{minipage}{0.18\textwidth} % Adjusted width
            \includegraphics[width=\linewidth]{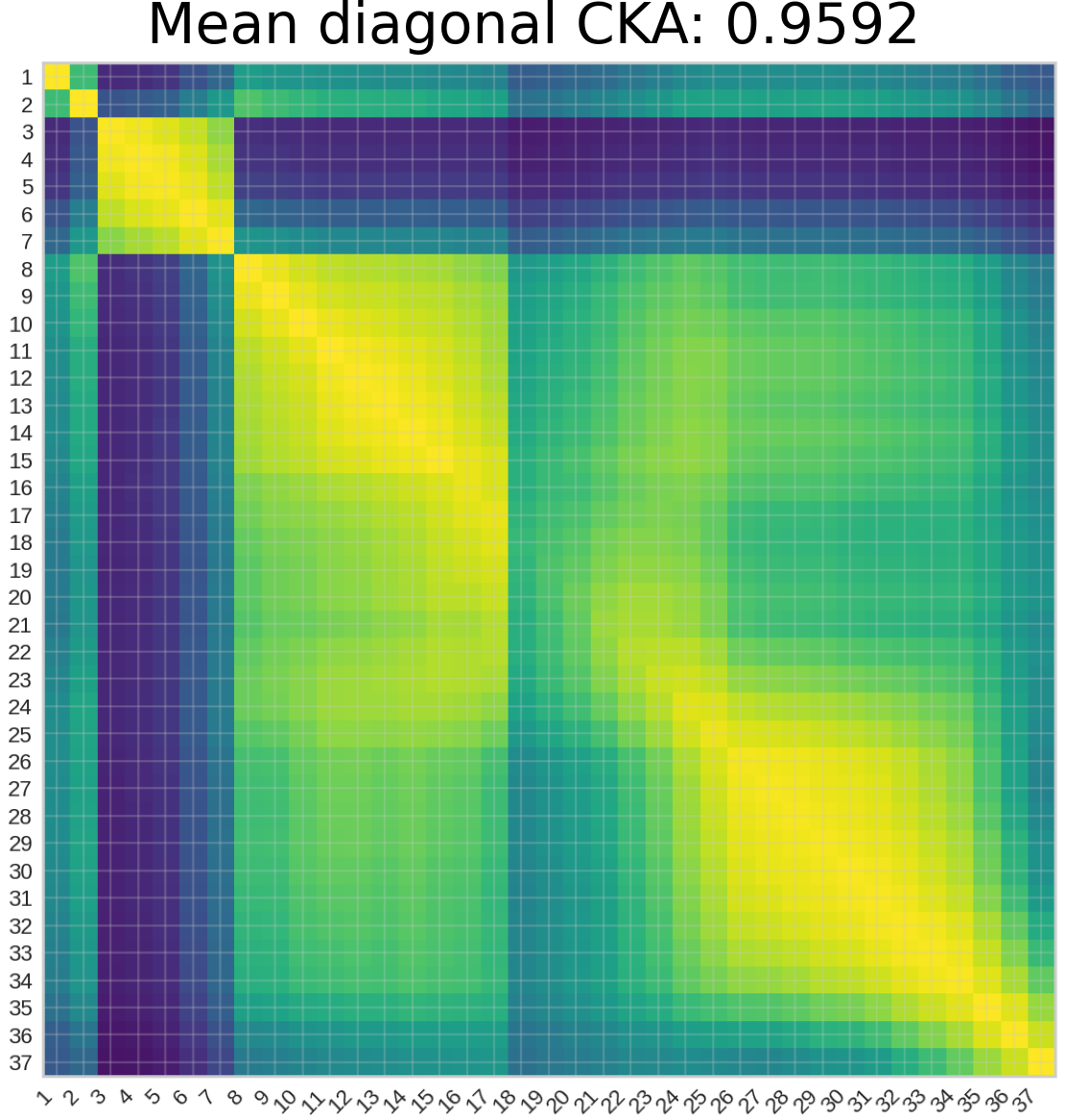} \\
            \centering \scriptsize \textbf{Qwen3-4B-PSR-\textit{Emb}}
        \end{minipage} \\
    \end{tabular}%
    } % End of \makebox

    \vspace{0.5em}
    \rotatebox{-90}{\includegraphics[height=6cm, width=\linewidth, keepaspectratio]{figures/re_experiments/cka_evaluation/cka_heatmap.png}} 

    % The second section and other elements from the original code were commented out or misplaced,
    % and have been removed for clarity and to focus on the main table.

    \caption{Additional Results on Linear CKA separated by dataset. The vertical axis and horizontal axis are Base Model Layer Index and Reasoning Model Layer Index, respectively. The \textbf{\textcolor{red}{red}} background indicates their backbone LLMs are SFT-tuned pairs.}
    \label{fig: appendix-corr-base-embedding-vs-reasoning-embedding}
\end{figure*}

% ================================================================================================ 
% $k$-NN Neighborhood Overlap 
% Base Model vs. Reasoning Model
%================================================================================================ 

\begin{figure*}[p]
    \centering
    {\large \textbf{$k$-NN Overlap}} \par\smallskip
    {\large \textbf{Base Models $\mathcal{M}_{base}$ vs. Reasoning Models $\mathcal{M}_{reason}$}} \par\medskip

    % --- Configuration ---
    \setlength{\tabcolsep}{1pt}

    \sectionbox{Dataset: CoT Datset}
    \vspace{1ex} % Adds a small vertical space for better separation

    % Center the main content grid and ensure it does not exceed the text width
    \makebox[\textwidth][c]{%
    \begin{tabular}{
        c @{\hspace{1pt}} c  @{\hspace{0.5em}}
        c @{\hspace{1pt}} c  @{\hspace{0.5em}}
        c @{\hspace{1pt}} c  @{\hspace{0.5em}}
        c @{\hspace{1pt}} c
    }
        % --- Row 1 ---
        % === MODIFIED CELL START ===
        % We wrap the minipage in a colorbox. 
        % \fboxsep controls the padding between the color edge and the image.
        \rotlabel{Qwen2.5-Math-1.5B} &
        \setlength{\fboxsep}{3pt}% 
        \colorbox{red!20}{%  <-- CHANGE COLOR HERE (e.g., yellow!20, blue!10)
            \begin{minipage}{0.22\textwidth}
                \centering
                \includegraphics[width=\linewidth]{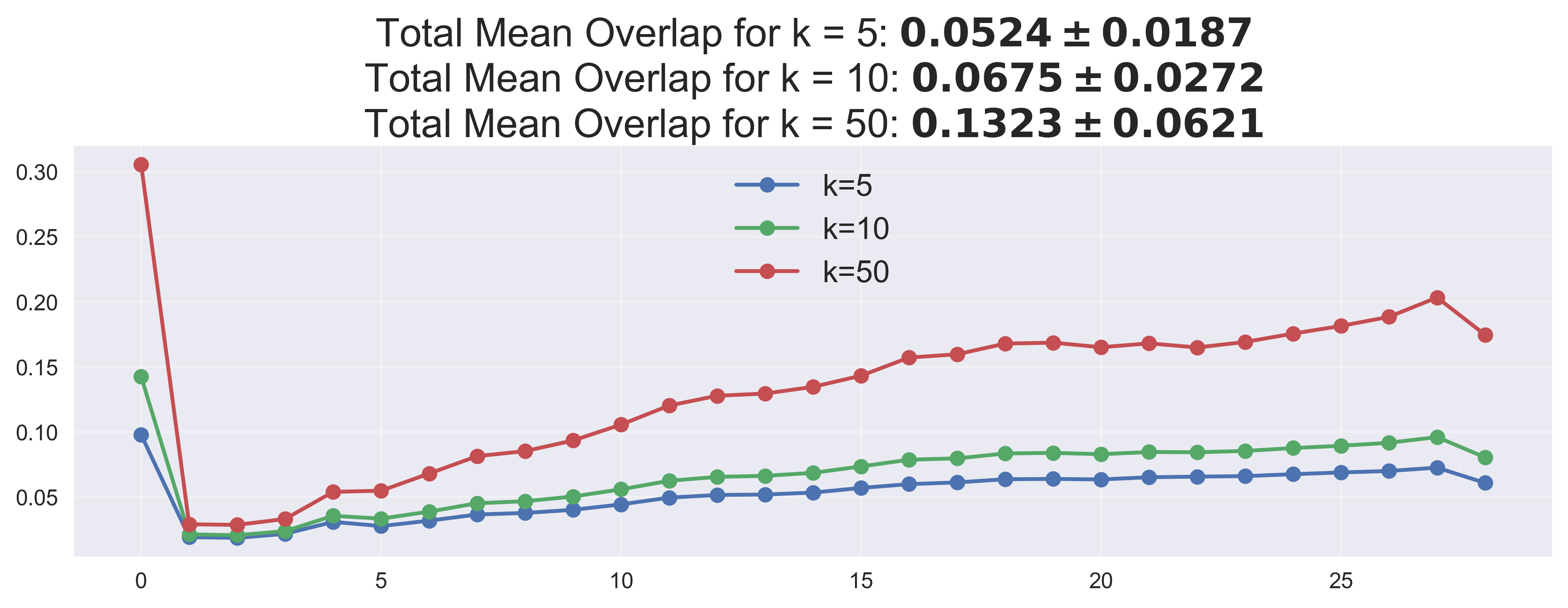}
                \par\vspace{2pt} % Small space between image and caption
                \scriptsize \textbf{DeepSeek-R1-Distill-Qwen-1.5B}
            \end{minipage}%
        } &

        \rotlabel{Qwen3-4B} &
        \begin{minipage}{0.22\textwidth} % Adjusted width
            \includegraphics[width=\linewidth]{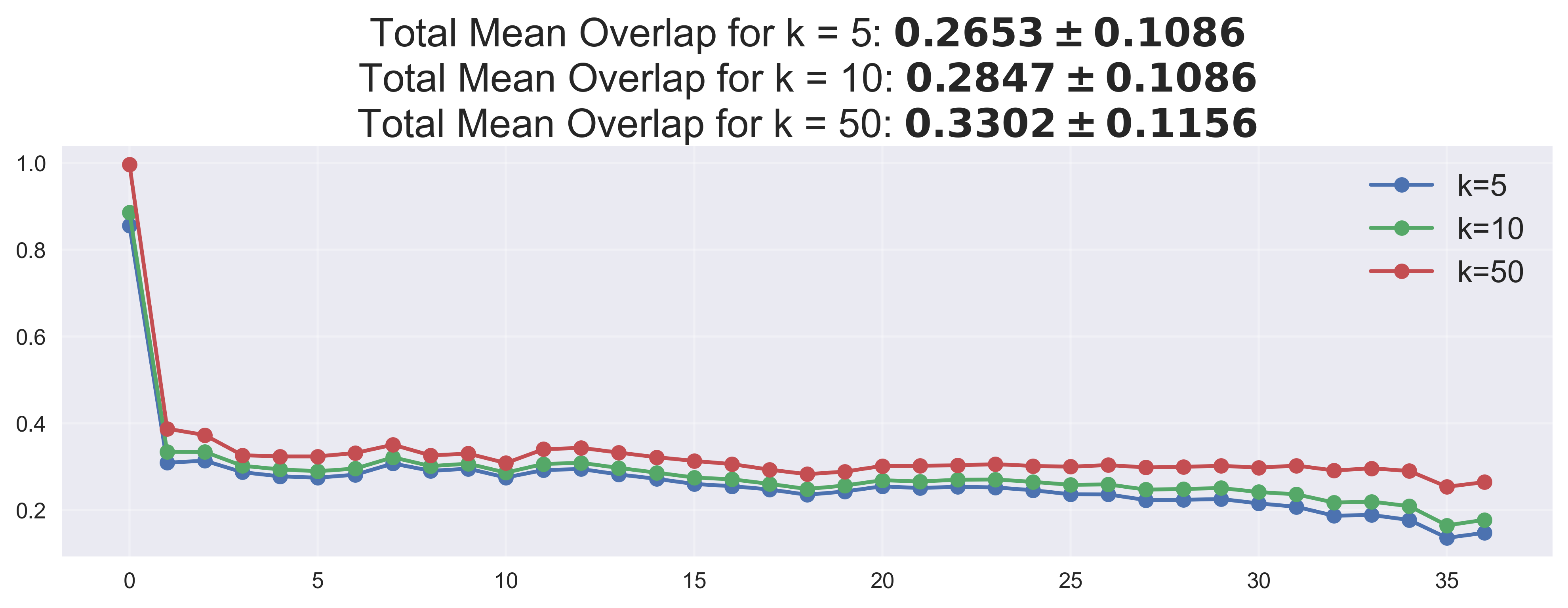} \\
            \centering \scriptsize \textbf{Polaris-4B-Preview}
        \end{minipage} &

        \rotlabel{DeepSeek-R1-Distill-Qwen-7B} &
        \begin{minipage}{0.22\textwidth} % Adjusted width
            \includegraphics[width=\linewidth]{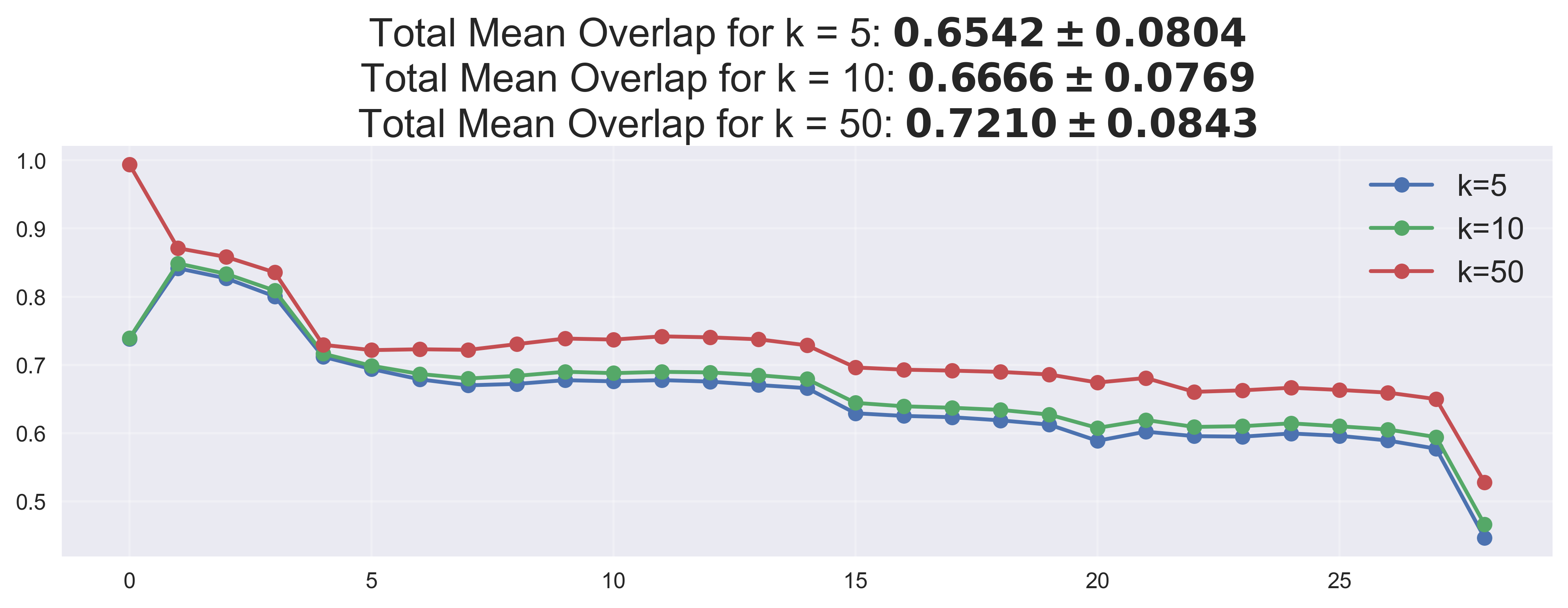} \\
            \centering \scriptsize \textbf{Polaris-7B-Preview}
        \end{minipage} &

        \rotlabel{Qwen2.5-7B} &
        \begin{minipage}{0.22\textwidth} % Adjusted width
            \includegraphics[width=\linewidth]{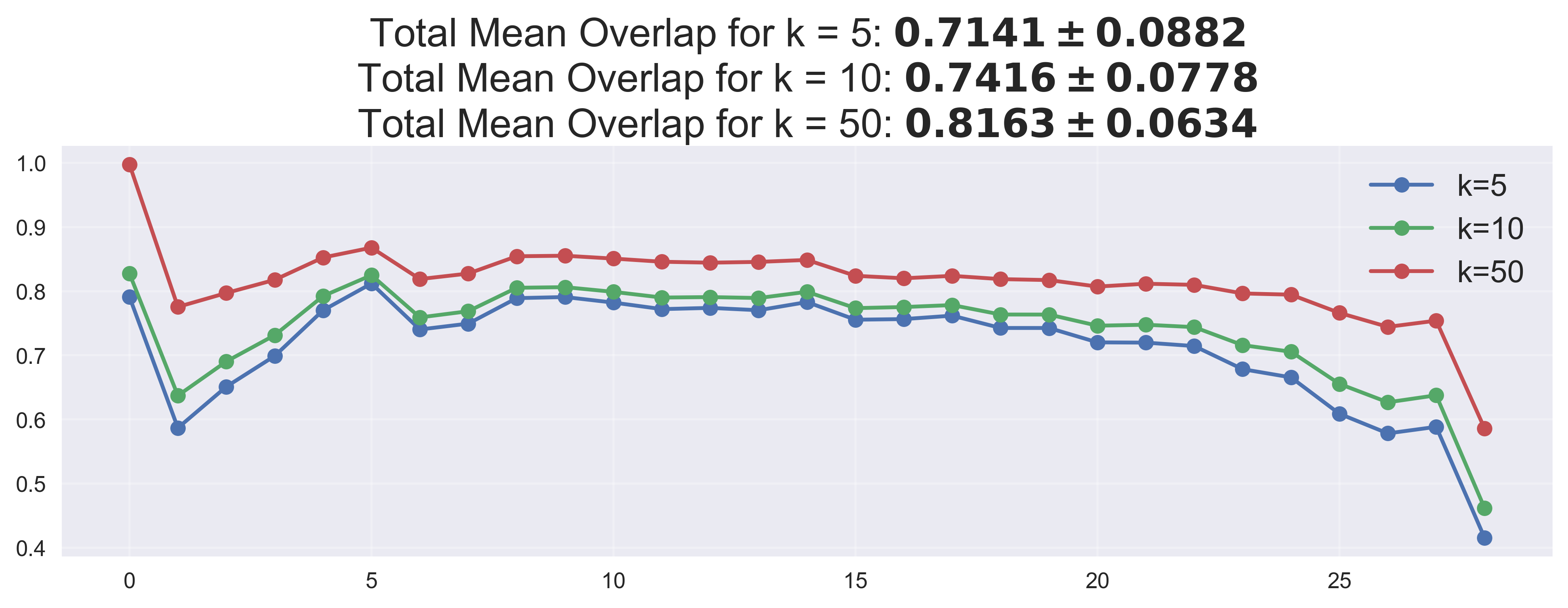} \\
            \centering \scriptsize \textbf{zero\_\_ppo\_\_think\_\_Qwen2.5-7B}
        \end{minipage} \\

        \multicolumn{8}{c}{\vspace{0.1em}} \\ % Spacing between rows

        % --- Row 2 ---
        \rotlabel{Qwen2.5-1.5B} &
        \begin{minipage}{0.22\textwidth} % Adjusted width
            \includegraphics[width=\linewidth]{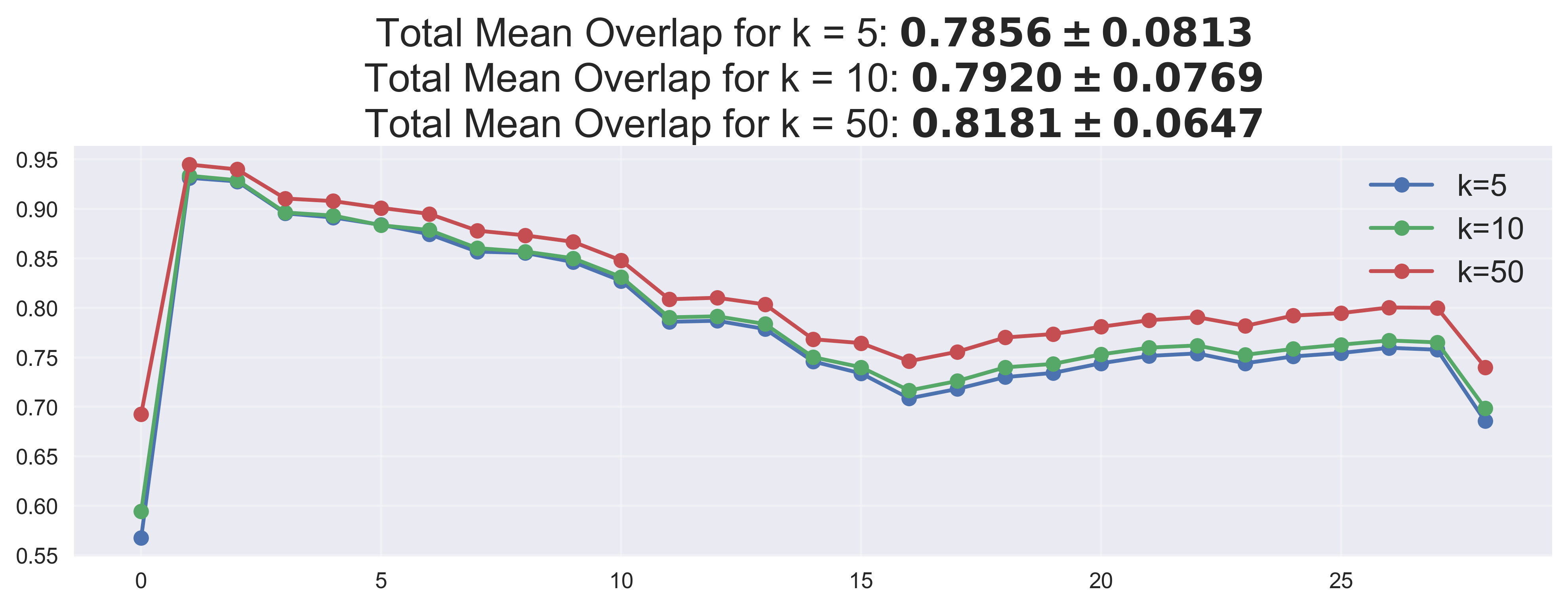} \\
            \centering \scriptsize \textbf{Qwen-2.5-1.5B-SimpleRL-Zoo}
        \end{minipage} &

        \rotlabel{Qwen2.5-0.5B} &
        \begin{minipage}{0.22\textwidth} % Adjusted width
            \includegraphics[width=\linewidth]{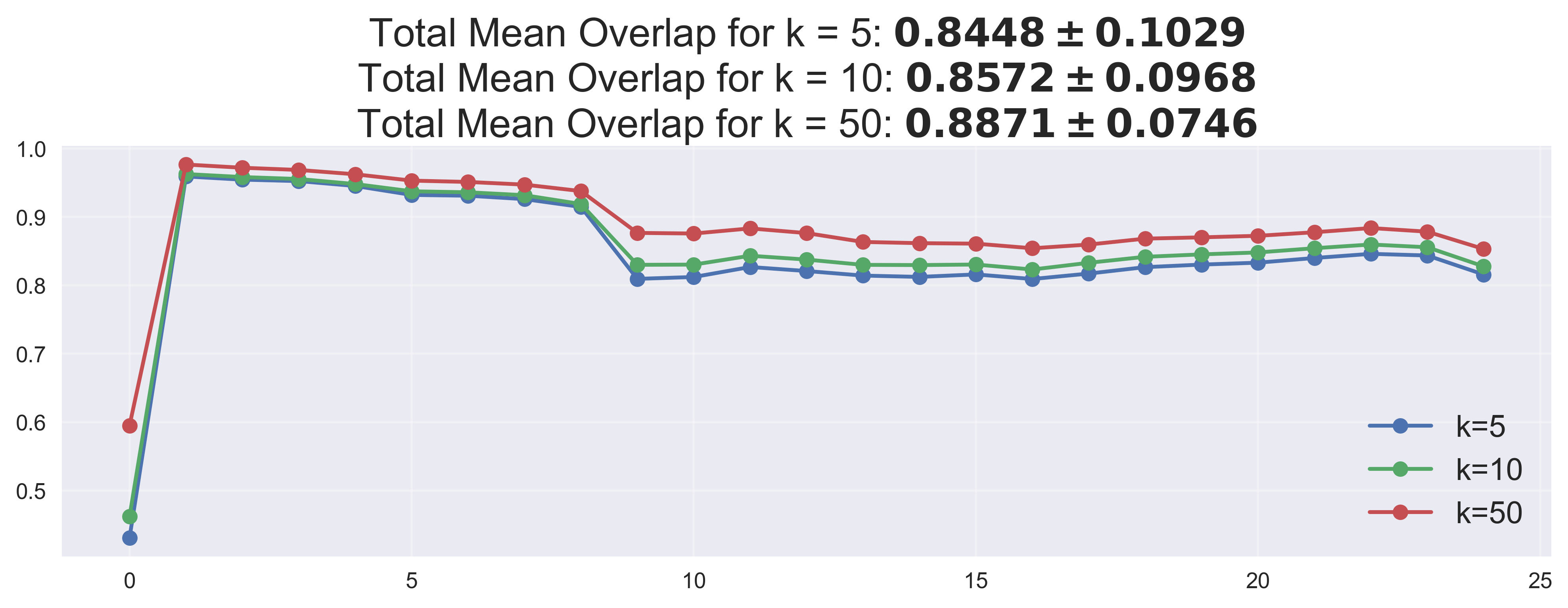} \\
            \centering \scriptsize \textbf{Qwen-2.5-0.5B-SimpleRL-Zoo}
        \end{minipage} &

        \rotlabel{DeepSeek-R1-Distill-Qwen-1.5B} &
        \begin{minipage}{0.22\textwidth} % Adjusted width
            \includegraphics[width=\linewidth]{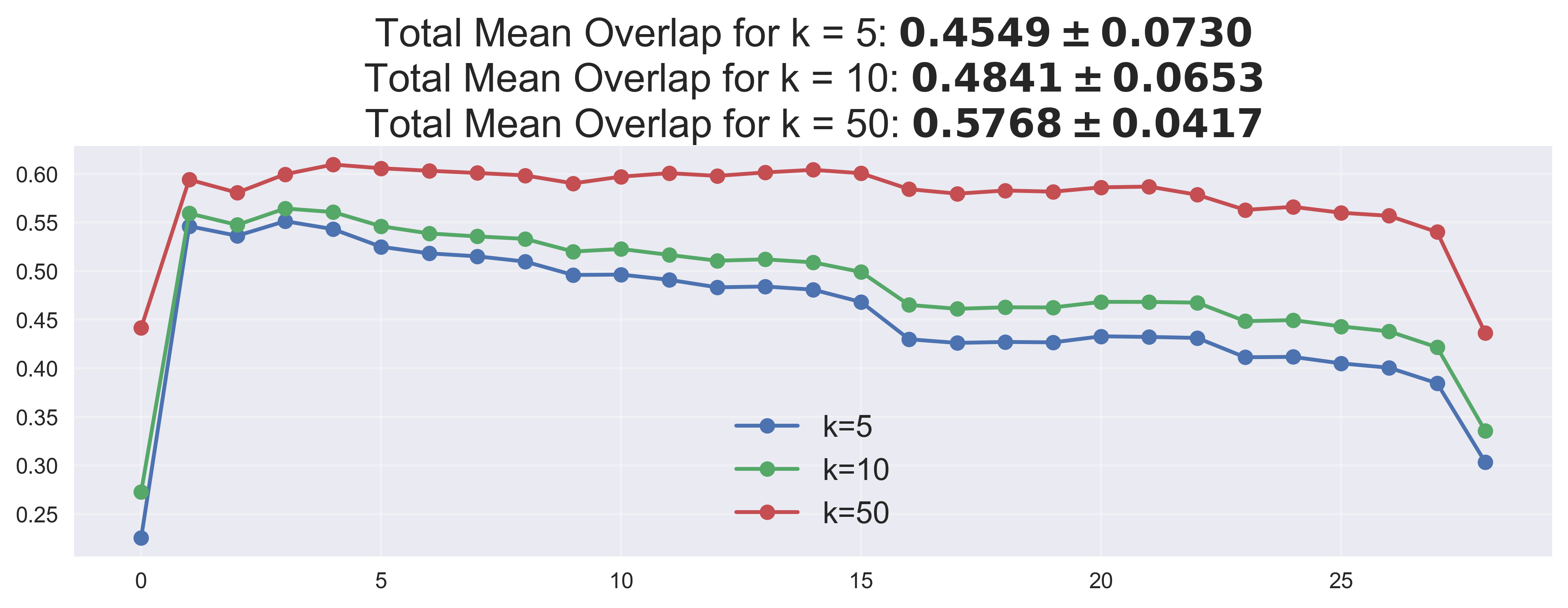} \\
            \centering \scriptsize \textbf{Nemotron-Research-Reasoning-Qwen-1.5B}
        \end{minipage} &

        \rotlabel{Qwen3-4B} &
        \begin{minipage}{0.22\textwidth} % Adjusted width
            \includegraphics[width=\linewidth]{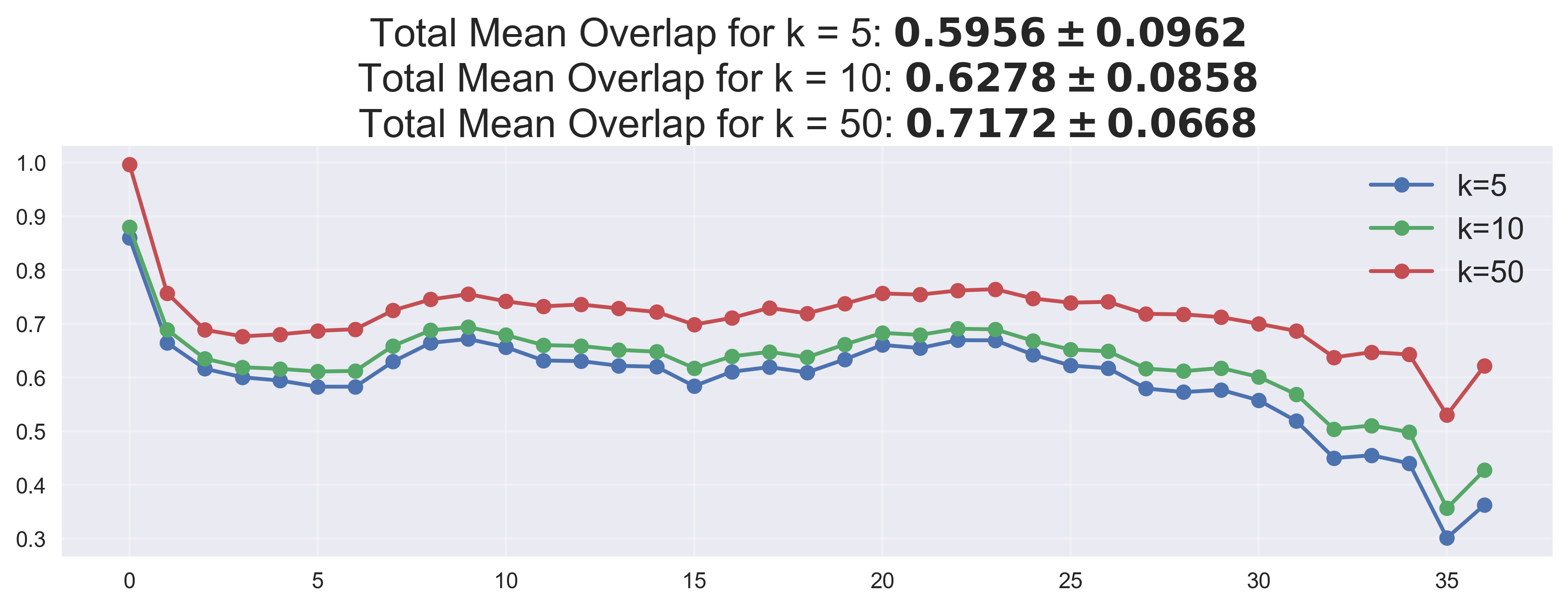} \\
            \centering \scriptsize \textbf{Qwen3-4B-PSR}
        \end{minipage} \\
    \end{tabular}%
    } % End of \makebox

    \vspace{0.5em}
    \textbf{Model Layer Index}
    
    \vspace{1.0em}

    \sectionbox{Dataset: MMLU-Pro}
    \vspace{1ex} % Adds a small vertical space for better separation

    % Center the main content grid and ensure it does not exceed the text width
    \makebox[\textwidth][c]{%
    \begin{tabular}{
        c @{\hspace{1pt}} c  @{\hspace{0.5em}}
        c @{\hspace{1pt}} c  @{\hspace{0.5em}}
        c @{\hspace{1pt}} c  @{\hspace{0.5em}}
        c @{\hspace{1pt}} c
    }
        % --- Row 1 ---
        \rotlabel{Qwen2.5-Math-1.5B} &
        \setlength{\fboxsep}{3pt}% 
        \colorbox{red!20}{%  <-- CHANGE COLOR HERE (e.g., yellow!20, blue!10)
            \begin{minipage}{0.22\textwidth}
                \centering
                \includegraphics[width=\linewidth]{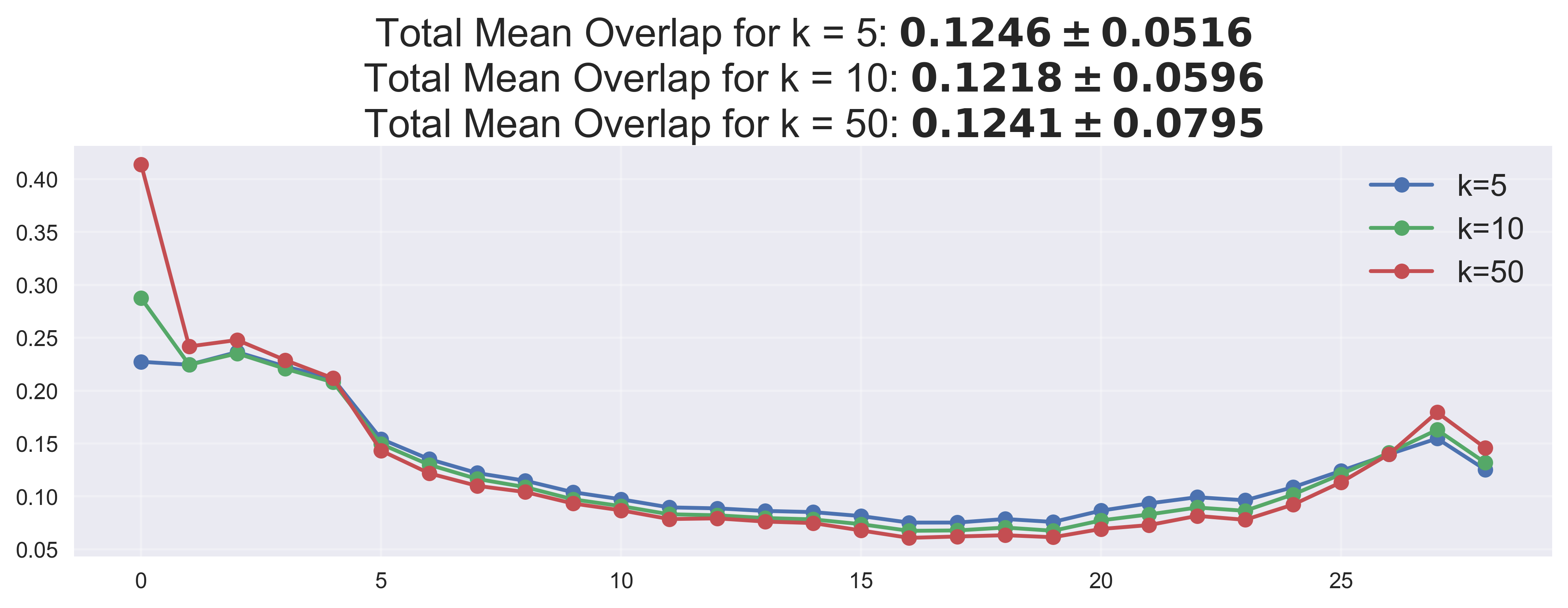}
                \par\vspace{2pt} % Small space between image and caption
                \scriptsize \textbf{DeepSeek-R1-Distill-Qwen-1.5B}
            \end{minipage}%
        } &

        \rotlabel{Qwen3-4B} &
        \begin{minipage}{0.22\textwidth} % Adjusted width
            \includegraphics[width=\linewidth]{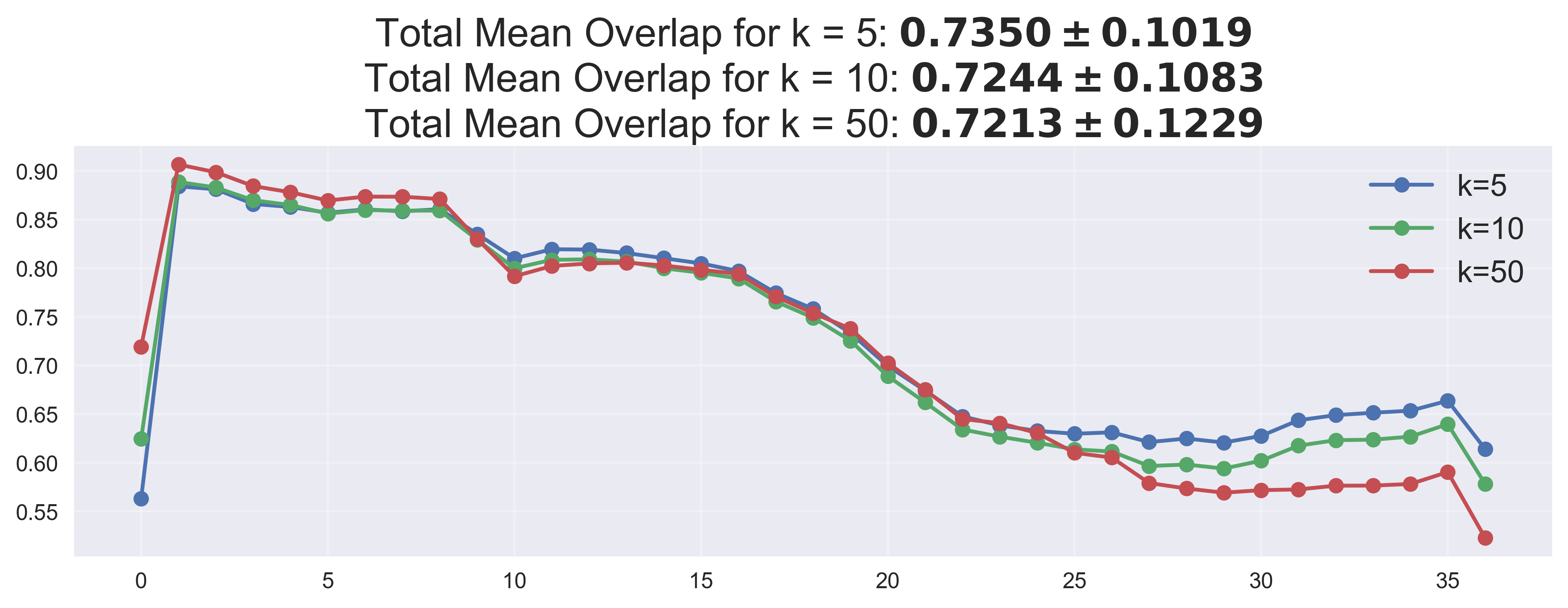} \\
            \centering \scriptsize \textbf{Polaris-4B-Preview}
        \end{minipage} &

        \rotlabel{DeepSeek-R1-Distill-Qwen-7B} &
        \begin{minipage}{0.22\textwidth} % Adjusted width
            \includegraphics[width=\linewidth]{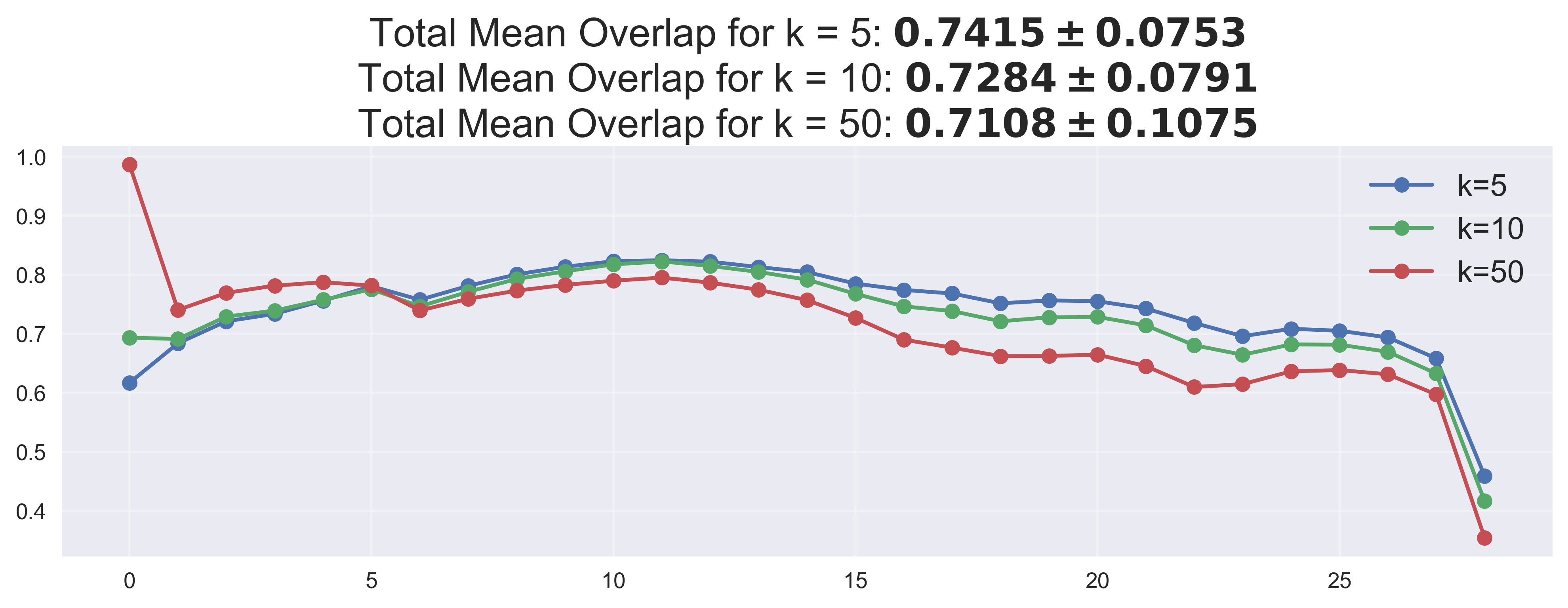} \\
            \centering \scriptsize \textbf{Polaris-7B-Preview}
        \end{minipage} &

        \rotlabel{Qwen2.5-7B} &
        \begin{minipage}{0.22\textwidth} % Adjusted width
            \includegraphics[width=\linewidth]{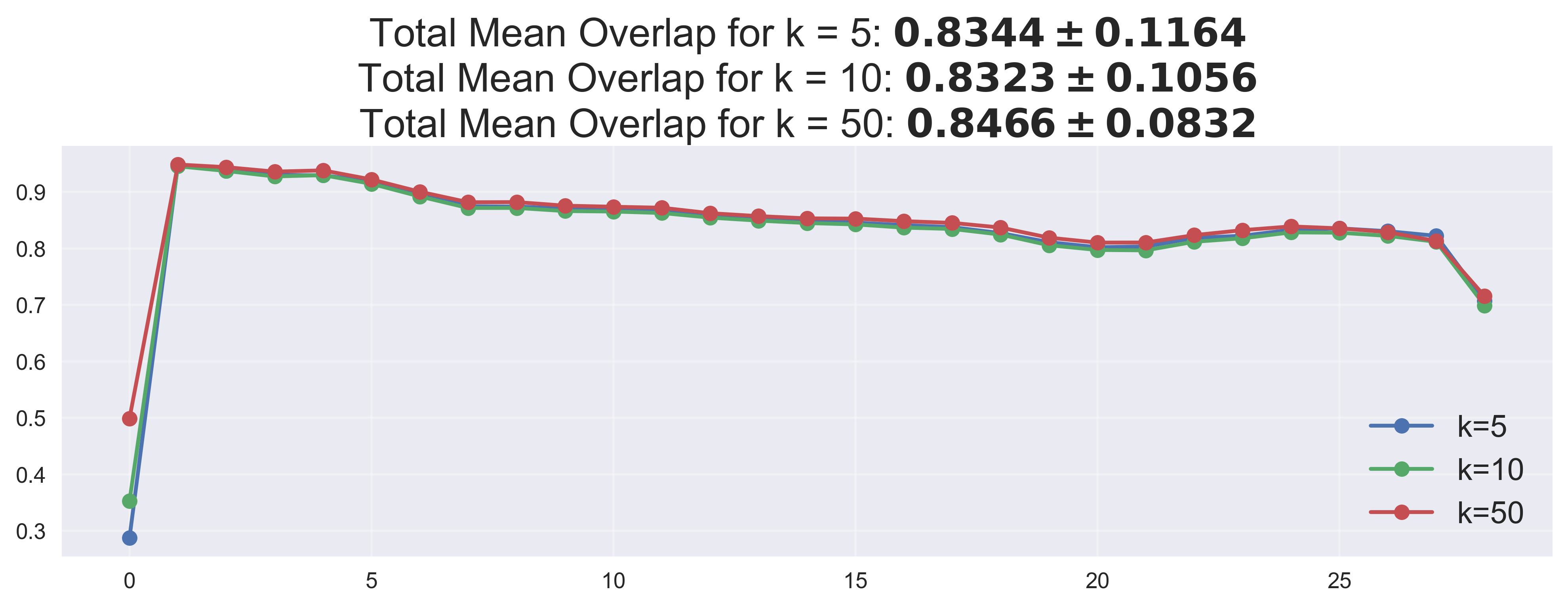} \\
            \centering \scriptsize \textbf{zero\_\_ppo\_\_think\_\_Qwen2.5-7B}
        \end{minipage} \\

        \multicolumn{8}{c}{\vspace{0.1em}} \\ % Spacing between rows

        % --- Row 2 ---
        \rotlabel{Qwen2.5-1.5B} &
        \begin{minipage}{0.22\textwidth} % Adjusted width
            \includegraphics[width=\linewidth]{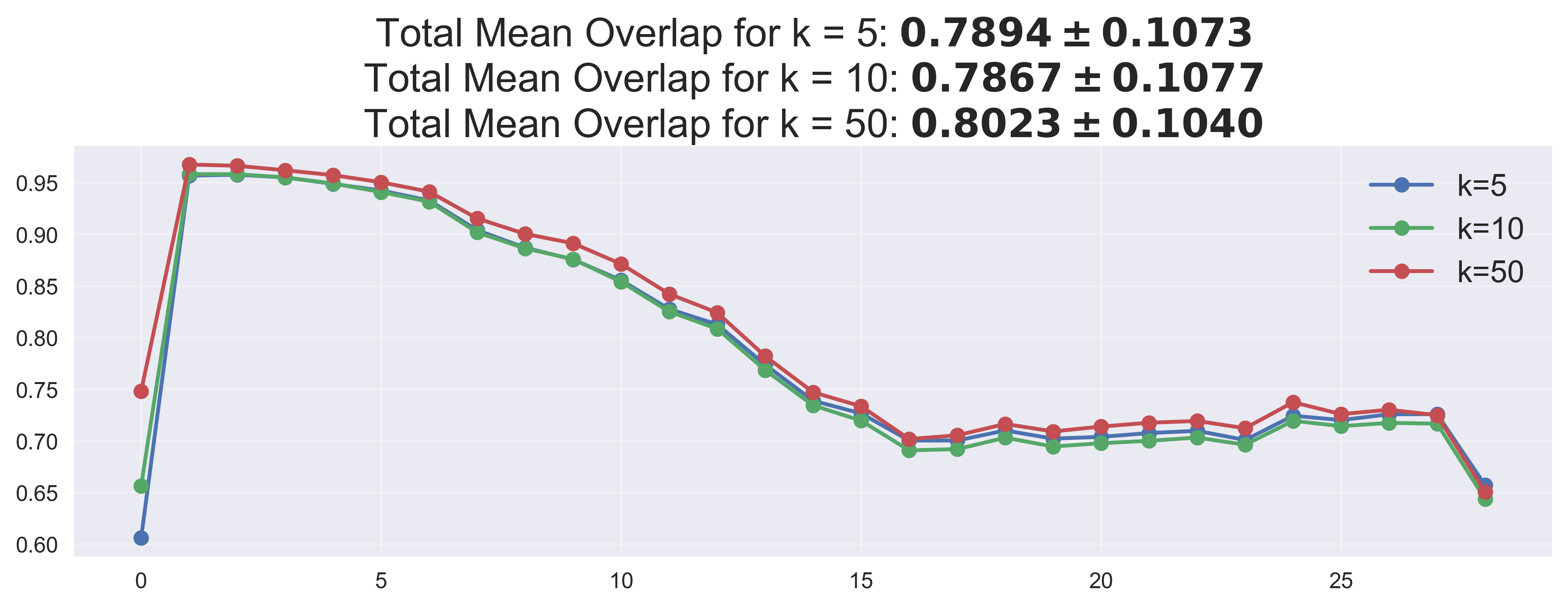} \\
            \centering \scriptsize \textbf{Qwen-2.5-1.5B-SimpleRL-Zoo}
        \end{minipage} &

        \rotlabel{Qwen2.5-0.5B} &
        \begin{minipage}{0.22\textwidth} % Adjusted width
            \includegraphics[width=\linewidth]{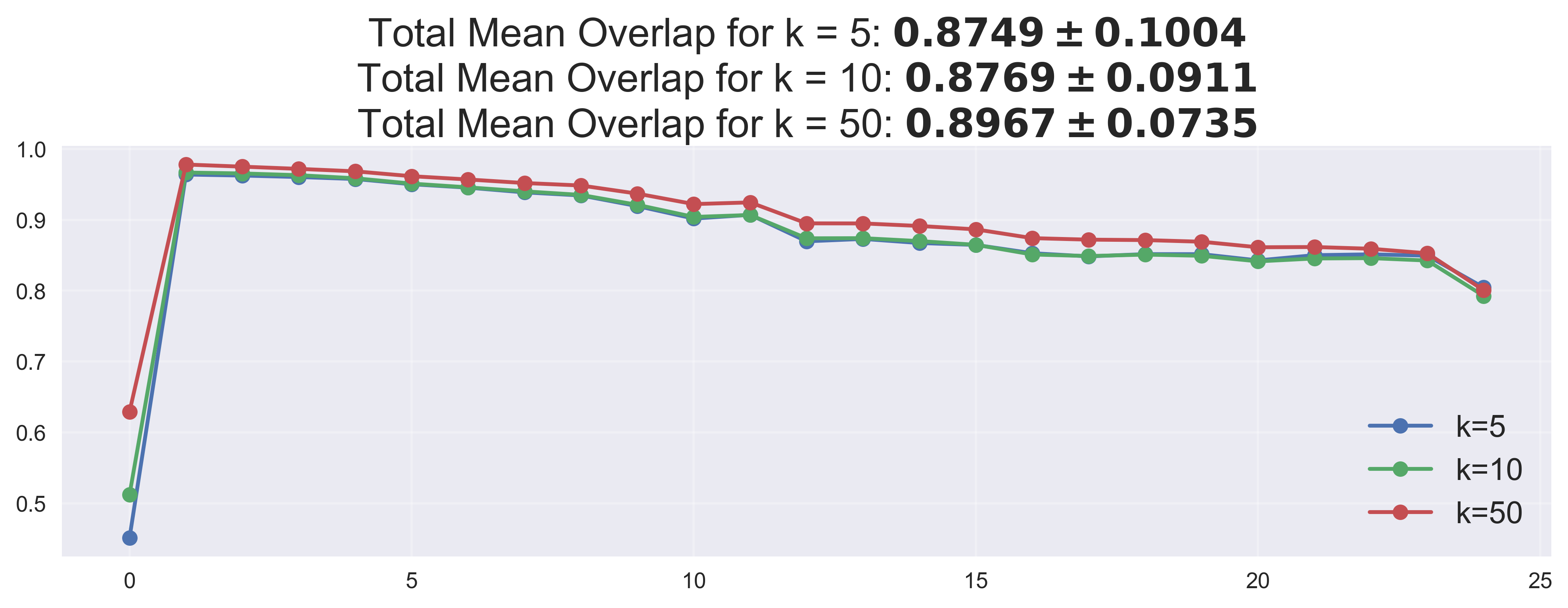} \\
            \centering \scriptsize \textbf{Qwen-2.5-0.5B-SimpleRL-Zoo}
        \end{minipage} &

        \rotlabel{DeepSeek-R1-Distill-Qwen-1.5B} &
        \begin{minipage}{0.22\textwidth} % Adjusted width
            \includegraphics[width=\linewidth]{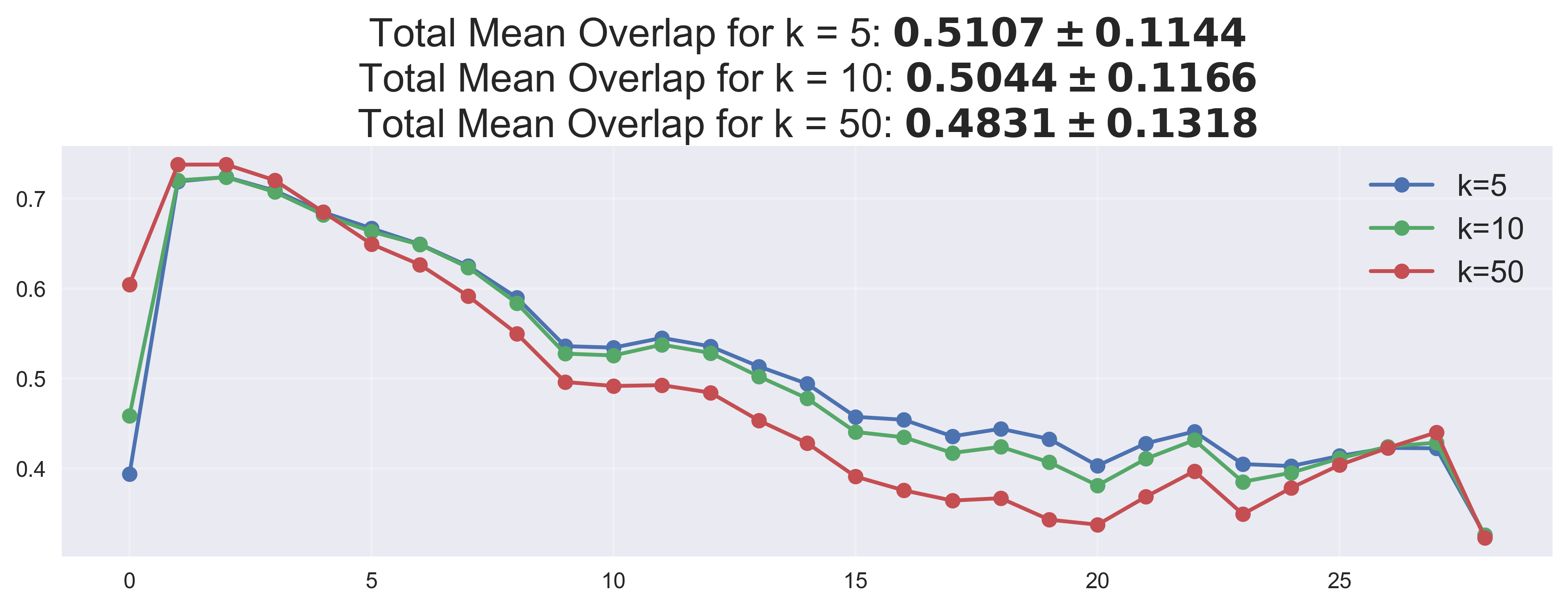} \\
            \centering \scriptsize \textbf{Nemotron-Research-Reasoning-Qwen-1.5B}
        \end{minipage} &

        \rotlabel{Qwen3-4B} &
        \begin{minipage}{0.22\textwidth} % Adjusted width
            \includegraphics[width=\linewidth]{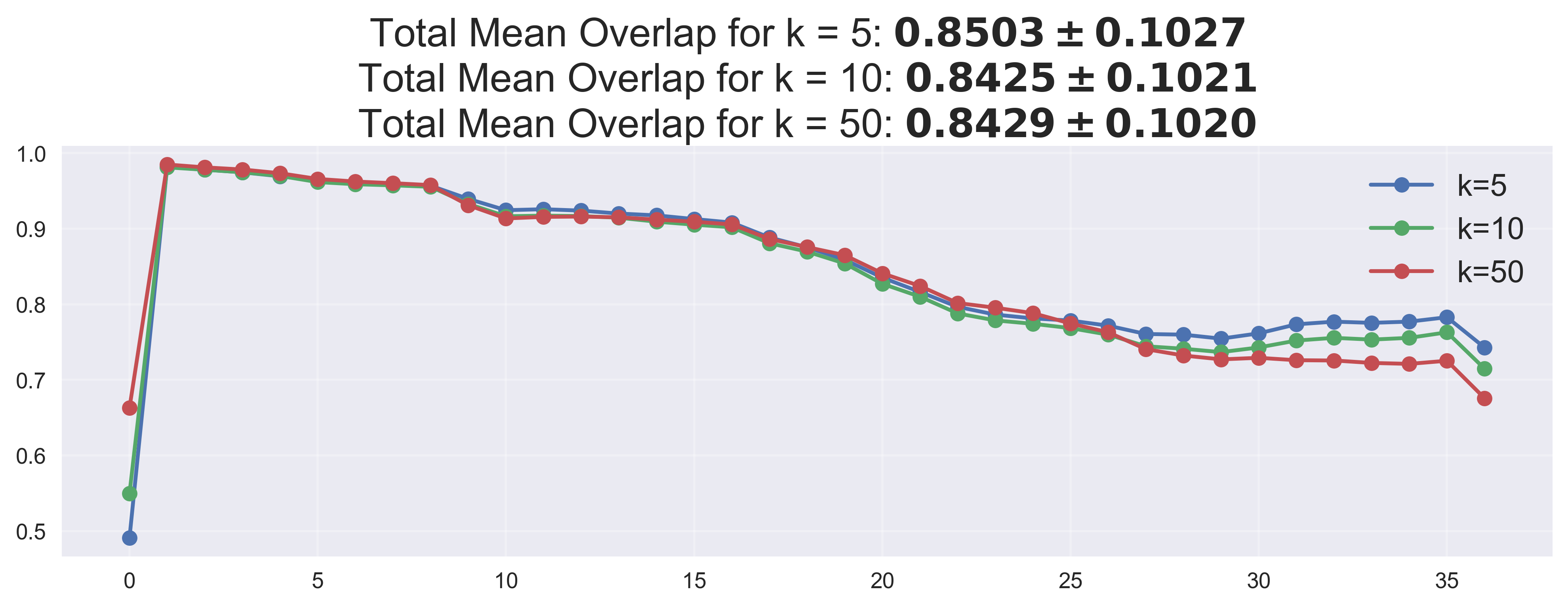} \\
            \centering \scriptsize \textbf{Qwen3-4B-PSR}
        \end{minipage} \\
    \end{tabular}%
    } % End of \makebox

    \vspace{0.5em}
    \textbf{Model Layer Index}

    % The second section and other elements from the original code were commented out or misplaced,
    % and have been removed for clarity and to focus on the main table.

    \caption{Additional Results on $k$-NN Overlap separated by dataset. The vertical axis and horizontal axis are Mean Overlap and Model Layer Index, respectively. The \textbf{\textcolor{red}{red}} background indicates SFT-tuned pairs.}
    \label{fig: appendix-knn-overlap-base-vs-reasoning}
\end{figure*}
% ================================================================================================ 
% Dimension-Wise Correlation 
% Base Embedding Model vs. Reasoning Embedding Model
%================================================================================================ 

\begin{figure*}[p]
    \centering
    {\large \textbf{$k$-NN Overlap}} \par\smallskip
    {\large \textbf{Base Embedding Models $\mathcal{M}_{base}^{Emb}$ vs. Reasoning Embedding Models $\mathcal{M}_{reason}^{Emb}$}} \par\medskip

    % --- Configuration ---
    \setlength{\tabcolsep}{1pt}

    \sectionbox{Dataset: CoT Datset}
    \vspace{1ex} % Adds a small vertical space for better separation

    % Center the main content grid and ensure it does not exceed the text width
    \makebox[\textwidth][c]{%
    \begin{tabular}{
        c @{\hspace{1pt}} c  @{\hspace{0.5em}}
        c @{\hspace{1pt}} c  @{\hspace{0.5em}}
        c @{\hspace{1pt}} c
    }
        % --- Row 1 ---
        \rotlabel{Qwen2.5-Math-1.5B-\textit{Emb}} &
        \setlength{\fboxsep}{3pt}% 
        \colorbox{red!20}{%  <-- CHANGE COLOR HERE (e.g., yellow!20, blue!10)
            \begin{minipage}{0.30\textwidth} % Adjusted width for better fit
            \includegraphics[width=\linewidth]{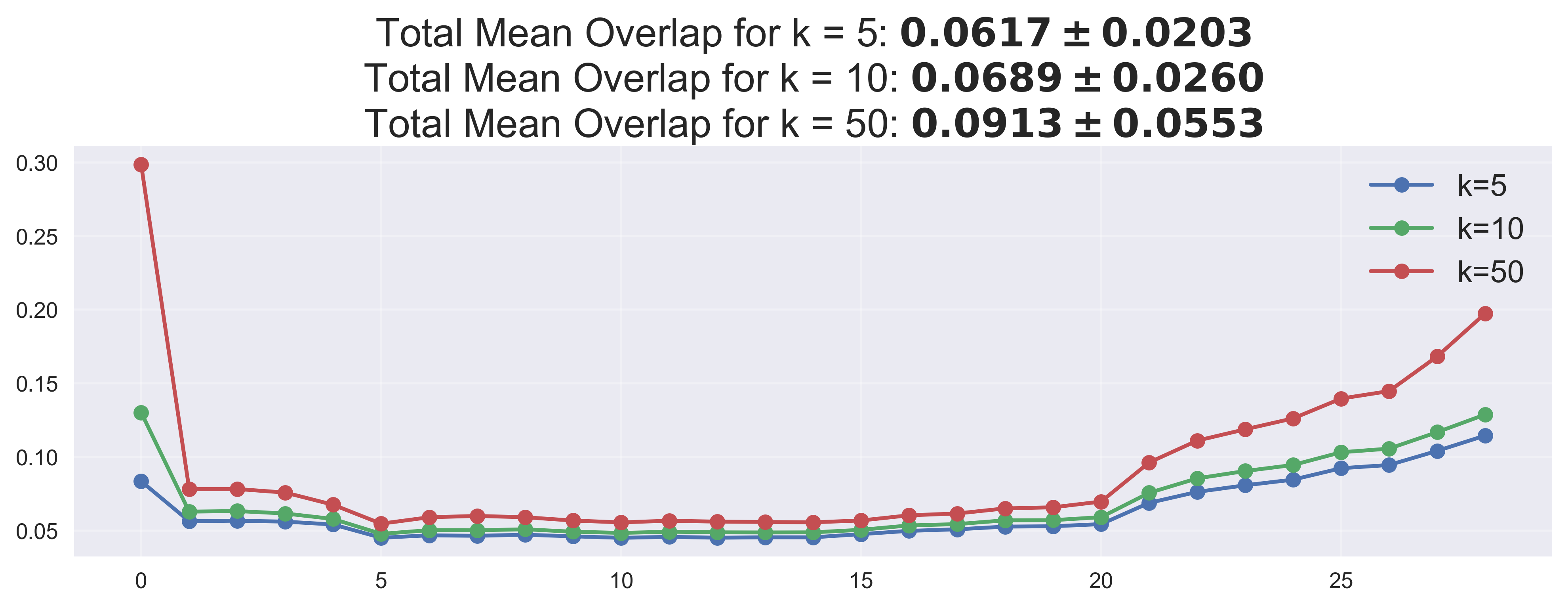}
            \centering \scriptsize \textbf{DeepSeek-R1-Distill-Qwen-1.5B-\textit{Emb}}
        \end{minipage}%
        } &

        \rotlabel{Qwen3-0.6B-Base-\textit{Emb}} &
        \setlength{\fboxsep}{3pt}% 
        \colorbox{red!20}{%  <-- CHANGE COLOR HERE (e.g., yellow!20, blue!10)
            \begin{minipage}{0.30\textwidth} % Adjusted width
            \includegraphics[width=\linewidth]{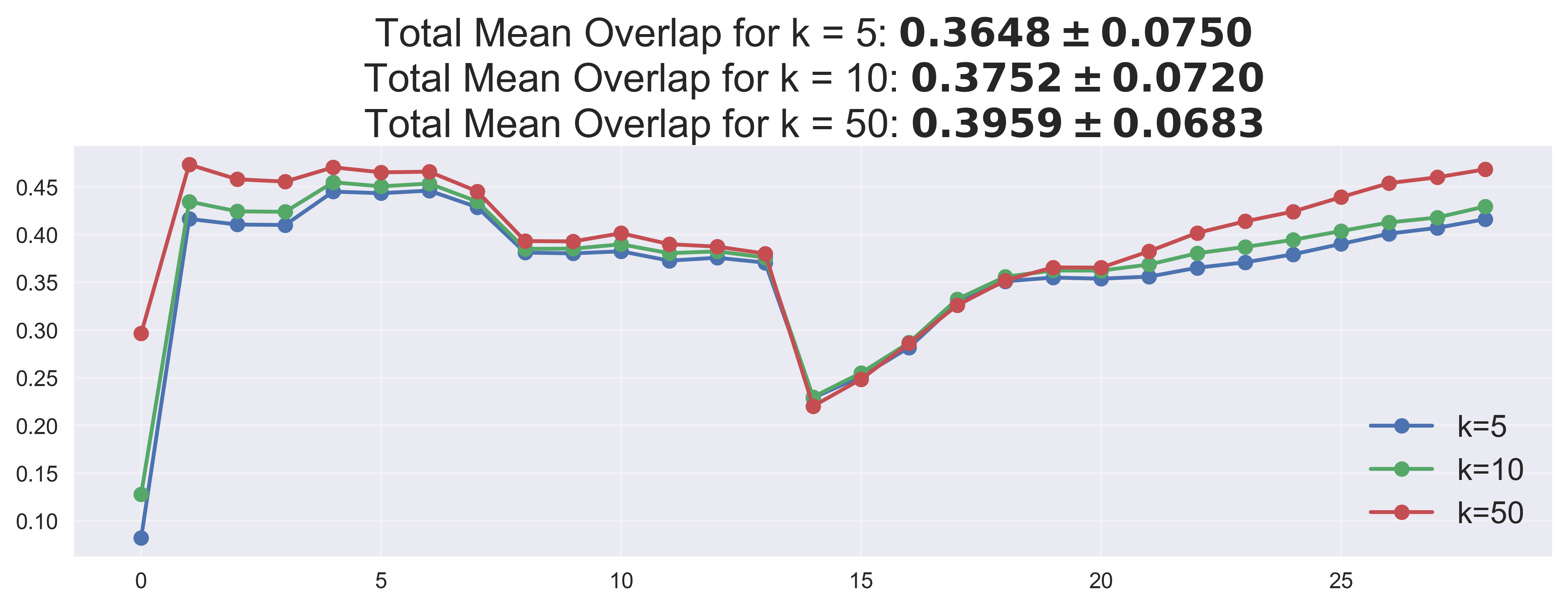} \\
            \centering \scriptsize \textbf{Qwen3-0.6B-\textit{Emb}}
        \end{minipage}%
        } &

        \rotlabel{Qwen2.5-1.5B-\textit{Emb}} &
        \begin{minipage}{0.30\textwidth} % Adjusted width
            \includegraphics[width=\linewidth]{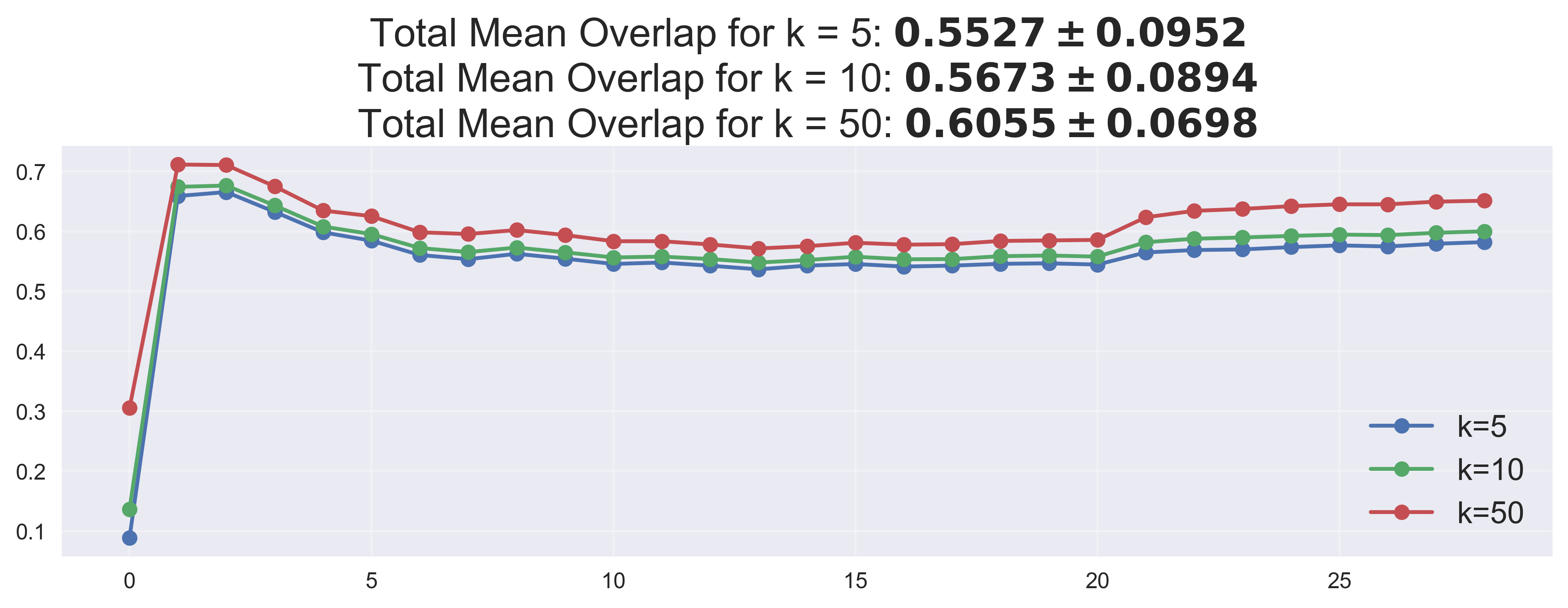} \\
            \centering \scriptsize \textbf{Qwen-2.5-1.5B-SimpleRL-Zoo-\textit{Emb}}
        \end{minipage} \\

        % \multicolumn{6}{c}{\vspace{0.1em}} \\ % Spacing between rows

        % --- Row 2 ---
        \rotlabel{Qwen2.5-0.5B-\textit{Emb}} &
        \begin{minipage}{0.30\textwidth} % Adjusted width
            \includegraphics[width=\linewidth]{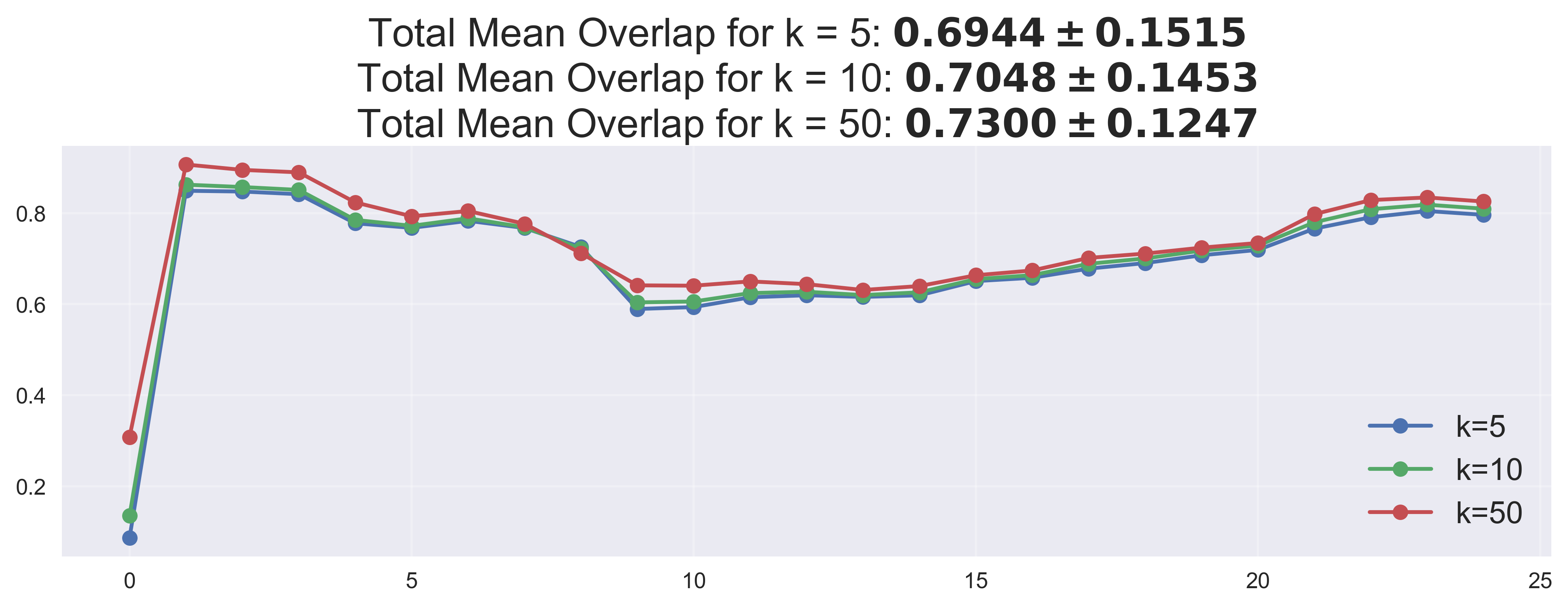} \\
            \centering \scriptsize \textbf{Qwen-2.5-0.5B-SimpleRL-Zoo-\textit{Emb}}
        \end{minipage} &

        \rotlabel{DeepSeek-R1-Distill-Qwen-1.5B-\textit{Emb}} &
        \begin{minipage}{0.30\textwidth} % Adjusted width
            \includegraphics[width=\linewidth]{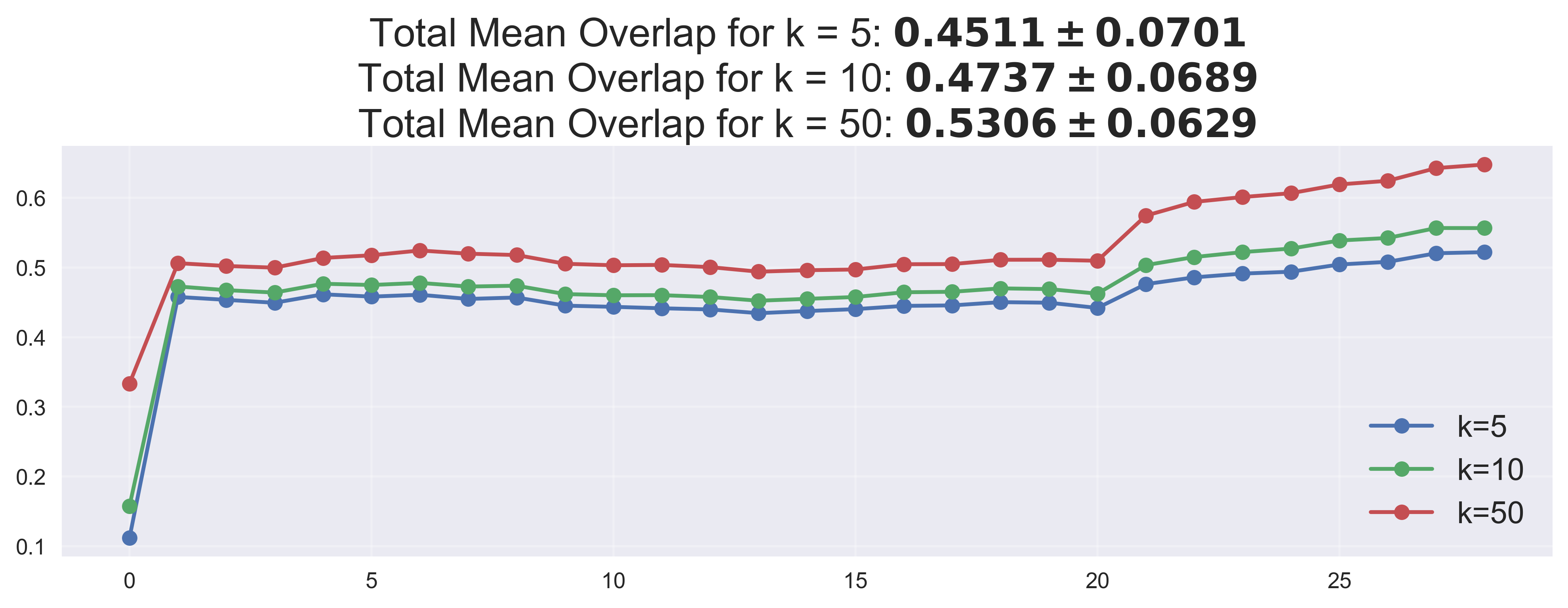} \\
            \centering \scriptsize \textbf{Nemotron-Research-Reasoning-Qwen-1.5B-\textit{Emb}}
        \end{minipage} &

        \rotlabel{Qwen3-4B-\textit{Emb}} &
        \begin{minipage}{0.30\textwidth} % Adjusted width
            \includegraphics[width=\linewidth]{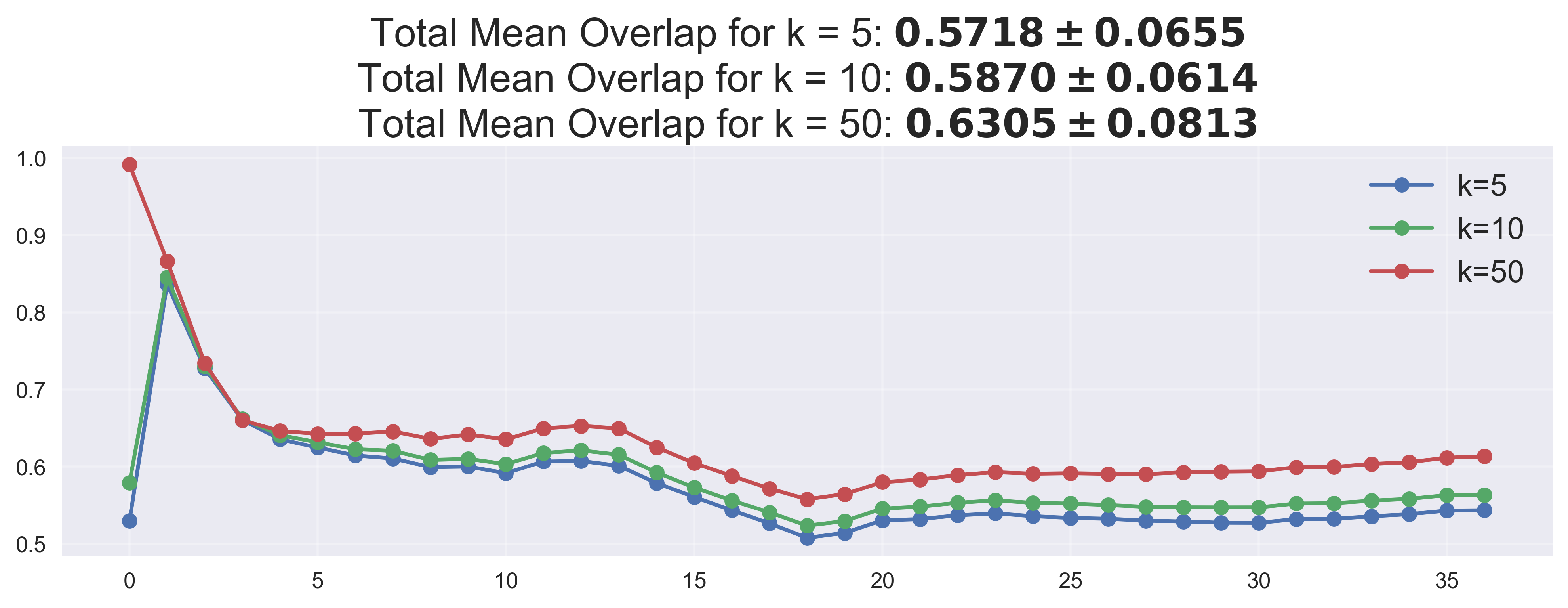} \\
            \centering \scriptsize \textbf{Qwen3-4B-PSR-\textit{Emb}}
        \end{minipage} \\
    \end{tabular}%
    } % End of \makebox

    \vspace{0.5em}
    \textbf{Model Layer Index}
    
    \vspace{1.0em}

    \sectionbox{Dataset: MMLU-Pro}
    \vspace{1ex} % Adds a small vertical space for better separation

    % Center the main content grid and ensure it does not exceed the text width
    \makebox[\textwidth][c]{%
    \begin{tabular}{
        c @{\hspace{1pt}} c  @{\hspace{0.5em}}
        c @{\hspace{1pt}} c  @{\hspace{0.5em}}
        c @{\hspace{1pt}} c
    }
        % --- Row 1 ---
        \rotlabel{Qwen2.5-Math-1.5B-\textit{Emb}} &
        \setlength{\fboxsep}{3pt}% 
        \colorbox{red!20}{%  <-- CHANGE COLOR HERE (e.g., yellow!20, blue!10)
            \begin{minipage}{0.30\textwidth} % Adjusted width for better fit
            \includegraphics[width=\linewidth]{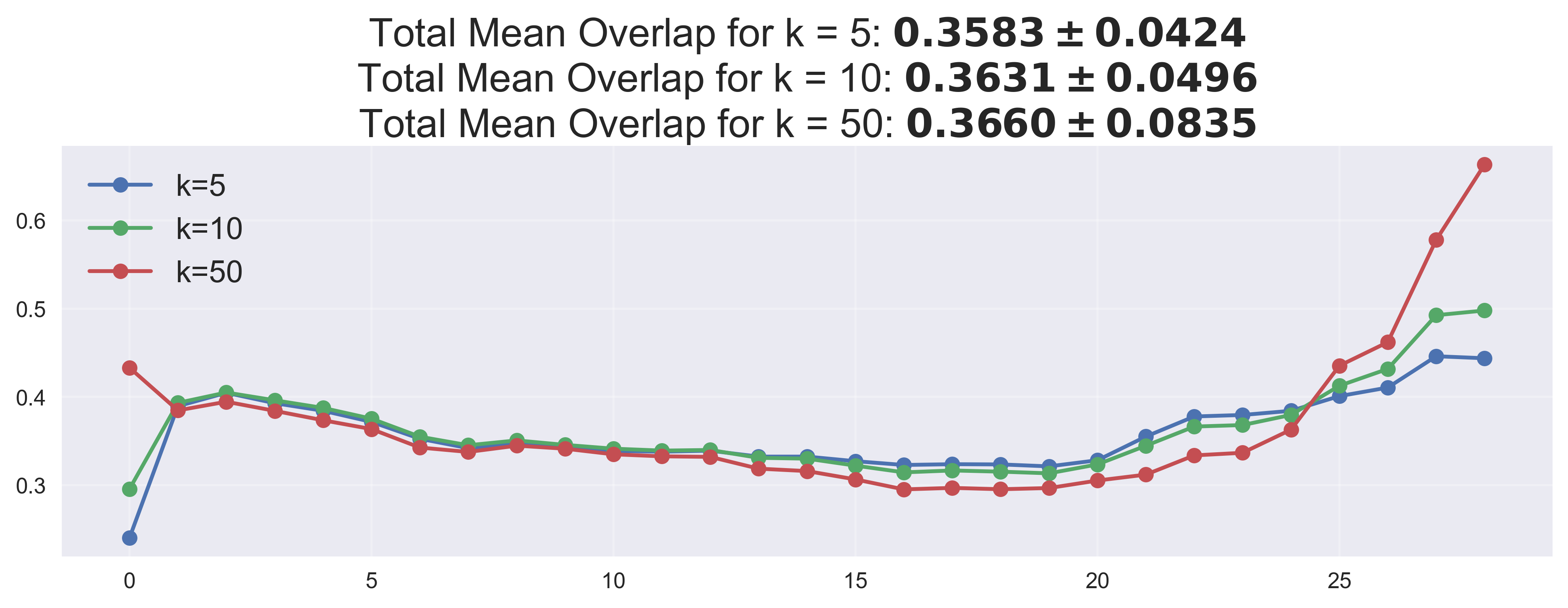}
            \centering \scriptsize \textbf{DeepSeek-R1-Distill-Qwen-1.5B-\textit{Emb}}
        \end{minipage}%
        } &

        \rotlabel{Qwen3-0.6B-Base-\textit{Emb}} &
        \setlength{\fboxsep}{3pt}% 
        \colorbox{red!20}{%  <-- CHANGE COLOR HERE (e.g., yellow!20, blue!10)
            \begin{minipage}{0.30\textwidth} % Adjusted width
            \includegraphics[width=\linewidth]{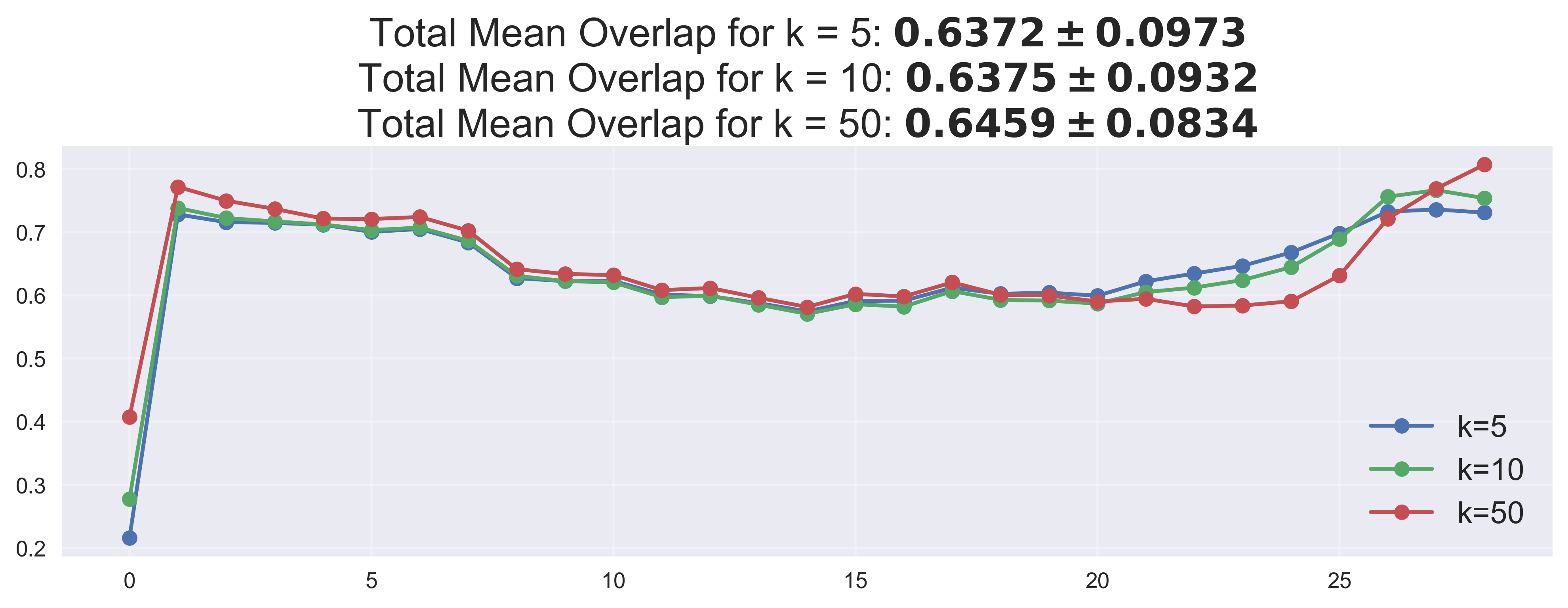} \\
            \centering \scriptsize \textbf{Qwen3-0.6B-\textit{Emb}}
        \end{minipage}%
        } &

        \rotlabel{Qwen2.5-1.5B-\textit{Emb}} &
        \begin{minipage}{0.30\textwidth} % Adjusted width
            \includegraphics[width=\linewidth]{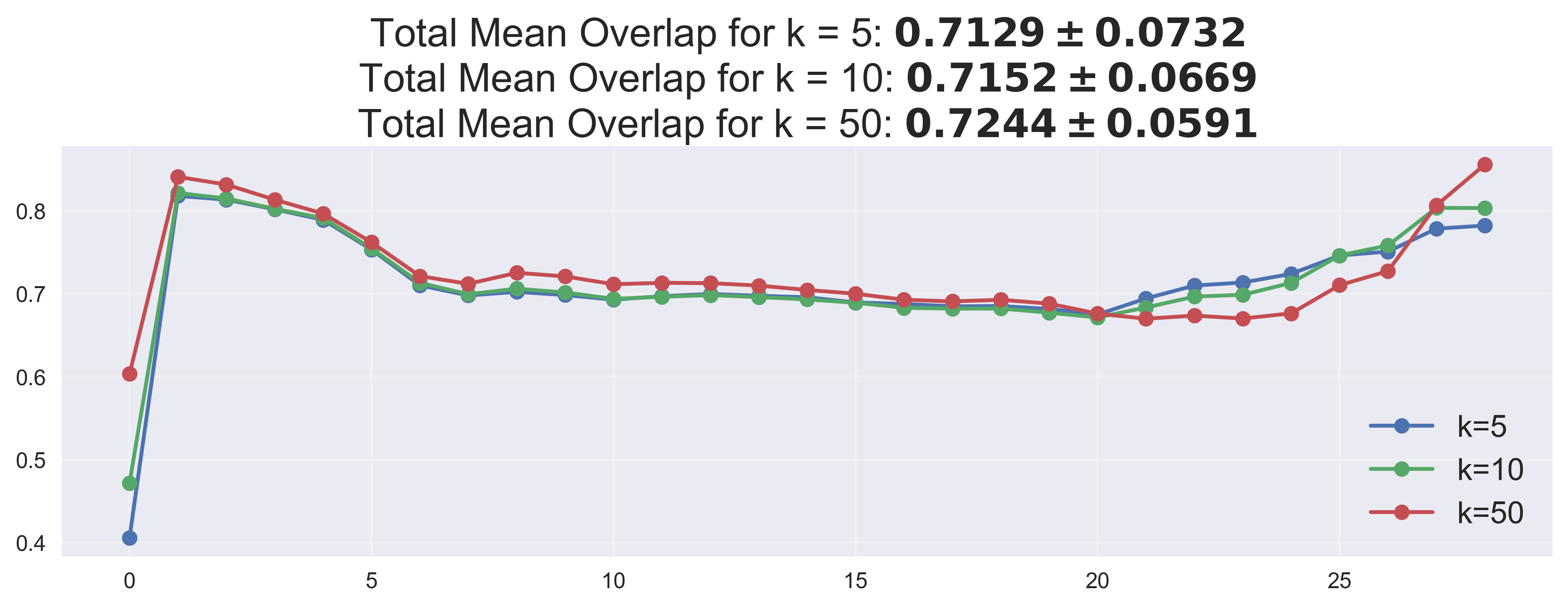} \\
            \centering \scriptsize \textbf{Qwen-2.5-1.5B-SimpleRL-Zoo-\textit{Emb}}
        \end{minipage} \\

        % \multicolumn{6}{c}{\vspace{0.1em}} \\ % Spacing between rows

        % --- Row 2 ---
        \rotlabel{Qwen2.5-0.5B-\textit{Emb}} &
        \begin{minipage}{0.30\textwidth} % Adjusted width
            \includegraphics[width=\linewidth]{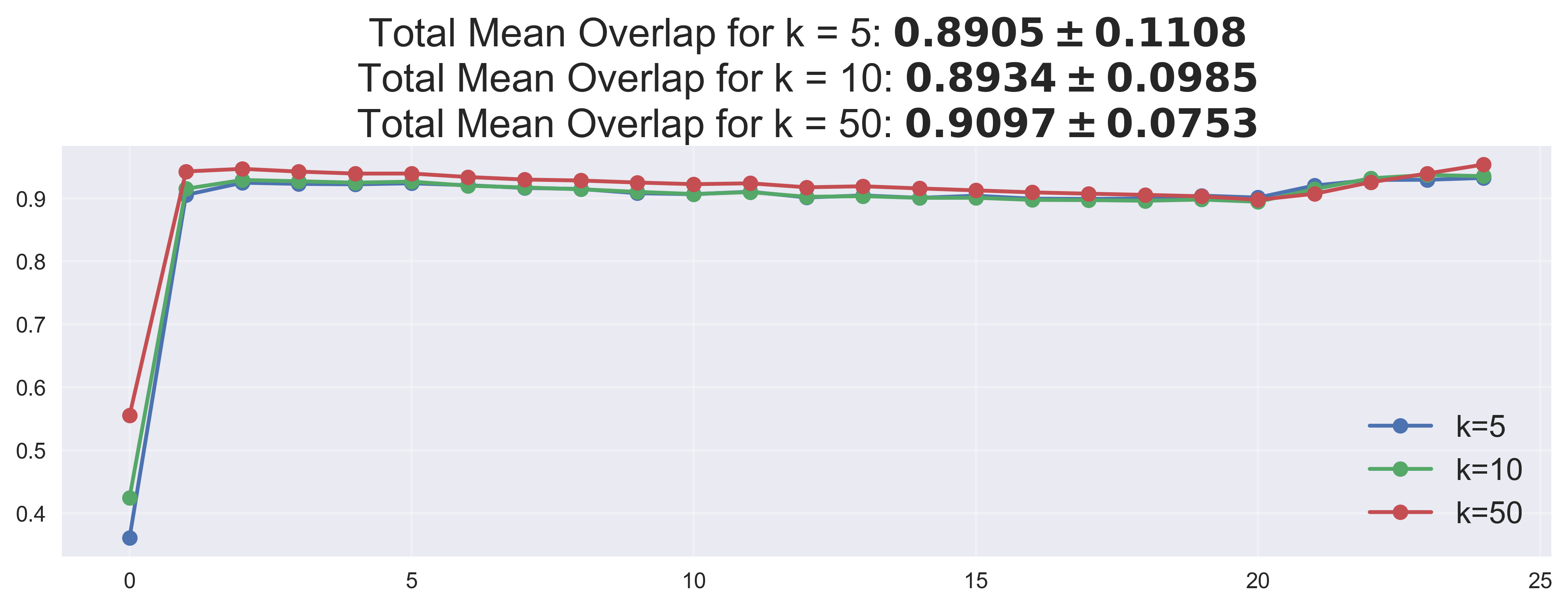} \\
            \centering \scriptsize \textbf{Qwen-2.5-0.5B-SimpleRL-Zoo-\textit{Emb}}
        \end{minipage} &

        \rotlabel{DeepSeek-R1-Distill-Qwen-1.5B-\textit{Emb}} &
        \begin{minipage}{0.30\textwidth} % Adjusted width
            \includegraphics[width=\linewidth]{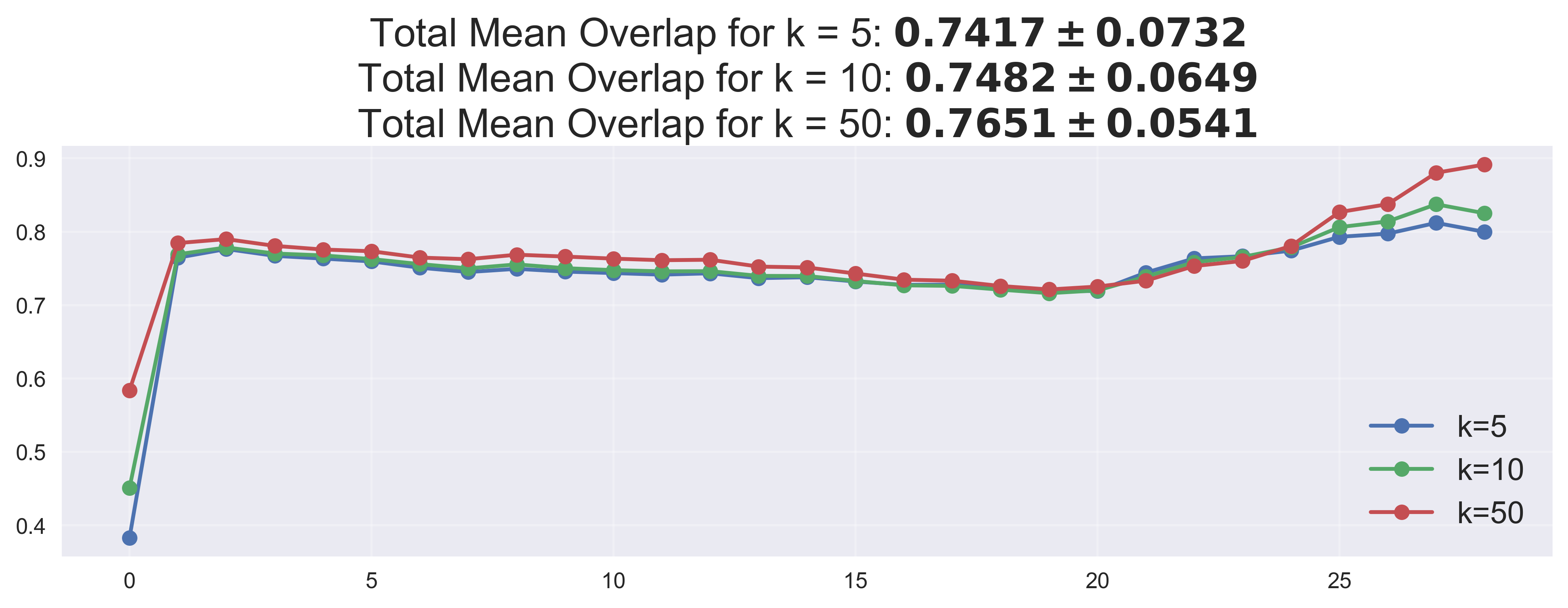} \\
            \centering \scriptsize \textbf{Nemotron-Research-Reasoning-Qwen-1.5B-\textit{Emb}}
        \end{minipage} &

        \rotlabel{Qwen3-4B-\textit{Emb}} &
        \begin{minipage}{0.30\textwidth} % Adjusted width
            \includegraphics[width=\linewidth]{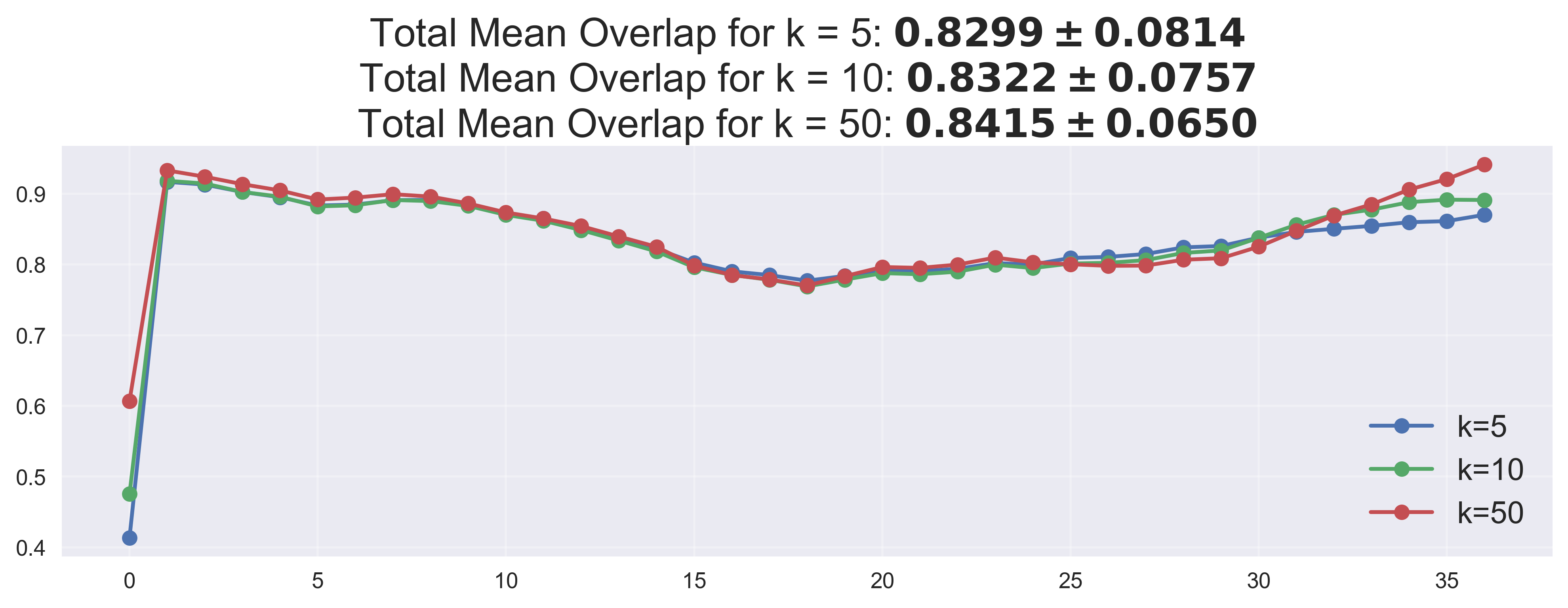} \\
            \centering \scriptsize \textbf{Qwen3-4B-PSR-\textit{Emb}}
        \end{minipage} \\
    \end{tabular}%
    } % End of \makebox

    \vspace{0.5em}
    \textbf{Model Layer Index}

    % The second section and other elements from the original code were commented out or misplaced,
    % and have been removed for clarity and to focus on the main table.

    \caption{Additional Results on $k$-NN Overlap separated by dataset. The vertical axis and horizontal axis are Mean overlap and Model Layer Index, respectively. The \textbf{\textcolor{red}{red}} background indicates their backbone LLMs are SFT-tuned pairs.}
    \label{fig: appendix-knn-overlap-base-embedding-vs-reasoning-embedding}
\end{figure*}

% ================================================================================================ 
% Cross-Model Linear Probe 
% Base Model vs. Reasoning Model
%================================================================================================ 

\begin{figure*}[p]
    \centering
    {\large \textbf{Cross-Model Linear Probes}} \par\smallskip
    {\large \textbf{Base Models $\mathcal{M}_{base}$ vs. Reasoning Models $\mathcal{M}_{reason}$}} \par\medskip

    % --- Configuration ---
    \setlength{\tabcolsep}{1pt}

    \sectionbox{Dataset: AG's News Topic Classification}
    \vspace{1ex} % Adds a small vertical space for better separation

    % Center the main content grid and ensure it does not exceed the text width
    \makebox[\textwidth][c]{%
    \begin{tabular}{
        c @{\hspace{1pt}} c  @{\hspace{0.5em}}
        c @{\hspace{1pt}} c  @{\hspace{0.5em}}
        c @{\hspace{1pt}} c  @{\hspace{0.5em}}
        c @{\hspace{1pt}} c
    }
        % --- Row 1 ---
        % === MODIFIED CELL START ===
        % We wrap the minipage in a colorbox. 
        % \fboxsep controls the padding between the color edge and the image.
        \rotlabel{Qwen2.5-Math-1.5B} &
        \setlength{\fboxsep}{3pt}% 
        \colorbox{red!20}{%  <-- CHANGE COLOR HERE (e.g., yellow!20, blue!10)
            \begin{minipage}{0.22\textwidth}
                \centering
                \includegraphics[width=\linewidth]{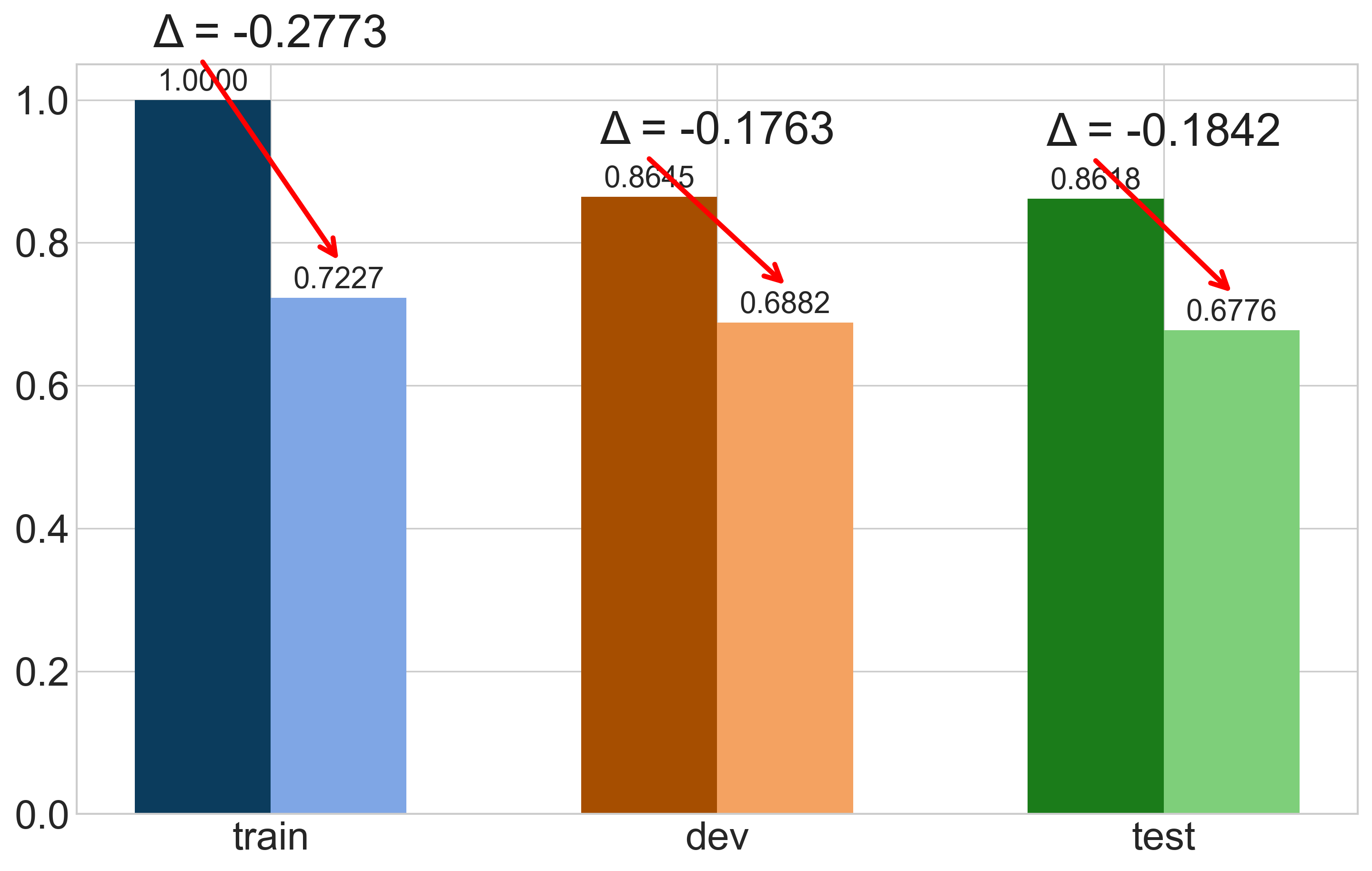}
                \par\vspace{2pt} % Small space between image and caption
                \scriptsize \textbf{DeepSeek-R1-Distill-Qwen-1.5B}
            \end{minipage}%
        } &

        \rotlabel{Qwen3-4B} &
        \begin{minipage}{0.22\textwidth} % Adjusted width
            \includegraphics[width=\linewidth]{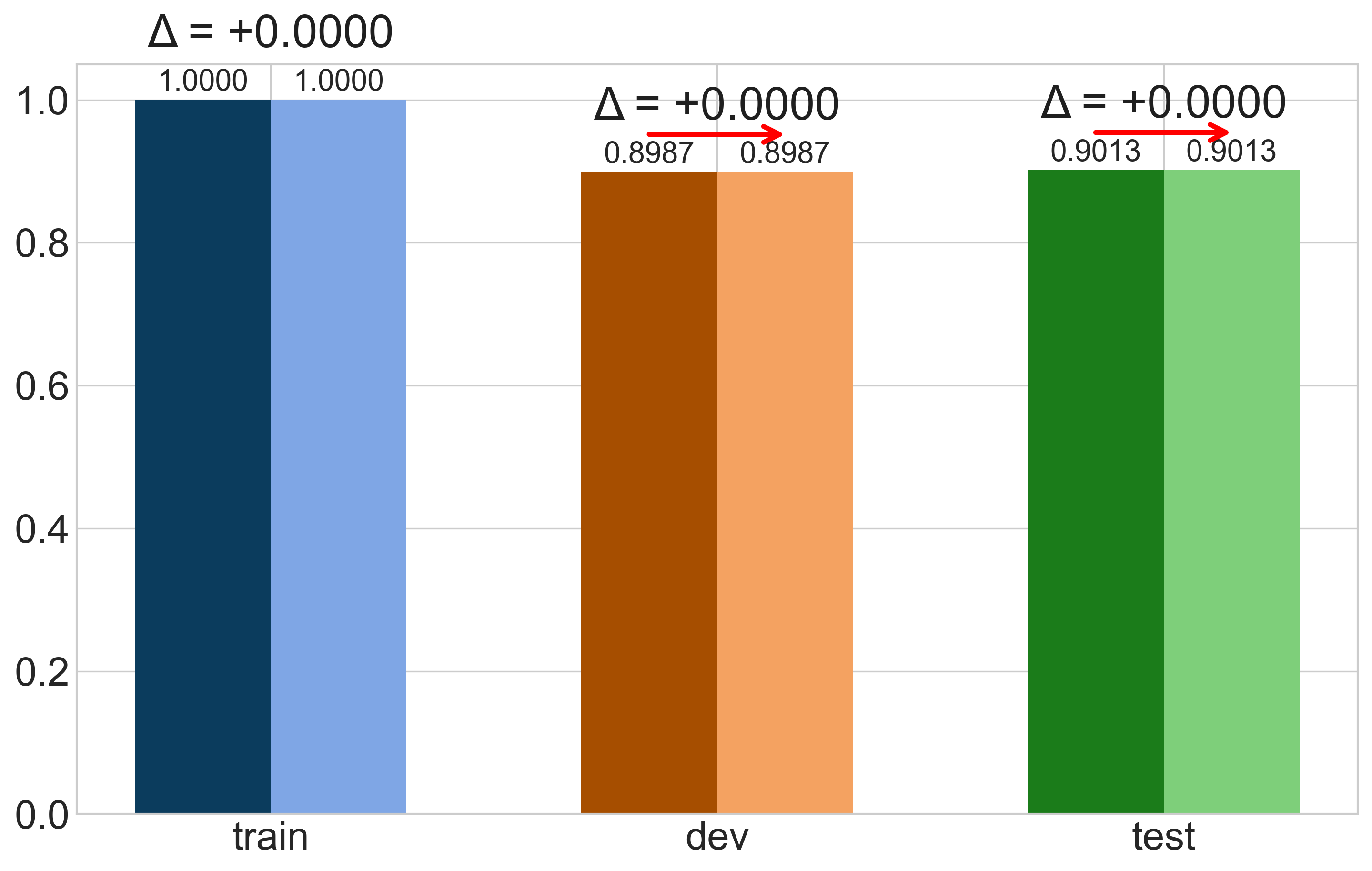} \\
            \centering \scriptsize \textbf{Polaris-4B-Preview}
        \end{minipage} &

        \rotlabel{DeepSeek-R1-Distill-Qwen-7B} &
        \begin{minipage}{0.22\textwidth} % Adjusted width
            \includegraphics[width=\linewidth]{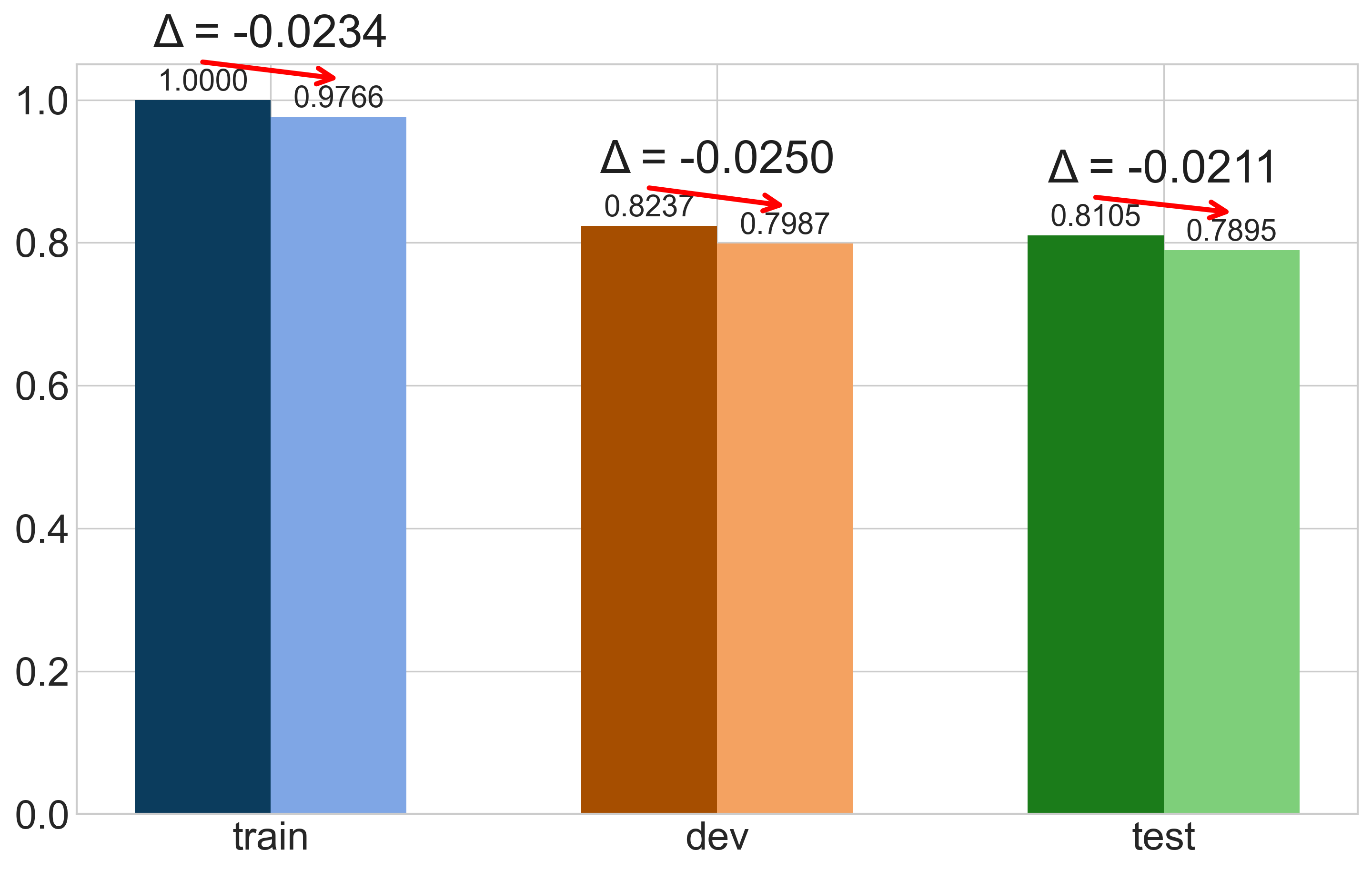} \\
            \centering \scriptsize \textbf{Polaris-7B-Preview}
        \end{minipage} &

        \rotlabel{Qwen2.5-7B} &
        \begin{minipage}{0.22\textwidth} % Adjusted width
            \includegraphics[width=\linewidth]{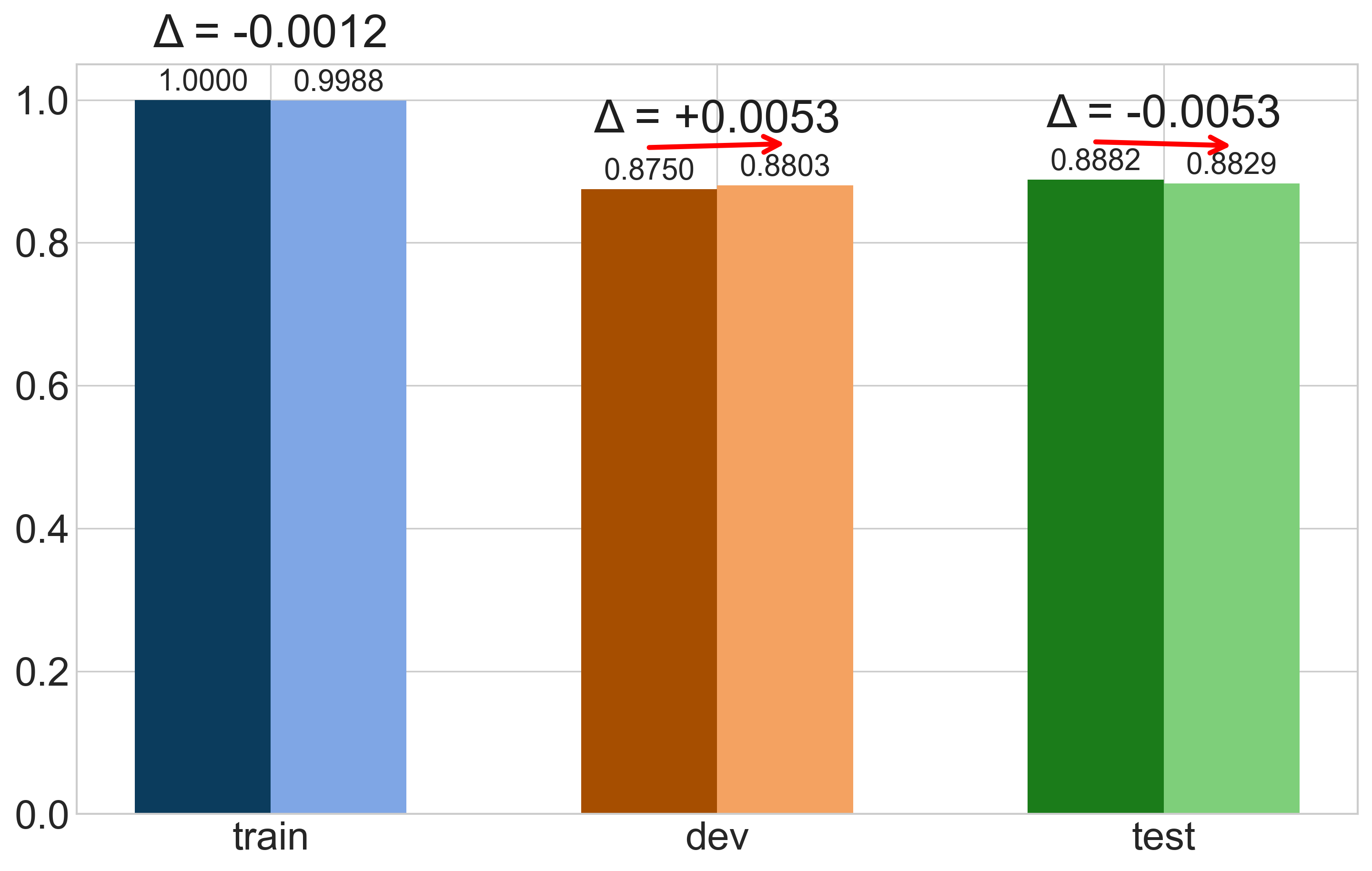} \\
            \centering \scriptsize \textbf{zero\_\_ppo\_\_think\_\_Qwen2.5-7B}
        \end{minipage} \\

        \multicolumn{8}{c}{\vspace{0.1em}} \\ % Spacing between rows

        % --- Row 2 ---
        \rotlabel{Qwen2.5-1.5B} &
        \begin{minipage}{0.22\textwidth} % Adjusted width
            \includegraphics[width=\linewidth]{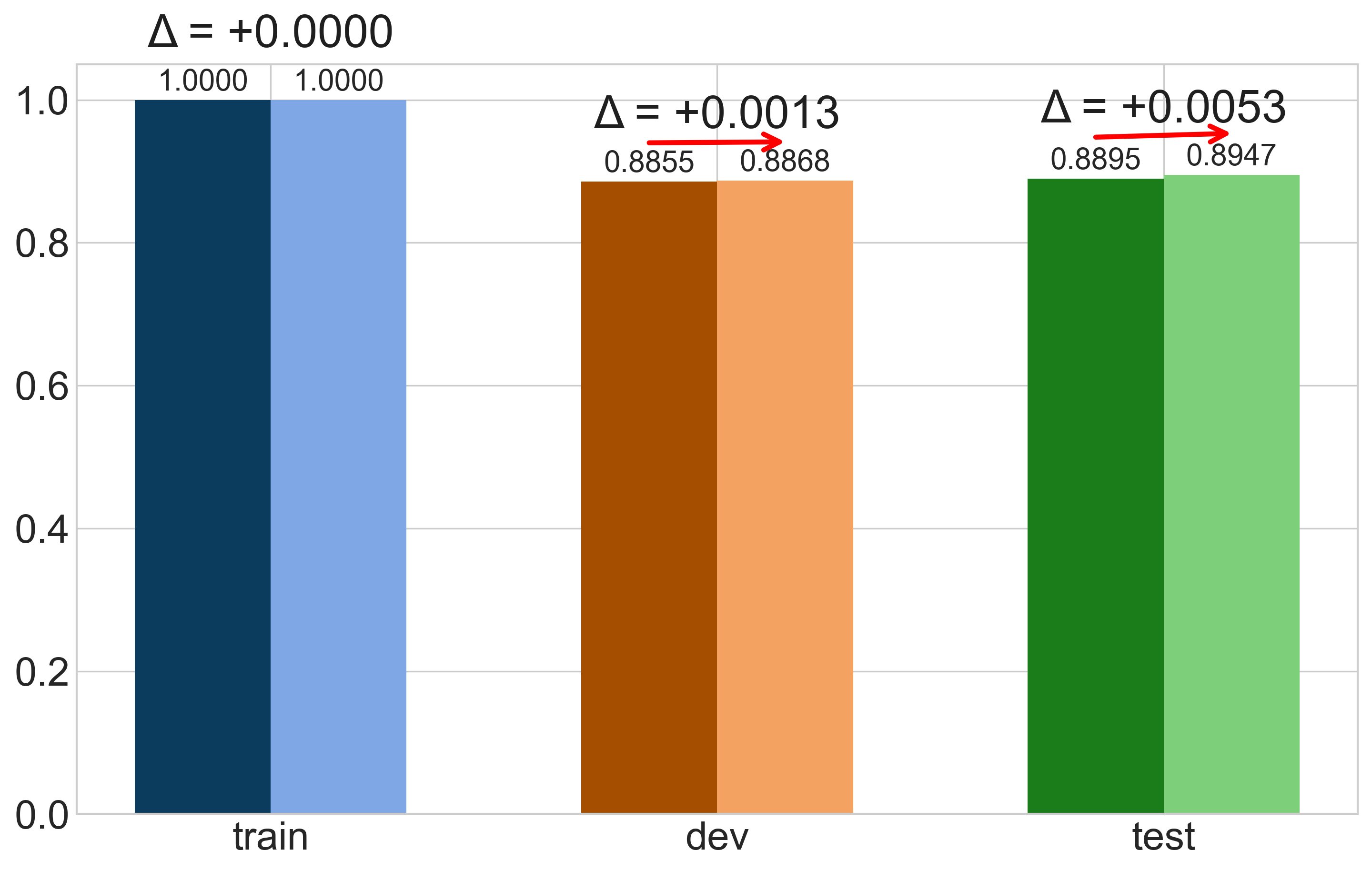} \\
            \centering \scriptsize \textbf{Qwen-2.5-1.5B-SimpleRL-Zoo}
        \end{minipage} &

        \rotlabel{Qwen2.5-0.5B} &
        \begin{minipage}{0.22\textwidth} % Adjusted width
            \includegraphics[width=\linewidth]{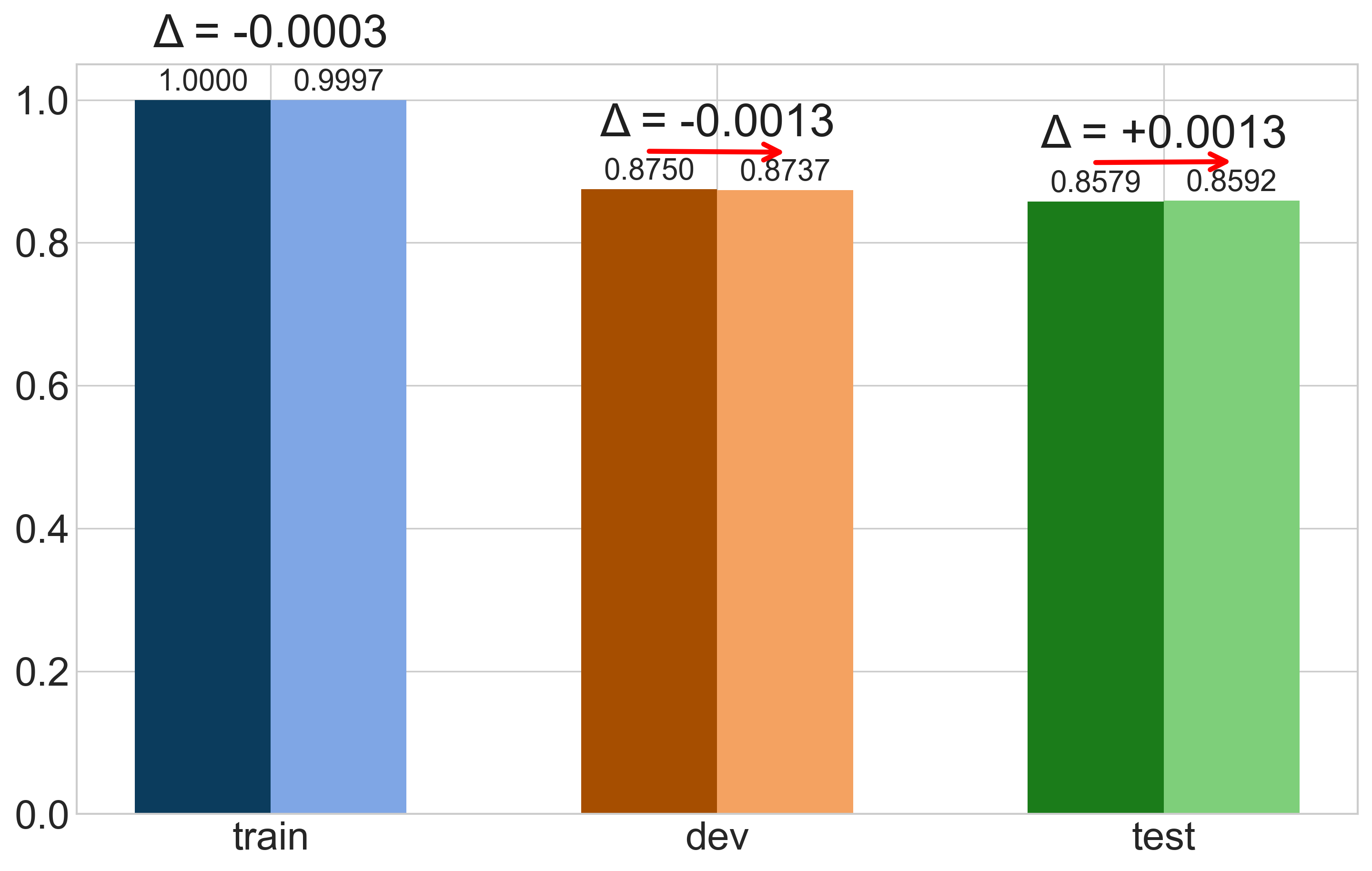} \\
            \centering \scriptsize \textbf{Qwen-2.5-0.5B-SimpleRL-Zoo}
        \end{minipage} &

        \rotlabel{DeepSeek-R1-Distill-Qwen-1.5B} &
        \begin{minipage}{0.22\textwidth} % Adjusted width
            \includegraphics[width=\linewidth]{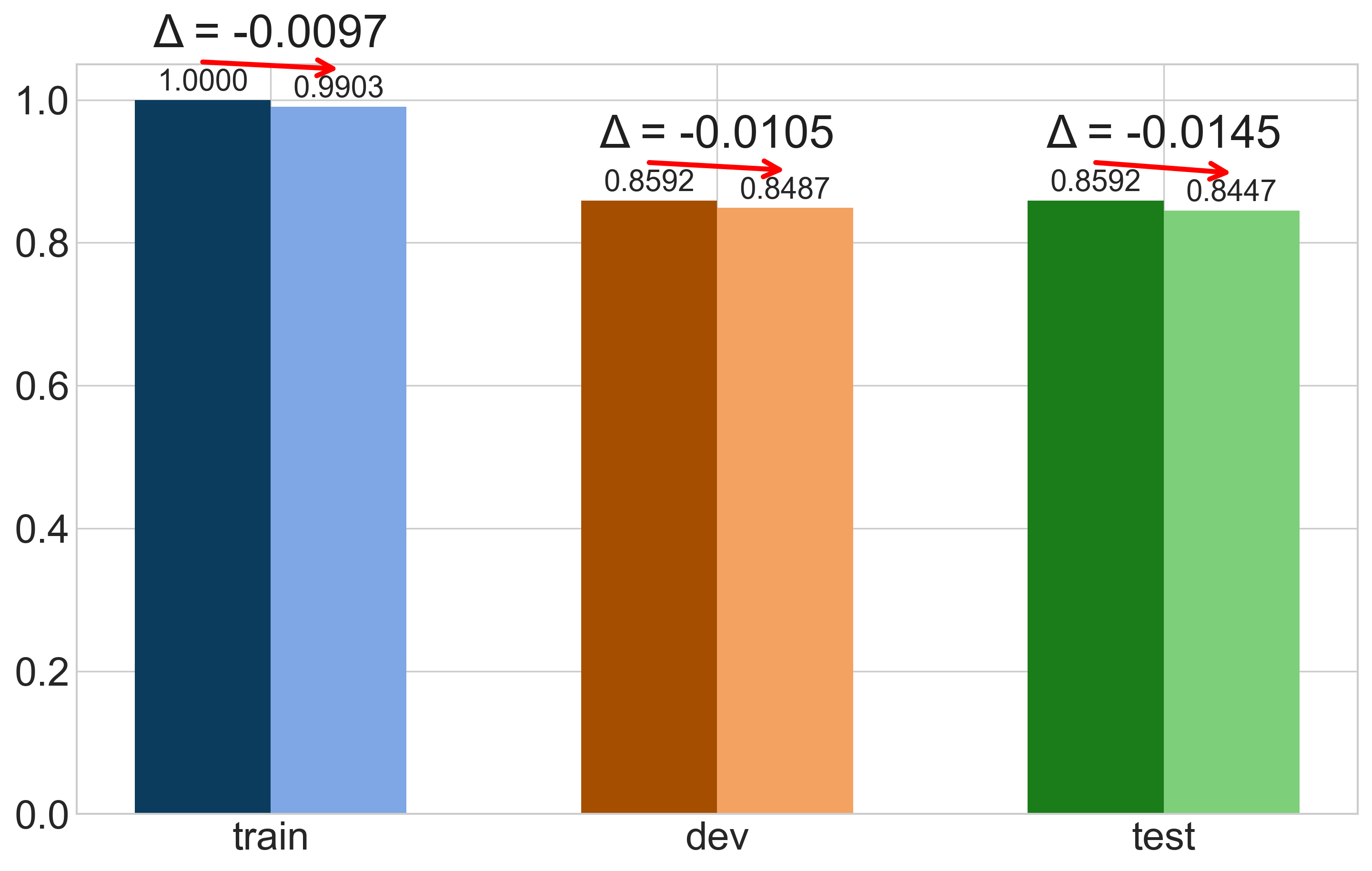} \\
            \centering \scriptsize \textbf{Nemotron-Research-Reasoning-Qwen-1.5B}
        \end{minipage} &

        \rotlabel{Qwen3-4B} &
        \begin{minipage}{0.22\textwidth} % Adjusted width
            \includegraphics[width=\linewidth]{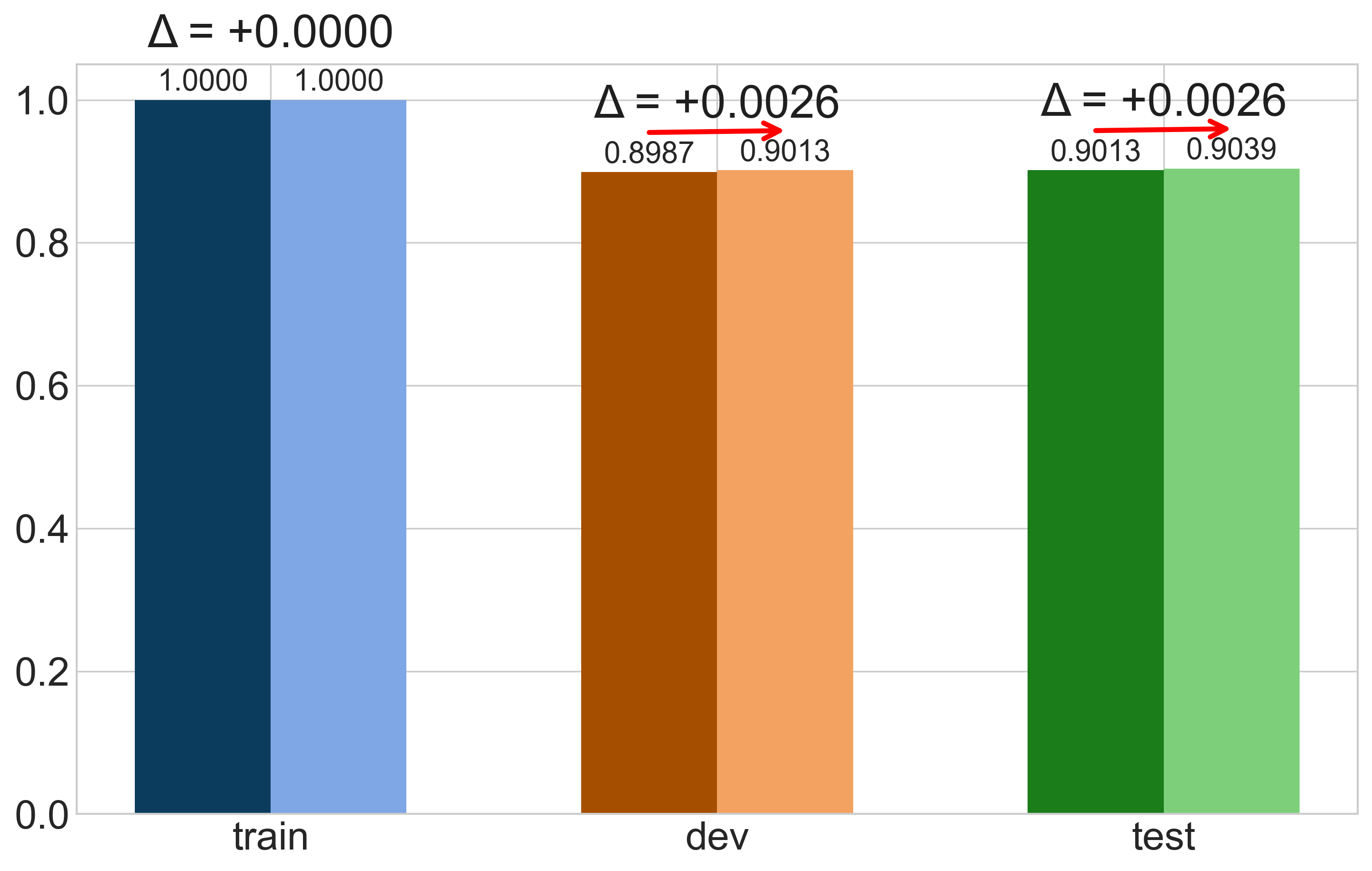} \\
            \centering \scriptsize \textbf{Qwen3-4B-PSR}
        \end{minipage} \\
    \end{tabular}%
    } % End of \makebox

    \vspace{1em}
    \textbf{Dataset Types}
    % The second section and other elements from the original code were commented out or misplaced,
    % and have been removed for clarity and to focus on the main table.

    \vspace{1em}
    \includegraphics[width=0.5\linewidth, keepaspectratio]{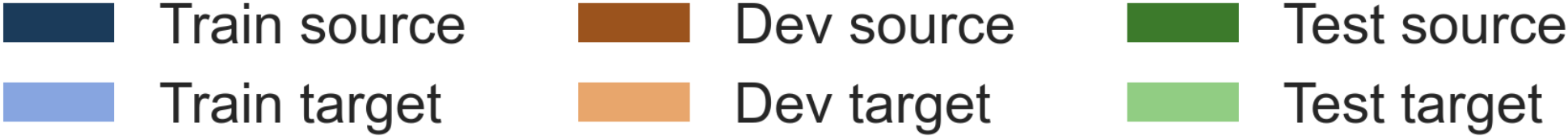} 

    \caption{Additional Results on Cross-Model Linear Probe. The vertical axis and horizontal axis are the Accuracy of the linear probe and Dataset types (train, dev, test), respectively. The \textbf{\textcolor{red}{red}} background indicates SFT-tuned pairs.}
    \label{fig: appendix-linear-probe-base-vs-reasoning}
\end{figure*}
% ================================================================================================ 
% Cross-Model Linear Probe 
% Base Embedding Model vs. Reasoning Embedding Model
%================================================================================================ 

\begin{figure*}[p]
    \centering
    {\large \textbf{Cross-Model Linear Probes}} \par\smallskip
    {\large \textbf{Base Embedding Models $\mathcal{M}_{base}^{Emb}$ vs. Reasoning Embedding Models $\mathcal{M}_{reason}^{Emb}$}} \par\medskip

    % --- Configuration ---
    \setlength{\tabcolsep}{1pt}

    \sectionbox{Dataset: AG's News Topic Classification}
    \vspace{1ex} % Adds a small vertical space for better separation

    % Center the main content grid and ensure it does not exceed the text width
    \makebox[\textwidth][c]{%
    \begin{tabular}{
        c @{\hspace{1pt}} c  @{\hspace{0.5em}}
        c @{\hspace{1pt}} c  @{\hspace{0.5em}}
        c @{\hspace{1pt}} c
    }
        % --- Row 1 ---
        \rotlabel{Qwen2.5-Math-1.5B-\textit{Emb}} &
        \setlength{\fboxsep}{3pt}% 
        \colorbox{red!20}{%  <-- CHANGE COLOR HERE (e.g., yellow!20, blue!10)
            \begin{minipage}{0.3\textwidth} % Adjusted width for better fit
                \includegraphics[width=\linewidth]{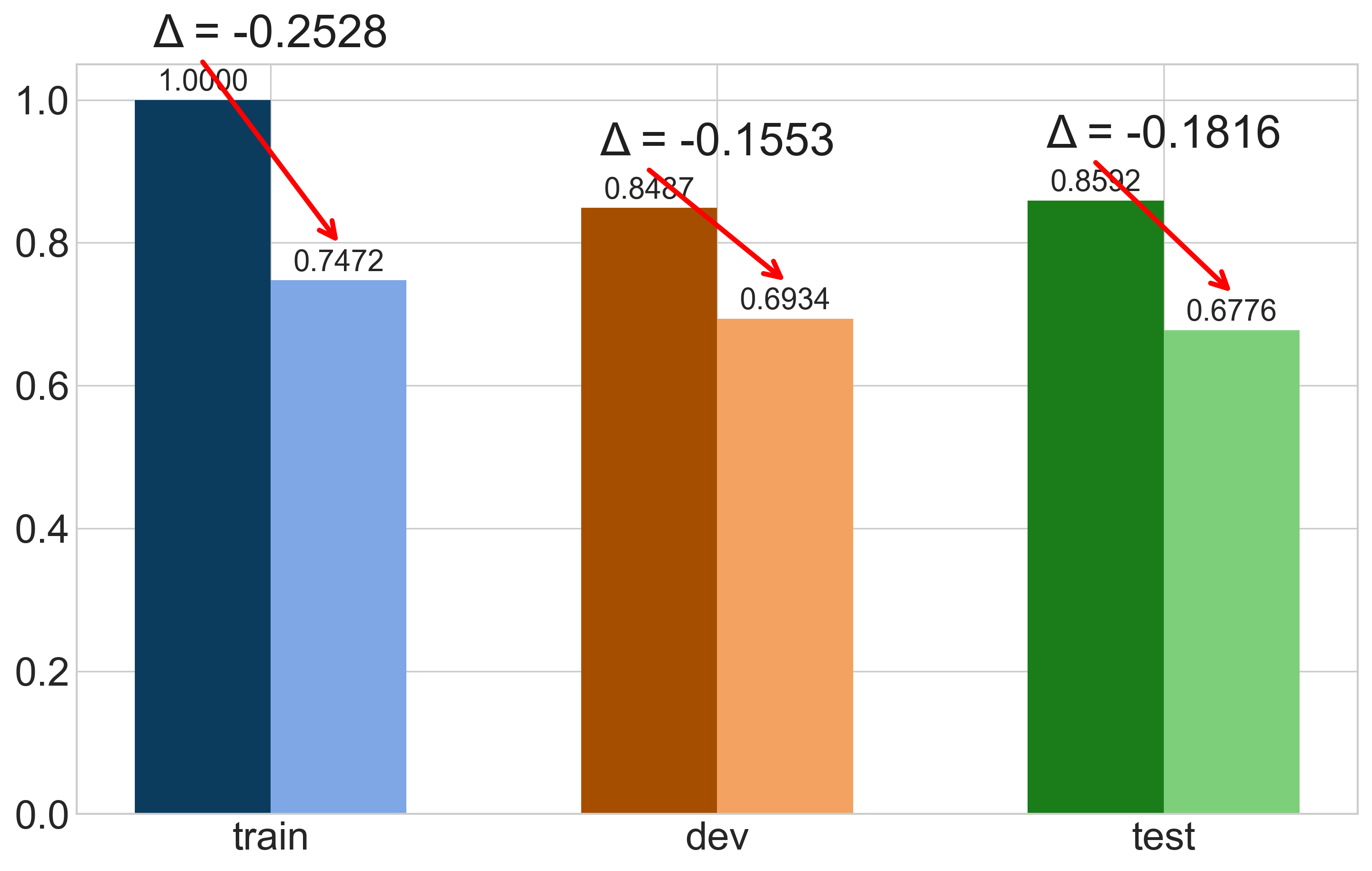}
                \centering \scriptsize \textbf{DeepSeek-R1-Distill-Qwen-1.5B-\textit{Emb}}
            \end{minipage}%
        } &

        % \hspace{0.5em}

        \rotlabel{Qwen3-0.6B-Base-\textit{Emb}} &
        \setlength{\fboxsep}{3pt}% 
        \colorbox{red!20}{%  <-- CHANGE COLOR HERE (e.g., yellow!20, blue!10)
            \begin{minipage}{0.3\textwidth} % Adjusted width
                \includegraphics[width=\linewidth]{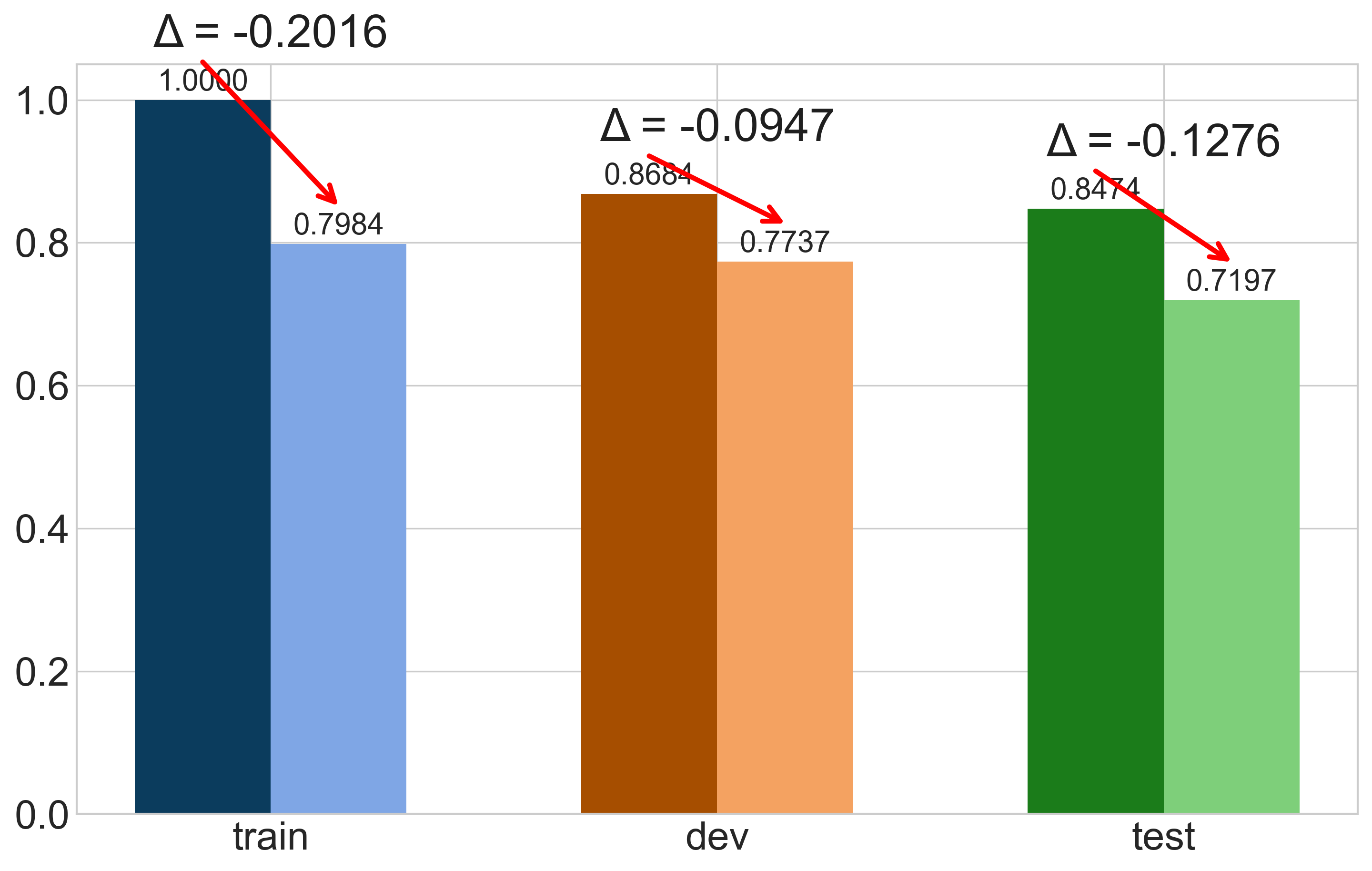} \\
                \centering \scriptsize \textbf{Qwen3-0.6B-\textit{Emb}}
            \end{minipage}%
        } &

        % \hspace{0.5em}

        \rotlabel{Qwen2.5-1.5B-\textit{Emb}} &
        \begin{minipage}{0.3\textwidth} % Adjusted width
            \includegraphics[width=\linewidth]{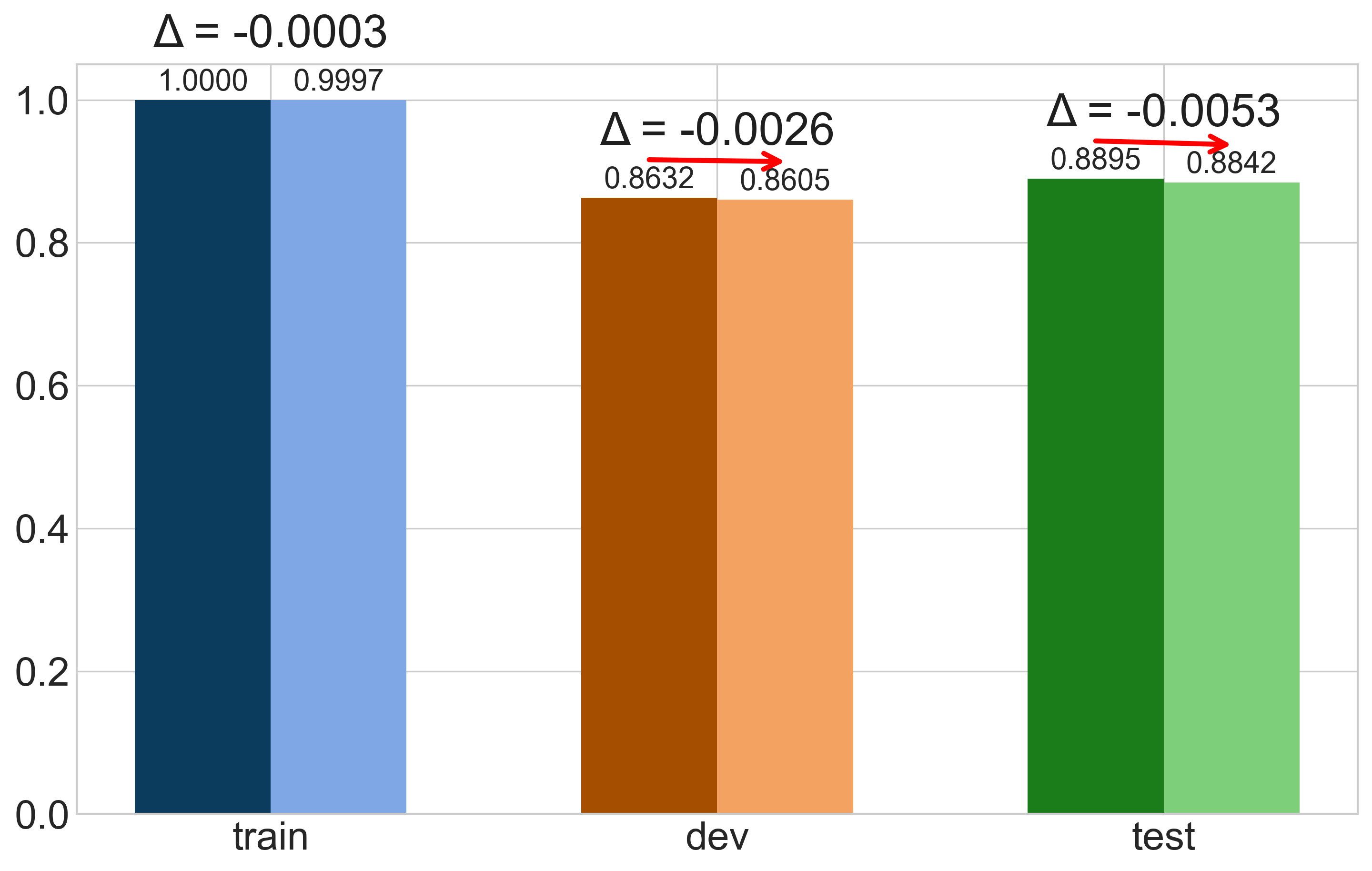} \\
            \centering \scriptsize \textbf{Qwen-2.5-1.5B-SimpleRL-Zoo-\textit{Emb}}
        \end{minipage} \\

        \multicolumn{6}{c}{\vspace{0.1em}} \\ % Spacing between rows

        % --- Row 2 ---
        \rotlabel{Qwen2.5-0.5B-\textit{Emb}} &
        \begin{minipage}{0.3\textwidth} % Adjusted width
            \includegraphics[width=\linewidth]{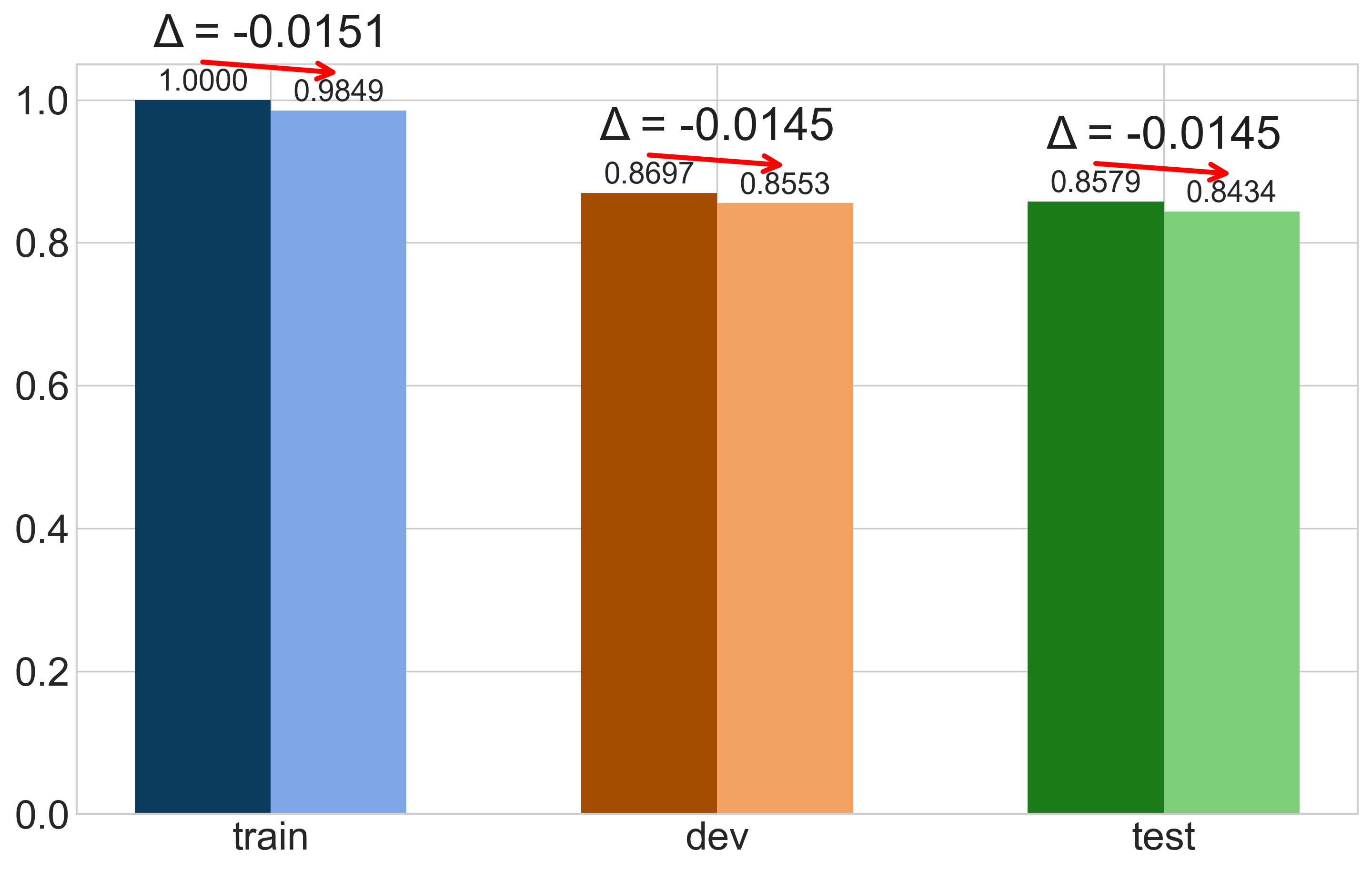} \\
            \centering \scriptsize \textbf{Qwen-2.5-0.5B-SimpleRL-Zoo-\textit{Emb}}
        \end{minipage} &

        \rotlabel{DeepSeek-R1-Distill-Qwen-1.5B-\textit{Emb}} &
        \begin{minipage}{0.3\textwidth} % Adjusted width
            \includegraphics[width=\linewidth]{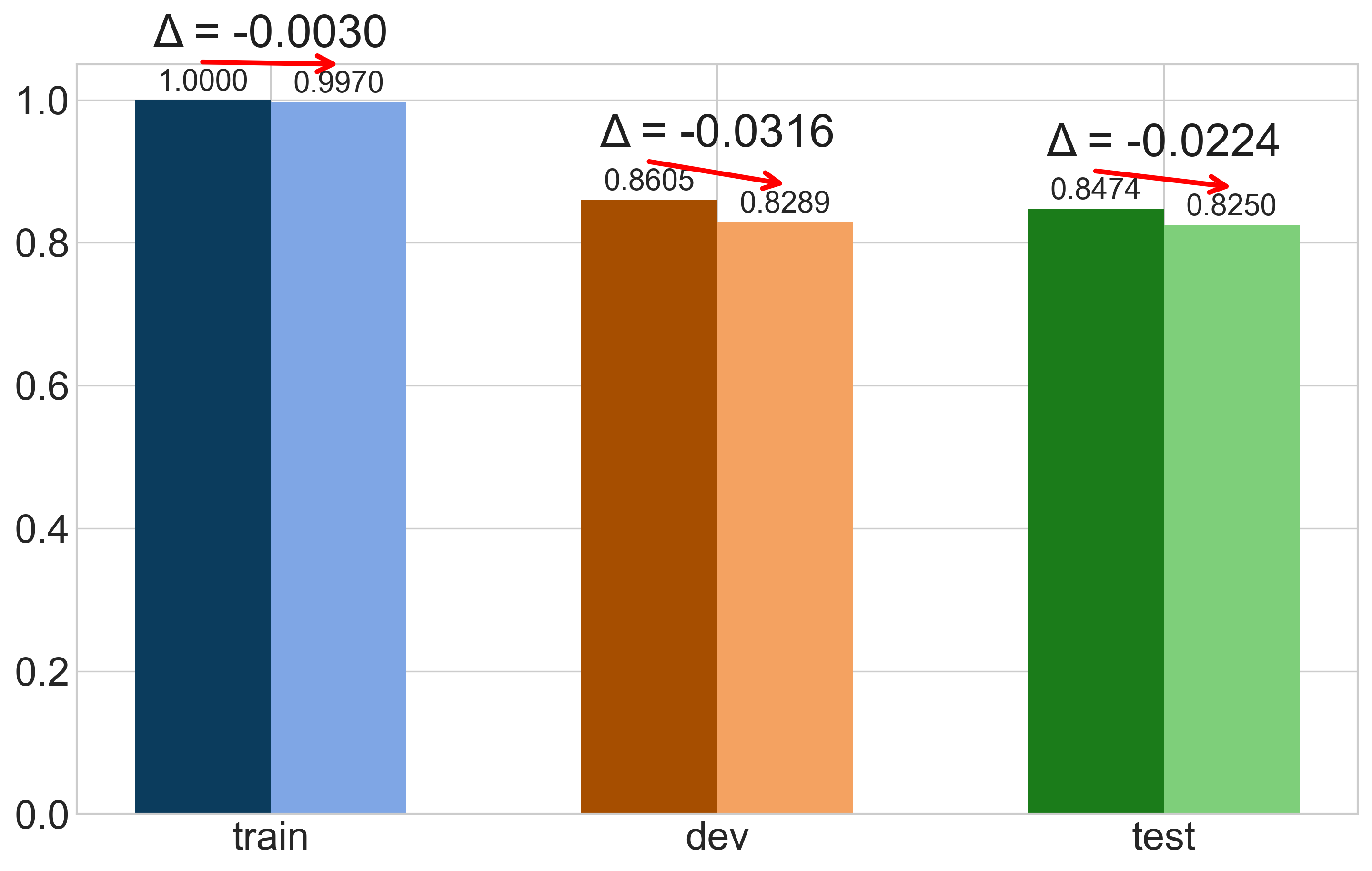} \\
            \centering \scriptsize \textbf{Nemotron-Research-Reasoning-Qwen-1.5B-\textit{Emb}}
        \end{minipage} &

        \rotlabel{Qwen3-4B-\textit{Emb}} &
        \begin{minipage}{0.3\textwidth} % Adjusted width
            \includegraphics[width=\linewidth]{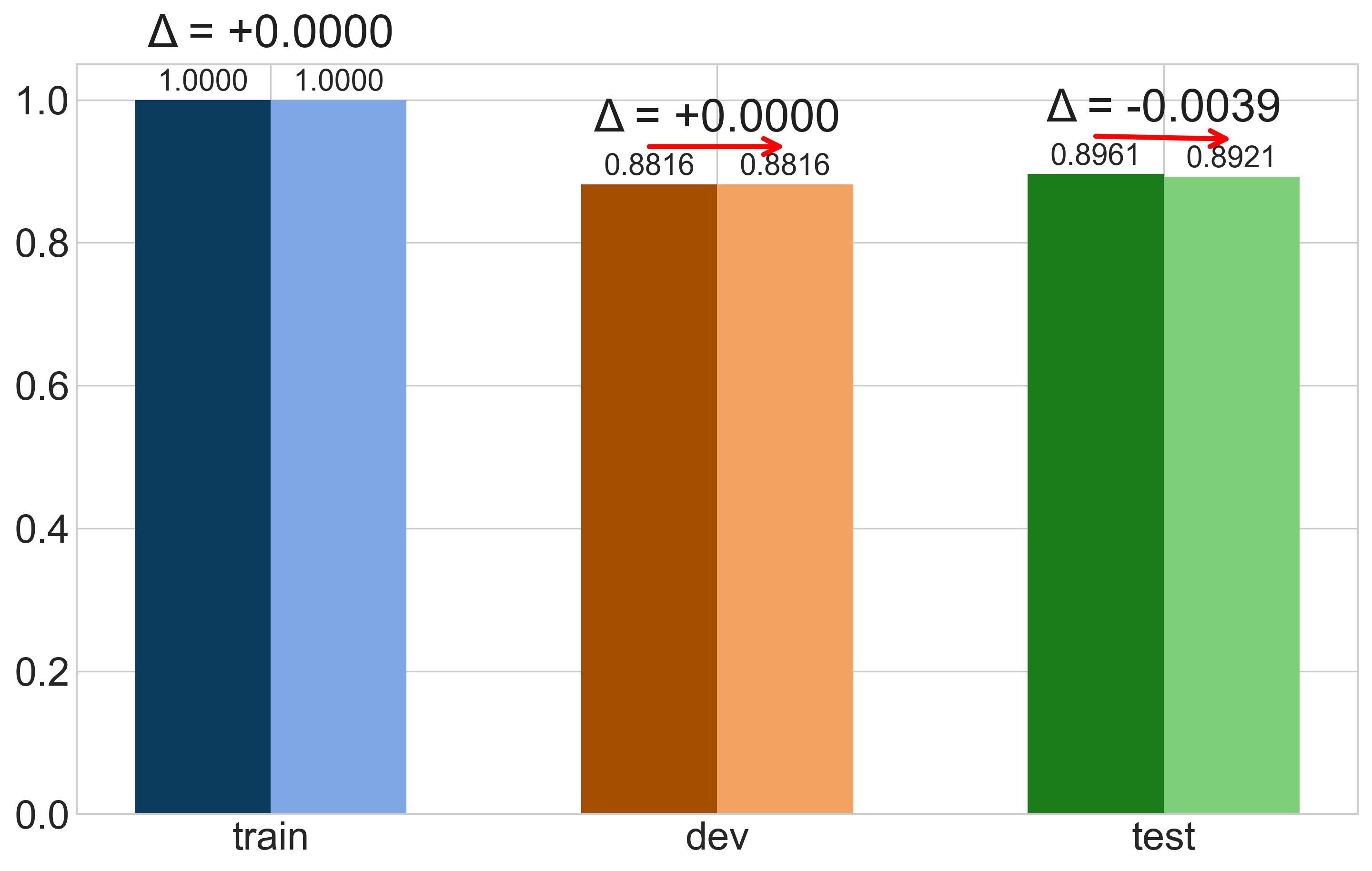} \\
            \centering \scriptsize \textbf{Qwen3-4B-PSR-\textit{Emb}}
        \end{minipage} \\
    \end{tabular}%
    } % End of \makebox

    \vspace{1em}
    \textbf{Dataset Types}

    \vspace{1em}
    \includegraphics[width=0.5\linewidth, keepaspectratio]{figures/re_experiments/linear_probe/cross_linear_probe_label.png}

    \caption{Additional Results on Cross-Model Linear Probe. The vertical axis and horizontal axis are the Accuracy of the linear probe and Dataset types (train, dev, test), respectively. The \textbf{\textcolor{red}{red}} background indicates their backbone LLMs are SFT-tuned pairs.}
    \label{fig: appendix-linear-probe-base-embedding-vs-reasoning-embedding}
\end{figure*}

\end{document}